%% file: main.tex
\documentclass[a4paper,11pt,twoside,openright,notitlepage,colorlinks=true,linkcolor=blue,pagebackref=true,linktoc=page]{book}
\usepackage[a-1b]{pdfx}

\makeatletter
\def\bstctlcite{\@ifnextchar[{\@bstctlcite}{\@bstctlcite[@auxout]}}
\def\@bstctlcite[#1]#2{\@bsphack
 \@for\@citeb:=#2\do{%
   \edef\@citeb{\expandafter\@firstofone\@citeb}%
   \if@filesw\immediate\write\csname #1\endcsname{\string\citation{\@citeb}}\fi}%
 \@esphack}
\makeatother

\usepackage{babel}
\usepackage[utf8]{inputenc}
\usepackage[T1]{fontenc}			
\usepackage{lmodern}
\usepackage{amsmath}					
\usepackage{amssymb}
\usepackage{units}				
\usepackage{amstext}        
\usepackage{amsfonts}
\usepackage{amsthm}
\usepackage{bm}
\usepackage{stmaryrd}
\usepackage{array}						 
\usepackage{multicol}					
\usepackage{multirow}
\usepackage{lscape}
\usepackage{xspace}						
\usepackage[pdftex]{graphicx}	
\usepackage{subcaption}
\usepackage{indentfirst} 			
\usepackage{latexsym}					
\usepackage{setspace}					
\usepackage{algorithm}				
\usepackage{algorithmic}			
\usepackage{makeidx}					
\usepackage{vmargin}					
\usepackage{float}            
\usepackage{enumerate}        
\usepackage{verbatim}         
\usepackage{floatflt}         
\usepackage{rotating}					
\usepackage{longtable}        
\usepackage{varioref} 
\usepackage{diagbox}        
\usepackage{url}
\usepackage{pdfpages}           
\usepackage{titlesec} 
\usepackage{graphicx}
\usepackage[tight]{minitoc}
\usepackage{xcolor,lipsum}
\usepackage[framemethod=tikz]{mdframed}
\usepackage{ae,aecompl}
\usepackage{graphicx}
\usepackage{eso-pic}
\usepackage{array}	
\usepackage{doi}
\usepackage[export]{adjustbox}
\usepackage{rotating}
\usepackage{bbm}
\usepackage{fancyhdr}
\usepackage{booktabs}
\usepackage{wrapfig}
\usepackage[capitalise]{cleveref}
\usepackage{nomencl}
\usepackage{etoolbox}
\usepackage{colortbl}
\usepackage[top=2.5cm, bottom=2cm, 
			left=3cm, right=2.5cm,
			headheight=15pt]{geometry}

\definecolor{redHESAM}{RGB}{210,0,37}
\makeindex

   {\par}
\makeatother   

\makenomenclature

\renewcommand\nomgroup[1]{%
  \item[\bfseries
  \ifstrequal{#1}{A}{Acronyms}{%
  \ifstrequal{#1}{R}{Roman symbols}{%
  \ifstrequal{#1}{G}{Greek symbols}{}}}%
]
\vspace{1cm}%
}
\newenvironment{vcenterpage}
{\newpage\vspace*{\fill}}
{\vspace*{\fill}\par\pagebreak}
\makeindex

\mtcsettitle{minitoc}{Contents}

\def\eg{\textit{e.g.}\xspace}
\def\ie{\textit{i.e.}\xspace}
\definecolor{Gray}{gray}{0.85}
\newtheorem{definition}{Definition}[chapter]
\newtheorem{proposition}{Proposition}[chapter]
\newcommand{\ubold}[1]{\fontseries{b}\selectfont#1}
\newcommand{\uda}{\,$\triangleright$}
\newcommand{\x}{\bm{x}}

\newcommand{\y}{\bm{y}}
\newcommand{\bw}{\mathbf{w}}
\newcommand{\btheta}{\bm{\theta}}
\newcommand{\bphi}{\bm{\varphi}}
\newcommand{\bpsi}{\bm{\psi}}
\newcommand{\balpha}{\bm{\alpha}}
\newcommand{\bomega}{\bm{\omega}}
\newcommand{\bpi}{\bm{\pi}}
\newcommand{\so}{\text{s}}
\newcommand{\tg}{\text{t}}
\newcommand{\cD}{\mathcal{D}}
\newcommand{\cH}{\mathcal{H}}

\newcommand{\cL}{\mathcal{L}}
\newcommand{\cX}{\mathcal{X}}
\newcommand{\cY}{\mathcal{Y}}
\newcommand{\bbH}{\mathbb{H}}
\newcommand{\bbI}{\mathbb{I}}
\newcommand{\xso}{\x_{\so}}
\newcommand{\xson}{\x_{\so,n}}
\newcommand{\xtg}{\x_{\tg}}
\newcommand{\xtgn}{\x_{\tg,n}}
\newcommand{\Pmap}{\mathsf{P}_{\x}^{\btheta}}

\newcommand{\Pmapt}{\mathsf{P}_{\xtg}^{\btheta}}
\newcommand{\Cmap}{\mathsf{C}_{\x}^{\bomega}}
\newcommand{\Cmapt}{\mathsf{C}_{\xtg}^{\bomega}}

\newcommand{\Cmaps}{\mathsf{C}_{\xso}^{\bomega}}
\newcommand{\Cmapsn}{\mathsf{C}_{\xson}^{\bomega}}
\DeclareMathOperator*{\argmax}{arg\,max}
\DeclareMathOperator*{\argmin}{arg\,min}
\usepackage[titletoc]{appendix}

\begin{document}

\frontmatter
\dominitoc
\thispagestyle{empty}

\newgeometry{
        showframe,
		top=20mm,
		bottom=20mm,
		inner=20mm,
		outer=20mm,
		includehead,
		ignorefoot,
		nomarginpar,
		headsep=10mm,
		footskip=10mm,
	}

\setmarginsrb{5mm}{15mm}{18mm}{5mm}{0mm}{0mm}{0mm}{0mm}

\hspace{0.5cm}\includegraphics[width=5.5cm]{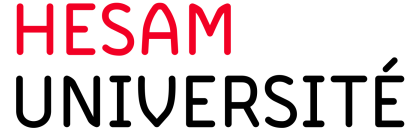} \hspace{6cm} \includegraphics[width=5.5cm]{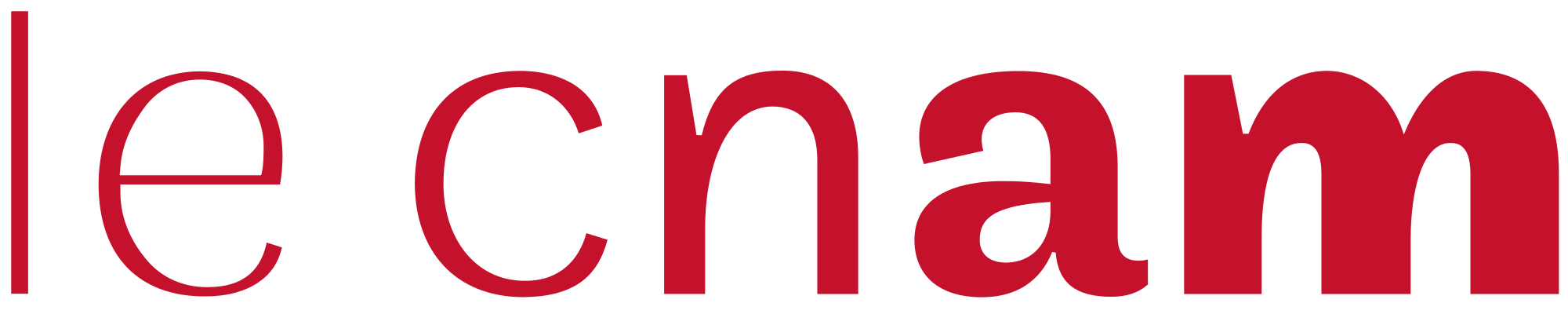}

\vspace{1cm}
\makebox[\textwidth][c]{\centering{\LARGE{\textbf{\'ECOLE DOCTORALE SCIENCES DES M\'{E}TIERS DE L'ING\'{E}NIEUR}}}}\\

\vspace{-0.2cm}
\hspace{0.5cm}\makebox[\textwidth][c]{\centering{\LARGE{\textbf{Centre d’études et de recherche en informatique et communications}}}}\\
\vspace{0.3cm}

\makebox[\textwidth][c]{{\Huge{\textbf{TH\`ESE}}}}\\

\vspace{0.3cm}
\makebox[\textwidth][c]{\centering{\large{\textit{présentée par} : }} {\Large{\textbf{Charles CORBI\`ERE}}}}\\
\vspace{.4cm}
\makebox[\textwidth][c]{\centering{\large{\textit{soutenue le} : }} {\Large{\textbf{16 mars 2022}}}}\\
\vspace{.4cm}
\makebox[\textwidth][c]{\centering{\large{\textit{pour obtenir le grade de} : }} {\Large{\textbf{Docteur d'HESAM Université}}}}\\
\vspace{.2cm}
\makebox[\textwidth][c]{\centering{\large{\textit{préparée au} : }} {\Large{\textbf{Conservatoire national des arts et métiers}}}}\\
\vspace{.2cm}
\makebox[\textwidth][c]{\centering{\large{\textit{Discipline} :}} {\large{\textbf{Mathématiques, Informatique et Systèmes}}}}
\makebox[\textwidth][c]{\centering{\large{\textit{Spécialité} :}} {\large{\textbf{Informatique}}}}
\vspace{0.2cm}

\resizebox{\textwidth}{!}{
\fcolorbox{redHESAM}{white}{
\parbox{\dimexpr \linewidth-2\fboxsep-2\fboxrule}{%
\begin{tabular}{>{\centering\arraybackslash}p{16.5cm}}
\begin{minipage}{16.5cm}
\centering
{
~\\
~\\
\LARGE{\textbf{{Robust Deep Learning for Autonomous Driving}}}
~\\
~\\}
\end{minipage}	
\end{tabular}
}}}~\\

\vspace{0.3cm}

\makebox[\textwidth][c]{\centering{\textbf{\textsc{\large TH\`{E}SE dirig\'{e}e par :}}}}\\
\vspace{0.5cm}
\makebox[\textwidth][c]{\centering{\large{\textbf{[M. THOME Nicolas]}}} {\large{Professeur, Cedric, Cnam}}}\\

\makebox[\textwidth][c]{\centering{\textbf{\textsc{\large et co-encadr\'{e}e par :}}}}\\
\vspace{0.5cm}
\makebox[\textwidth][c]{\centering{\large{\textbf{[M. P\'EREZ Patrick]}}} {\large{Directeur scientifique, valeo.ai}}}\\

\vfill

\begin{minipage}[c]{0.9\linewidth}
   \centering
   \begin{tabular}{p{0.05cm} p{5.3cm} p{6.4cm} p{0.0cm} p{4.5cm}}
	\textbf{Jury} & & & \\
	& \textbf{Mme Florence d'ALCH\'E-BUC} & Professeur, LTCI, Télécom Paris & & Présidente du jury  \\
	& \textbf{M. Stéphane CANU} & Professeur, LITIS, INSA Rouen & & Rapporteur  \\
	& \textbf{M. Graham TAYLOR} & Professeur, University of Guelph &  &Rapporteur  \\
	& \textbf{M. Alex KENDALL} & Chercheur associé, University of Cambridge et co-fondateur, Wayve & &  Examinateur  \\
	& \textbf{M. Matthieu CORD} & Professeur, ISIR, Sorbonne Université et chercheur senior, valeo.ai & & Examinateur \\
	& \textbf{M. Patrick P\'EREZ} & Directeur scientifique, valeo.ai & & Co-directeur de thèse  \\
	& \textbf{M. Nicolas THOME} & Professeur, Cedric, Cnam & & Directeur de thèse\\

\end{tabular}  
  \end{minipage}
\hfill
\hspace{1.8cm}
  \begin{minipage}[c]{0.1\linewidth}
   \centering
   \includegraphics[height=5cm]{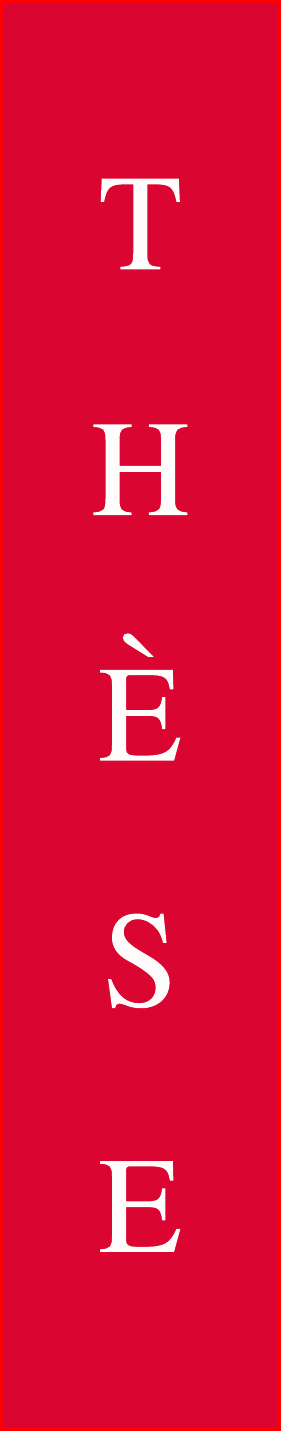}     
  \end{minipage}

\pagestyle{empty}
\cleardoublepage
\thispagestyle{empty}
\textbf{Declaration}\\

I, undersigned, Charles Corbière, hereby declare that the work presented in this manuscript is my own work, carried out under the scientific direction of Nicolas Thome (thesis director) and of Patrick Pérez (co-thesis director), in accordance with the principles of honesty, integrity and responsibility inherent to the research mission. The research work and the writing of this manuscript have been carried out in compliance with the French charter for Research Integrity.
This work has not been submitted previously either in France or abroad in the same or in a similar version to any other examination body.\\

{\raggedleft Paris (France), March 2022\\}
\begin{figure}[h]
    \includegraphics[width=.2\textwidth,right]{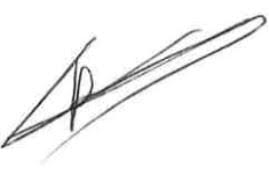}
\end{figure}

\vspace{2cm}

\textbf{Affidavit}\\

Je soussigné / soussignée, Charles Corbière, déclare par la présente que le travail présenté dans ce manuscrit est mon propre travail, réalisé sous la direction scientifique de Nicolas Thome (directeur de thèse) et de Patrick Pérez (co-directeur de thèse), dans le respect des principes d’honnêteté, d'intégrité et de responsabilité inhérents à la mission de recherche. Les travaux de recherche et la rédaction de ce manuscrit ont été réalisés dans le respect de la charte nationale de déontologie des métiers de la recherche.
Ce travail n'a pas été précédemment soumis en France ou à l'étranger dans une version identique ou similaire à un organisme examinateur.\\

{\raggedleft Fait à Paris (France), mars 2022\\}
\begin{figure}[h]
    \includegraphics[width=.2\textwidth,right]{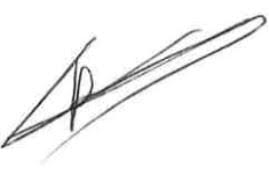}
\end{figure}


\pagestyle{empty}


\setlength{\baselineskip}{1.1\baselineskip} 
\setlength{\voffset}{0pt} 				
\setlength{\topmargin}{0pt}				
\setlength{\headheight}{30pt}			
\setlength{\headsep}{30pt}				
\setlength{\textheight}{620pt}		
\setlength{\footskip}{30pt}				
\setlength{\hoffset}{-15pt} 				
\setlength{\oddsidemargin}{0pt}	
\setlength{\evensidemargin}{0pt}	
\setlength{\textwidth}{420pt}			
\setlength{\marginparsep}{10pt}		
\setlength{\marginparwidth}{40pt}	
\parskip=4pt									   





\chapter*{Remerciements}
\addcontentsline{toc}{chapter}{Remerciements}
\adjustmtc
\markright{\MakeUppercase{Remerciements}}

Je souhaite remercier ici les nombreuses personnes qui m’ont accompagné, orienté, conseillé et soutenu tout au long de cette thèse.

Tout d’abord, je tiens à remercier mon directeur de thèse Nicolas Thome pour m’avoir accueilli au sein de son équipe et de m’avoir encadré pendant ces presque quatre ans. Sa disponibilité, son accompagnement et nos nombreuses discussions techniques m’ont été précieux. J’ai toujours pu compter sur son soutien à l’approche des échéances et l'en remercie encore.

Je voudrais également remercier mon co-directeur de thèse Patrick Pérez pour avoir cru en moi au moment de la création valeo.ai et des premiers recrutements au sein de son équipe de recherche. Patrick m'a été d'une grande aide via ses conseils avisés et pertinents tout au long de la thèse.

Conduire mes travaux entre l'équipe VERTIGO au Cnam et l'équipe de valeo.ai a été pour moi une expérience très enrichissante. Je remercie tous les chercheurs, doctorants et stagiaires de valeo.ai avec qui j'ai eu de nombreux échanges instructifs. En particulier, je tiens à remercier Tuan-Hung Vu et Antoine Saporta pour notre travail en commun sur ConDA, Andrei Bursuc pour ses conseils et son impressionnante capacité à indiquer des papiers pertinents de la littérature, et bien évidemment Arthur Ouaknine, Simon Roburin et Maxime Bucher pour les bons moments passés ensemble. Côté VERTIGO, je remercie mes camarades Laura Calem, Olivier Petit et Vincent Le Guen avec qui j'ai partagé une grande partie de mes aventures au Cnam. Mais aussi plus récemment Elias Ramzi, Loïc Themyr, Perla Doubinski, Yannis Karmin, Clément Rambour et Nicolas Audebert. Enfin, je remercie Marc Lafon pour tous nos échanges sur l'incertitude depuis son stage et lui souhaite bonne chance pour sa thèse.

Je remercie le jury pour l’intérêt porté à mes travaux, pour avoir accepté de relire ma thèse et d’assister à la soutenance qui finalise le travail de ces presque quatre années de recherche.

Pour finir, je tiens à remercier bien évidemment ma famille et mes amis pour leur soutien et leur présence tout au long de la thèse. Tout particulièrement Marlène Pivard avec qui je partage ma vie pour son importance vitale à mes côtés et ses encouragements. Mais aussi mes parents, Evelyne et Jean-Paul, qui m'ont toujours poussé à réaliser mes rêves tout en me laissant m'épanouir dans ce qui me plaisait. Enfin, la grande équipe des Equidés et la Troupe du Rire pour ces moment chaleureux d'une vie heureuse.

\begin{vcenterpage}
\chapter*{Abstract}
\addcontentsline{toc}{chapter}{Abstract}
\adjustmtc
\markright{\MakeUppercase{Abstract}}

The last decade's research in artificial intelligence and hardware development had a significant impact on the advance of autonomous driving. Yet, safety remains a major concern when it comes to deploying such systems in high-risk environments. Modern neural networks have been shown to struggle to correctly identify their mistakes and to provide over-confident predictions instead of abstaining when exposed to unseen situations. Progress on these issues is crucial to achieve certification from transportation authorities but also to arouse enthusiasm from users. 

The objective of this thesis is to develop methodological tools which provide reliable uncertainty estimates for deep neural networks. In particular, we aim to improve the detection of erroneous predictions and anomalies at test time.

First, we introduce a novel target criterion for model confidence, the true class probability (TCP). We show that TCP offers better properties than current uncertainty measures for the task of failure prediction. Since the true class is by essence unknown at test time, we propose to learn TCP criterion from data with an auxiliary model (\emph{ConfidNet}), introducing a specific learning scheme adapted to this context. The relevance of the proposed approach is validated on image classification and semantic segmentation datasets, demonstrating superiority with respect to strong uncertainty quantification baselines on failure prediction.

Then, we extend our learned confidence approach to the task of domain adaptation for semantic segmentation. A popular strategy, self-training, relies on selecting predictions on the unlabeled data and re-training a model with these pseudo-labels. Termed \emph{ConDA}, the proposed adaptation improves self-training methods by providing effective confidence estimates used to select pseudo-labels. To meet the challenge of domain adaptation, we equipped the auxiliary model with a multi-scale confidence architecture and supplemented the confidence loss with an adversarial training scheme to enforce alignment between confidence maps in source and target domains.

\sloppy
Finally, we consider the presence of anomalies and we tackle the ultimate practical objective of jointly detecting misclassification and out-of-distributions samples. To this end, we introduce \emph{KLoS}, an uncertainty measure based on evidential models and defined on the class-probability simplex. By keeping the full distributional  information, KLoS captures both uncertainty due to class confusion and lack of knowledge, which is related to out-of-distribution samples. We further improve performance across various image classification datasets by using an auxiliary model with a learned confidence approach.

\end{vcenterpage}

\begin{vcenterpage}
\chapter*{Résumé}
\markright{\MakeUppercase{Resume}}

Le véhicule autonome est revenu récemment sur le devant de la scène grâce aux avancées fulgurantes de l’intelligence artificielle. Pourtant, la sécurité reste une préoccupation majeure pour le déploiement de ces systèmes dans des environnements à haut risque. Il a été démontré que les réseaux de neurones actuels peinent à identifier correctement leurs erreurs et fournissent des prédictions sur-confiantes, au lieu de s'abstenir, lorsque exposés à des anomalies. Des progrès sur ces questions sont essentiels pour obtenir la certification des régulateurs mais aussi pour susciter l'enthousiasme des utilisateurs. 

L'objectif de cette thèse est de développer des outils méthodologiques permettant de fournir des estimations d'incertitudes fiables pour les réseaux de neurones profonds. En particulier, nous visons à améliorer la détection des prédictions erronées et des anomalies lors de l'inférence. Tout d'abord, nous introduisons un nouveau critère pour estimer la confiance d'un modèle dans sa prédiction : la probabilité de la vraie classe (TCP). Nous montrons que TCP offre de meilleures propriétés que les mesures d'incertitudes actuelles pour la prédiction d'erreurs. La vraie classe étant, par essence, inconnue à l'inférence, nous proposons d'apprendre TCP avec un modèle auxiliaire (\emph{ConfidNet}), introduisant un schéma d'apprentissage spécifique adapté à ce contexte. La qualité de l'approche proposée est validée sur des jeux de données de classification d'images et de segmentation sémantique., démontrant une supériorité par rapport aux méthodes de quantification incertitude utilisées pour la prédiction de d'erreurs.

Ensuite, nous étendons notre approche d'apprentissage de confiance à la tâche d'adaptation de domaine. Une stratégie populaire, l'auto-apprentissage, repose sur la sélection de prédictions sur données non étiquetées puis le réentraînement d'un modèle avec ces pseudo-étiquettes. Appelée \emph{ConDA}, l'adaptation proposée améliore la sélection de pseudo-labels grâce à des meilleures estimations de confiance. Afin de relever le défi de l'adaptation de domaine, nous avons équipé le modèle auxiliaire d'une architecture multi-échelle et complété la fonction de perte par un schéma d'apprentissage contrastif afin de renforcer l'alignement entre les cartes de confiance des domaines source et cible.

\sloppy
Enfin, nous considérons la présence d'anomalies et nous attaquons au défi pratique de la détection conjointe des erreurs de classification et des échantillons hors distribution. A cette fin, nous introduisons \emph{KLoS}, une mesure d'incertitude définie sur le simplexe et basée sur des modèles évidentiels. En conservant l'ensemble des informations de distribution, KLoS capture à la fois l'incertitude due à la confusion de classe et au manque de connaissance du modèle, cette dernière type d'incertitude étant liée aux échantillons hors distribution. En utilisant ici aussi un modèle auxiliaire avec apprentissage de confiance, nous améliorons les performances sur divers ensembles de données de classification d'images.

\end{vcenterpage}

\tableofcontents
\listoftables            
\addcontentsline{toc}{chapter}{List of tables}
\adjustmtc
\listoffigures           
\addcontentsline{toc}{chapter}{List of figures}
\adjustmtc

\nomenclature[A]{AI}{Artificial intelligence}
\nomenclature[A]{CV}{Computer vision}
\nomenclature[A]{NLP}{Natural language processing}
\nomenclature[A]{ML}{Machine learning}
\nomenclature[A]{DL}{Deep learning}
\nomenclature[A]{DNN}{Deep neural network}
\nomenclature[A]{BNN}{Bayesian neural network}
\nomenclature[A]{MC}{Monte-Carlo}
\nomenclature[A]{ConvNet}{Convolutional neural network}
\nomenclature[A]{ENN}{Evidential neural network}
\nomenclature[A]{KL}{Kullback-Leibler}
\nomenclature[A]{NLL}{Negative log-likelihood}
\nomenclature[A]{OOD}{Out-of-distribution}
\nomenclature[A]{MCP}{Maximum class probability}
\nomenclature[A]{TCP}{True class probability}
\nomenclature[A]{MLP}{Multi-layer Perceptron}
\nomenclature[A]{NN}{Neural network}
\nomenclature[A]{\textit{i.i.d.}}{independent and identically distributed}

\nomenclature[R]{$P$, $p$}{probability measure, density function}
\nomenclature[R]{$X$, $Y$}{random variables}
\nomenclature[R]{$\cX$, $\x$, $\x_n$}{input space, input}
\nomenclature[R]{$\cY$, $y$, $y_n$}{output space, output}
\nomenclature[R]{$\cD$}{training dataset}
\nomenclature[R]{$\cH$, $h$}{hypothesis space, hypothesis}
\nomenclature[R]{$\cL$, $\ell$}{loss function}
\nomenclature[R]{$\mathcal{N}$}{Normal distribution}
\nomenclature[R]{$\mathbb{R}$}{set of real numbers}
\nomenclature[R]{$\mathbb{E}[X]$}{expected value of random variable $X$}
\nomenclature[R]{$\textrm{Var}[X]$}{variance of random variable $X$}
\nomenclature[R]{$\bbH[X]$}{entropy of random variable $X$}
\nomenclature[R]{$\bbI[X,Y]$}{mutual information between random variables $X$ and $Y$}
\nomenclature[R]{$\mathbbm{1}[\cdot]$}{indicator function}
\nomenclature[R]{K}{number of classes}
\nomenclature[R]{$\mathbb{KL}(p ~\vert \vert~ q)$}{KL divergence from distribution $p$ to $q$}

\nomenclature[G]{$\balpha$}{concentration parameters of a Dirichlet distribution}
\nomenclature[G]{$\btheta$}{classification network's parameters}
\nomenclature[G]{$\bomega$}{auxiliary network's parameters}
\nomenclature[G]{$\bphi$}{confidence module's parameters}
\nomenclature[G]{$\bpi$}{parameters of a categorical distribution}

\printnomenclature
\addcontentsline{toc}{chapter}{Nomenclature}

\mainmatter
\pagestyle{fancy}


\chapter{Introduction}
\label{chap1}
\minitoc
\input{chapter01/chapter01}

\chapter{Uncertainty Estimation in Deep Learning for Classification}
\label{chap2}
\adjustmtc
\input{chapter02/chapter02}

\chapter{Learning A Model's Confidence via An Auxiliary Model}
\label{chap3}
\input{chapter03/chapter03}

\chapter{Self-Training with Learned Confidence for Domain Adaptation}
\label{chap4}
\input{chapter04/chapter04}

\chapter{Simultaneous Detection of Misclassifications and Out-of-Distribution Samples with Evidential Models}
\label{chap5}
\input{chapter05/chapter05}

\chapter{Conclusion and Perspectives}
\label{chap6}
\input{chapter06/chapter06}

\cleardoublepage\phantomsection
\addcontentsline{toc}{chapter}{Bibliography}
\thispagestyle{empty}
\clearpage
\bibliographystyle{IEEEtran}
\bibliography{references}
\thispagestyle{empty}

\chapter*{Résumé de la Thèse}
\label{synthese}
\markboth{\MakeUppercase{Résumé de la Thèse}}{\MakeUppercase{Résumé de la Thèse}}
\addcontentsline{toc}{chapter}{Résumé de la Thèse}
\input{synthese}

\newpage

\begin{appendices}
    \addtocontents{toc}{\protect\setcounter{tocdepth}{0}}
    
    \chapter{Additional Analysis for Failure Prediction Experiments}
    \label{appxA}
    \markright{\MakeUppercase{Appendix A}}
    \input{appendixA/appendixA}
    
    \chapter{Details and Further Experiments for KLoS}
    \label{appxB}
    \markright{\MakeUppercase{Appendix B}}
    \input{appendixB/appendixB}

\end{appendices}

\addtocontents{toc}{\protect\setcounter{tocdepth}{2}}


\cleardoublepage\phantomsection
\thispagestyle{empty}
\newpage $\ $
\newpage
\thispagestyle{empty}
\setlength{\baselineskip}{11pt}
\setpapersize{A4}

\setmarginsrb{0mm}{0mm}{15mm}{0mm}{0mm}{0mm}{0mm}{0mm}
\begin{center}

\fbox{
\begin{tabular}{p{3cm} |p{10.2cm}|p{3cm}}
		\begin{minipage}{3cm}
			\includegraphics[width=1\linewidth]{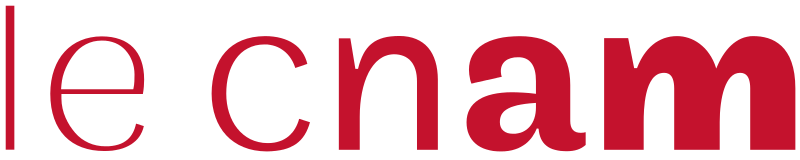}  
		\end{minipage}
	&
		\begin{minipage}{10.2cm}
			\begin{center}
				\textbf{\Large{Charles CORBIERE}}\\
				\vspace{0.2cm}
				\textbf{\Large{Robust Deep Learning \\ for Autonomous Driving}}\\
				\vspace{0.2cm}
			\end{center}
		\end{minipage}
		&
		\begin{minipage}{3cm}
			\includegraphics[width=1\linewidth]{images/logo_hesam.png} 
		\end{minipage}
\end{tabular}
}
\end{center}
\vspace{0.3cm}
\begin{center}

\fcolorbox{redHESAM}{white}{
\begin{minipage}{17.5cm}
\vspace{0.2cm}
\textbf{\large{Résumé :}} Le véhicule autonome est revenu récemment sur le devant de la scène grâce aux avancées fulgurantes de l’intelligence artificielle. Pourtant, la sécurité reste une préoccupation majeure lorsqu'il s'agit de déployer de tels systèmes dans des environnements à haut risque. L'objectif de cette thèse est de développer des outils méthodologiques permettant de fournir des estimations d'incertitudes fiables pour les réseaux de neurones profonds. Tout d'abord, nous introduisons un nouveau critère pour estimer la confiance d'un modèle dans sa prédiction: la probabilité de la vraie classe (TCP). Nous montrons que TCP offre de meilleures propriétés que les mesures d'incertitudes actuelles pour la prédiction d'erreurs. La vraie classe étant par essence inconnue pour un exemple de test, nous proposons d'apprendre TCP à partir de données avec un modèle auxiliaire en introduisant un schéma d'apprentissage spécifique adapté à ce contexte. La qualité de l'approche proposée est validée sur des jeux de données de classification d'images et de segmentation sémantique. Nous étendons ensuite notre approche d'apprentissage de confiance à la tâche d'adaptation de domaine où elle améliore la sélection de pseudo-labels dans les méthodes d'auto-apprentissage. Enfin, nous nous attaquons au défi de la détection conjointe des erreurs de classification et des échantillons hors distribution en présentant une nouvelle mesure d'incertitude définie sur le simplexe et basée sur des modèles évidentiels. \\
\\
\\
\textbf{\large{Mots clés:}} apprentissage profond, incertitude, confiance, robustesse, voiture autonome
\vspace{0.2cm}
\end{minipage}
}
\end{center}

\vspace{0.2cm}

\begin{center}
\fcolorbox{redHESAM}{white}{
\begin{minipage}{17.5cm}
\vspace{0.2cm}
\textbf{\large{Abstract :}} The last decade's research in artificial intelligence had a significant impact on the advance of autonomous driving. Yet, safety remains a major concern when it comes to deploying such systems in high-risk environments. The objective of this thesis is to develop methodological tools which provide reliable uncertainty estimates for deep neural networks. First, we introduce a new criterion to reliably estimate model confidence: the true class probability (TCP). We show that TCP offers better properties for failure prediction than current uncertainty measures. Since the true class is by essence unknown at test time, we propose to learn TCP criterion from data with an auxiliary model, introducing a specific learning scheme adapted to this context. The relevance of the proposed approach is validated on image classification and semantic segmentation datasets. Then, we extend our learned confidence approach to the task of domain adaptation where it improves the selection of pseudo-labels in self-training methods. Finally, we tackle the challenge of jointly detecting misclassification and out-of-distributions samples by introducing a new uncertainty measure based on evidential models and defined on the simplex. \\
\\
\\
\textbf{\large{Keywords:}} deep learning, uncertainty, confidence, robustness, autonomous driving
\vspace{0.2cm}

\end{minipage}
}
\end{center}
\setlength{\voffset}{0pt}


\end{document}

%% file: chapter01/chapter01.tex
\section{Context}
\label{chap1:sec:context}

From Rosenblatt's Perceptron \cite{Rosenblatt1958ThePA} to the rise of attention-based neural networks (`Transformers') \cite{attention2017}, the field of artificial intelligence (AI) has been experiencing alternative periods of hype cycles followed by disappointment, reduced funding and interest. However, the current advances in deep learning (DL) \cite{lecun2015deeplearning} not only raised interest among AI researchers but also drive technological progress in many science disciplines, including physics, biology, as well as in manufacturing and other industrial applications\footnote{To grasp the impact of AI on today's world, the `State of AI Report' published every year is a thorough compilation of developments in research, industry and politics: \url{https://www.stateof.ai}.}. Since the stunning victory of a convolutional neural network architecture, AlexNet \cite{krizhevsky2012imagenet}, at the Large Scale Visual Recognition Challenge (LSVRC) in 2012, deep learning is now ubiquitous in the fields of computer vision (CV) \cite{He2015,SSD2016,ChenPK0Y16}, natural language processing (NLP) \cite{MikolovKBCK10,DevlinCLT19}, speech recognition \cite{hinton2012deep,hannun2014speech} and reinforcement learning \cite{Silver_2016} (see \cref{chap1:fig:ai_applications}), accounting for a larger portion of papers published in each respective conference.

\begin{figure}[ht]
\centering
    \includegraphics[width=0.95\linewidth]{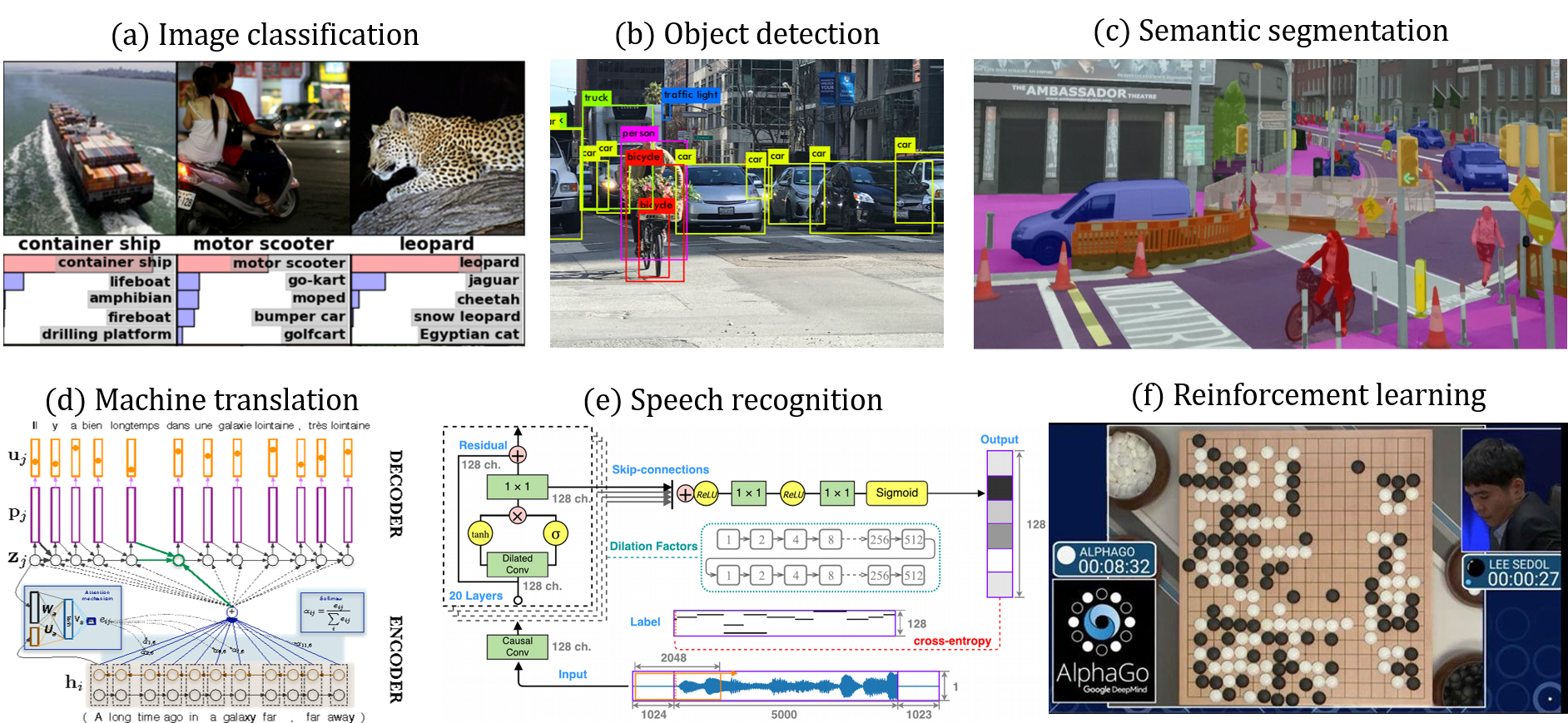}
    \caption[Examples of successful applications of AI]{\textbf{Examples of successful applications of AI:} in computer vision (a,b,c), natural language processing (d), speech recognition (e) and reinforcement learning mastering the game of Go (f). \textit{Image credits to Krizhevsky et al. \cite{krizhevsky2012imagenet}, Hui \cite{HuiYoLO}, Neuhold et al. \cite{neuhold-iccv2017}, van den Oord et al. \cite{WaveNet}, Garcia-Martinez et al. \cite{Garcia-Martinez16} and DeepMind \cite{AlphaGo2020}.}}
\label{chap1:fig:ai_applications}
\end{figure}

Recent breakthroughs in computer vision thanks to deep learning largely explain the spectacular revival of autonomous driving with major tech players such as Waymo, Tesla, Baidu and Yandex investing in self-driving car programs. Deep learning is now being used in various autonomous driving modules. In perception, convolutional neural networks process information coming from visual cameras to understand a scene and detect crucial aspects of the environment in real-time \cite{JanaiGBG17}: nature and position of other vehicles, bicycles, pedestrians; position and meaning of road markings, signs, lights; navigable space; position of obstacles; etc. To enrich scene analysis, perception systems in autonomous driving complements traditional cameras with active sensors such as radar sensors and LiDARs (Light Detection and Ranging) measuring sparse but direct tri-dimensional aspects of dynamic scenes. Perception with these sensors also tend to rely more and more and deep learning models \cite{Ouaknine_2021_ICCV,puy20flot}, hence they can be combined to improve the quality of the higher-level decision system via sensor fusion \cite{Varghese2015OverviewOA}. As one of the world leaders in automotive sensors, Valeo, which is funding this thesis, is positioned at the heart of this current revolution, developing high-quality LiDARs with their SCALA technology. While deep learning is mostly used in perception modules, promising research aims to apply it also in planning, such as trajectory forecasting for the objects present in the car's environment. End-to-end approaches, from perception to control, by predicting steering angle and acceleration is also starting to emerge with deep reinforcement learning \cite{marin2020}. While these incredible progresses are undeniable, at the time of writing of this manuscript, robotaxis are not deployed yet and many challenges still need to be resolved for autonomous driving before large scale commercialisation, in particular by addressing issues related to safety.

\section{Motivations}
\label{chap1:sec:motivations}

Despite its clear benefits to many applications, the deployment of machine learning (ML) models in high-stakes environments raises serious questions about its impacts on our society. \emph{AI safety} \cite{AmodeiOSCSM16,rudner2021} is an area of research that aims to identify causes of unintended and harmful behaviour in machine learning systems and to develop tools to ensure these systems work safely and reliably. Such behaviour may emerge from machine learning systems \cite{Hendrycks2021UnsolvedPI} when:
\begin{itemize}
    \item exposed to unusual situations, distributional changes on inputs \cite{wilds2020} or long-tail scenarios \cite{Anguelov2019} (\emph{`Robustness'});
    \item subjected to corruptions during training, such as data poisoning \cite{Steinhardt2017}, or during inference with adversarial attacks from malicious opponents \cite{realadvattacks2021} (\emph{`External Safety'});
    \item the learning objective is not aligned with human values -- which may be hard to specify though -- or the model uses shortcuts during optimization \cite{Geirhos2020a} which are not transferable to more challenging testing conditions (\emph{`Alignment'}).
\end{itemize}
When deployed in the wild, a machine learning system should be able to detect these hazardous cases (\emph{`Monitoring'}) and estimate whether it is confident in its prediction in order to prevent accidents.  

\sloppy
Accidents occurring with autonomous cars are typical examples where repercussions can be catastrophic. During the perception step, low-level feature extraction such as image segmentation and image localisation are used to process raw sensory inputs \cite{BojarskiTDFFGJM16}. Outputs of such models are then fed into higher-level decision-making procedures. However, mistakes done by lower-level machine learning components can propagate up the decision-making process and lead to devastating results. One striking example is the tragic incident which happened on May 7\textsuperscript{th} 2016 near Williston (Florida, USA) and resulted in the first death caused by a car with highly automated driving assistance \cite{NTSB2017}. Tesla, the car manufacturer, stated that the accident originated from the vision system which incorrectly classified the white side of a turning trailer truck as a bright sky\footnote{\url{https://www.tesla.com/blog/tragic-loss}} (\cref{chap1:fig:tesla_accident}). Visual signals can indeed be fooled or not adapted in some arduous conditions such as heavy raining, bright sky or night time. As the NTSB noted in their report \cite{NTSB2017}, ``introducing automation in complex and unstructured environment is very challenging" and they recommended to manufacturers of vehicles equipped with automation systems to ``incorporate system safeguards that limit the use of automated vehicle control systems to those conditions for which they were designed". Since then, several other accidents and crashes with self-driving cars continue to occur, including the first recorded case of a pedestrian fatality in 2018 \cite{NTSB2019}. Among the factors explaining the collusion, the NTSB report stated that the system of the Uber test vehicle failed to recognize the woman, first identifying her as an unknown object, next as a vehicle, then as the bicycle she was pushing. Correctly monitoring and assessing system confidence in its predictions appears to be more than necessary to safely deploy ML models in high-stakes environments \cite{mcallister2017}. Progress on these issues is crucial in autonomous driving to achieve certification from transportation authorities but also to arouse enthusiasm from users.

\begin{figure}[t]
\centering
\begin{minipage}[c]{0.45\linewidth}
\centering
    \includegraphics[width=\linewidth]{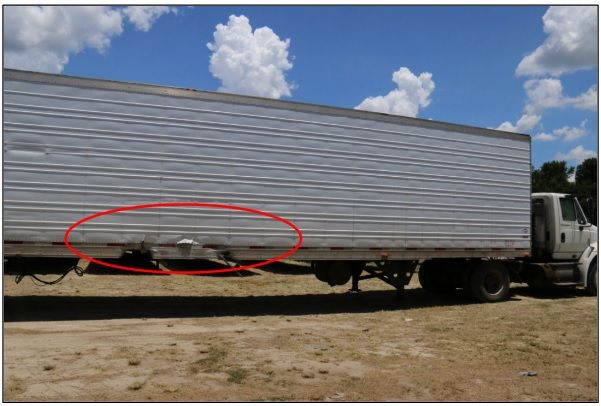}
    \subcaption{Photo of the involved white truck}
    \label{chap1:fig:tesla_accident_photo}
\end{minipage}%
\hspace{0.1cm}
\begin{minipage}[c]{0.45\linewidth}
\centering
    \includegraphics[width=\linewidth]{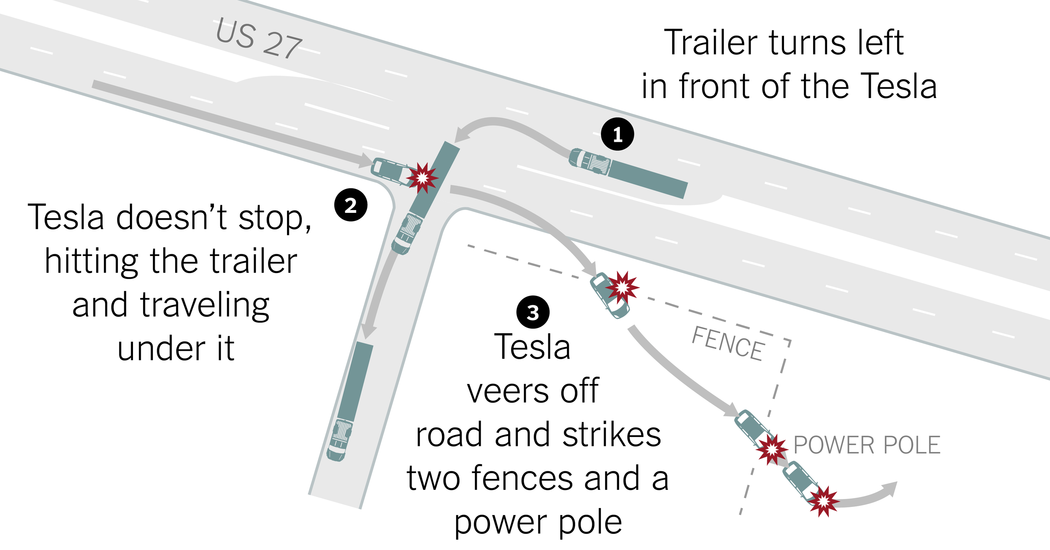}
    \vspace{0.5cm}
    \subcaption{Outline of the accident}
    \label{chap1:fig:tesla_accident_scheme}
\end{minipage}%
\caption[Tesla's deadly crash in May 2016: photo of the white truck and outline of the accident]{Tesla's deadly crash in May 2016: photo of the white truck confused with bright sky by Tesla's vision system (\cref{chap1:fig:tesla_accident_photo}) and outline of the accident (\cref{chap1:fig:tesla_accident_scheme}). Image credits: New York Times\footnotemark.}
\label{chap1:fig:tesla_accident}
\end{figure}
\footnotetext{\url{https://www.nytimes.com/interactive/2016/07/01/business/inside-tesla-accident.html}}

\begin{figure}[t]
\centering
\begin{minipage}[c]{0.30\linewidth}
\centering
    \includegraphics[width=0.9\linewidth]{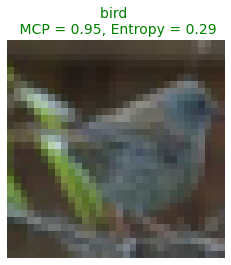}
    \subcaption{Correct prediction}
    \label{chap1:fig:overconfidence_correct}
\end{minipage}%
\hspace{0.1cm}
\begin{minipage}[c]{0.30\linewidth}
\centering
    \includegraphics[width=0.9\linewidth]{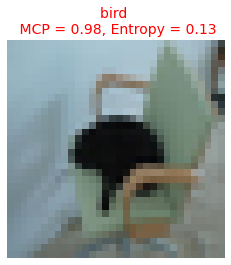}
    \subcaption{In-distribution error}
    \label{chap1:fig:overconfidence_error}
\end{minipage}%
\hspace{0.1cm}
\begin{minipage}[c]{0.30\linewidth}
\centering
    \includegraphics[width=0.9\linewidth]{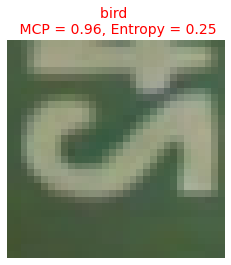}
    \subcaption{Out-of-distribution sample}
    \label{chap1:fig:overconfidence_ood}
\end{minipage}%
\caption[Comparison of confidence estimates for a model trained on CIFAR-10 dataset between a correct prediction, an in-distribution (CIFAR-10) error, and an out-of-distribution sample taken from SVHN dataset]{Comparison of confidence estimates using MCP or predictive entropy for a model trained on CIFAR-10 dataset \cite{Krizhevsky09} between a correct prediction (\cref{chap1:fig:overconfidence_correct}), an in-distribution error (\cref{chap1:fig:overconfidence_error}), and an out-of-distribution (OOD) sample taken from SVHN dataset \cite{svhn-dataset}. The model assigned a higher MCP value and a lower entropy, hence a large confidence score, to the error and the OOD sample than to the correct prediction, which is not desirable in uncertainty estimation.}
\label{chap1:fig:overconfidence}
\end{figure}

Knowing when a model doesn't know is important to improve trustworthiness and safety \cite{mcallister2017}. By assigning high levels of uncertainty to erroneous predictions, a ML system could have been able to avoid previous catastrophes by sending a trigger alarm or giving back control to users. Related to the example of sensor fusion mentioned earlier, when evaluating high uncertainty for a prediction output by the visual camera during night time, a system could decide to rely more on active sensors predictions which are more robust to these light conditions. One key of a good uncertainty-based fusion of multiple sensor's predictions is to ensure that probabilities are well \emph{calibrated} (see \cref{chap2:subsec:calibration} for details about the meaning of probability calibration). One would also like the confidence criterion to correlate successful predictions with high values. Some paradigms, such as self-training with pseudo-labeling \cite{lee-icml2013, Li_2019_CVPR}, consist in picking and labeling the most confident samples before retraining the network accordingly. The performance improves by selecting successful predictions thanks to an accurate confidence criterion. 

For practical systems, it may also be important to understand what the model does not know. Common classification in uncertainty estimation distinguishes two types of uncertainty.
Aleatoric uncertainty, also termed \emph{data} uncertainty, is due to the inherent stochasticity of the outcome of an experiment. 
This type of uncertainty arises due to class confusion, sensor noise, or non-discriminant features, such as in \cref{chap1:fig:overconfidence_error} where the model confuses a cat on a chair with a bird (\emph{known-unknown}). Epistemic uncertainty refers to uncertainty caused by a lack of knowledge of the model, for instance an input from another distribution or from an unknown class (\emph{unknown-unknown}), as illustrated in \cref{chap1:fig:overconfidence_ood}. A good estimation of uncertainty is also useful to discriminate unusual situations from regular inputs, such as driving conditions in snowy roads in Russia while the car's ML system has been trained on data collected in California. By providing more data to a model, we can reduce this uncertainty. Identifying samples with large epistemic uncertainty is also beneficial for classification improvements in active learning \cite{gal-active2017} and for efficient exploration in reinforcement learning \cite{Gal2016}.

While confidence estimation\footnote{The terms \emph{uncertainty} and \emph{confidence} estimation are used alternatively in this manuscript, the latter referring to the opposite of uncertainty.} has a long history in machine learning \cite{Chow1957AnOC,zaragoza1998confidence,Blatz2004,elyaniv10a}, a series of recent works showed that modern neural networks (NNs) suffer from several conceptual drawbacks which make them unreliable \cite{hendrycks17baseline,nguyen2015,Gal2016,kendall2015bayesian,Hein_2019_CVPR}. In classification, the output of the last layer is fed to the softmax function, which produces a probability distribution over class labels. However, with modern NNs, these  probabilities have been shown to be non-calibrated \cite{guo2017} which makes NNs unsuited for a larger decision-making pipelines. To obtain uncertainty estimates, a widely used baseline with NNs is to take the value of the predicted class’ probability \cite{hendrycks17baseline}, namely the \emph{maximum class probability} (MCP), or to use the predicted entropy given the predicted probability distribution. But as shown in \cref{chap1:fig:overconfidence}, these measures can produce high confident predictions for in-distribution errors, which hardens error detection or selective classification where one would filter out these samples. When deployed in real conditions, machine learning models often encounter samples that are away from the training distribution, such as covariate shift or new classes. However, NNs are known to be brittle to distribution shifts \cite{wilds2020} with their prediction performances severely decreasing as they tend to rely on spurious correlations \cite{Geirhos2020a}. Multiple works also showed that NNs provide over-confident predictions for samples far from the training data \cite{Hein_2019_CVPR}, including fooling images \cite{nguyen2015} or adversarial inputs \cite{Szegedy2014}. \cref{chap1:fig:overconfidence_ood} shows the example of an out-of-distribution sample taken from SVHN dataset \cite{svhn-dataset} and predicted as a bird with high confidence by a model trained on CIFAR-10 dataset \cite{Krizhevsky09}. 

The development of principled methods for deep learning models such as Bayesian neural networks (BNNs) \cite{Ghahramani2015Nature,Gal2016,NIPS2019_9472} and ensembles \cite{deepensembles2017} enable deep neural networks to capture epistemic uncertainty more accurately. Such as with ensemble, predictions with BNNs are obtain by averaging multiple forward pass due to their finite approximation of the predictive distribution (Monte Carlo sampling). But this comes at the expense of an increased computational cost to obtain uncertainty estimates. In addition, recent works \cite{ovadia2019,Ashukha2020Pitfalls} show they still fall short in giving useful estimates of their predictive uncertainty. Despite these progress in uncertainty estimation, there remains a gap to be filled in detecting in-distribution errors and abnormal samples to avoid serious repercussions when deploying a fleet of driverless robotaxis.

\section{Contributions and outline}
\label{chap1:sec:contrib}

In this thesis, we tackle the challenge of providing reliable uncertainty estimates along with deep neural network predictions with applications for autonomous driving. In particular, we aim to improve the detection of erroneous predictions at test time by distinguishing them from correct ones. Errors can be of different natures and the following contributions will firstly address the task of in-distribution misclassification detection, also known as \emph{failure prediction} (\cref{chap3}). Along with the detection of such examples at test time, we also elaborate on leveraging our proposed approach in the case of domain adaptation (\cref{chap4}), where self-training approaches rely on uncertainty estimates to select samples in the re-labelling phase. Finally, we consider the presence of anomalies and consequently propose to detect both in-distribution errors and out-of-distribution samples with a single uncertainty measure (\cref{chap5}). 

\paragraph{Outline.} In regards with the challenges mentioned above, our contributions are the following:
\begin{itemize}
    \item \hyperref[chap3]{Chapter 3: Learning A Model’s Confidence via An Auxiliary Model} \\
After exposing the limits of standard uncertainty measures with deep neural networks in classification, we define a new confidence criterion, \emph{True Class Probability}, which provides theoretical guarantees and empirical evidence for confidence estimation. We  propose  to  design  an  auxiliary  neural  network, coined \emph{ConfidNet}, which aims to learn this confidence criterion from data. An exploration of the classification-with-rejection framework strengthens the rationale of the proposed approach. Extensive experiments are conducted for validating the relevance of the proposed approach on image classification and semantic segmentation datasets. An analysis of the impact of loss function, criterion and learning scheme is also presented.
\\ 
\item \hyperref[chap4]{Chapter 4: Self-Training with Learned Confidence for Domain Adaptation} \\
Self-training has recently proven a potent strategy to improve the effectiveness of Unsupervised Domain Adaptation (UDA) in semantic segmentation. This line of work mostly relies on the generation of pseudo-labels over the unannotated target domain to incorporate target images and learn a better segmentation adaptation model. A crucial issue is to base the pseudo-label selection on reliable confidence measures. We propose to adapt our learned confidence approach to estimate the confidence of the segmentation network in its predictions and to use these confidence estimates as a criterion  for  pseudo-label selection. Named \emph{ConDA}, the proposed adaptation of our original approach to this new context includes two further contributions: (1) an adversarial training scheme to reduce the gap between confidence maps in source and target domains; (2) an enhanced architecture for the confidence network to perform multi-scale confidence estimation. We show that this strategy produces more accurate pseudo-labels and outperforms strong baselines on challenging UDA segmentation benchmarks.
\\
\item \hyperref[chap5]{Chapter 5: Simultaneous Detection of Misclassifications and Out-of-Distribution Samples with Evidential Models} \\
Beyond errors due to misclassifications by deep neural networks, models may encounter data that is unlike the model’s training data when deployed in the wild. In this chapter, we tackle the task of jointly detecting errors and anomalies in a single uncertainty measure. To this end, we leverage the second-order uncertainty representation provided by evidential models \cite{sensoy2018,malinin2018}, a Bayesian method based on subjective logic, and we introduce \emph{KLoS}, a KL-divergence criterion defined on the class-probability simplex. We show that KLoS quantifies in-distribution and out-of-distribution uncertainty more accurately than first-order measures such as the predictive entropy. In a similar spirit to the previous contribution, we design an auxiliary neural network, \emph{KLoSNet}, to learn a refined measure directly aligned with the evidential training objective. Our experiments show that KLoSNet acts as a class-wise density estimator and outperforms current first-order and second-order uncertainty measures to simultaneously detect misclassifications and OOD samples. We study the impact of the choice of OOD training samples on our method and concurrent measures, which sheds a new light on the impact of the vicinity of this data with OOD test data.

\end{itemize}

Before delving in the core of the thesis, we present in \cref{chap2} an overview of the recent progress in uncertainty estimation with deep neural networks, including a thorough characterization of the source of uncertainty, methods to model uncertainty and the various ways of evaluating the quality of uncertainty estimates. Finally, in \cref{chap6}, we conclude this thesis with an overview of the contributions of each chapter and we propose several interesting perspectives for future works.

\section{Related publications}
\label{chap1:sec:publi}
This thesis is based on the material published in the following papers: \\

\begin{tabular}{p{0.90\textwidth} c}
\toprule
Publication & Chapter \\
\midrule
Charles Corbière, Nicolas Thome, Avner Bar-Hen, Matthieu Cord, Patrick Pérez. ``Addressing Failure Prediction by Learning Model Confidence", in \textit{Advances in Neural Information Processing Systems (NeurIPS)}, 2019.  & \ref{chap3} \\
\midrule
Charles Corbière, Nicolas Thome, Antoine Saporta, Tuan-Hung Vu, Matthieu Cord, Patrick Pérez. ``Confidence Estimation via Auxiliary Models", in \textit{IEEE Transactions on Pattern Analysis and Machine Intelligence}, 2021.  & \ref{chap3},\ref{chap4} \\
\midrule
Charles Corbière, Marc Lafon, Nicolas Thome, Matthieu Cord, Patrick Pérez. ``Beyond First-Order Estimation with Evidential Models for Open-World Recognition", \textit{ICML 2021 Workshop on Uncertainty and Robustness in Deep Learning.} & \ref{chap5} \\
\bottomrule
\end{tabular}

%% file: chapter02/chapter02.tex
\begin{center}
  \textsc{Chapter Abstract}
\end{center}
\begin{quote}
\noindent \textit{In this chapter, we propose a general overview of the literature regarding uncertainty estimation with deep neural networks. The sources of uncertainties are primarily discussed in \cref{chap2:sec:sources_unc} with a formalisation in the context of supervised learning. Traditionally, uncertainty is modeled in a probabilistic way and probabilistic methods have always been perceived as the natural tool to handle uncertainty. In particular, the Bayesian framework provides a probabilistic representation of uncertainty by incorporating degrees of belief. After reviewing the basic concepts of deep learning, we will see in \cref{chap2:sec:representation} how Bayesian approaches model the different sources of uncertainty. In addition, we enumerate the measures proposed along with standard and Bayesian approaches to quantify uncertainty. We describe their behavior and, above all, their limits that will be addressed in this thesis. While reliable uncertainty estimates are crucial in many safety-critical applications, their evaluation remains challenging as the ‘ground truth’ uncertainty estimates are usually not available. One would expect them to truly reflect probabilities in a multi-sensor perception system where late fusion rely on these probabilities. On the other hand, the goal may be to detect errors or anomalies and thus a reliable ranking between correct predictions and abnormal samples is desired. In \cref{chap2:sec:evaluation}, we present the existing tasks commonly used in the literature to evaluate the quality of uncertainty estimates with deep neural networks.}
\end{quote}

\clearpage
\minitoc[tight]

\section{Sources of uncertainty}
\label{chap2:sec:sources_unc}

Uncertainty can arise from various reasons and may require a different handling depending of their nature. After introducing a traditional categorization of sources in uncertainty in machine learning literature, we dive into a more precise identification within the setting of supervised learning.

\subsection{General characterisation}
\label{chap2:subsec:definition}

The nature of uncertainties has been a topic of discussion by statisticians~\cite{lindley2000}, economists \cite{Knight1921}, engineers \cite{engineerstatistic1970} and other specialists facing random processes. In the machine learning literature \cite{Kiureghian2009,kendall2017,hullermeier2020,SENGE201416,malinin2018}, sources of uncertainty are traditionally characterized as either \emph{aleatoric} or \emph{epistemic}. When an outcome of an experiment may variate due to intrinsic randomness of a phenomenon, \eg coin flipping, we refer to aleatoric uncertainty. Another term used alternatively is \emph{data uncertainty}, which emphasizes that the stochasticity is inherent to the observed object rather than the model. This type of uncertainty arises due to class confusion, noise, non-discriminant features \eg, sun glare or rain drop in autonomous driving images (see \cref{chap2:fig:aleatoric_unc}).  

\begin{figure}[ht]
\centering
\begin{minipage}[c]{0.54\linewidth}
\centering
    \includegraphics[width=\linewidth]{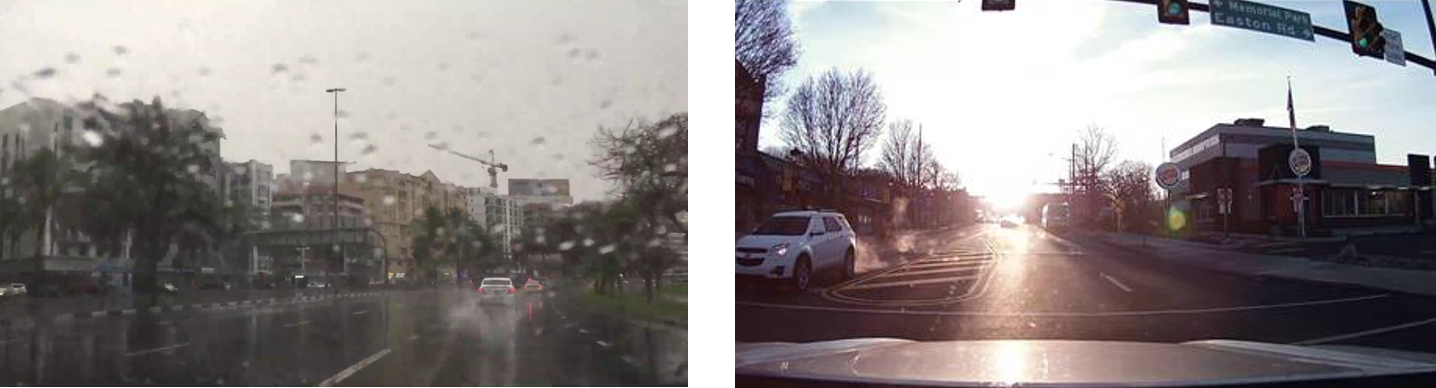}
    \subcaption{Rain drops and sun glare}
    \label{chap2:fig:rain_sunglare}
\end{minipage}%
\hspace{0.1cm}
\begin{minipage}[c]{0.22\linewidth}
\centering
    \includegraphics[width=\linewidth]{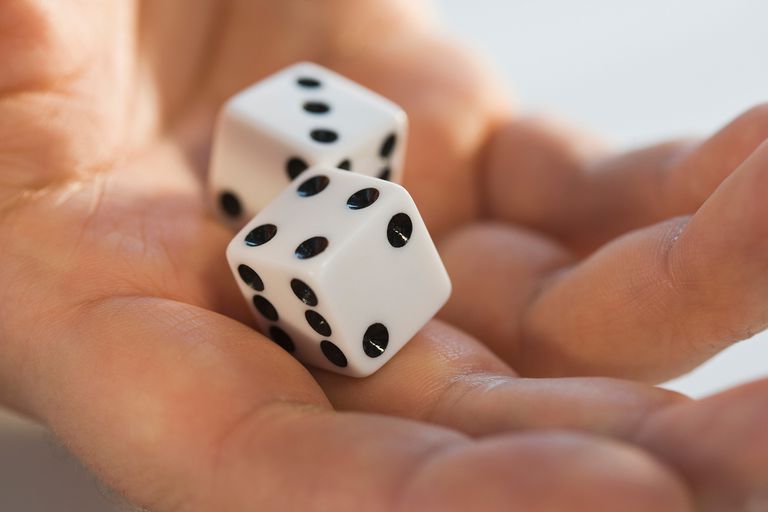}
    \subcaption{Dice roll}
    \label{chap2:fig:dice_roll}
\end{minipage}%
\hspace{0.1cm}
\begin{minipage}{0.21\linewidth}
\centering
    \includegraphics[width=\linewidth]{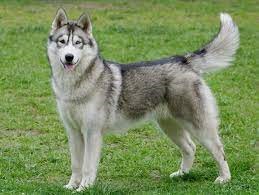}
    \subcaption{Dog/wolf confusion}
    \label{chap2:fig:husky}
\end{minipage}
\caption[Examples of aleatoric uncertainty in computer vision]{\textbf{Examples of aleatoric uncertainty in computer vision.} (a) Adverse weather conditions in autonomous driving \cite{Zendel_2018_ECCV} harden perception for image-based modules; (b) a dice roll is inherently stochastic due to the extreme sensitivity to initial conditions that cannot be measured with sufficient precision; (c) although being a breed of dog, huskies share many similarities with wolves.}
\label{chap2:fig:aleatoric_unc}
\end{figure}

Epistemic uncertainty refers to uncertainty caused by a lack of knowledge, hence intricately linked to the model representing the random process. For models trained on computer vision tasks, we show some examples of samples with large epistemic uncertainty in \cref{chap2:fig:epistemic_unc}. In contrast with aleatoric uncertainty, it can be reduced by providing additional information, here in the form of training data. To illustrate the distinction between the types of uncertainty, let us consider weather forecasting \cite{kull2014} which discriminates a predicted probability score from the uncertainty in that prediction: “[...] a weather forecaster can be very certain that the chance of rain is 50\%; or her best estimate at 20\% might be very uncertain due to lack of data.”. Here, the amount of aleatoric uncertainty corresponds to 50\%, due to the complex and multi-variate factors resulting to rain; while the weather forecaster acknowledges not being confident in his prediction (20\%), which relates to epistemic uncertainty. With machine learning predictions, epistemic uncertainty is expected to be high for samples far from the training distribution. Also known as distribution shifts, mismatch between input data in deployment stage and the original training distribution can arise in many real-world tasks \cite{wilds2020}. By providing more data to a model, we can reduce this type of uncertainty. In summary, epistemic uncertainty refers to the reducible part of the total uncertainty, whereas aleatoric uncertainty refers to the non-reducible part.

\begin{figure}[ht]
\centering
\begin{minipage}[c]{0.32\linewidth}
\centering
    \includegraphics[width=\linewidth]{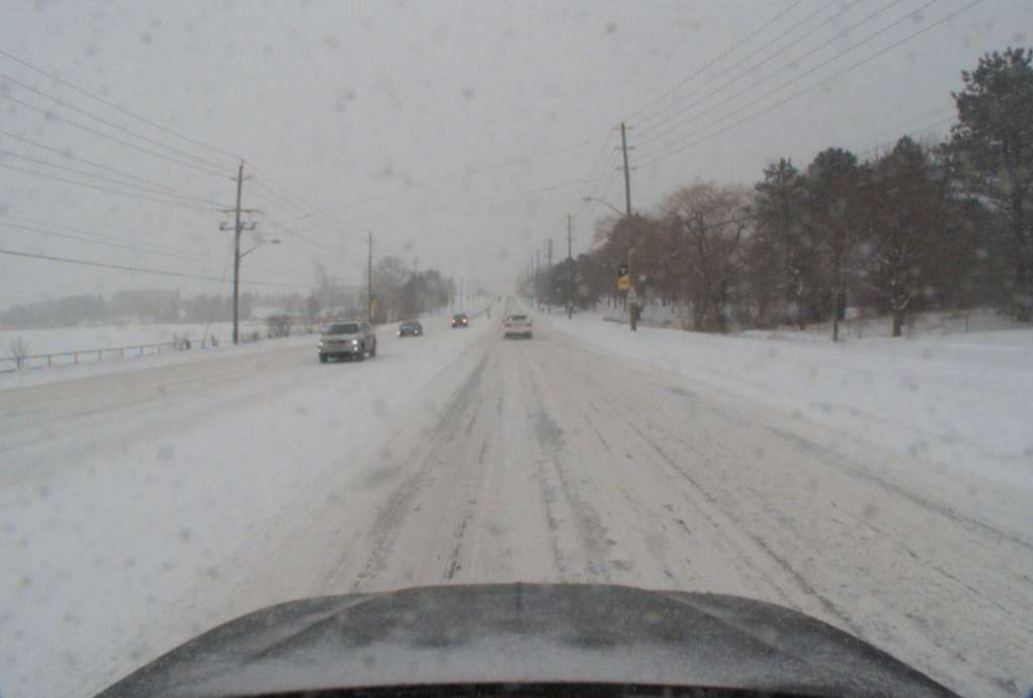}
    \subcaption{Snowy road}
    \label{chap2:fig:snow_road}
\end{minipage}%
\hspace{0.1cm}
\begin{minipage}[c]{0.32\linewidth}
\centering
    \includegraphics[width=\linewidth]{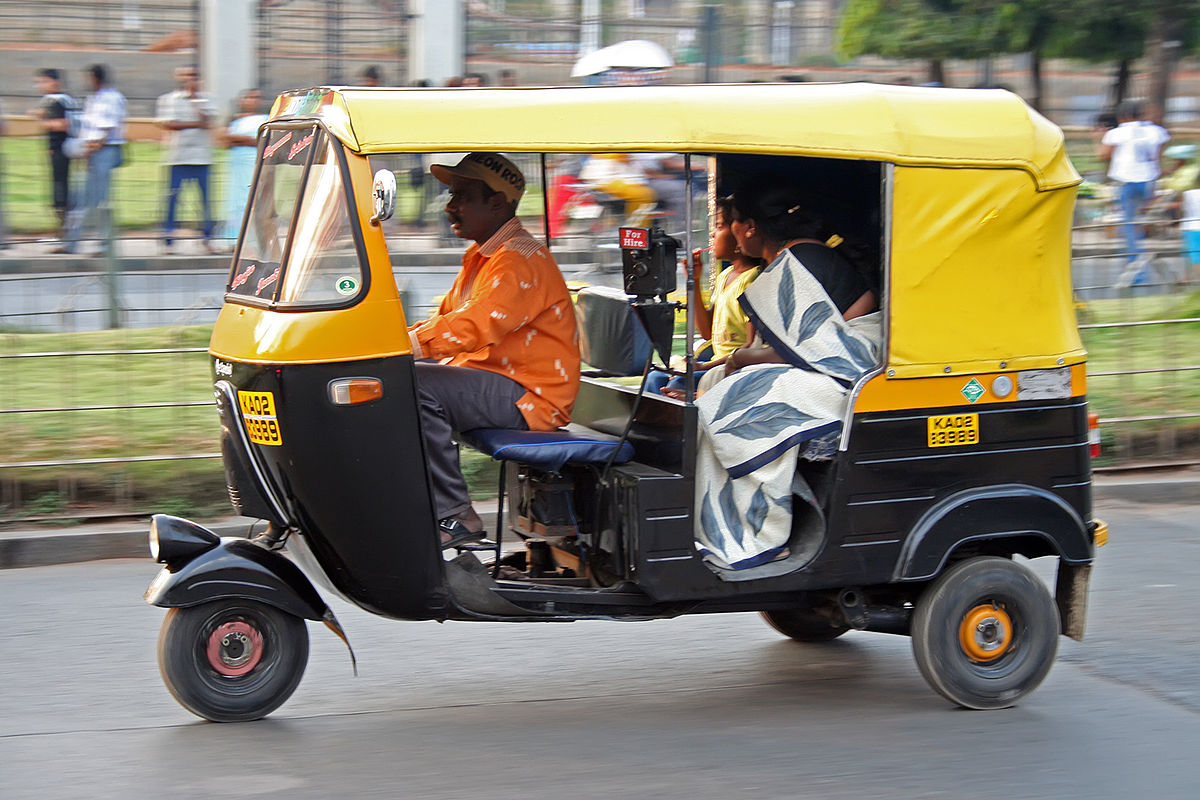}
    \subcaption{Rickshaw}
    \label{chap2:fig:rickshaw}
\end{minipage}%
\hspace{0.1cm}
\begin{minipage}{0.30\linewidth}
\centering
    \includegraphics[width=\linewidth]{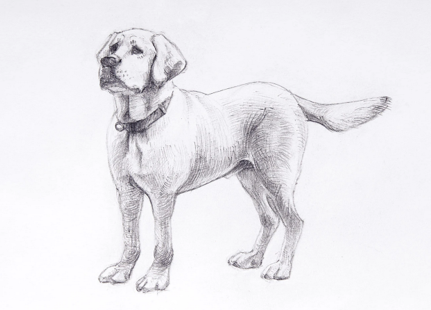}
    \subcaption{Drawing of a dog}
    \label{chap2:fig:drawing}
\end{minipage}
\caption[Examples of epistemic uncertainty in autonomous driving vision]{\textbf{Examples of epistemic uncertainty in computer vision.} An autonomous car can be exposed to unseen conditions (\cref{chap2:fig:snow_road}) or new semantic class (\cref{chap2:fig:rickshaw}). While being only trained on natural images, an ImageNet-trained classifier might encounter images of known classes but with different rendering \eg, drawings (\cref{chap2:fig:drawing}).}
\label{chap2:fig:epistemic_unc}
\end{figure}

The idea between distinguishing the two types of uncertainty is to characterize the uncertainty coming from the model and to take adequate actions to reduce it. For example, in active learning, selecting and labelling regions with large epistemic uncertainty should better improve the model's capacity to generalize, while focusing on large aleatoric uncertainty would be inefficient \cite{gal-active2017}. However, in some cases, such distinction may be unnecessary. For instance, when deploying a ML system for autonomous driving applications, the source of uncertainty might be inconsequential: the main purpose is to decide whether the agent should send a trigger alarm -- or give back control to user -- if the total uncertainty estimation regarding its prediction is high. This case is studied in \cref{chap5}. On a related matter, aleatoric and epistemic uncertainties are not absolute notions but are context-dependent. As illustrated by \cite{hullermeier2020} and replicated in \cref{chap2:fig:unc_add_feat}, embedding data in a higher-dimensional space can reduce aleatoric uncertainty but it may also increase epistemic uncertainty as more data are required to fit a model.

\begin{figure}[t]
\centering
\begin{minipage}[c]{0.49\linewidth}
    \centering
    \includegraphics[width=\linewidth]{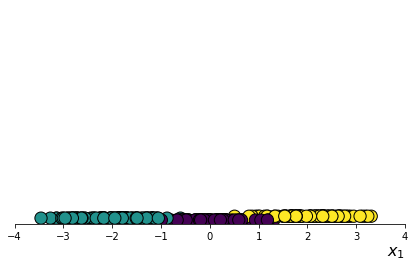}
\end{minipage}%
\hspace{0.1cm}
\begin{minipage}[c]{0.49\linewidth}
    \centering
    \includegraphics[width=\linewidth]{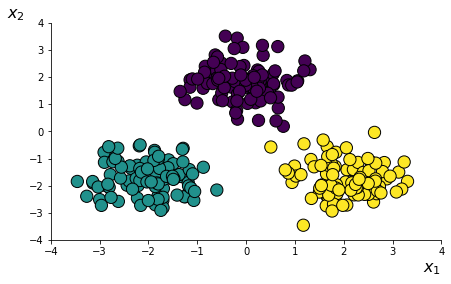}
\end{minipage}%
\caption[Illustration of the versatility of aleatoric and epistemic uncertainty]{\textbf{Embedding data in a higher-dimensional space can reduce aleatoric uncertainty.} While the two classes are overlapping on the left plot, by adding a second feature $\x_2$ they become separable, and consequently aleatoric uncertainty is reduced \cite{hullermeier2020}. }
\label{chap2:fig:unc_add_feat}
\end{figure}

This is why uncertainty modelling in machine learning should come with a clear description of the setting of the learning problem.

\subsection{Uncertainty in supervised learning}
\label{chap2:subsec:supervised_learning}

Let us consider a training dataset $\cD= \{ (\x_n, y_n) \}_{n=1}^N$ composed of $N$ \textit{i.i.d.} training samples, where $\x_n \in \cX$ is an input,  deep feature maps from an image or the image itself for instance, and $y_n \in \cY$ its corresponding output. These samples are drawn from an unknown joint distribution $P(X,Y)$ over $(\cX, \cY)$. Given loss function $\ell: \cY \times \cY \rightarrow \mathbb{R}^+$, the goal of supervised learning is to find a \emph{hypothesis} $h^*:\cX \rightarrow \cY$ within a fixed class $\cH$ of functions that minimizes the \emph{true risk}:
\begin{equation}
    h^* \in \argmin_{h \in \cH} \int_{\cX \times \cY} \ell(h(\x),y)dP(\x,y).
    \label{chap2:eq:true_risk}
\end{equation}
Because the distribution $P(X,Y)$ is unknown to the learning algorithm, we restricted the goal to find the hypothesis $\hat{h}$ that minimizes the \emph{empirical risk} given training data $\cD$:
\begin{equation}
    \hat{h} \in \argmin_{h \in \cH} \frac{1}{N} \sum_{n=1}^N \ell(h(\x_n),y_n).
    \label{chap2:eq:erm}
\end{equation}
Hypotheses $h^*$ and $\hat{h}$ are also known respectively as the \emph{Bayes estimator} and the \emph{empirical Bayes estimator} as they minimize the posterior expected value of loss $\ell$. Obviously, $\hat{h}$ is only an estimate of $h^*$ whose quality depends on the amount and diversity of training data. The uncertainty that arises due to this discrepancy is referred as \emph{approximation uncertainty}. 

Eventually, given an input $\x$, what we're interested in is to evaluate the \emph{predictive uncertainty} $p(y \vert \x)$, \ie the uncertainty related to predicting an outcome. In the case of a stochastic dependency between $\cX$ and $\cY$, even with a perfect knowledge of $P$ there still remain uncertainty, which is characterized as \emph{aleatoric uncertainty}. Given an input $\x$ with true label $y$, the best prediction would be the \emph{point-wise Bayes estimator} $f^*$:
\begin{equation}
    f^*(\x) = \argmin_{\hat{y} \in \cY} \int_\cY \ell(y,\hat{y})dP(y \vert \x).
    \label{chap2:eq:pointwise_risk}
\end{equation}
Due to the choice of the hypothesis space $\cH$, the Bayes estimator $h^*$ does not coincide with the point-wise Bayes estimator $f^*$, which give rise to \emph{model uncertainty}. Approximation uncertainty and model uncertainty are related to model aspects, either due to model design or to training data. Hence, they can be grouped as epistemic uncertainty. In practice, epistemic uncertainty is reduced to approximation uncertainty as deep neural networks with non-linear activations can theoretically approximate any continuous function (`universal approximation theorem' \cite{Hornik1989MultilayerFN}), thus $h^* \approx f^*$.

\section{Modelling uncertainty with deep neural networks}
\label{chap2:sec:representation}

Capturing both aleatoric and epistemic uncertainty is crucial to provide accurate uncertainty estimates. The Bayesian framework provides a natural probabilistic representation of uncertainty by incorporating degrees of belief. After an overview of recent developments of deep learning, we will see in this section how Bayesian approaches model uncertainty and which measures are used to quantify uncertainty in practice.

\subsection{Deep neural networks}
\label{chap2:subsec:dl}

Inspired by the simplified modelling of a biological neuron \cite{mcculloch43a}, an artificial feed-forward neural network, simply shortened here as neural network (NN), is a non-linear function $f: \cX \rightarrow \cY$ composed of a succession of non-linear mathematical functions, called \emph{layers}, that progressively transforms an input $\x$ to an output $y$:
\begin{equation}
    f(\x) = f^{(L)}\circ  f^{(L-1)}\circ  f^{(L)} \circ  \cdot \cdot \cdot \circ   f^{(1)}(\x),
\end{equation}
where $L$ is the number of layers. Each layer $l \in \llbracket 1,L \rrbracket$ is parametrized by $\btheta_l$ and we denote the overall set of parameters of the neural network as $\btheta = (\btheta_1,...,\btheta_L)$. 

A classic layer is the \emph{fully-connected} layer which consists in a linear combination of the input followed by a nonlinear activation $\mathbf{h}_l = \phi(\bw_l \mathbf{h}_{l-1} + \mathbf{b}_l)$ applied element-wise and where $\btheta_l = (\bw_l, b_l)$. The typical nonlinearities are the sigmoid function, the hyperbolic tangent and the Rectified Linear Unit (ReLU), the latter being currently the most popular. A neural network composed of at least one hidden layer is called \emph{multi-layer Perceptron} (MLP).  Thanks to their depth, DNNs are able to transform raw input data into more and more complex representations, from the low-level concepts \eg, colours or contours in computer vision, to high-level concepts such as objects, which is particularly useful for image classification. Interestingly, even with one sufficiently large hidden layer, MLPs can model any arbitrary function of the input thanks to the universal approximation theorem \cite{Hornik1989MultilayerFN}. The last layer, also called the output layer, is followed by an activation reflecting the desired output. For instance, in multi-class classification where $\cY = \llbracket 1, K \rrbracket$ with $K$ being the number of classes, the softmax function is commonly used to output a vector of categorical probabilities\footnote{In the following, we write $f(\x, \btheta)$ to denote explicitly the dependence of $f$ on its parameters $\btheta$}:
\begin{equation}
    \forall k \in \cY, \quad P(Y=k \vert~ \x, \btheta) = \frac{\exp \big ( f_k(\x, \theta) \big )}{\sum_{j \in \cY} \exp \big ( f_j(\x, \theta) \big )}.
\end{equation}

\paragraph{Training.} Neural networks are trained using gradient descent optimizers, such as stochastic gradient descent (SGD) with momentum \cite{Rumelhart1986we, Bottou10large}, Adagrad \cite{adagrad}, AdaDelta \cite{zeiler2012adadelta} and Adam \cite{kingma-iclr2015}. The gradient of the loss with respect to the model’s parameters $\btheta$ is obtained via back-propagation \cite{Rumelhart1986we}. In classification tasks, a NN with softmax activation is commonly trained via \emph{maximum likelihood estimation} on the training data, or equivalently minimizing the \emph{negative log-likelihood}:
\begin{equation}
    \hat{\btheta} \in \argmin_{\btheta} \frac{1}{N} \sum_{n=1}^N - \log P(y_n \vert \x_n, \btheta).
\end{equation}

\begin{figure}[t]
\centering
    \includegraphics[width=0.7\linewidth]{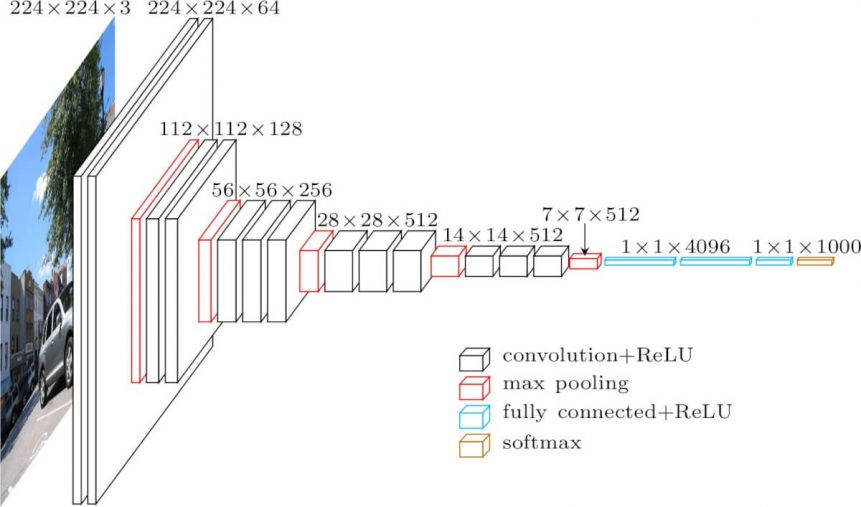}
    \caption[VGG-16 architecture]{\textbf{VGG-16 \cite{Simonyan15} architecture} consists in a succession of convolutional layers and max-pooling layers followed by fully-connected layers for image classification. \textit{Image credits: Durand \cite{DurandArchi2017}.}}
    \label{chap2:fig:vgg16}
\end{figure}

\paragraph{Convolutional neural networks.} Deep learning became ubiquitous in computer vision in 2012 when AlexNet \cite{krizhevsky2012imagenet} won the ImageNet Large Scale Visual Recognition Challenge. AlexNet is an architecture that is part of a class of neural networks named \emph{convolutional neural networks} (ConvNets). ConvNets are composed of a succession of convolutional and pooling layers, the former being a special case of matrix multiplication with circulant structure. By sharing a convolutional filter for all spatial positions, convolutional layers reduce the storage requirements of the model and encode translation equivariance. Pooling layers, or alternatively adding stride in a convolution, progressively aggregates spatial information as we go deeper in the network and produce invariance to small translations, which is particularly relevant for computer vision as we often want the response of a classifier to be independent of the location of objects in the image. In \cref{chap2:fig:vgg16}, we show the architecture of a typical ConvNet, VGG-16 \cite{Simonyan15}. As deep neural networks became deeper and deeper, training issues due to vanishing gradient started to emerge: through a large number of layers, the loss gradient becomes smaller and smaller. ResNets \cite{resnet2015} overcome this issue by adding skip connections between blocks of layers. Due to their strong performance on ImageNet, ResNets are now considered as a standard architecture for computer vision. In \cref{chap5}, we use one instance of ResNets with 18 layers, ResNet-18, which reaches 69.8\% top-1 accuracy on ImageNet, compared to 56.5\% with AlexNet and 71.6\% with VGG-16 but with considerably fewer parameters (12M for ResNet-18 vs. 138M VGG-16).

\begin{figure}[t]
\centering
    \includegraphics[width=\linewidth]{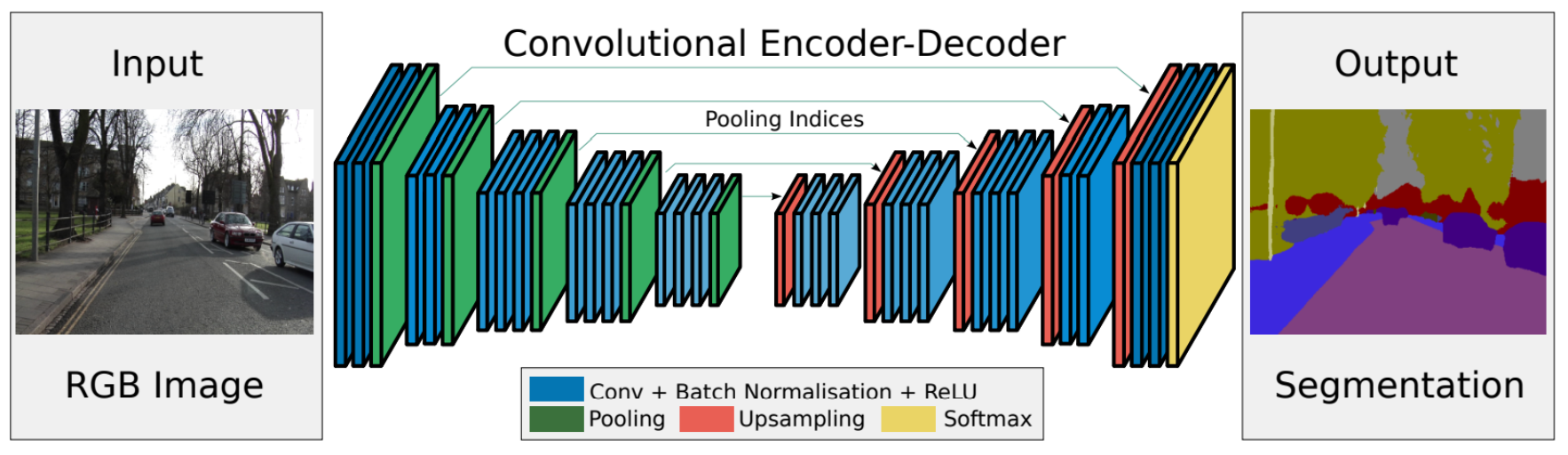}
    \caption[SegNet architecture for semantic segmentation]{\textbf{SegNet \cite{kendall2015bayesian} architecture} is a fully-convolutional neural network with the particularity that the decoder upsamples its input using the pool indices from its encoder. \textit{Image credits: Badrinarayanan et al. \cite{kendall2015bayesian}.}}
    \label{chap2:fig:segnet}
\end{figure}

\paragraph{Fully-convolutional neural networks.} Semantic segmentation can be seen a pixel-wise classification problem. The desired output is a semantic map of the same size as the image input. To meet this challenge, \emph{fully-convolutional neural networks} adopt an encoder-decoder structure where the encoder which reduce the spatial resolution and encodes a meaningful intermediate representation, then the decoder progressively recovers the spatial information by using successive upsampling operations. An example of fully-convolutional architecture used in \cref{chap3} is shown in \cref{chap2:fig:segnet} with SegNet \cite{kendall2015bayesian} which is based on VGG architecture. In \cref{chap4}, we also use DeepLab \cite{ChenPK0Y16} which showed tremendous performance on various benchmarks for semantic segmentation. While new architectures now outperform DeepLab, this architecture became a standard, used in autonomous driving benchmarks such as Cityscapes \cite{cordts-cvpr2016}.

\subsection{Bayesian approaches}
\label{chap2:subsec:bayesian}

In Bayesian statistics, a probability expresses a degree of belief or information about an event. Given hypothesis $h$, we fix a prior distribution $p(h)$ over $h$ and learning consists in updating that prior with the probability of the data given $h$, \ie likelihood, according to Bayes' rule:
\begin{equation}
    p(h \vert \cD) = \frac{p(\cD \vert h) p(h)}{p(\cD)} \propto p(\cD \vert h) p(h).
\end{equation}
The posterior distribution $p(h \vert \cD)$ captures the model's knowledge regarding hypothesis $h$ given data $\cD$. The more peaked this distribution is, the more certain the model will be in regards to epistemic uncertainty.

\begin{figure}[t]
\centering
    \includegraphics[width=\linewidth]{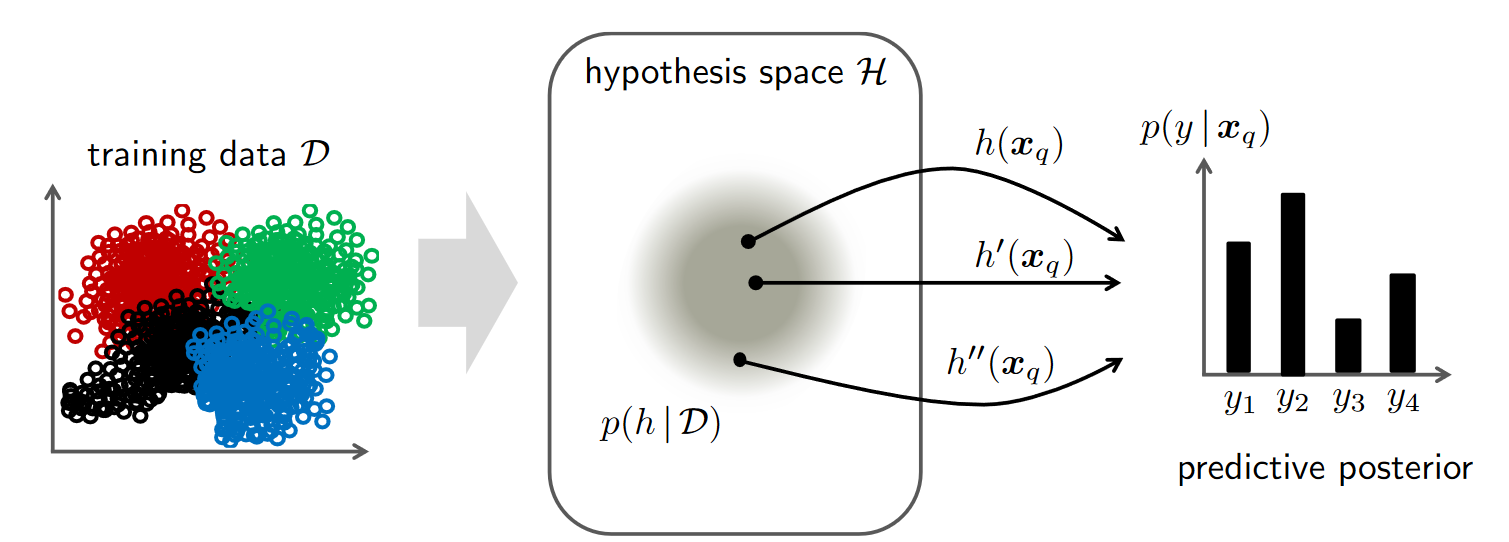}
    \caption[Illustration of Bayesian inference]{\textbf{Illustration of Bayesian inference}. Training data $\cD$ is used to update the posterior distribution $p(h \vert \cD)$. Then, given an input $\x_q$, the posterior predictive distribution is obtained by Bayesian model averaging. \textit{Image credits: Hüllermeier et al.} \cite{hullermeier2020}.}
    \label{chap2:fig:bayesian_inference}
\end{figure}

In Bayesian inference, given an unknown input $\x$, the posterior predictive distribution $p(y \vert \x, \cD)$ is obtained by \emph{Bayesian model averaging}:
\begin{equation}
    p(y \vert \x, \cD) = \int_{\cH} p(y \vert \x,h) dp(h \vert \cD).
    \label{chap2:eq:pred_distrib}
\end{equation}
Hence, the posterior predictive distribution is a weighted average over its probabilities under all hypotheses in $\cH$, weighted by the posterior probability $p(h \vert \cD)$ (see \cref{chap2:fig:bayesian_inference}).

\paragraph{Bayesian neural networks.} While traditional deep neural networks output a single point-wise prediction, Bayesian neural networks \cite{MacKay92bayesianmethods,bnn1996} (BNNs) propose to apply Bayesian inference by considering distributions over a network's parameters $\btheta$ and learning the posterior distribution $p(\btheta \vert \cD)$. The posterior predictive distribution $p(y \vert \x^*, \cD)$ is obtained by marginalizing  over the parameters $\btheta$:
\begin{equation}
    p(y \vert \x, \cD) = \int p(y \vert \x,\btheta)p(\btheta \vert \cD)d\btheta.
    \label{chap2:eq:bma}
\end{equation}
When modelling complex real-world data, exact inference may be intractable because the previous integrals cannot be expressed in closed form, since the parameters are mapped through non-linearities
in deep neural network architectures. 

\sloppy
Since the posterior distribution cannot usually be evaluated analytically, a few approximation methods have been considered to compute the posterior predictive. A first simple approach consists in approximating the posterior distribution $p(\btheta \vert \cD)$ with a Dirac distribution centered on the maximum likelihood estimator $\hat{\btheta}_{\textrm{MLE}}$: 
\begin{align}
\btheta_{\textrm{MLE}} \in  \argmax_{\btheta} p(\cD \vert \btheta) &=  \argmax_{\btheta} \prod_{n=1}^N p(y_n \vert \x_n, \btheta) \\
&= \argmax_{\btheta} \sum_{n=1}^N \log p(y_n \vert \x_n, \btheta).
\end{align}
Then, the posterior predictive distribution is simply the evaluation of the prediction on this point: $p(y \vert \x, \cD) \approx p(y \vert \x, \hat{\btheta}_{\textrm{MLE}})$\footnote{Note that this method actually correspond to a standard neural network trained with maximum likelihood.}. However, we obtain a point estimate for parameters $\btheta$ which may overfit. For instance, with a dataset composed of 3 tosses landed head, we would then estimate $\hat{\btheta}_{\textrm{MLE}}$ such that $p(y \vert \x^*, \cD)=1$ for any toss! To mitigate this issue, one could instead compute the \emph{maximum-a-posteriori} estimate $\btheta_{\textrm{MAP}} =  \argmax_{\btheta} p(\btheta \vert \cD)$ but it still remains a point estimate that underestimates epistemic uncertainty.

With deep neural networks, a few methods have been explored including Laplace approximation \cite{MacKay92bayesianmethods}, Hamiltonian Monte Carlo sampling \cite{neal2011}, and expectation-propagation \cite{hernandezlobatoc15,jylanki14a}. In particular, variational inference \cite{Hinton1993,graves2011} gained a lot of popularity in the recent years due to better scaling. The goal is to learn to approximate the exact posterior distribution by defining a simpler variational distribution $q(\btheta)$ and minimizing the Kullback-Leibler (KL) divergence between $q(\btheta)$ and $p(\btheta \vert \cD)$. For instance, \emph{Variational Bayes} \cite{Blundell2015} defines the variational distribution $q(\btheta)$ as a Gaussian distribution with a diagonal covariance, \ie a fully factorized Gaussian. Another important example is \emph{Monte-Carlo Dropout} (MC Dropout) where Gal and Ghahramani \cite{Gal2016} establish a connection between variational inference and dropout layers \cite{srivastava14a}, commonly used in neural networks for regularization. At inference, the predictive distribution is approximated by Monte Carlo sampling and averaging over all $M$ forward predictions:
\begin{equation}
    p(y \vert \x, \cD) \approx \frac{1}{M} \sum_{m=1}^M p(y \vert \x,\btheta_m),
\end{equation}
where $\btheta_m \sim p(\btheta \vert \cD)$ are the sampled weights from forward pass $m$. The total uncertainty can then be quantified in terms of variance in the case of regression and entropy as detailed in the following section. 

With Bayesian Neural Networks, the crucial aspect is how well the posterior distribution $p(\btheta \vert \cD)$ is approximated \cite{izmailov2021bayesian}. Unfortunately, MCDropout has been shown to be a poor approximation to the true posterior \cite{wenzel20a}, resulting in unreliable uncertainty estimates \cite{ovadia2019,Osband2016RiskVU,liu2021peril}.

\paragraph{Ensembles.} Lakshminarayanan \textit{et al.}  \cite{deepensembles2017} propose a simple but effective approach named \emph{Deep Ensembles} which outperforms Bayesian neural networks for uncertainty representation \cite{ovadia2019,Ashukha2020Pitfalls}. An ensemble of $M$ models is trained independently with random initialization. Such as with MC Dropout, predictions are obtained by averaging the $M$ samples and uncertainty estimates can be derived from the spread of the ensemble. While being originally considered as non-Bayesian, Deep Ensembles can actually be seen as a Bayesian model average \cite{WilsonI20}, whose samples provides a richer functional diversity in the predictive integral \cref{chap2:eq:bma}. Further works \cite{HuangLPLHW17,swag2019} explored ways to avoid training multiple models to reduce training time by leveraging intermediate checkpoints of a model during training. The main drawback of these approaches is the computational expense of training and storing weights of $M$ models, which is not convenient for embedded systems such as autonomous vehicles. 

\paragraph{Gaussian processes.} Considered as the gold standard of uncertainty estimation \cite{rasmussen2005}, Gaussian processes are non-parametric Bayesian models. Unlike BNNs which define probability distributions over networks' weights, they directly specify distributions over the \emph{function} $f(\cdot, \btheta)$ induced by the network. This distribution is a joint Gaussian distribution defined over a collection of function values $f(\x_1),\cdots,f(\x_n)$. The computation of the covariance function (or kernel) of the distribution requires access to the full training dataset at inference time. Although some approximations \cite{benton2019functionspace} have been proposed, this family of probabilistic methods does not scale well with the dimension of the data.

Previous methods proposed adapting neural networks to capture epistemic uncertainty thanks to the spread of the posterior distribution $p(\btheta \vert \cD)$. But how do we derive uncertainty estimates from these Bayesian approaches? This will be addresses in \ref{chap2:subsec:unc_measures}.

\subsection{Evidential models}
\label{chap2:subsec:enn}

To overcome the issue of approximation due to sampling, a recent class of models, named \emph{evidential} \cite{malinin2018,sensoy2018} proposes instead to explicitly represent the distribution over probabilities. This line of work is based on subjective logic \cite{josan2016sublogic}, a probabilistic framework which formalizes the Dempster-Shafer \cite{dempster2008} theory's notion of belief as a Dirichlet distribution. In the multi-class setting, the subjective opinion of a multinomial random variable $y \in \cY$ is given by a triplet:
\begin{equation}
    \omega = (\mathbf{b}, u, \mathbf{a}) \quad\quad \textrm{with} \quad \sum_{k \in \cY} b_k + u = 1,
\end{equation}
where $\mathbf{b} = (b_1, \cdots, b_K)^T$ denotes the belief mass over $\cY$, $u \geq 0$ is the overall uncertainty mass and $a$ is the base rate distribution. Let $e_k \geq 0$ be the evidence derived for class $k$. The class belief $b_k$ and the uncertainty $u$ are computed as:
\begin{equation}
b_k = \frac{e_k}{S} \quad\quad \textrm{and} \quad\quad u = \frac{K}{S},
\end{equation}
where $S = \sum_{k=1}^K (e_k +1)$. Note that the uncertainty $u$ is inversely proportional to the total evidence.

The link to the Dirichlet distribution can be grasped by first considering the simpler problem of inferring from a set $\cD$ of $N$ rolls the probability that a dice with $K$ sides comes up as face $k$ \cite{murphy2012machine}. We denote $\bpi =(\pi_1, \cdots, \pi_K)$ the random variable over categorical probabilities, where $\sum_{k=1}^K \pi_k$ = 1, and which lives on the (K-1)-dimensional simplex $\triangle^{K-1}$. Assuming \textit{i.i.d.} data, its likelihood reads $p(\cD \vert \bpi) =  \prod_{k=1}^K \pi_k^{N_k}$ where $N_k$ is the count of class $k$ among the $N$ draws. For its conjugate properties with the categorical distribution, the prior $p(\bpi)$ can be modeled as a Dirichlet distribution with concentration parameters $\boldsymbol{\beta}$. Then, the posterior $p(\bpi \vert \cD)$ is also a Dirichlet distribution with parameters ($\beta_1 + N_1$,..., $\beta_K + N_K$) and the posterior predictive distribution for a single multinoulli trial has the closed form $P(Y=k ~\vert~ \cD) = \mathbb{E} \big [\pi_k \vert \mathcal{D}] = \frac{\beta_k + N_k}{\sum_k \beta_k + N}$. We observe that the prior distribution acts as a \textit{Bayesian smoothing} by adding pseudo-counts $\boldsymbol{\beta}$ to the empirical counts. \\

\begin{figure}[t]
\centering
\captionsetup[subfigure]{justification=centering}
\begin{minipage}[c]{0.30\linewidth}
\centering
    \includegraphics[width=\linewidth]{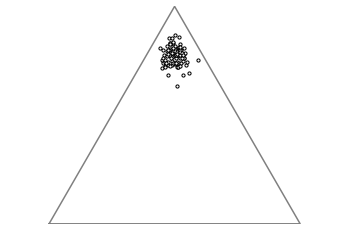}
    \subcaption{Confident prediction}
    \label{chap2:fig:simplex_confident}
\end{minipage}%
\hfil
\begin{minipage}[c]{0.30\linewidth}
\centering
    \includegraphics[width=\linewidth]{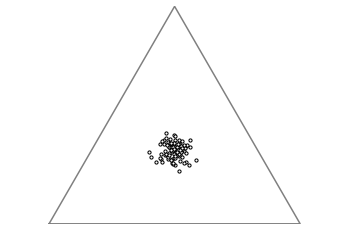}
    \subcaption{Conflicting evidence}
    \label{chap2:fig:simplex_aleatoric}
\end{minipage}%
\hfil
\begin{minipage}[c]{0.30\linewidth}
\centering
    \includegraphics[width=\linewidth]{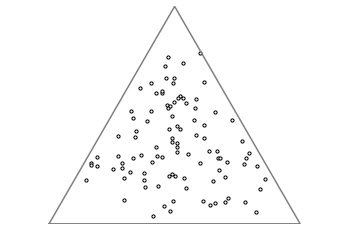}
    \subcaption{Lack of evidence}
    \label{chap2:fig:simplex_epistemic}
\end{minipage}%

\medskip
\begin{minipage}[c]{0.30\linewidth}
\centering
    \includegraphics[width=\linewidth]{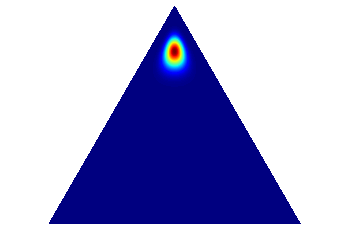}
\end{minipage}%
\hfil
\begin{minipage}[c]{0.30\linewidth}
\centering
    \includegraphics[width=\linewidth]{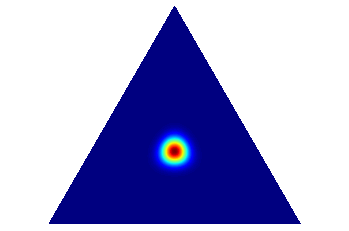}
\end{minipage}%
\hfil
\begin{minipage}[c]{0.30\linewidth}
\centering
    \includegraphics[width=\linewidth]{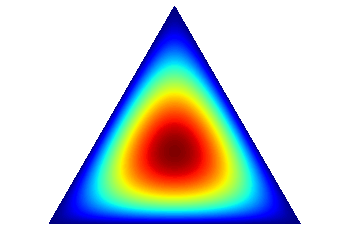}
 \end{minipage}%
\caption[Uncertainty representation on the simplex]{\textbf{Uncertainty representation on the simplex}. Top row shows samples drawn from an ensemble or a BNN. Bottom row illustrates the implicit distribution they are sampled from. A confident prediction will have a distribution focused on a corner of the simplex and with a low dispersion (\cref{chap2:fig:simplex_confident}). Conflicting evidence (aleatoric uncertainty) will result in a distribution close to the simplex center (\cref{chap2:fig:simplex_aleatoric}), reflecting a high class confusion. Finally, a lack of evidence (\cref{chap2:fig:simplex_epistemic}) corresponds to a distribution with high dispersion (epistemic uncertainty): each sample can yield very different probabilities.}
\label{chap2:fig:simplex}
\end{figure}

Let us extend the Bayesian treatment of a single categorical distribution to classification, \ie, the goal is to predict the class label $y$ from a categorical distribution that depends on input $\x$. The training dataset $\cD$ consists of $N$ \textit{i.i.d.} samples $(\x,y)$ drawn from an unknown joint distribution $P(X,Y)$. Obviously, for a test sample $\x^*$, its label frequency count is now unknown and we are not able to estimate the posterior predictive distribution $P(Y \vert \x^*, \cD)$. Bayesian models and ensembling methods approximate the posterior predictive distribution by marginalizing over the network's parameters thanks to sampling. But this comes at the cost of multiple forward passes. 

Evidential Neural Networks (ENNs) propose instead to model explicitly the posterior distribution over categorical probabilities $p(\bpi \vert \x, y)$ by a variational Dirichlet distribution,
\begin{equation}
    q_{\btheta}(\bpi \vert \x) \!=\! \text{Dir} \big ( \bpi \vert \balpha(\x, \btheta) \big )
    \!=\! 
    \frac{\Gamma(\alpha_0 (\x, \btheta))}
    {\prod_{k=1}^K \Gamma(\alpha_k(\x, \btheta)) }\prod_{k=1}^K \pi_k^{\alpha_k(\x, \btheta) - 1},
\end{equation}
whose concentration parameters $\balpha(\x, \btheta) ~{=}~ \exp f(\x,\btheta)$ are output by a neural network $f$ with parameters $\btheta$; $\Gamma$ is the Gamma function and $\alpha_0(\x, \btheta) ~{=}~ \sum_{k=1}^K \alpha_k(\x, \btheta)$ with $\alpha_k = \exp f_k(\x, \btheta)$ indexing the $k^{\text{th}}$ element of the vector of all $K$ concentration parameters $\balpha$. Precision $\alpha_0$ controls the sharpness of the density with more mass concentrating around the mean as $\alpha_0$ grows. By conjugate property, the predictive distribution for a new point $\x^*$ is 
\begin{equation}
P(Y~{=}~k ~\vert~ \x^*, \cD)~
     {\approx}~\mathbb{E}_{q_{\btheta}(\bpi \vert \x^*)} [\pi_k] ~{=}~ 
     \frac{\exp f_k(\x^*,\btheta)}{\sum_{j=1}^K \exp f_j(\x^*,\btheta)},
\end{equation}
which is the usual output of a network $f$ with softmax activation.

Instead of reasoning on first-order probabilities, we can now derive second-order uncertainty measures on the Dirichlet distribution. Evidential models provide a second-order uncertainty representation as shown in \cref{chap2:fig:simplex} where the expectation of the Dirichlet distribution relates to aleatoric uncertainty and its criterion concerning its dispersion can measure the the amount of evidence in a prediction, hence epistemic uncertainty

\paragraph{\textbf{Training Objective}} The ENN training is formulated as a variational approximation to minimize the KL divergence between the distribution $q_{\btheta}(\bpi \vert \x) $ and the true posterior distribution $p(\bpi \vert \x, y)$:
\begin{equation}
    \cL_{\text{var}}(\btheta;\cD) = \mathbb{E}_{(\x,y) \sim P(X, Y)} \big [ \mathbb{KL} \big ( q_{\btheta}(\bpi \vert \x)~\|~p(\bpi \vert \x, y) \big ) \big ]
\end{equation}
The training objective and its derivation are further study in \cref{chap5}.

\subsection{Uncertainty measures}
\label{chap2:subsec:unc_measures}

\sloppy
In regression, while the predictive distribution $p(y \vert \x, \cD)$ remains intractable (\cref{chap2:eq:pred_distrib}), likelihood is assumed Gaussian and one can estimate the predictive distribution's first two moments empirically \cite{Gal2016}. Given likelihood $p(y \vert \x,\btheta) = \mathcal{N}(y; f(\x, \btheta), \tau^{-1} \boldsymbol{I})$, the first moment $\mathbb{E}_{p(y \vert \x, \cD)}[y]$ can be approximated by the unbiased estimator $\tilde{\mathbb{E}}[y] = \frac{1}{M}\sum_{m=1}^M f(\x, \btheta_m)$ following Monte-Carlo sampling. The model’s predictive variance $\textrm{Var}_{p(y \vert \x, \cD)}[y]$ -- the second moment -- is given by the unbiased estimator $\tilde{\textrm{Var}}[y] = \tau^{-1} \boldsymbol{I} +\frac{1}{M} \sum_{m=1}^M f(\x, \btheta_m)^T f(\x, \btheta_m) - \tilde{\mathbb{E}}[y]^T \tilde{\mathbb{E}}[y]$. In particular, we note that the predictive variance accounts both for the aleatoric uncertainty with $\tau^{-1} \boldsymbol{I}$ and for the epistemic uncertainty with the second term.

When it comes to classification, the aleatoric uncertainty at an input point $\x$ is defined as the entropy of the \emph{true} conditional distribution $p(Y \vert \x, \cD)$:
\begin{equation}
    \bbH[Y \vert \x, \cD] = - \sum_{k \in \cY} p(Y=k \vert \x, \cD) \log p(Y=k \vert \x, \cD).
\end{equation}
The entropy attains its maximum value when all classes have equal uniform probability and its minimum value of zero when one class has probability 1 and all others probability 0. But in contrast with regression, we cannot rely on the previous derivation to estimate the predictive moments: likelihood is now a categorical distribution and we cannot estimate its first two moments anymore:
\begin{equation}
    p(Y \vert \x,\btheta) = \mathrm{Cat} \Big (Y; \phi \big ( f(\x, \btheta) \big ) \Big ),
\end{equation}
where the softmax operator, $\phi : \mathbb{R}^{K} \rightarrow \triangle^{K-1}$, transforms logits into probabilities on the (K-1)-dimensional unit simplex $\triangle^{K-1}$, thanks to an exponential form.

A first possibility is to use the entropy of the \emph{predictive} distribution estimated by the model. In the case of Monte Carlo sampling, this corresponds to averaging the probability
vectors from the $M$ stochastic forward passes:
\begin{align}
    \bbH[Y \vert \x, \cD] &= - \sum_{k \in \cY} p(Y=k \vert \x, \cD) \log p(Y=k \vert \x, \cD) \\
    &\approx - \sum_{k \in \cY} \Big ( \frac{1}{M} \sum_{m=1}^M p(Y=k \vert \x,\btheta_m) \Big ) \log \Big ( \frac{1}{M} \sum_{m=1}^M p(Y=k \vert \x,\btheta_m) \Big )
\end{align}
with $\btheta_m \sim p(\btheta \vert \cD)$ are the sampled weights from forward pass $m$. 

Given an input $\x$, it is also possible to estimate aleatoric uncertainty by looking at the likelihood of the class predicted by the model $m$, which is by design the class associated with the \emph{maximum} probability:
\begin{equation}
    MCP_m(\x) = \max_{k \in \cY} p(Y=k \vert \x, \btheta_m).
\end{equation}
In the case of Monte Carlo sampling, MCP is computed on the average probability vector. MCP values range from $1/K$ to its maximum value of one when all probabilities are 0 except for the predicted class, hence no uncertainty. 

\begin{figure}[t]
\centering
\begin{minipage}[t]{0.31\linewidth}
    \centering
    \includegraphics[width=\linewidth]{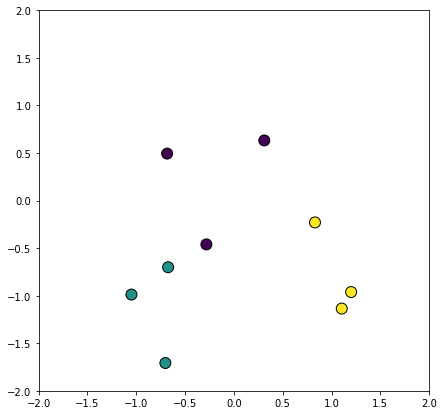}
    \subcaption{Train data}
    \label{chap2:fig:train_data}
\end{minipage}%
\begin{minipage}[t]{0.31\linewidth}
    \centering
    \includegraphics[width=\linewidth]{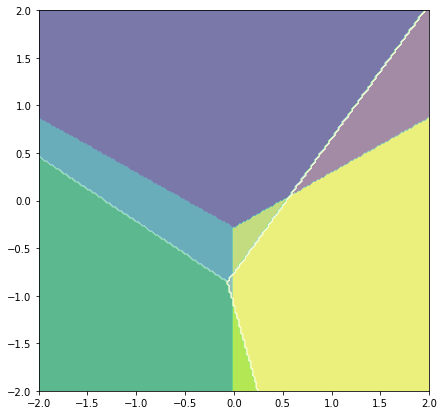}
    \captionsetup{justification=centering}
    \subcaption{Bayes optimal's vs. model's decision frontiers}
    \label{chap2:fig:bayes_optimal}
\end{minipage}%
\begin{minipage}[t]{0.36\linewidth}
    \centering
    \includegraphics[width=\linewidth]{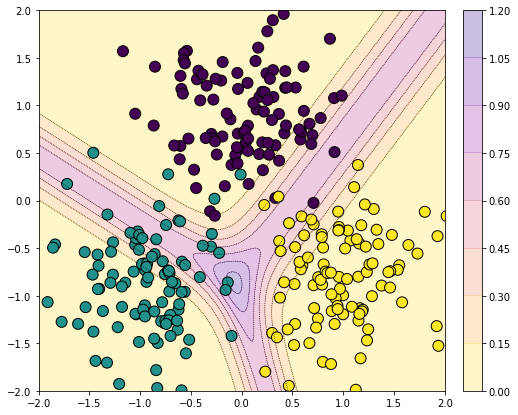}
    \subcaption{Entropy map}
    \label{chap2:fig:entropy_map}
\end{minipage}%
\caption[Illustration of unreliable uncertainty estimates with MCP and entropy due to poor fitting of the conditional distribution]{\textbf{Illustration of unreliable uncertainty estimates with MCP and entropy due to poor fitting of the conditional distribution.} A logistic regression classifier trained on only nine inputs sampled from the distribution (\cref{chap2:fig:train_data}) will have its decision frontier differ greatly from the Bayes optimal classifier $h^*$ given the marginal distribution of the Gaussian mixture (\cref{chap2:fig:bayes_optimal}). Consequently, uncertainty measures estimated on the predictive distribution such as entropy (\cref{chap2:fig:entropy_map}) poorly reflect the true aleatoric uncertainty of the conditional distribution. }
\label{chap2:fig:limit_mcp}
\end{figure}

But the quality of uncertainty estimates given by MCP and predictive entropy depends on the quality of the approximation of the posterior distribution $p(\btheta \vert \cD)$ and consequently may be inaccurate. For instance, \cref{chap2:fig:limit_mcp} shows an uncertainty map computed from the predictive entropy output by a neural network trained only on nine samples on a toy dataset. This toy dataset is composed of a Gaussian mixture with three equally weighted components having equidistant centers and equal spherical covariance matrices. As the model's decision frontiers do not coincide with the Bayes optimal ones -- given by the true conditional distribution $p(y \vert \x)$ --, the predictive entropy might be incorrectly low in regions of high aleatoric uncertainty (close to Bayes optimal's decision frontiers) and high in confident regions. 

To estimate the epistemic uncertainty, an intuitive idea with BNNs and ensembling is to consider the variance of the predictions produced by the $M$ stochastic forward passes. Gal \cite{Gal2016} proposes to compute the \emph{variation-ratio} which is based on the frequency of prediction of the most predicted class:
\begin{equation}
    \textrm{var-ratio}(\x) = 1 - \max_{k \in \cY} \Big (\frac{1}{M}\sum_{m=1}^M \mathbbm{1}[\hat{y}_m=k] \Big ).
\end{equation}
More interestingly, Depeweg \textit{et al.} \cite{depeweg18a} propose to measure the \emph{mutual information} $\bbI[y, \btheta \vert \x, \cD]$ between the prediction $y$ and the posterior distribution, based on the decomposition of the predictive uncertainty. Assuming that predictive entropy $\bbH[y \vert \x, \cD]$ contains aleatoric and epistemic uncertainty as it depends on dataset $\cD$, the mutual information corresponds to the difference between predictive entropy and the expected entropy of each member of the ensemble $\mathbb{E}_{p(\btheta \vert \cD)} \Big [ \bbH[y \vert \x, \btheta] \Big ]$, which does not depend on the model anymore:
\begin{align}
    \bbI[Y, \btheta \vert \x, \cD] &= \bbH \Big [ \mathbb{E}_{p(\btheta \vert \cD)} \big [ p(Y \vert \x, \btheta) \big ] \Big ] - \mathbb{E}_{p(\btheta \vert \cD)} \Big [ \bbH[Y \vert \x, \btheta] \Big ] \\
    &\approx - \sum_{k \in \cY} \Big ( \frac{1}{M} \sum_{m=1}^M p(Y=k \vert \x,\btheta_m) \Big ) \log \Big ( \frac{1}{M} \sum_{m=1}^M p(Y=k \vert \x,\btheta_m) \Big ) \nonumber \\
    &\quad\quad\quad+ \frac{1}{M} \sum_{m=1}^M \sum_{k \in \cY}  p(Y=k \vert \x,\btheta_m)\log p(Y=k \vert \x,\btheta_m).
\end{align}
Consequently, mutual information is a dispersion measure which accounts for the variance of the predictions produced by the $M$ stochastic forward passes.

To gain intuition about the behavior of the previous measures, let us consider a classification task with 3 classes and the following samples:
\begin{enumerate}
    \item a sample where the model outputs probability vectors with maximum probability on the same class: $$p_1 = \{(1,0,0),(1,0,0),...,(1,0,0)\}$$;
    \item a sample where the model outputs probability vectors with uniform probability: $$p_2 = \{(\frac{1}{3},\frac{1}{3},\frac{1}{3}), (\frac{1}{3},\frac{1}{3},\frac{1}{3}),..., (\frac{1}{3},\frac{1}{3},\frac{1}{3})\}$$;
    \item a sample where the model produces inconsistent probability vectors: $$p_3 = \{(1,0,0), (0,1,0),..., (0,0,1)\}$$.
\end{enumerate}
This first sample $p_1$ represents an input predicted with high confidence by the model and where the aleatoric uncertainty is low. The second sample $p_2$ presents high aleatoric uncertainty due to class confusion but low epistemic uncertainty as the model always predicted the same probability vector. In contrast, the model outputs very different predictions regarding the third sample $p_3$, denoting a large epistemic uncertainty. When computing the predictive entropy, we find that obviously $\bbH[p_1]=0$ but $\bbH[p_2] = \bbH[p_3] = 1.09$, which does not enable us to separate the two sources of uncertainty. Now, the mutual information gives us more information about the third sample as $\bbI[p_1] = \bbI[p_3] = 0$ and $\bbI[p_2] = 1.09$, showing here it measures solely the dispersion between predictions. Again here, this theoretical decomposition depends on the approximation to this posterior distribution which may re-introduce approximation uncertainty in both terms.

Finally, with evidential models, a series of uncertainty measures based on the second-order Dirichlet distribution allows one to measure different sources of uncertainty \cite{josang2018,Shi2020MultifacetedUE}. In particular, the \emph{vacuity} is due to insufficient or unreliable information received from sources and represented by uncertainty mass $u$ in subjective logic. On the second-order Dirichlet distribution, this is equivalent to its precision $\alpha_0$ which is a measure of its dispersion, hence capturing epistemic uncertainty. Related to aleatoric uncertainty, \emph{dissonance} corresponds to contradicting belief, such as in class confusion, and is defined as:
\begin{equation}
    diss(\omega) = \sum_{k \in \cY} \Big ( \frac{b_k \sum_{j \neq k} b_j \textrm{Bal}(b_j,b_k)}{\sum_{j\neq k b_j}} \Big ),
\end{equation}
where $\textrm{Bal}(b_j,b_k) = 1 - \frac{\vert b_j - b_k \vert}{b_j + b_k}$ if $b_k b_j \neq 0$ and 0 if $\min(b_j, b_k)=0$ is the relative mass balance function between two belief masses. For instance, given opinion $(b_1,b_2,b_3,u,\mathbf{a}) = (0.3,0.3,0.3,0.1, \mathbf{a})$, the dissonance value is equal to $0.9$. Its maximum value is 1 and its minimum value is 0.

\section{Evaluation of the quality of uncertainty estimates}
\label{chap2:sec:evaluation}

Evaluating the quality of predictive uncertainties is challenging as the ‘ground truth’ uncertainty estimates are usually not available. Depending on the application, the desirable properties of uncertainty estimates can vary: in a multi-modal system, we aim for calibrated uncertainty estimates before fusion while one may only be interested in a reliable ranking between correct and erroneous predictions. 

We present in this section the existing tasks commonly used in the literature to evaluate the quality of uncertainty estimates with deep neural networks. 

\subsection{Selective classification}
\label{chap2:subsec:selective}

The idea of a reject option with ML systems has been around for ages \cite{Chow1957AnOC}. 
Classification with a reject option, also known as \emph{selective classification} \cite{elyaniv10a}, consists in a scenario where a classifier can abstain on points where its confidence is below a certain threshold. By abstaining from predicting when in doubt, the main motivation is to reduce the error rate while keeping as many correct samples as possible.

To select which sample to reject, a \textit{confidence-rate function} $\kappa_f$ is associated to the classifier $f$ in order to evaluate the degree of confidence of its predictions, the higher the value the more certain the prediction. Uncertainty estimates are used here to assess this degree of confidence. Then, given a threshold $\delta$, an input $\x$ is rejected if its degree of confidence is lower than the threshold value, \begin{equation}
    g(\bm{x})=
    \begin{cases}
      1 & \text{if}\ \kappa_f(\bm{x}) \geq \delta \, , \\
      0 & \text{otherwise.} \\
    \end{cases}
\end{equation}
Ideally, uncertainty estimates should enable the selection function to split the test set in a subset containing all errors and the other set containing all correct predictions. 

The performance of a selective model is quantified using coverage and risk. Re-using the notations introduced in \cref{chap2:subsec:supervised_learning}, we also consider explicitly a test set $\cD_{\mathrm{test}}$ composed of labeled samples also following $P(X,Y)$. Coverage is defined to be the probability mass of the non-rejected region in $\cX$, which can be approximated empirically by the number of non-rejected samples:
\begin{equation}
    \hat{\phi}(g) = \frac{1}{\vert \cD_{\textrm{test}} \vert} \sum_{(\x,y) \in \cD_{\textrm{test}}} g(\x),
\end{equation}
where $\vert \cD_{\textrm{test}} \vert$ is the number of samples in the test set. The selective risk corresponds to the evaluation of the loss $\ell$ on the non-rejected samples, which is commonly the 0/1 error with classification, divided by coverage. Its empirical approximation writes as:
\begin{equation}
    \hat{R}(f,g) = \frac{1}{\vert \cD_{\textrm{test} \vert}} \sum_{(\x,y) \in \cD_{\textrm{test}}} \frac{\ell \big ( f(\x),y \big )g(\x)}{\hat{\phi}(g)}.
\end{equation}

Given test set $\cD_{\textrm{test}}$, the task evaluation is based on risk-coverage curves such as shown in \cref{chap3:subsec:results}. These curves are obtained by computing the empirical selective risk for various values of coverage. The threshold $\delta$ depends on a user-specified cost for abstention. Consequently, to free ourselves from choosing this threshold, we compare methods by computing the following metrics:
\begin{itemize}
    \item \textbf{AURC} measures the Area Under the Risk-Coverage curve. This metric is threshold-independent. The higher the AURC, the better the selective classifier.
    \item \textbf{Excess-AURC (E-AURC)}. This is a normalized AURC metrics defined in \cite{geifman2018biasreduced}. It takes into account the optimal ranking given the error rate of the classifier. More specifically, if we denote $\kappa_f^*$ the perfect confidence-rate function and $\hat{r}$ the risk of classifier $f$, it writes as:
    \begin{align}
        \text{E-AURC}(\kappa_f) &= \text{AURC}(\kappa_f) - \text{AURC}(\kappa^*_f) \\ &\approx \text{AURC}(\kappa_f) -  \big ( \hat{R} + (1-\hat{R})\log(1-\hat{R}) \big ).
    \end{align}
\end{itemize}

With deep neural networks, we denote two types of approaches. The first one considers a trained prediction model and constructs a selection mechanism \cite{NIPS2017_7073}. Most of the time, the confidence-rate function used is the value of MCP given by the softmax layer output. The second type of approaches aims to jointly learn the classifier and the selection function \cite{cortes2016}. In particular, Geifman \& El-Yaniv \cite{Geifman2019SelectiveNetAD} train a DNN to optimize classification and rejection simultaneously. The reject function corresponds to the output of a second head of the DNN.

\subsection{Misclassification detection}
\label{chap2:subsec:misclassif}

 Given a trained model, \emph{misclassification detection}, also referred as failure prediction \cite{HeckerDG18}, is the task of predicting at run-time whether the model has taken a correct decision or not for a given input. Uncertainty estimates are used here as confidence score to compare to a threshold $\delta$. We say that the input $\x$ is estimated to be a correct prediction if its confidence score is above the threshold and to be an error, otherwise. Consequently, misclassification detection boils down to a binary classification task where instead of rejecting samples, we assign them a binary label:
 \begin{equation}
     g(\bm{x})=
    \begin{cases}
      1 & \text{if}\ \kappa_f(\bm{x}) \geq \delta \, , \\
      0 & \text{otherwise.} \\
    \end{cases}
\end{equation}
Such as for selective classification, the choice of the threshold impacts misclassification detection. Given a threshold $\delta$, the test set can be split into true positives (TP), false positives (FP), false negative (FN) and true negatives (TN).  From this confusion matrix, a common evaluation metrics is to choose a threshold such that the True Positive Rate (TPR) is equal to 95\% and then evaluate the False Positive Rate (FPR):
\begin{equation}
    FPR = \frac{FP}{FP+TN} \quad\quad\textrm{with} \quad TPR = \frac{TP}{TP + FN} = 95\%.
\end{equation}
Threshold-independent evaluation metrics include the Area Under the ROC curve (AUROC), where the latter is a graphical plot showing the TPR and FPR against each other. It illustrates the ability of a binary classifier as its discrimination threshold is varied and it can be interpreted as the probability that a positive example has a greater detector score/value than a negative example. However, AUROC may suffer from unbalanced dataset, for instance when there is a larger amount of good predictions than wrong ones. In that case, AUROC will be close to 100\% and the impact of wrongly ranking a misclassification is mitigated \cite{hendrycks17baseline}. 

Alternatively, the Area Under the Precision-Recall curve (AUPR) is based on the graph between precision and recall:
\begin{equation}
    \textrm{precision} = \frac{TP}{TP+FP} \quad\quad \textrm{and} \quad\quad \textrm{recall}= \frac{TP}{TP+FN}.
\end{equation}
AUPR is a metric that adjusts for different positive and negative base rates. As such, there is AUPR-Success where good predictions are considered positive and AUPR-Error where misclassification are now the positive class. In the second case, confidence scores are multiplied by $-1$. 

As we will see in \cref{chap3}, a widely-used baseline method with deep neural networks \cite{hendrycks17baseline} is to take the value of MCP given by the softmax layer output. A detailed review of proposed methods is presented in \cref{chap3:sec:related_work}.

While these evaluation metrics can be used to assess the misclassification detection performance of a model, they cannot be used to directly compare performance across different models \cite{Ashukha2020Pitfalls}. Correct and incorrect predictions are specific for every model, therefore, every model induces its own binary classification problem. The induced problems can differ significantly, since different models produce different confidences and misclassify different objects.

\subsection{Calibration}
\label{chap2:subsec:calibration}

In a number of applications of machine learning, it is of increasing importance to know whether the classifier output can be interpreted as actual probabilities. For instance, self-driving car with a multi-modal prediction system needs its individual components to provide comparable probabilities. Alternatively, in medical diagnosis, a ML system could request an additional analysis from human doctors if its output probability of a disease diagnosis is too low \cite{Jiang2012CalibratingPM}.

Given the probability distribution $\hat{p}(\x) = p(Y \vert \x, \btheta)$ output by the model for a sample $\x$, a probabilistic classifier is calibrated if any predicted class probability is equal to the true class probability according to the underlying data distribution \cite{vaicenavicius19a}:
\begin{equation}
    \forall \x \in \cX, \quad \mathbb{P} \big [ Y~\vert~\hat{p}(\x) \big ] = \hat{p}(\x).
    \label{eq:calibration}
\end{equation}
Any deviation from the perfect calibration is called miscalibration.
A weaker condition \cite{guo2017} is to consider only the probability, or confidence estimate $\kappa_f$, associated with the predicted class $\hat{y}$:
\begin{equation}
   \forall \x \in \cX, \quad \mathbb{P} \big [ Y = \hat{y}~\vert~\kappa_f(\x) \big ] = \kappa_f(\x).
    \label{eq:weaker_calibration}
\end{equation}
For instance, given 100 predictions with a confidence $\kappa_f(\x)=0.7$, we expect that 70 samples should be correctly classified if the classifier is perfectly calibrated according to \cref{eq:weaker_calibration}.

Expected calibration error (ECE) \cite{naeini2015} is a metric that estimates model miscalibration by splitting the probability scores into $M$ bins $B_m$ and comparing them to average accuracies inside these bins:
\begin{equation}
    \text{ECE} = \sum_{m=1}^M \frac{B_m}{N} ~\vert~ \textrm{acc}(B_m) - \textrm{conf}(B_m) ~\vert,
\end{equation}
where
\begin{equation*}
    \textrm{acc}(B_m) = \frac{1}{\vert B_m \vert} \sum_{\x \in B_m} \delta(\hat{y}(\x) - y(\x)) 
    \quad\quad \textrm{and} \quad\quad
    \textrm{conf}(B_m) = \frac{1}{\vert B_m \vert}  \sum_{\x \in B} \kappa_f(\x).
\end{equation*}
But ECE metric suffers from certain shortcomings. Due to binning, it does not monotonically increase as predictions approach ground truth (a biased estimator of the true calibration). Then, it only estimates miscalibration in terms of the maximum probability and does not evaluate the first condition in \cref{eq:calibration}. Worse, a model may attain a perfect ECE score while being not accurate. For instance, on a binary classification, a model always predicting the first class $y=1$ with confidence $\kappa_f(x)=0.3$ may be perfectly calibrated on a dataset with 70\% inputs of class $0$ and 30\% inputs of class $1$, although its accuracy is only $0.3$.

Alternatively, Lakshminarayanan \textit{et al.} \cite{deepensembles2017} argues that models should be trained and evaluated using a proper scoring rule to achieve good calibration. For instance, the \emph{Brier score} measures the squared error between the predictive probability of a label and
one-hot encoding of the correct label:
\begin{equation}
    \text{BS} = \frac{1}{\vert \cD_{\textrm{test}} \vert} \sum_{(\x,y) \in \cD_{\textrm{test}}} \Big ( \frac{1}{K} \sum_{k \in \cY} \big [ p(Y = k ~\vert~ \x, \btheta) - \delta(y - k) \big ]^2 \Big ).
\end{equation}
Finally, a popular metric for measuring the quality of in-distribution uncertainty is to measure the negative log-likelihood (NLL):
\begin{equation}
     \text{NLL} = - \frac{1}{\vert \cD_{\textrm{test}} \vert} \sum_{(\x,y) \in \cD_{\textrm{test}}} \log p(y \vert \x, \mathcal{D}).  
\end{equation}
In classification, NLL boils down to computing the cross-entropy loss on the test set. It directly penalizes high probability scores assigned to incorrect labels and low probability scores assigned to the correct labels.

Recent work \cite{guo2017} revealed that deep neural networks are poorly calibrated. Among recalibration methods, a popular approach is to apply temperature scaling on models' logit \cite{guo2017}. The temperature parameter $T$ is learned on a validation dataset $\cD_{\textrm{val}}$ by minimizing the negative log-likelihood and keeping model's weight fixed:
\begin{equation}
    \min_{T \in \mathbb{R}^{+}} \frac{1}{\vert \cD_{\textrm{val}} \vert} \sum_{(\x,y) \in \cD_{\textrm{val}}} \log \frac{\exp \big ( f_y(\x, \btheta)/T \big)}{\sum_{k \in \cY} \exp \big ( f_k(\x, \btheta) / T \big )}.
\end{equation}
Even though temperature scaling improves calibration, it does not affect the ranking of the confidence score between inputs. Consequently, temperature scaling is not effective to improve misclassification detection or selective classification mentioned in the previous sections.

\subsection{Out-of-distribution detection}
\label{chap2:subsec:ood}

Until now, we reviewed tasks that evaluate uncertainty estimation on test samples assumed to be drawn \textit{i.i.d.} from the same distribution than the training data. But as motivated in \cref{chap1}, in many safety-critical applications, ML systems are deployed in an open-world scenario \cite{Bendale_2015_CVPR}. Inputs can be subject to distributional shifts which are categorized either as covariate shifts maintaining semantic consistency, such as a drawing of a dog when training data was only composed of natural images, or semantic shifts where the label space $\cY$ is different from in-distribution data, such as input from a new class.

\begin{figure}[t]
\centering
    \includegraphics[width=\linewidth]{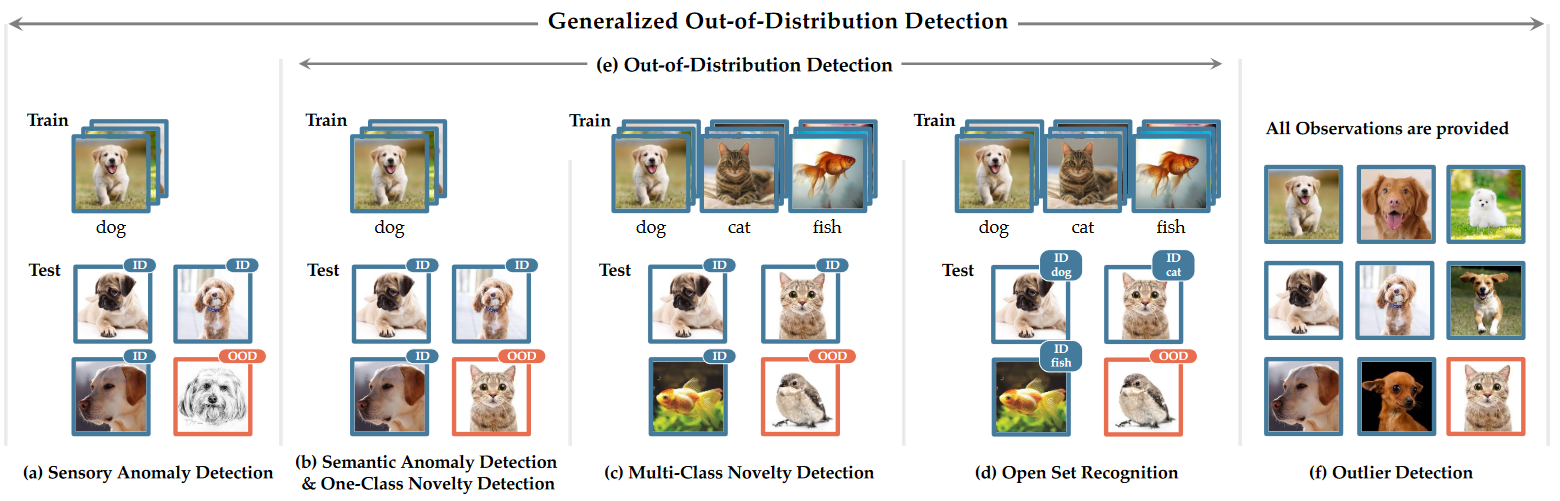}
    \caption[Out-of-distribution framework]{\textbf{Reproduction of the out-of-distribution framework proposed in \cite{yang2021oodsurvey}}. Five detection tasks are represented depending on their problem setting. Semantic anomaly detection, multi-class novelty detection and open-set recognition are considered as subcategories of out-of-distribution detection.}
    \label{chap2:fig:ood_survey}
\end{figure}

In the ML literature, several fields attempt to address the issue of identifying the unknowns/outliers/anomalies samples in the open-world setting. In their survey, Yang \textit{et al.} \cite{yang2021oodsurvey} provide an interesting unified framework of these subtopics, summarized in \cref{chap2:fig:ood_survey} reproduced from their paper. In classification tasks, we can define five subcategories depending on the problem setting:
\begin{enumerate}
    \item \textbf{Sensory anomaly detection}: training data is composed of only one class and test data may present non-semantic covariate shift, the goal is to detect these anomalies;
    \item  \textbf{Semantic anomaly detection}: training data is still composed of only one class and test data may now present semantic shift, such as new classes, the goal is to detect these anomalies; \label{chap2:task:AD}
    \item  \textbf{Multi-class novelty detection}: training data is composed of $C$ classes and test data may present semantic shift, such as new classes, the goal is still to detect these anomalies; \label{chap2:task:ND}
    \item  \textbf{Open-set recognition}: training data is composed of $C$ classes and test data may present semantic shift, such as new classes, but now the goal is twofold: correctly classify in-distribution data while detecting these anomalies; \label{chap2:task:OSR}
    \item  \textbf{Outlier detection}: a transductive problem where all observations are provided, we do not consider a train/test split anymore, and some samples can present any distributional shift, the goal is still to detect these outliers, for instance to clean data.
\end{enumerate}
Among these previous tasks, we commonly refer as \emph{out-of-distribution detection} the sensory/semantic anomaly detection and multi-class novelty detection. 

Out-of-distribution (OOD) detection shares similarities with misclassification detection as they both aim to detect errors or abnormal samples in a given test set. AUROC and AUPR metrics where the positive class is composed of OOD samples are used to evaluate the capacity of a model independently of a specific threshold to separate OOD samples from in-distribution samples according to a confidence score.

With deep neural networks, post-processing logits with temperature scaling using a large temperature $T$ on a pre-trained model has been shown to be effective to reduce model's over-confidence on OOD samples \cite{odin2018}. At test time, the authors use the MCP as confidence score to detect OOD samples after applying the temperature scaling. Also in the family of post-processing methods, Lee \textit{et al.} \cite{mahalanobis2018} assumed that intermediate feature maps -- in particular the penultimate before last classification layer -- of a trained deep neural network follow class-conditional Gaussian distributions with a tied covariance matrix. They estimate its parameters on training data:
\begin{equation}
    \hat{\mu}_k = \frac{1}{N_k} \sum_{n: y_n=k}^N f(\x_n, \btheta) \quad\textrm{and}\quad \hat{\Sigma} = \frac{1}{N} \sum_{k \in \cY} \sum_{n: y_n=k}^N \big ( f(\x_n, \btheta)  - \hat{\mu}_k \big ) \big ( f(\x_n, \btheta)  - \hat{\mu}_k \big )^T,
\end{equation}
where $N_k$ is the number of training samples with label $k$. Their confidence score correspond to the maximum Mahalanobis distance between input $\x$ and the closest class-conditional Gaussian distribution:
\begin{equation}
    M(\x) = \max_{k \in \cY} - \big ( f(\x, \btheta)  - \hat{\mu}_k \big )^T \hat{\Sigma}^{-1} \big ( f(\x, \btheta)  - \hat{\mu}_k \big ).
\end{equation}
These two previous methods also used adversarial perturbations to improve the separability of OOD samples from in-distribution data. In contrast to the original literature on adversarial perturbation, they use fast gradient sign method \cite{goodfellow2015} (FGSM) to increase the probability of the model on the predicted class (see \cref{chap2:subsec:adversarial}). A limitation of these methods is that they need relevant OOD samples to find the right hyper-parameters $T$ and $\varepsilon$.

On the other hand, a range of methods assumed that a set of OOD samples may be available during training. For instance, Hendrycks \textit{et al.} \cite{hendrycks2019oe} proposed to train a deep neural network to
simultaneously classify in-distribution samples and to produce high predictive entropy for samples from a known large out-of-distribution dataset $\cD_{out}$:
\begin{equation}
    \cL_{\text{OE}}(\btheta, \cD, \cD_{out}) =  \mathbb{E}_{(\x, y) \sim \cD} \big [ \log p(y \vert \x, \btheta) \big ] + \lambda \mathbb{E}_{\x \sim \cD_{out}} \big [ \bbH[p(\cdot \vert \x, \btheta)] \big ].    
\end{equation}
While this method remains the best OOD detector so far, the assumption of available OOD data during training may be unrealistic in many applications \cite{charpentier2020,sensoy2020}. In addition, we show in \cref{chap5} that all these previous methods are brittle to the choice of the OOD dataset.

Along with Mahalanobis-based OOD detection, a set of density-based methods relying on generative models attempt to model in-distribution data and to detect anomalous test data assuming OOD samples have low likelihood. In the context of deep learning, a classic method is to use an auto-encoder (AE) or a variational auto-encoder \cite{Kingma2014} (VAE) as generative model. However, Nalisnick \emph{et al.} \cite{nalisnick2018do} find that the density learned by flow-based models \cite{DinhKB14}, VAEs \cite{NIPS2016_b1301141} and PixelCNNs \cite{Salimans2017PixeCNN} may assign a larger likelihood to OOD samples than in-distribution samples in some vision benchmarks (CIFAR-10 vs. SVHN, FashionMNIST vs MNIST, CelebA vs. SVHN, ImageNet vs. CIFAR-10). Finally, recent works \cite{Grathwohl2020Your,NEURIPS2020_f5496252} explored using energy-based models (EBMs) for OOD detection, due to their natural fit within a discriminative framework. EBMs are generative models that use a scalar energy score to express probability density through unnormalized negative log probability \cite{LeCun06atutorial}. But their learning process can be computationally unstable as they requires approximations, such as stochastic gradient Langevin dynamics \cite{WellingT11} to estimate integrals.

\subsection{Adversarial robustness}
\label{chap2:subsec:adversarial}

\begin{figure}[t]
\centering
    \includegraphics[width=\linewidth]{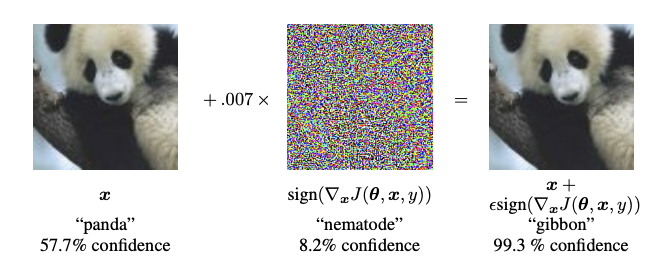}
    \caption[Example of an adversarial attack with FGSM in classification]{\textbf{Example of an adversarial attack with FGSM in classification}. An original input correctly classified as a panda becomes a misclassified input when applying the adversarial perturbation. Worse, its confidence score, here MCP, is arbitrarily high, wrongly indicating a confident prediction. \textit{Image credits: Goodfellow et al. \cite{goodfellow2015}.}}
    \label{chap2:fig:adversarial}
\end{figure}

In contrast with anomalies, \emph{adversarial examples} are inputs which are indistinguishable to the human eye but confuse a neural network, resulting in a misclassification (see \cref{chap2:fig:adversarial}). Adversarial examples are crafted by applying small perturbations to the input and restricting the magnitude of the attack to a value inferior to a bit of an 8-bit image encoding. In the original paper, the adversarial attack used rely on the fast gradient sign method \cite{goodfellow2015} (FGSM):
\begin{equation}
    \tilde{\x} = \x + \varepsilon\cdot\textrm{sign} \nabla_{\x} f(\x, \btheta)
\end{equation}
Since, a profusion of different adversarial attacks has been proposed in the literature \cite{Papernot2016,tramer2017space,CarliniW16a,dong2018}.

In adversarial robustness, the goal of an adversarial defense mechanism is to improve robustness of deep neural networks against adversarial attacks, \ie reduce the gap between `clean' accuracy on original inputs and `adversarial` accuracy on adversarial examples. Multiple defense mechanisms have been proposed over the years but they almost all end up being defeated by new adversarial attacks, except for adversarial training \cite{madry2018towards} which remains correct under certain conditions. Consequently, recent advances in adversarial robustness tend to construct certified defenses where neural networks are provably robust against adversaries \cite{raghunathan2018certified}.

Recently, Tsipras \emph{et al.} \cite{tsipras2018robustness} showed the goal of adversarial robustness might fundamentally be at odds with that of standard generalization. For instance, adversarial training improves adversarial accuracy but also produce a slight decrease in original test accuracy. Instead of considering robustness to adversarial attacks, a different line of work \cite{Carlini2017AdversarialEA,malinin2019} investigates detection of these adversarial attacks. As with misclassification detection and OOD detection, the evaluation metric used are threshold-independent metrics such AUROC and AUPR.

\section{Conclusion}
\label{chap2:sec:conclusion}

Uncertainty estimation is a wide research area, ranging from theoretical perspectives with Bayesian approaches to practical considerations with the detection of abnormal samples to avoid critical failures. Uncertainty can arise due to the stochasticity of the latent data generative process (\emph{aleatoric uncertainty}) or due to the lack of knowledge of the model on an input (\emph{epistemic uncertainty}). While `ground-truth' uncertainty estimates are usually not available, different tasks aims at evaluating the capacity of the model to provide accurate uncertainty estimates. Selective classification, misclassification detection and calibration evaluate in-distribution uncertainty, either for rejecting/detecting errors or to ensure a classifier which outputs correct probabilities. Out-of-distribution detection considers an open-world setting where distribution shifts and inputs from unknown classes may occur. Finally, adversarial robustness is a particular task where inputs have been corrupted to fool the classifier. In particular, one may see these adversarial examples as a worst-case analysis of distribution shift \cite{gilmer2019a}. 

In this thesis, we will start by addressing in-distribution uncertainty estimation by proposing a learning confidence approach with auxiliary model to improve misclassification and selective classification in \cref{chap3}. The task of selective classification is also useful for self-training methods presented in \cref{chap4}. Finally, we tackle the challenge of jointly quantifying in-distribution and out-of-distribution (OOD) uncertainties in \cref{chap5} with an uncertainty measure which account both for aleatoric and epistemic uncertainty.

%% file: chapter03/chapter03.tex
\begin{center}
  \textsc{Chapter Abstract}
\end{center}
\begin{quote}
\noindent \textit{Reliably quantifying the confidence of deep neural classifiers is a challenging yet fundamental requirement for deploying ML models in safety-critical applications. In this chapter, we are interested in the problem of detecting in-distribution erroneous predictions of deep neural networks in the context of classification. We introduce a novel target criterion for a model's confidence, namely the \emph{True Class Probability} (TCP) and show that TCP offers better properties for failure prediction than standard uncertainty measures. Since the true class is by essence unknown at test time, we propose to learn the TCP criterion from data with an auxiliary model, \emph{ConfidNet}, introducing a specific learning scheme adapted to this context. A major benefit of ConfidNet is to be a separate network which can estimate the model confidence of any trained classifier. We evaluate our approach on the task of failure prediction and selective classification and we validate that the proposed approach provides accurate confidence estimates. We study various network architectures and experiment with small and large datasets for image classification and semantic segmentation. In every tested benchmark, our approach outperforms strong baselines.}

\textit{The work described in this chapter is based on the following publications:}
\begin{itemize}
    \item \textit{Charles Corbière, Nicolas Thome, Avner Bar-Hen, Matthieu Cord, Patrick Pérez. ``Addressing Failure Prediction by Learning Model Confidence". In Advances in Neural Information Processing Systems (NeurIPS), 2019.}
    \item \textit{Charles Corbière, Nicolas Thome, Antoine Saporta, Tuan-Hung Vu, Matthieu Cord, Patrick Pérez. ``Confidence Estimation via Auxiliary Models". In IEEE Transactions on Pattern Analysis and Machine Intelligence, 2021.} 
\end{itemize}
\end{quote}

\clearpage
\minitoc

\section{Context}
\label{chap3:sec:context}

Propagating an erroneous prediction of a machine learning system or over-estimating its confidence may carry serious repercussions in critical visual-recognition applications such as in autonomous driving, medical diagnosis \cite{medicaldiag2018} or nuclear power plant monitoring \cite{Linda2009}. In classification, \emph{failure prediction} is the task of predicting at run-time whether a trained model has taken a correct decision or not for a given input. By detecting an erroneous prediction, a system could decide to stick to the prediction or, on the contrary, to hand it over to a human or a back-up system with, \eg other sensors, or simply to trigger an alarm. Closely related to failure prediction, classification with a reject option \cite{Chow1957AnOC}, also known as \emph{selective classification} \cite{elyaniv10a}, consists in a scenario where the classifier is given the option to reject an instance instead of predicting its label. These two tasks refer to the same problem of \emph{ordinal ranking}, which aims to estimate confidence values whose ranking of samples is effective to distinguish correct from incorrect predictions (see \cref{chap3:fig:ordinal_ranking}).  Then, the user can specify a threshold so that some inputs with predicted confidence is below it are considered as erroneous predictions. 

In failure prediction, a widely used baseline with neural-network classifiers is to take the value of the predicted class' probability, namely the \textit{maximum class probability} (MCP), given by the softmax layer output. Although recent evaluations of MCP with modern deep models reveal reasonable performance~\cite{hendrycks17baseline}, they still suffer from several conceptual drawbacks. In particular, MCP leads by design to high confidence values, even for erroneous predictions, since the largest softmax output is used. This design tends to make erroneous and correct predictions overlap in terms of confidence and thus limits the capacity to distinguish them. Another common uncertainty measure is the predictive entropy \cite{shannon} which captures the average amount of information contained in the probability vector output by the model. It is worth mentioning that these entropy-based criteria measure the softmax output dispersion, where the uniform distribution has maximum entropy. But it is not clear how well these dispersion measures are adapted to distinguishing failures from correct predictions. We elaborate on these limits in \cref{chap3:subsec:limits}.

\begin{figure}[t]
\centering
    \includegraphics[width=0.9\linewidth]{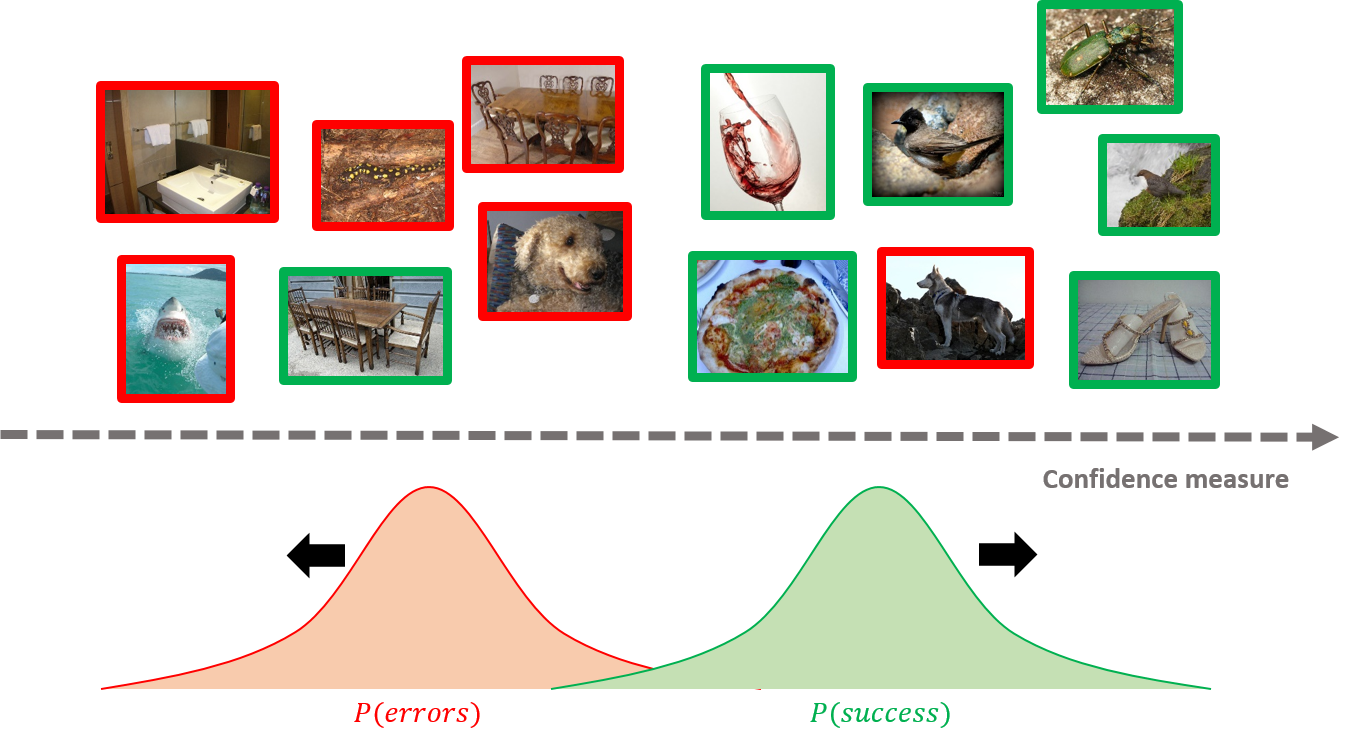}
    \caption[Illustration of an effective confidence measure for ordinal ranking on  in-distribution samples]{\textbf{Illustration of an effective confidence measure for ordinal ranking on in-distribution samples.} When ranking samples according to their confidence score, correct predictions (in green) should have higher values on average than misclassifications (in red) to enable the model to distinguish them.}
\label{chap3:fig:ordinal_ranking}
\end{figure}

In this chapter, we identify a better confidence criterion, the \emph{true class probability} (TCP), for deep neural network classifiers with a reject option (\cref{chap3:sec:confidence_measure}). We provide simple guarantees of the quality of this criterion regarding confidence estimation. Since the true class is obviously unknown at test time, we propose a novel approach, \emph{ConfidNet}, which consists in designing an auxiliary network specifically dedicated to estimate the confidence of a prediction (\cref{chap3:sec:confidnet}). Given a trained classifier $f$, this auxiliary network learns the TCP criterion from data. When applied to failure prediction, we observe significant improvements over strong baselines (\cref{chap3:subsec:results}). A thorough analysis of our approach, including relevant variations, ablation studies and qualitative evaluations of confidence estimates, helps to gain insight about its behavior in \cref{chap3:subsec:variants}.

\section{Defining a confidence measure for effective ordinal ranking}
\label{chap3:sec:confidence_measure}

In this section, we first briefly introduce the task of classification with a reject option, along with necessary notations. We also address semantic image segmentation, which can be seen as a pixel-wise classification problem, where a model outputs a dense segmentation mask with a predicted class assigned to each pixel. As such, all the following material is formulated for classification, and implementation details for segmentation are specified when necessary. We point out the limits of current measures and present our effective confidence-rate function for neural-net classifiers.

\subsection{Problem formulation}
\label{chap3:subsec:formulation}

Following notations introduced in \cref{chap2}, we consider a training dataset $\cD= \{ (\x_n, y_n) \}_{n=1}^N$ composed of $N$ \textit{i.i.d.} training samples, where $\x_n \in \cX \subset \mathbb{R}^D$ is a $D$-dimensional data representation, deep feature maps from an image or the image itself for instance, and $y_n \in \cY=\llbracket 1, K \rrbracket$ is its true class among the $K$ predefined categories. These samples are drawn from an unknown joint distribution $P(X,Y)$ over $(\cX, \cY)$.

\begin{definition}[Selective classifier]
    A \textit{selective classifier} \cite{elyaniv10a, NIPS2017_7073} is a pair $(f,g)$ where $f: \cX \rightarrow \cY$ is a \textit{prediction function} and $g: \cX \rightarrow \{0,1\}$ is a \textit{selection function} which enables to reject a prediction:
\begin{equation}
    (f,g)(\bm{x}) =
    \begin{cases}
      f(\x), & \text{if}\ g(\x)=1 \, , \\
      \mathrm{reject}, &\text{if}\ g(\x)=0 \, . \\
    \end{cases}
\end{equation}
\end{definition}

We focus on classifiers based on artificial neural networks. Given an input $\x$, such a network $F$ with parameters $\btheta$ outputs non-negative scores over all classes, which are normalized through softmax. If well trained, this output can be interpreted as the predictive distribution
$F(\x;\hat{\btheta}) = P(Y \vert \x, \hat{\btheta}) \in \Delta^{K-1}$, with $\Delta^{K-1}$ the probability (K-1)-simplex in $\mathbb{R}^{K}$ and $\hat{\btheta}$ the learned weights.
Based on this distribution, the predicted sample class is usually the \textit{maximum-a-posteriori} estimate:  
\begin{equation}
f(\bm{x}) = \mathrm{arg}\!\max_{k \in \cY}~P(Y = k \vert \bm{x}, \hat{\btheta}) = \mathrm{arg}\!\max_{k \in \cY} F(\x;\hat{\btheta})[k].
\label{chap3:eq:F2f}
\end{equation}

\begin{figure}[t]
\centering
\captionsetup[subfigure]{justification=centering}
\begin{minipage}[c]{0.5\linewidth}
\centering
    \includegraphics[width=\linewidth]{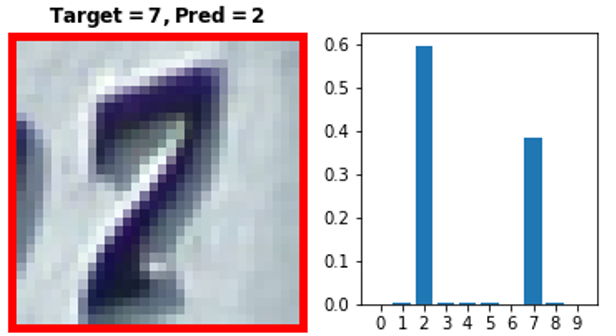}
    \subcaption{
    Erroneous prediction, entropy $= 0.79$}
    \label{chap3:fig:visu-entropy_a}
\end{minipage}%
\begin{minipage}[c]{0.5\linewidth}
\centering
    \includegraphics[width=\linewidth]{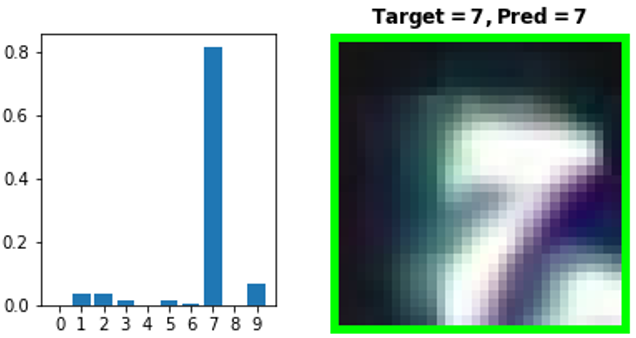}
    \subcaption{
    Correct prediction, entropy $= 0.79$}
    \label{chap3:fig:visu-entropy_b}
\end{minipage}%
  \caption[Illustration of the limits of predictive entropy as a confidence measure on SVHN samples.] {\textbf{Illustrating the limits of predictive entropy as confidence estimation on the SVHN test samples}. Red-border image (\cref{chap3:fig:visu-entropy_a}) is misclassified by the classification model; green-border image (\cref{chap3:fig:visu-entropy_b}) is correctly classified. Predictions exhibit similar high entropy in both cases. For each sample, we provide a plot of their softmax predictive distribution.}
  \label{chap3:fig:visu-entropy}
\end{figure}

We are not interested here in trying to improve the accuracy of the already-trained model $F$, but rather in making its future use more reliable by endowing the system with the ability to recognize when the prediction might be wrong.

To this end, a \textit{confidence-rate function} $\kappa_f:\cX \rightarrow \mathbb{R}^{+}$ is associated to $f$ so as to assess the degree of confidence of its predictions, the higher the value the more certain the prediction \cite{elyaniv10a, NIPS2017_7073}. A suitable confidence-rate function should correlate erroneous predictions with low values and successful predictions with high values. 
Finally, given a user-defined threshold $\delta \in \mathbb{R}^+$, the selection function $g$
can be simply derived from the confidence rate:
\begin{equation}
    g(\bm{x})=
    \begin{cases}
      1 & \text{if}\ \kappa_f(\bm{x}) \geq \delta \, , \\
      0 & \text{otherwise.} \\
    \end{cases}
\end{equation}

\subsection{Limits of current uncertainty measures}
\label{chap3:subsec:limits}

\begin{figure}[t]
\centering
\begin{minipage}[c]{0.5\linewidth}
\centering
    \includegraphics[width=\linewidth]{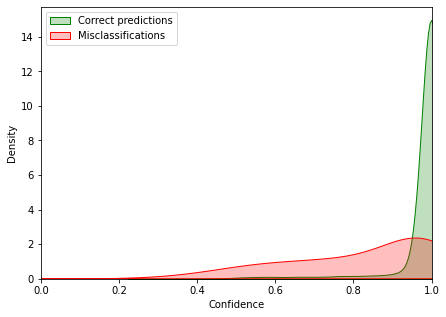}
    \subcaption{Maximum Class Probability}
    \label{chap3:fig:density-plot-mcp}
\end{minipage}%
\begin{minipage}{0.5\linewidth}
\centering
    \includegraphics[width=\linewidth]{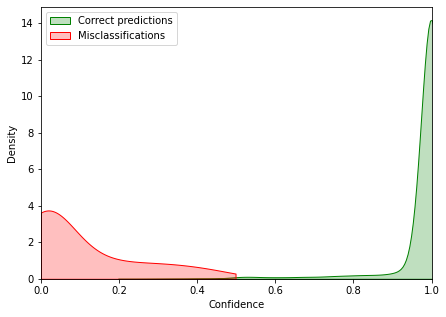}
    \subcaption{True Class Probability}
    \label{chap3:fig:density-plot-tcp}
\end{minipage}
\caption[Distributions of MCP and TCP confidence estimates computed over correct and erroneous predictions by a trained VGG-16 model on CIFAR-10]{\textbf{Distributions of MCP and TCP confidence estimates computed over correct and erroneous predictions by a trained VGG-16 model on CIFAR-10.} When ranking according to MCP (a) the test predictions of a convolutional model trained on CIFAR-10, we observe that correct ones (in green) and misclassifications (in red) overlap considerably, making it difficult to distinguish them. On the other hand, ranking samples according to TCP (b) alleviates this issue and allows a much better separation.}
\label{chap3:fig:density-plot}
\end{figure}

For a given input $\x$, a standard uncertainty measure for a classifier $F$ is the probability associated to the predicted max-score class, that is the \textit{maximum class probability}:

\begin{definition}[Maximum Class Probability]
For a given input $\x$ and a classifier $F$, the Maximum Class Probability (MCP) is defined as: 
\begin{equation}
\mathrm{MCP}_F(\x) = 
\max_{k \in \cY} P(Y=k \vert \x, \hat{\btheta}) =  
\max_{k \in \cY} F(\x;\hat{\btheta})[k].
\end{equation} 
\end{definition}

However, by taking the largest softmax probability as a confidence estimate, MCP leads to high confidence values both for correct and erroneous predictions alike, making it hard to distinguish them, as shown in \cref{chap3:fig:density-plot-mcp}.

Taking the predictive entropy as uncertainty measure may not also be always adequate. In \cref{chap3:fig:visu-entropy}, we show side-by-side two samples with a similar distribution entropy taken from a small convolutional network trained on SVHN, a street-view numbers dataset~\cite{svhn-dataset}.  Left image (red-border) is misclassified while the right one enjoys a correct prediction (green-border). Predictions exhibit similar high entropy in both cases. But as entropy is a symmetric measure in regards to class probabilities: a correct prediction with $[0.65, 0.35]$ distribution is evaluated as confident as an incorrect one with $[0.35, 0.65]$ distribution, which is undesirable for accurate failure prediction.

\subsection{The True Class Probability}
\label{chap3:subsec:tcp}

When the model misclassifies an example, the probability associated to the true class $y$ is lower than the maximum one and likely to be low. Based on this simple observation, we propose to consider instead this \emph{true class probability} as a suitable confidence-rate function.

\begin{definition}[True Class Probability]
Given a classifier $F$, for any admissible input $\x\in\cX$ we assume the \textit{true} class $y(\x)$ is known, which we denote $y$ for simplicity. The True Class Probability (TCP) is defined as  
\begin{equation}
    \mathrm{TCP}_F(\x,\,y) = P(Y=y \vert\,\x, \hat{\btheta}) = F(\x;\hat{\btheta})[y].
\end{equation}
\end{definition}

In \cref{chap3:fig:density-plot}, we can observe that TCP allows a much better separation than MCP. In particular, TCP offers the following interesting guarantees regarding its ability to characterize correct or erroneous predictions of a model.

\begin{proposition}
Given a properly labelled example $(\bm{x},y)$, then:
\begin{itemize}
    \item $\mathrm{TCP}_F(\bm{x},y) > 1/2$ $~\Rightarrow$ $f(\x) = y$, \ie the example is correctly classified by the model;
    \item $\mathrm{TCP}_F(\bm{x},y) < 1/K$ $\Rightarrow$ $f(\x) \neq y$, \ie the example is wrongly classified by the model,
\end{itemize} 
where class prediction $f(\x)$ is defined by \cref{chap3:eq:F2f}.
\end{proposition} 

\begin{proof}
Let $F$ be a trained neural network classifier with learned weights $\hat{\btheta}$, $K$ be the number of labels and $\bm{x} \in \mathbb{R}^D$ a sample with its associated true label $y \in \mathcal{Y}$ such that  $\mathrm{TCP}_F(\bm{x},y) > \frac{1}{2}$. Starting from the definition of TCP we have:
\begin{align}
   \mathrm{TCP}_F(\bm{x},y) = P(Y=y \vert \bm{x}, \hat{\btheta}) &> \frac{1}{2} \\
   \iff 1 - \sum_{k \in \mathcal{Y}, k \neq y} P(Y=k \vert \bm{x}, \hat{\btheta}) &> \frac{1}{2} \\
   \iff \sum_{k \in \mathcal{Y}, k \neq y} P(Y=k \vert \bm{x},\hat{\btheta}) &< \frac{1}{2}.
\end{align}
Since probabilities are positive, we obtain that  $\forall k \neq y$, $P(Y=k \vert \bm{x}, \hat{\btheta}) <\frac{1}{2}<P(Y=y \vert \bm{x}, \hat{\btheta})$.
Denoting $\hat{y}=f(\x)$ the class predicted by the network, we have $\hat{y} = \argmax_{k \in \cY} P(Y=k \vert \bm{x}, \hat{\btheta})$. Hence $\hat{y}=y$.

In the same way, for any $(\bm{x},y) \in \mathbb{R}^D\times\mathcal{Y}$, such that $\mathrm{TCP}_F(\bm{x},y) < \frac{1}{K}$, we have:
\begin{align}
   P(Y=y \vert \bm{x}, \hat{\btheta}) &< \frac{1}{K} \\
   \iff 1 - \sum_{k \in \mathcal{Y}, k \neq y} P(Y=k \vert \bm{x}, \hat{\btheta}) &< \frac{1}{K} \\
   \iff \sum_{k \in \mathcal{Y}, k \neq y} P(Y=k \vert \bm{x}, \hat{\btheta}) &> \frac{k-1}{K}.\label{eq1}
\end{align}
If the model correctly classifies this sample, \textit{i.e.}, $\hat{y}=y$, then $\forall k \neq y$, $P(Y=y \vert \bm{x}, \hat{\btheta}) \geq P(Y=k \vert \bm{x}, \hat{\btheta})$. We have:
\begin{equation}
    \sum_{K \in \mathcal{Y}, K \neq y} P(Y=k \vert \bm{x}, \hat{\btheta}) \leq (K-1)P(Y=y \vert \bm{x}, \hat{\btheta}) \leq \frac{K-1}{K},
\end{equation}
which contradicts \cref{eq1}. Hence, there exists at least one $k$ such that $P(Y=k \vert \bm{x}, \hat{\btheta}) > P(Y=y \vert \bm{x}, \hat{\btheta})$, which results in $\hat{y} \neq y$.
\end{proof}

Within the range $[1/K, 1/2]$, there is no guarantee that correct and incorrect predictions will not overlap in terms of TCP. However, when using deep neural networks, we observe that the actual overlap area is extremely small in practice, as illustrated in \cref{chap3:fig:density-plot-tcp} on the CIFAR-10 dataset. One possible explanation comes from the fact that modern deep neural networks output overconfident predictions and therefore non-calibrated probabilities (see \cref{chap2:subsec:calibration}). We provide consolidated analyses on this aspect in \cref{chap3:sec:failure} and further results on other datasets in \cref{appxA:sec:plots}. 

We also introduce a normalized variant of the TCP confidence criterion, which consists in computing the \textit{ratio} between TCP and MCP: 

\begin{definition}[Normalized True Class Probability]
Given a classifier $F$, for any admissible input $\x\in\cX$ we assume the \textit{true} class $y(\x)$ is known, which we denote $y$ for simplicity. The normalized True Class Probability (nTCP) is defined as  
\begin{equation}
    n\mathrm{TCP}_F(\x,\,y) = \frac{P(Y=y \vert\,\x, \hat{\btheta})}{P(Y=\hat{y} \vert\,\x, \hat{\btheta})}.
\end{equation}
\end{definition}

The normalized criterion \textit{n}TCP presents stronger theoretical guarantees than TCP, since correct predictions will be, by  design, assigned the value of $1$, whereas errors will range in $[0,1[$. On the other hand, learning this criterion may be more challenging since all correct predictions must match a single scalar value.

\section{ConfidNet: learning to predict TCP with an auxiliary model}
\label{chap3:sec:confidnet}

Using TCP as a confidence-rate function on a model's output would be of great help when it comes to reliably estimate its confidence. However, the true classes $y$ are obviously not available when estimating confidence on test inputs. 

\begin{figure}[t]
  \centering
  \includegraphics[width=\linewidth]{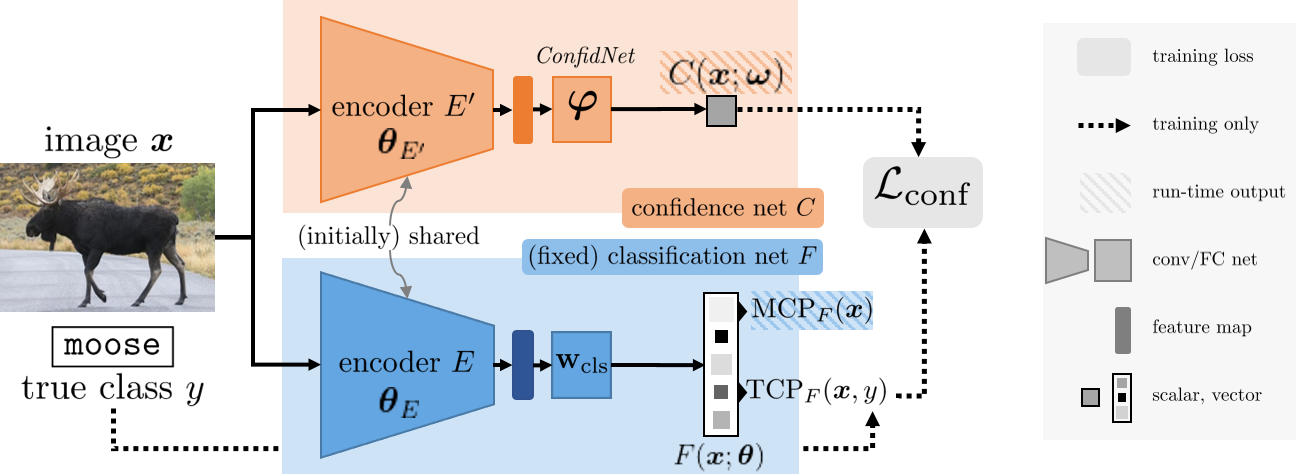}
  \caption[Overview of the learning confidence approach]{\textbf{Learning confidence approach.} The fixed classification network $F$, with parameters $\btheta =(\btheta_{E},\btheta_{\text{cls}})$, is composed of a succession of convolutional and fully-connected layers (encoder $E$) followed by last classification layers with softmax activation. The auxiliary confidence network $C$, with parameters $\bomega$, builds upon the feature maps extracted by the encoder $E$, or its fine-tuned version $E'$ with parameters $\btheta_{\text{E'}}$: they are passed to ConfidNet, a trainable multi-layer module with parameters $\bphi$. The auxiliary model outputs a confidence score  $C(\bm{x};\bomega)\in[0,1]$, with $\bomega = \bphi$ in absence of encoder fine-tuning and $\bomega =(\btheta_{E'},\bphi)$ in case of fine-tuning.}
  \label{chap3:fig:confidnet_network}
\end{figure}

\subsection{Principle}
\label{chap3:subsec:principle}

We propose to \emph{learn TCP confidence from data}. More formally, for the classification task at hand, we consider a parametric selective classifier $(f,g)$, with $f$ based on an already-trained neural network $F$. We aim at deriving its companion selection function $g$ from a learned estimate of the TCP function of $F$. 
To this end, we introduce an \textit{auxiliary model} $C$, with parameters $\bomega$, that is intended to predict $\text{TCP}_F$ and to act as a confidence-rate function for the selection function $g$. 
An overview of the proposed approach is available in \cref{chap3:fig:confidnet_network}. This model is trained such that, at runtime, for an input $\x \in \cX$ with (unknown) true label $y$, we have: 
\begin{equation}
    C(\x;\bomega) \approx \text{TCP}_F(\x,y).
\end{equation}

In practice, this auxiliary model $C$ will be a neural network trained under full supervision on $\cD$ to produce this confidence estimate. To design this network, we can transfer knowledge from the already-trained classification network. Throughout its training, $F$ has indeed learned to extract increasingly-complex features that are fed to its final classification layers. Calling $E$ the encoder part of $F$, a simple way to transfer knowledge consists in defining and training a multi-layer head with parameters $\bphi$ that regresses $\mathrm{TCP}_F$ from features encoded by $E$. We call \textit{ConfidNet} this module. As a result of this design, the complete confidence network $C$ is composed of a frozen encoder followed by trained ConfidNet layers. The complete architecture can be later fine-tuned, including the encoder, as in classic transfer learning. In that case, $\bomega$ will encompass the parameters of both the encoder and the ConfidNet's layers.

\subsection{Architecture}
\label{chap3:subsec:architecture}

Standard image classification models are composed of convolutional layers followed by one or more fully-connected layers and a final softmax operation. In order to work with such a classification network $F$, we build ConfidNet upon a late intermediate representation of $F$. ConfidNet is designed as a small multilayer perceptron composed of a succession of dense layers with a final sigmoid activation that outputs $C(\bm{x};\bomega)\in[0,1]$. ConfidNet is train in a supervised manner, such that it predicts well the true-class probability assigned by $F$ to the input image. Regarding the capacity of ConfidNet, we have empirically found that increasing further its depth leaves performance unchanged for estimating the confidence of the classification network (see \cref{chap3:subsec:variants}).

\subsection{Loss function}
\label{chap3:subsec:loss}

As we want to regress a score between $0$ and $1$, we use a mean-square-error (MSE) loss to train the confidence model:
\begin{equation} 
\cL_{\text{conf}}(\bomega;\cD) = \frac{1}{N} \sum_{n=1}^N \big(C(\x_n;\bomega) - \text{TCP}_F(\x_n,y_n)\big)^2.
\label{chap3:eq:loss-conf}
\end{equation}

Since the final task here is the prediction of failures, with confidence prediction being only a means toward it, a more explicit supervision with failure/success information could be considered. In that case, the previous regression loss could still be used, with 0 (failure) and 1 (success) target values instead of TCP. Alternatively, a binary cross entropy loss (BCE) for the error-prediction task using the predicted confidence as a score could be used. Seeing failure detection as a ranking problem, where good predictions must be ranked before erroneous ones according to the predicted confidence, a batch-wise ranking loss can also be utilized~\cite{Mohapatra_2018_CVPR}. We experimentally assessed all these alternative losses, including a focal version \cite{focaloss} of the BCE to focus on hard examples, as discussed in \cref{chap3:subsec:variants}. They lead to inferior performance compared to using \cref{chap3:eq:loss-conf}. This might be due to the fact that TCP conveys more detailed information than a mere binary label on the quality of the classifier's prediction for a sample. Hinton \textit{et al}. \cite{hinton2015distilling} make a similar observation when using soft targets in knowledge distillation. In situations where only very few error samples are available, this finer-grained information improves the performance of the final failure detection (see \cref{chap3:subsec:variants}).

\subsection{Learning scheme}
\label{chap3:subsec:learning}

We decompose the parameters of the classification network $F$ into $\btheta = (\btheta_{E}, \btheta_{\text{cls}})$, where $\btheta_{E}$ denotes its encoder's weights and $\btheta_{\text{cls}}$ the weights of its last classification layers. Such as in transfer learning, the training of the confidence network $C$ starts by fixing the shared encoder and training only ConfidNet's  weights $\bphi$. In this phase, the loss \cref{chap3:eq:loss-conf} is thus minimized only w.r.t. $\bomega = \bphi$.

In a second phase, we further fine-tune the complete network $C$, including its encoder which is now untied from the classification encoder $E$ (the main classification model must remain unchanged, by definition of the addressed problem). Denoting $E'$ this now independent encoder, and $\btheta_{E'}$ its weights, this second training phase optimizes \cref{chap3:eq:loss-conf} w.r.t. $\bomega = (\btheta_{E'}, \bphi)$ with $\btheta_{E'}$ initially set to $\btheta_{E}$.

We also deactivate dropout layers in this last training phase and reduce learning rate to mitigate stochastic effects that may lead the new encoder to deviate too much from the original one used for classification. Data augmentation can thus still be used. ConfidNet can be trained using either the original training set or a validation set. The impact of this choice is evaluated in \cref{chap3:subsec:variants}.

\section{Related work}
\label{chap3:sec:related_work}

Confidence estimation \cite{Blatz2004,zaragoza1998confidence} has a long history in the machine learning community, tightly related to classification with a reject option \cite{Chow1957AnOC}. The following works \cite{bartlettreject2008,NIPS2016_6336,cortes2016} explored alternative rejection criteria. In particular, \cite{cortes2016} proposes to jointly learn the classifier and the selection function. El-Yaniv \cite{elyaniv10a} provides an analysis of the risk-coverage trade-off that occurs when classifying with a reject option. More recently, \cite{NIPS2017_7073,Geifman2019SelectiveNetAD} extend the approach to deep neural networks, considering various confidence measures. Since the wide adoption of deep learning methods, confidence estimation has raised even more interest as recent works \cite{nguyen2015,Hein_2019_CVPR} reveal that modern neural networks tend to be overconfident and provide unreliable uncertainty estimates. 

Bayesian neural networks \cite{bnn1996} offer a principled approach for confidence estimation by adopting a Bayesian formalism which models the weight posterior distribution. As the true posterior cannot be evaluated analytically in complex models, various approximations have been developed, such as variational inference \cite{Blundell2015, NIPS2019_9472, Gal2016} or expectation propagation \cite{hernandezlobatoc15}. In particular, MC Dropout \cite{Gal2016} has raised a lot of interest due to the simplicity of its implementation. Predictions are obtained by averaging softmax vectors from multiple feed-forward passes through the network with dropout layers. When applied to regression, the predictive distribution uncertainty can be summarized by computing statistics, \eg, variance. However, when using MC Dropout for uncertainty estimation in classification tasks, the predictive distribution is averaged to a point-wise softmax estimate before computing standard uncertainty criteria such as entropy. It is worth mentioning that these entropy-based criteria measure the softmax output dispersion, where the uniform distribution has maximum entropy. It is not clear how well these dispersion measures are adapted to distinguishing failures from correct predictions as we will see in \cref{chap3:subsec:limits}. \cite{kendall2017} presented a framework to decompose the uncertainty into aleatoric and epistemic terms. But it requires multiple forward passes for inference. Lakshminarayanan \textit{et al}. \cite{deepensembles2017} propose an alternative to Bayesian neural networks by leveraging an ensemble of neural networks to produce well-calibrated uncertainty estimates. However, it requires training multiple classifiers, which has a considerable computing cost in training and inference time.

In failure prediction, a widely used baseline is to take the value of the predicted class' probability given by the softmax layer output, namely the \emph{maximum class probability} (MCP), suggested by \cite{oldconfmcp1995} and revised by \cite{hendrycks17baseline}. As stated before, MCP presents several limits regarding both failure prediction and out-of-distribution detection, as it outputs unduly high confidence values. More recently, \cite{NIPS2018_7798} proposed a new confidence measure, `Trust Score', which measures the agreement between the classifier and a modified nearest-neighbor classifier on the test examples. More precisely, the confidence criterion used in Trust Score is the ratio between the distance from the sample to the nearest class different from the predicted class and the distance to the predicted class. One clear drawback of this approach is its lack of scalability, since computing nearest neighbors in large datasets is extremely costly in both computations and memory. Another more fundamental limitation related to the Trust Score itself is that local distance computation becomes less meaningful in high dimensional spaces~\cite{distance-curse}, which is likely to negatively affect the performances of this method as shown in experiments. Finally, DeVries \& Taylor \cite{devries2018learning} share with us the same purpose of learning confidence in neural networks. Their work differs by focusing on out-of-distribution detection and learning jointly a distribution confidence score and classification probabilities. In addition, they use predicted confidence scores to interpolate output probabilities and target whereas we specifically craft our proposed criterion for failure prediction.

\begin{wrapfigure}{r}{0.5\textwidth}
    \captionsetup{justification=centering}
    \centering
    \vspace{-0.5cm}
    \includegraphics[width=\linewidth]{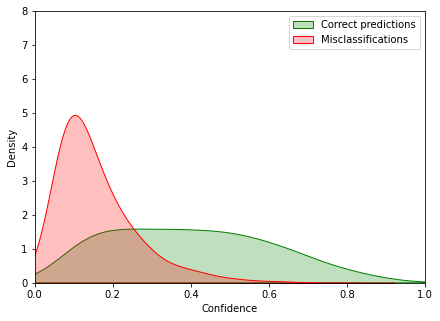}
    \caption[Distribution of MCP after temperature scaling for a VGG-16 on CIFAR-100]{Distribution of MCP after temperature scaling for a VGG-16 on CIFAR-100.}
    \vspace{-0.8cm}
    \label{chap3:fig:plot_tscaling}
\end{wrapfigure}
In tasks closely related to failure prediction, Guo \textit{et al}. \cite{guo2017}, for confidence calibration, and Liang \textit{et al}. \cite{odin2018}, for out-of-distribution detection, proposed to use temperature scaling to mitigate confidence values. However, this does not affect the ranking of confidence scores and therefore the separability between errors and correct predictions. For instance, we plot in \cref{chap3:fig:plot_tscaling} the distribution of MCP confidence estimates after temperature scaling ith a VGG-16 model on CIFAR-10. The temperature parameter $T$ has been found using validation data, such as described in \cite{guo2017}. Even though overconfidence is reduced, the separability between errors and correct predictions still remains problematic.

\section{Application to failure prediction}
\label{chap3:sec:failure}

We evaluate our approach to predict failures in both classification and segmentation settings. First, we run comparative experiments against strong confidence estimation and Bayesian uncertainty estimation methods on various datasets. These results are then completed by a thorough analysis of the influence of the confidence criterion, the training loss and the learning scheme in our approach. Finally, we provide a few visualizations to get additional insight into the behavior of our approach. Our code is available at \url{https://github.com/valeoai/ConfidNet}.

\subsection{Experiment setup}
\label{chap3:subsec:exp_setup}

\paragraph{Datasets.} We run experiments on image datasets of varying scale and complexity: MNIST \cite{mnist} and SVHN \cite{svhn-dataset} datasets provide relatively simple and small ($28\times 28$) images of digits (10 classes). They are split into 60,000 training samples and 10,000 testing samples. 
CIFAR-10 and CIFAR-100 \cite{Krizhevsky09} bring more complexity to classify low resolution images. In each dataset, we further keep 10\% of training samples as a validation dataset.
We also report experiments for semantic segmentation on CamVid \cite{camvid}, using ConfidNet's training and architecture introduced in \cref{chap3:subsec:architecture}, with dense layers replaced by $1\times 1$ convolutions with an adequate number of channels. CamVid is a standard road scene dataset. Images are resized to $360\times 480$ pixels and are segmented according to 11 classes such as `road', `building', `car' or `pedestrian'. 

\paragraph{Classification network.} For each dataset, we use standard neural network architectures as classifiers. Network architectures range from small convolutional networks for MNIST \cite{mnist} and SVHN \cite{svhn-dataset} to VGG-16 architectures \cite{Simonyan15} for CIFAR datasets \cite{Krizhevsky09}. We also added a multi-layer perceptron (MLP) with 1 hidden layer of size 100 for MNIST dataset in order to investigate performances on small models. Finally, we implemented SegNet following \cite{kendall2015bayesian}. All models are trained in a standard way with a cross-entropy loss and an SGD optimizer with a learning rate of $10^{-3}$, a momentum of 0.9 and a weight decay of $10^{-4}$. The number of training epochs depends on the dataset considered, varying from 100 epochs on MNIST to 250 epochs on CIFAR-100. As we also want to compute Monte Carlo samples following \cite{Gal2016}, we include dropout layers. Best models are selected on validation-set accuracy.

\paragraph{Baselines.} To demonstrate the effectiveness of our method, we implemented competitive confidence and uncertainty estimation approaches including Maximum Class Probability (MCP) as a baseline \cite{hendrycks17baseline}, Trust Score \cite{trustscore2018}, and Monte-Carlo Dropout (MC Dropout) \cite{Gal2016}. For Trust Score, we used the code provided by the authors\footnote{\url{https://github.com/google/TrustScore}}.
With MC Dropout, we use the same model as baseline (which already includes dropout layers) and we sample 100 times from the classification model at test time keeping dropout layers activated. We then compute the average softmax probability over all samples to conduct Monte Carlo integration. Model uncertainty is estimated by calculating the entropy of the averaged probability vector across the class dimension.

\paragraph{Evaluation metrics}
We measure the quality of failure prediction following standard metrics used in the literature~\cite{hendrycks17baseline}. We enumerate these metrics in the following and refer the reader to \cref{chap2:sec:evaluation} for a more detail description:
\begin{itemize}
    \item \textbf{FPR at 95\% TPR} measures the False Positive Rate ($\mathrm{FPR}$) when the True Positive Rate ($\mathrm{TPR}$) is equal to 95\%. True Positive Rate can be computed by $\mathrm{TPR} = \mathrm{TP} / (\mathrm{TP}+\mathrm{FN})$, where $\mathrm{TP}$ and $\mathrm{FN}$ denote numbers of true positives and false negatives respectively. The False Positive Rate can be computed by $\mathrm{FPR} = \mathrm{FP} / (\mathrm{FP}+\mathrm{TN})$, where $\mathrm{FP}$ and $\mathrm{TN}$ denote the number of false positives and true negatives respectively. This metric can be interpreted as the probability that an error is misclassified as a correct prediction when the True Positive Rate ($\mathrm{TPR}$) is as high as 95\%. \\
    \item \textbf{AUROC} measures the Area Under the Receiver Operating Characteristic curve (AUROC). The ROC curve is a graph showing True Positive Rate versus False Positive Rate. This metric is a threshold-independent performance evaluation, such as AUPR. It can be interpreted as the probability that a positive example has a greater prediction score than a negative example. \\
    \item \textbf{AUPR} measures the Area Under the Precision-Recall (PR) curve. The PR curve is a graph showing precision $= \mathrm{TP}/(\mathrm{TP} + \mathrm{FP})$ versus recall $= \mathrm{TP}/(\mathrm{TP} + \mathrm{FN})$. As we specifically want to detect failures, we use AUPR-Error (shortened here AUPR) as the primary metrics to assess performances.
\end{itemize}
As an additional, indirect way to assess the quality of the predicted classifier's confidence, we also consider the selective classification problem that was discussed in \cref{chap2:subsec:selective}. In this setup, the predictions by the classifier $F$ that get a predicted confidence below a defined threshold are rejected. Given a coverage rate (the fraction of examples that are not rejected), the performance of the classifier should improve. The impact of this selection is measured in average with:
\begin{itemize}
    \item \textbf{Area under the risk-coverage curve (AURC)}. In classification with a reject option, the risk-coverage curve is the graph of the empirical risk of the classifier given a loss (usually $0/1$ loss) as a function of the empirical coverage, which is the proportion of the non-rejected samples. This metric is threshold-independent, as AUROC and AUPR.
    \item \textbf{Excess-AURC (E-AURC)}. This is a normalized AURC metric defined in \cite{geifman2018biasreduced}. It takes into account the optimal ranking given the error rate of the classifier.
\end{itemize}

\paragraph{ConfidNet.} For each of the considered classification models, ConfidNet is built upon the penultimate layer, which is a convolutional layer with non-linear activation and optionally followed by a normalization layer. We train ConfidNet for 100 epochs with the Adam optimizer with a learning rate $1\times10^{-4}$, dropout, weight decay $10^{-4}$ and the same data augmentation as in the classifier's training. The relevance of selecting the same training dataset used for classifier learning or a hold-out dataset is specifically discussed in \cref{chap3:subsec:variants}. We select the best model based on the AUPR on the validation dataset. In the second training step involving encoder fine-tuning, the training is completed on very few epochs based on previous best model, using Adam optimizer with learning rate $1\times10^{-6}$ or $1\times10^{-7}$ and no dropout to mitigate stochastic effects that may lead the new encoder to deviate too much from the original one used for classification. Once again, the best model is selected on validation-set AUPR.

\subsection{Comparative results}
\label{chap3:subsec:results}

\begin{table}[t]
  \caption[Comparative results of confidence estimation methods for failure prediction and selective classification]{\textbf{Comparison of confidence estimation methods for failure prediction and selective classification}. For each dataset, all methods share the same classification network. For MC Dropout, test accuracy is averaged through random sampling. The first three metrics are percentages and concern failure prediction. The two last ones (the lower, the better) concern selective classification and their values have been multiplied by $10^3$ for clarity. Scores are averaged over 5 runs, best results are in \textbf{bold}, second best ones are \underline{underlined}.}
  \label{chap3:tab:comparative-results}
  \centering
  \begin{adjustbox}{max width=0.99\linewidth}
  \begin{tabular}{cl|rrr|rr}
    \toprule
    Dataset & Model & FPR\,@\,95\%\,TPR\,$\downarrow$& AUPR \,$\uparrow$ & AUROC \,$\uparrow$& AURC \,$\downarrow$& E-AURC\,$\downarrow$ \\
    \midrule
    \multirow{4}{*}{\shortstack[c]{\ubold{MNIST} \\ MLP}} & MCP~\cite{hendrycks17baseline} & 14.88 {\scriptsize $\pm 1.42$} & 47.25 {\scriptsize $\pm 1.67$} & 97.28 {\scriptsize $\pm 0.20$} & 0.83 {\scriptsize $\pm 0.07$} & 0.61 {\scriptsize $\pm 0.06$} \\
    & MC Dropout~\cite{Gal2016} & 15.17 {\scriptsize $\pm 1.08$} & 40.98 {\scriptsize $\pm 1.24$} & 97.10 {\scriptsize $\pm 0.18$} & 0.85 {\scriptsize $\pm 0.07$} & 0.63 {\scriptsize $\pm 0.06$} \\
    & TrustScore~\cite{trustscore2018} & \underline{14.80} {\scriptsize $\pm 2.03$} & \underline{52.13} {\scriptsize $\pm 1.79$} & \underline{97.36} {\scriptsize $\pm 0.10$} & \underline{0.82} {\scriptsize $\pm 0.04$} & \underline{0.59} {\scriptsize $\pm 0.03$} \\
    & ConfidNet & \ubold{11.61} {\scriptsize $\pm 1.96$} & \ubold{59.72} {\scriptsize $\pm 1.90$} & \ubold{97.89} {\scriptsize $\pm 0.14$} & \ubold{0.70} {\scriptsize $\pm 0.05$} & \ubold{0.47} {\scriptsize $\pm 0.04$} \\
    \midrule
    \multirow{4}{*}{\shortstack[c]{\ubold{MNIST} \\ SmallConvNet}} & MCP~\cite{hendrycks17baseline} & 5.53 {\scriptsize $\pm 1.25$} & 36.08 {\scriptsize $\pm 3.60$} & 98.49 {\scriptsize $\pm 0.07$} & \underline{0.15} {\scriptsize $\pm 0.01$} & \underline{0.12} {\scriptsize $\pm 0.01$} \\
    & MC Dropout~\cite{Gal2016} & \ubold{5.03} {\scriptsize $\pm 0.72$} & \underline{42.12} {\scriptsize $\pm 5.52$} & \underline{98.53} {\scriptsize $\pm 0.12$} & 0.16 {\scriptsize $\pm 0.01$} & \underline{0.12} {\scriptsize $\pm 0.01$} \\
    & TrustScore~\cite{trustscore2018} & 9.60 {\scriptsize $\pm 2.69$} & 33.47 {\scriptsize $\pm 3.82$} & 98.20 {\scriptsize $\pm 0.23$} & 0.18 {\scriptsize $\pm 0.03$} & 0.15 {\scriptsize $\pm 0.02$} \\
    & ConfidNet & \underline{5.32} {\scriptsize $\pm 1.14$} & \ubold{45.45} {\scriptsize $\pm 3.75$} & \ubold{98.72} {\scriptsize $\pm 0.07$} & \ubold{0.13} {\scriptsize $\pm 0.02$} & \ubold{0.10} {\scriptsize $\pm 0.01$} \\
    \midrule
    \multirow{4}{*}{\shortstack[c]{\ubold{SVHN} \\ SmallConvNet}} & MCP~\cite{hendrycks17baseline} & \underline{32.17} {\scriptsize $\pm 0.91$} & \underline{46.20} {\scriptsize $\pm 0.50$} & \underline{92.93} {\scriptsize $\pm 0.13$} & \underline{5.58} {\scriptsize $\pm 0.14$} & \underline{4.50} {\scriptsize $\pm 0.09$} \\
    & MC Dropout~\cite{Gal2016} & 33.54 {\scriptsize $\pm 1.06$} & 45.15 {\scriptsize $\pm 1.29$} & 92.84 {\scriptsize $\pm 0.08$} & 5.70 {\scriptsize $\pm 0.11$} & 4.61 {\scriptsize $\pm 0.09$} \\
    & TrustScore~\cite{trustscore2018} & 34.01 {\scriptsize $\pm 1.11$} & 44.77 {\scriptsize $\pm 1.30$} & 92.65 {\scriptsize $\pm 0.29$} & 5.72 {\scriptsize $\pm 0.11$} & 4.64 {\scriptsize $\pm 0.12$} \\
    & ConfidNet & \ubold{29.90} {\scriptsize $\pm 0.76$} & \ubold{48.64} {\scriptsize $\pm 1.08$} & \ubold{93.15} {\scriptsize $\pm 0.15$} & \ubold{5.51} {\scriptsize $\pm 0.09$} & \ubold{4.43} {\scriptsize $\pm 0.08$} \\
    \midrule
    \multirow{4}{*}{\shortstack[c]{\ubold{CIFAR-10} \\ VGG16}} & MCP~\cite{hendrycks17baseline} & \underline{49.19} {\scriptsize $\pm 1.42$} & \underline{48.37} {\scriptsize $\pm 0.69$} & \underline{91.18} {\scriptsize $\pm 0.32$} & \underline{12.66} {\scriptsize $\pm 0.61$} & \underline{8.71} {\scriptsize $\pm 0.50$} \\
    & MC Dropout~\cite{Gal2016} & 49.67 {\scriptsize $\pm 2.66$} & 48.08 {\scriptsize $\pm 0.99$} & 90.70 {\scriptsize $\pm 1.96$} & 13.31 {\scriptsize $\pm 2.63$} & 9.46 {\scriptsize $\pm 2.41$} \\
    & TrustScore~\cite{trustscore2018} & 54.37 {\scriptsize $\pm 1.96$} & 41.80 {\scriptsize $\pm 1.97$} & 87.87 {\scriptsize $\pm 0.41$} & 17.97 {\scriptsize $\pm 0.45$} & 14.02 {\scriptsize $\pm 0.34$} \\
    & ConfidNet & \ubold{45.08} {\scriptsize $\pm 1.58$} & \ubold{53.72} {\scriptsize $\pm 0.55$} & \ubold{92.05} {\scriptsize $\pm 0.34$} & \ubold{11.78} {\scriptsize $\pm 0.58$} & \ubold{7.88} {\scriptsize $\pm 0.44$} \\
    \midrule
    \multirow{4}{*}{\shortstack[c]{\ubold{CIFAR-100} \\ VGG16}} & MCP~\cite{hendrycks17baseline} & 66.55 {\scriptsize $\pm 1.56$} & 71.30 {\scriptsize $\pm 0.41$} & 85.85 {\scriptsize $\pm 0.14$} & 113.23 {\scriptsize $\pm 2.98$} & 51.93 {\scriptsize $\pm 1.20$} \\
    & MC Dropout~\cite{Gal2016} & \underline{63.25} {\scriptsize $\pm 0.66$} & \underline{71.88} {\scriptsize $\pm 0.72$} & \underline{86.71} {\scriptsize $\pm 0.30$} & \ubold{101.41} {\scriptsize $\pm 3.45$} & \ubold{46.45} {\scriptsize $\pm 1.91$} \\
    & TrustScore~\cite{trustscore2018} & 71.90 {\scriptsize $\pm 0.93$} & 66.77 {\scriptsize $\pm 0.52$} & 84.41 {\scriptsize $\pm 0.15$} & 119.41 {\scriptsize $\pm 2.94$} & 58.10 {\scriptsize $\pm 1.09$} \\
    & ConfidNet & \ubold{62.70} {\scriptsize $\pm 1.04$} & \ubold{73.55} {\scriptsize $\pm 0.57$} & \ubold{87.17} {\scriptsize $\pm 0.21$} & \underline{108.46} {\scriptsize $\pm 2.62$} & \underline{47.15} {\scriptsize $\pm 0.95$} \\
    \midrule
    \multirow{4}{*}{\shortstack[c]{\ubold{CamVid} \\ SegNet}} & MCP~\cite{hendrycks17baseline} & 63.87  {\scriptsize $\pm 0.76$} & 48.53 {\scriptsize $\pm 0.34$} & 84.42 {\scriptsize $\pm 0.09$} & & \\
    & MCDropout~\cite{Gal2016} & \underline{62.95} {\scriptsize $\pm 0.72$} & \underline{49.35} {\scriptsize $\pm 0.30$} & \underline{84.58} {\scriptsize $\pm 0.08$} & & \\
    & TrustScore~\cite{trustscore2018} & & 20.42 {\scriptsize $\pm 1.02$} & 68.33 {\scriptsize $\pm 1.17$} & & \\
    & ConfidNet & \ubold{61.52} {\scriptsize $\pm 0.67$} & \ubold{50.51} {\scriptsize $\pm 0.26$} & \ubold{85.02} {\scriptsize $\pm 0.08$} & & \\
    \bottomrule
  \end{tabular}
  \end{adjustbox}
\end{table}

Comparative results are summarized in \cref{chap3:tab:comparative-results}. We observe that our approach outperforms baseline methods in every setting, with a significant gap on small models/datasets. This confirms both that TCP is an adequate confidence criterion for failure prediction and that our approach ConfidNet is able to learn it. TrustScore also presents good results on small datasets/models such as MNIST where it improved the baseline. While ConfidNet still performs well on more complex datasets, Trust Score's performance drops, which might be explained by high dimensionality issues with distances as mentioned in \cref{chap3:subsec:exp_setup}. For its application to semantic segmentation where each training pixel is a `neighbor', computational complexity forced us to reduce drastically the number of training neighbors and of test samples. We sampled randomly in each train and test image a small percentage of pixels to compute TrustScore. ConfidNet, in contrast, is as fast as the original segmentation network. We also improve state-of-art performances at the time of publication from MCDropout, whose drawbacks regarding failure prediction have been highlighted previously in \cref{chap3:fig:visu-entropy}.

Risk-coverage curves \cite{elyaniv10a, NIPS2017_7073} depicting the performance of ConfidNet and other baselines for every tested dataset appear in \cref{chap3:fig:rc_curve}. `Coverage' corresponds to the probability mass of the non-rejected region after using a threshold as a selection function \cite{NIPS2017_7073}. For both datasets, ConfidNet presents a better coverage potential for each selective risk that a user can choose beforehand. In addition, we can see that the improvement is more pronounced at high coverage rates - \textit{e.g.} in $[0.8 , 0.95]$ for CIFAR-10 (\cref{chap3:fig:rc_cifar10}) and in $[0.86 , 0.96]$ for SVHN (\cref{chap3:fig:rc_svhn}) - which highlights the capacity of ConfidNet to identify successfully critical failures. 

\begin{figure}[ht]
\centering
\captionsetup[subfigure]{justification=centering}
\begin{minipage}{.5\linewidth}
  \centering
  \includegraphics[width=\linewidth]{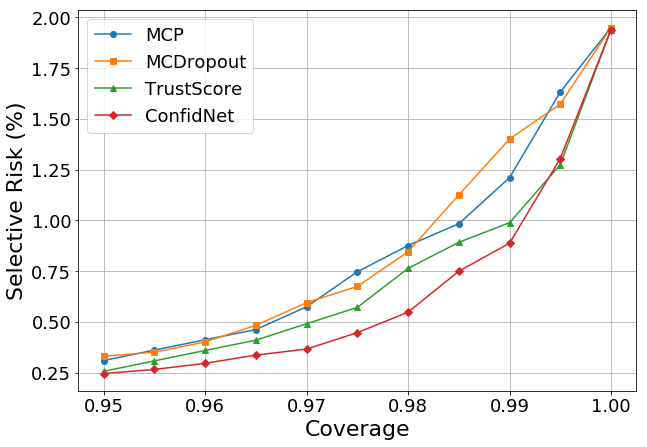}
  \vspace{-0.6cm}
  \subcaption{MLP on MNIST}
  \label{chap3:fig:rc_mlp}
\end{minipage}%
\begin{minipage}{.5\linewidth}
  \centering
  \includegraphics[width=\linewidth]{chapter03/images/RC_curve_mlp.png}
  \vspace{-0.6cm}
  \subcaption{Small ConvNet on MNIST}
  \label{chap3:fig:rc_mnist}
\end{minipage}\hfill
\vspace{0.3cm}
\begin{minipage}{.5\linewidth}
  \centering
  \includegraphics[width=\linewidth]{chapter03/images/RC_curve_mlp.png}
  \vspace{-0.6cm}
  \subcaption{Small ConvNet on SVN}
  \label{chap3:fig:rc_svhn}
\end{minipage}%
\begin{minipage}{.5\linewidth}
  \centering
  \includegraphics[width=\linewidth]{chapter03/images/RC_curve_mlp.png}
  \vspace{-0.6cm}
  \subcaption{VGG-16 on CIFAR-10}
  \label{chap3:fig:rc_cifar10}
\end{minipage}\hfill
\vspace{0.3cm}
\begin{minipage}{.5\linewidth}
  \centering
  \includegraphics[width=\linewidth]{chapter03/images/RC_curve_mlp.png}
  \vspace{-0.6cm}
  \subcaption{VGG-16 on CIFAR-100}
  \label{chap3:fig:rc_cifar100}
\end{minipage}%
\caption[Risk-coverage curves for various uncertainty measures on respective test sets]{\textbf{Comparison of risk-coverage curves for various uncertainty measures on respective test sets}. `Selective risk' ($y$-axis) represents the percentage of errors in the remaining test set for a given coverage percentage.}
\label{chap3:fig:rc_curve}
\end{figure}

\clearpage

\subsection{Effect of learning variants}
\label{chap3:subsec:variants}

\paragraph{Confidence loss function.}
In \cref{chap3:tab:loss-analysis}, we compare training ConfidNet with the MSE loss (\cref{chap3:eq:loss-conf}) to training with a binary-classification cross-entropy loss (BCE), a focal BCE loss \cite{focaloss} and a batch-wise approximate ranking loss. Even though BCE specifically addresses the failure prediction task, it achieves lower performances on CIFAR-10 datasets. Similarly, the focal loss and the ranking one yield results below TCP's performance in every tested benchmark. Similar results on SVHN and CamVid dataset are available in \cref{appxA:sec:ablation_loss}. Our intuition is that TCP regularizes the training by providing finer-grained information about the quality of the classifier's predictions. This is especially important in the difficult learning configuration where only very few error samples are available due to the good performance of the classifier.

\begin{table}[ht]
  \caption[Effect of loss function with ConfidNet on CIFAR-10]{\textbf{Effect of loss function with ConfidNet} on CIFAR-10. Results are percentages (\%).}
  \label{chap3:tab:loss-analysis}
  \centering
  \begin{adjustbox}{max width=\linewidth}
  \begin{tabular}{clrrrr}
    \toprule
    Dataset & Loss & FPR\,@\,95\%\,TPR\,$\downarrow$& AUPR \,$\uparrow$ & AUROC \,$\uparrow$ \\
    \midrule
    \multirow{5}{*}{\shortstack[c]{\ubold{CIFAR-10} \\ VGG-16}} & BCE & 45.20 & 47.95 & 91.94 \\
    & Focal & 45.20 & 47.76 & 91.93 \\
    & Ranking & 46.99 & 44.04 & 91.49 \\
    & \textit{n}$\mathrm{TCP}$ & 45.02 & 48.78 & 92.06 \\
    & $\mathrm{TCP}$ & \ubold{44.94} & \ubold{49.94} & \ubold{92.12} \\
    \bottomrule
  \end{tabular}
  \end{adjustbox}
\end{table}

We also evaluate the impact of regression to the normalized criterion \textit{n}TCP: performance is lower than the one of TCP on small datasets such as CIFAR-10 where few errors are present, but can be higher on larger datasets such as CamVid where each pixel is a sample (see \cref{appxA:tab:loss_variants}). This emphasizes once again the complexity of incorrect/correct classification training.

\paragraph{Hold-out dataset for training ConfidNet.}
Most neural networks used in our experiments tend to overfit. On small datasets such as MNIST and SVHN, convolutional neural networks already achieve nearly perfect accuracy on test set, above 96\%, which leaves very few errors available. For this reason, we also experimented with training ConfidNet on a hold-out dataset. We report results on all datasets in \cref{chap3:tab:validationset} for validation sets with 10\% of samples. We observe a general performance drop when using a validation set for training TCP confidence. The drop is especially pronounced for small datasets (MNIST), where models reach $>$97\% train and val accuracies. Consequently, with a high accuracy and a small validation set, we do not get a larger absolute number of errors using the validation set compared to the train set. One solution would be to increase validation set size but this would damage the model's prediction performance. By contrast, we take care with our approach to base our confidence estimation on models with levels of test predictive performance that are similar to those of baselines. On CIFAR-100, the gap between train accuracy and validation accuracy is substantial (95.56\% vs. 65.96\%), which may explain the slight improvement for confidence estimation using the validation set (+0.17\%). We think that training ConfidNet on the validation set with models reporting low/middle test accuracies could improve the approach.

\begin{table}[ht]
    \caption[Ablation study between training ConfidNet on train set or on validation set]{\textbf{Ablation study between training ConfidNet on train set or on validation set}. Comparison in AUPR (the higher, the better) on all benchmarks. Results are percentages (\%).}
    \centering
    \label{chap3:tab:validationset}
    \begin{adjustbox}{max width=\linewidth}
    \begin{tabular}{lcccccc}
        \toprule
        & \ubold{MNIST} & \ubold{MNIST} & \ubold{SVHN} & \ubold{CIFAR-10} & \ubold{CIFAR-100} & \ubold{CamVid} \\
        & MLP & SmallConvNet & SmallConvNet & VGG-16 & VGG-16 & SegNet\\
        \midrule
        ConfidNet (using train set) & 57.34 & 43.94 & 50.72 & 49.94 & 73.68 & 50.28 \\
        ConfidNet (using val set) & 33.41 & 34.22 & 47.96 & 48.93 & 73.85 & 50.15 \\
        \bottomrule
    \end{tabular}
    \end{adjustbox}
\end{table} 

\paragraph{ConfidNet's encoder fine-tuning.} 
We analyse in \cref{chap3:tab:learning_scheme} the effect of the encoder fine-tuning. Learning only ConfidNet on top of the pre-trained encoder $E$ (that is, $\bomega = \bphi$), our confidence network already achieves significant improvements w.r.t. the baselines. With a subsequent fine-tuning of both modules (that is, $\bomega = (\btheta_{E'},\bphi))$, its performance is further boosted in every setting, by around 1-2\%. Note that using a vanilla fine-tuning without the deactivation of the dropout layers did not bring any improvement.

\begin{table}[t]
\centering
  \caption[Impact of the encoder fine-tuning on the
error-prediction performance of ConfidNet]{\textbf{Impact of the encoder fine-tuning on the
error-prediction performance of ConfidNet}. Comparison in
AUPR (the higher, the better) on all benchmarks. Results are percentages (\%).}
  \label{cloning-analysis}
   \begin{adjustbox}{max width=\linewidth}
  \begin{tabular}{lcccccc}
    \toprule
    & \ubold{MNIST} & \ubold{MNIST} & \ubold{SVHN} & \ubold{CIFAR-10} & \ubold{CIFAR-100} &  \ubold{CamVid} \\
    & MLP & SmallConvNet & SmallConvNet & VGG-16 & VGG-16 & SegNet\\
    \midrule
    Confidence training & 58.42 & 44.54 & 48.49 & 50.18 & 71.30 & 50.12\\
    + Fine-tuning ConvNet & 59.72 & 45.45 & 48.64 & 53.72 & 73.55 & 50.51\\
    \bottomrule
  \end{tabular}
  \end{adjustbox}
  \label{chap3:tab:learning_scheme}
\end{table}

\paragraph{ConfidNet's architecture}
We experiment different architectures for ConfidNet on the SVHN dataset, varying the number of layers. Except for the first and last layers, whose dimensions respectively depend on the size of the input and of the output, each layer presents the same number of units (400). In \cref{chap3:fig:plot_architecture}, we observe that starting from 3 layers, ConfidNet already improves baseline performance.

\begin{figure}[ht]
  \centering
  \includegraphics[width=0.7\linewidth]{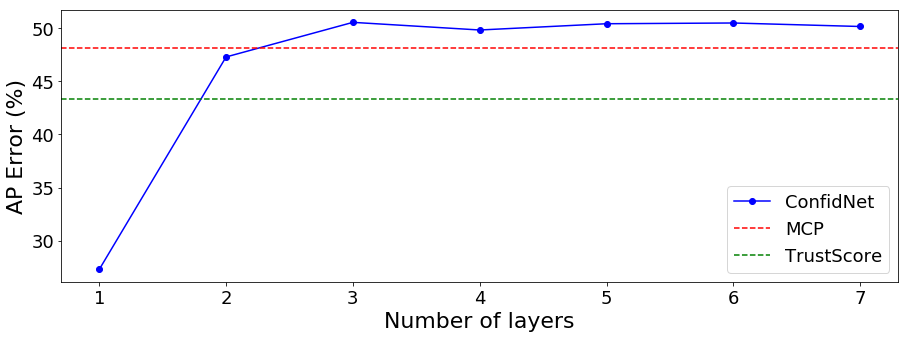}
  \caption[Influence of ConfidNet's depth on its performance]{\textbf{Influence of ConfidNet's depth on its performance}. Performance in AUPR as a function of the number of layers used in ConfidNet on SVHN test set and compared to the performance of MCP and True Score baselines.}
  \label{chap3:fig:plot_architecture}
\end{figure}

\subsection{Comparison with a two-fold ensemble}
\label{chap3:subsec:ensemble}

As ConfidNet with fine-tuning induces an increased capacity of the whole model by using a complete auxiliary network in conjunction with the original one, one might hypothesize that this contributes to its superior performance. To investigate this question, we compare its performance with the MCP metric taken from an ensemble of two neural networks. Comparative results for each dataset are presented in \cref{chap3:tab:deepensembles}. At equal computational cost, ConfidNet outperforms an ensemble of two neural networks in every setting.

\begin{table}[ht]
    \caption[Equal-capacity comparisons between ensemble and ConfidNet]{\textbf{Equal-capacity comparisons}. Comparison between ConfidNet trained on 1 neural network (1NN) and the MCP metric taken from the average prediction of an ensemble of two neural networks (2NNs). Results are percentages (\%).}
    \centering
    \label{chap3:tab:deepensembles}
    \begin{adjustbox}{max width=\linewidth}
    \begin{tabular}{l|cc|cc|cc|cc|cc}
        \toprule
        & \multicolumn{2}{c|}{\ubold{MNIST}} & \multicolumn{2}{c|}{\ubold{MNIST}} & \multicolumn{2}{c|}{\ubold{SVHN}} & \multicolumn{2}{c|}{\ubold{CIFAR-10}} & \multicolumn{2}{c}{\ubold{CIFAR-100}} \\
        & \multicolumn{2}{c|}{MLP} & \multicolumn{2}{c|}{SmallConvNet} & \multicolumn{2}{c|}{SmallConvNet} & \multicolumn{2}{c|}{VGG-16} & \multicolumn{2}{c}{VGG-16} \\
        \midrule
        Method & AUPR & AUROC & AUPR & AUROC & AUPR & AUROC & AUPR & AUROC & AUPR & AUROC \\
        \midrule
        MCP-1NN & 47.30 & 97.22 & 37.87 & 98.54 & 46.17 & 92.92 & 47.88  & 91.45 & 71.39 & 85.76  \\
        MCP-2NNs & 43.70 & 97.25 & 43.16 & 98.70 & 46.35 & 92.98 & 48.81 & 92.26 & 70.52 & 86.38  \\
        ConfidNet-1NN & \ubold{60.17} & \ubold{97.90} & \ubold{47.81} & \ubold{98.75} & \ubold{48.67}  & \ubold{93.17} & \ubold{54.00} & \ubold{92.39} & \ubold{73.30} & \ubold{87.00}  \\
        \bottomrule
    \end{tabular}
    \end{adjustbox}
\end{table}

\subsection{Effect on calibration}
\label{chap3:sec:calibration}

\begin{table}[ht]
  \caption[Comparative calibration results for failure prediction]{\textbf{Comparative calibration results}. Performance in ECE (the lower, the better) when using MCP baseline (`Baseline') or ConfidNet as confidence estimator on the six benchmarks, and when using dedicated temperature scaling (`T. Scaling'). Results are percentages (\%)}
  \label{chap3:tab:calibration}
  \centering
   \begin{adjustbox}{max width=\linewidth}
  \begin{tabular}{lcccccc}
    \toprule
    & \ubold{MNIST} & \ubold{MNIST} & \ubold{SVHN} & \ubold{CIFAR-10} & \ubold{CIFAR-100} & \ubold{CamVid} \\
    & MLP & SmallConvNet & SmallConvNet & VGG-16 & VGG-16 & SegNet\\
    \midrule
    Baseline & 0.37 & \ubold{0.20} & \ubold{0.50} & 4.48 & 22.37 & 9.65 \\
    ConfidNet & 0.66 & 0.30 & 1.11 & 3.45 & 15.61 & 7.57 \\
    Baseline + T. Scaling & \ubold{0.20} & 0.69 & 1.30 & \ubold{2.88} & \ubold{5.16} & \ubold{4.77} \\
    \bottomrule
  \end{tabular}
  \end{adjustbox}
\end{table}

\begin{figure}[t]
\centering
\begin{minipage}[t]{0.33\linewidth}
    \centering
    \includegraphics[width=\linewidth]{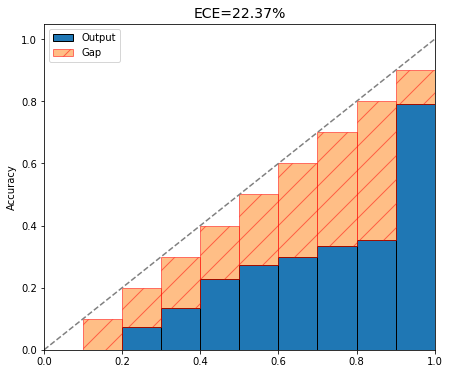}
    \subcaption{MCP}
    \label{chap2:fig:calibration_mcp}
\end{minipage}%
\begin{minipage}[t]{0.33\linewidth}
    \centering
    \includegraphics[width=\linewidth]{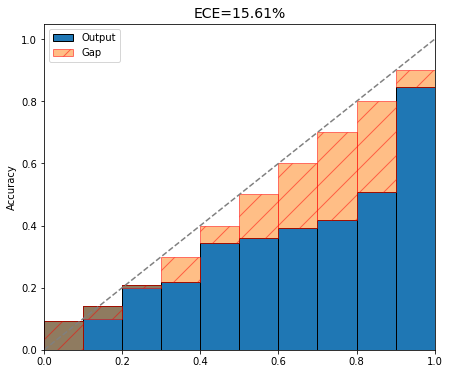}
    \subcaption{ConfidNet}
    \label{chap2:fig:calibration_confidnet}
\end{minipage}%
\begin{minipage}[t]{0.33\linewidth}
    \centering
    \includegraphics[width=\linewidth]{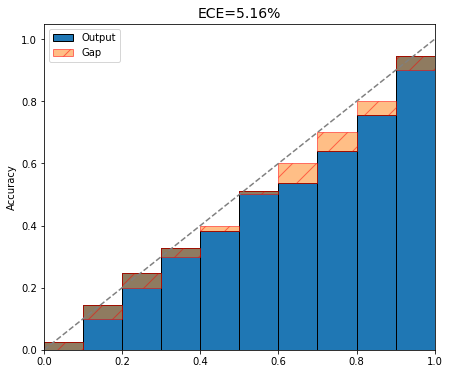}
    \subcaption{Temperature Scaling}
    \label{chap2:fig:calibration_calibrated}
\end{minipage}%
\caption[Reliability diagrams for a VGG-16 model on CIFAR-100]{\textbf{Reliability diagrams for a VGG-16 model on CIFAR-100.} Even though ConfidNet improves calibration over MCP, it remains less effective than dedicated methods for calibration, such as temperature scaling.}
\label{chap3:fig:calibration}
\end{figure}

We observed that ConfidNet tends to lower the confidence of an example that the model wrongly classified while being over-confident (high MCP). As a side experiment, we study whether using ConfidNet as confidence estimation can improve the calibration of deep neural networks. 

In \cref{chap3:tab:calibration}, we report the expected calibration error (ECE) which is an approximate measure of miscalibration between confidence and accuracy \cite{guo2017}. ConfidNet yields equivalent or better ECE results than the MCP baseline, with clear superiority on complex datasets such as CIFAR-10, CIFAR-100 and CamVid. On MNIST and SVHN, the baseline already offers a small ECE. These results confirm our intuition about the capacity of ConfidNet to address over-confident predictions, even though it has not been designed for. Nevertheless, dedicated methods such as temperature scaling used in \cite{guo2017} remain preferred for calibrating deep neural networks. Reliability diagrams in \cref{chap3:fig:calibration} illustrates this result with a VGG-16 architecture trained on CIFAR-100.

\subsection{Visualisations and failure cases}
\label{chap3:subsec:visu}

We provide an illustration on CamVid (\cref{chap3:fig:visu-camvid}) to better understand our approach for failure prediction. Compared to MCP baseline, our approach produces higher confidence scores for correct pixel predictions and lower ones on erroneously predicted pixels, which allows a user to better detect error area in semantic segmentation.

\begin{figure}[ht]
\centering
\captionsetup[subfigure]{justification=centering}
\begin{minipage}[c]{0.33\linewidth}
\centering
    \includegraphics[width=\linewidth]{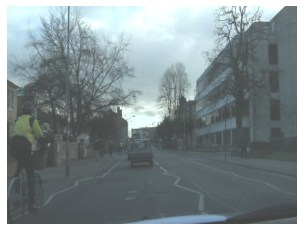}
    \vspace{-0.6cm}
    \subcaption{Input image}
    \label{chap3:fig:visu-camvid_a}
\end{minipage}%
\begin{minipage}[c]{0.33\linewidth}
\centering
    \includegraphics[width=\linewidth]{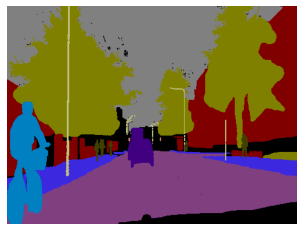}
    \vspace{-0.6cm}
    \subcaption{Ground truth map}
    \label{chap3:fig:visu-camvid_b}
\end{minipage}%
\begin{minipage}[c]{0.33\linewidth}
\centering
    \includegraphics[width=\linewidth]{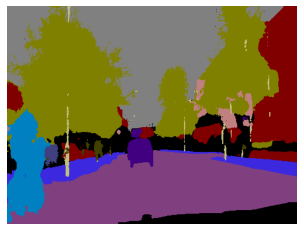}
    \vspace{-0.6cm}
    \subcaption{Prediction map}
    \label{chap3:fig:visu-camvid_c}
\end{minipage}%
\vspace{0.2cm}
\begin{minipage}[c]{0.33\linewidth}
\centering
    \includegraphics[width=\linewidth]{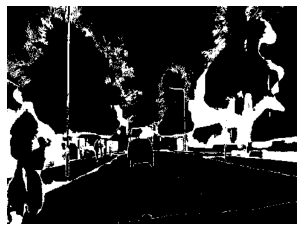}
    \vspace{-0.6cm}
    \subcaption{Model errors}
    \label{chap3:fig:visu-camvid_d}
\end{minipage}%
\begin{minipage}[c]{0.33\linewidth}
\centering
    \includegraphics[width=\linewidth]{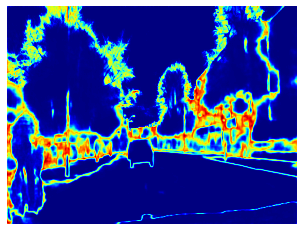}
    \vspace{-0.5cm}
    \subcaption{ConfidNet}
    \label{chap3:fig:visu-camvid_e}
\end{minipage}%
\begin{minipage}[c]{0.33\linewidth}
\centering
    \includegraphics[width=\linewidth]{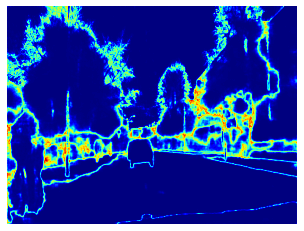}
    \vspace{-0.6cm}
    \subcaption{MCP}
    \label{chap3:fig:visu-camvid_f}
\end{minipage}%
  \caption[Visualization of uncertainty map for ConfidNet and MCP on a CamVid scene]{\textbf{Visualization of inverse confidence (uncertainty) map for ConfidNet (\cref{chap3:fig:visu-camvid_e}) and MCP (\cref{chap3:fig:visu-camvid_f}) on one CamVid scene}. The top row shows the input image (\cref{chap3:fig:visu-camvid_a}) with its ground-truth (\cref{chap3:fig:visu-camvid_b}) and the semantic segmentation mask (\cref{chap3:fig:visu-camvid_c}) predicted by the original classification model. The error map associated with the predicted segmentation is shown in (\cref{chap3:fig:visu-camvid_d}), with erroneous predictions flagged in white. ConfidNet (55.53\% AP-Error) allows a better prediction of these errors than MCP (54.69\% AP-Error).}
  \label{chap3:fig:visu-camvid}
\end{figure}

\begin{figure}[ht]
\centering
\captionsetup[subfigure]{justification=centering}
\begin{minipage}[c]{0.18\linewidth}
\centering
    \includegraphics[width=\linewidth]{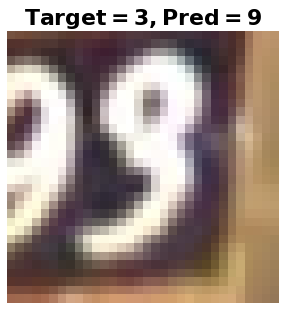}
    \vspace{-0.3cm}
    \subcaption{$\textrm{MCP}=0.93~$, 
    ConfidNet $= 0.74$, 
    $\textrm{TCP}= 0.06$}
    \label{chap3:fig:failure_case_1}
\end{minipage}%
\hspace{0.3cm}
\begin{minipage}[c]{0.18\linewidth}
\centering
    \includegraphics[width=\linewidth]{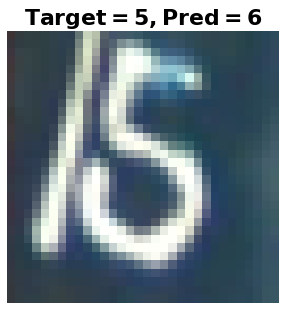}
    \vspace{-0.3cm}
    \subcaption{$\textrm{MCP}=0.94~$, 
    ConfidNet $= 0.78$, 
    $\textrm{TCP}= 0.06$}
    \label{chap3:fig:failure_case_2}
\end{minipage}%
\hspace{0.3cm}
\begin{minipage}[c]{0.18\linewidth}
\centering
    \includegraphics[width=\linewidth]{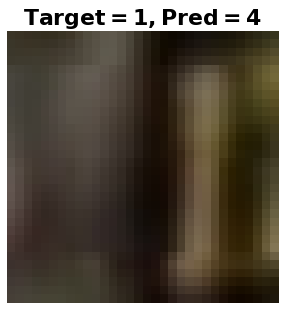}
    \vspace{-0.3cm}
    \subcaption{$\textrm{MCP}=0.89~$, 
    ConfidNet $= 0.90$, 
    $\textrm{TCP}= 0.11$}
    \label{chap3:fig:failure_case_3}
\end{minipage}%
\hspace{0.3cm}
\begin{minipage}[c]{0.18\linewidth}
\centering
    \includegraphics[width=\linewidth]{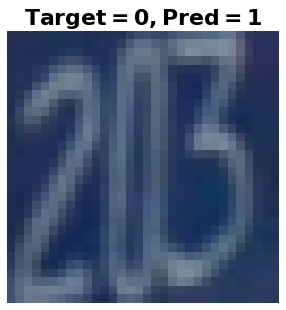}
    \vspace{-0.3cm}
    \subcaption{$\textrm{MCP}=0.82~$, 
    ConfidNet $= 0.91$, 
    $\textrm{TCP}= 0.17$}
    \label{chap3:fig:failure_case_4}
\end{minipage}%
\hspace{0.3cm}
\begin{minipage}[c]{0.18\linewidth}
\centering
    \includegraphics[width=\linewidth]{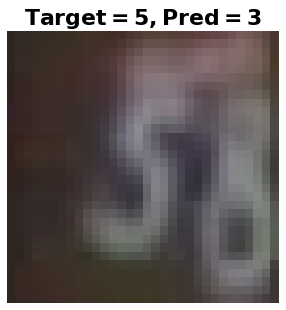}
    \vspace{-0.3cm}
    \subcaption{$\textrm{MCP}=0.90~$, 
    ConfidNet $= 0.94$, 
    $\textrm{TCP}= 0.10$}
    \label{chap3:fig:failure_case_5}
\end{minipage}%
\hspace{0.3cm}
\caption[Failure cases of ConfidNet]{\textbf{Failure cases of ConfidNet}. On these misclassified digits from SVHN's test images, ConfidNet fails to regress the corresponding TCP values, hence to predict the low confidence that should be assigned to the classifier's decisions.}
\label{chap3:fig:failure_cases}
\end{figure}

In \cref{chap3:fig:failure_cases}, we present some failures of ConfidNet on the SVHN dataset. On these selected samples, the classifier outputs an erroneous prediction with high MCP, but Confidnet fails to predict the low TCP values, hence to identify these misclassifications. We can observe that these samples are hard to classify, due to low image quality and confusing shapes. For instance, the classifier has assigned \cref{chap3:fig:failure_case_1} to a 9 while the correct label was 3. In this example, even though ConfidNet output remains high (0.74), it manages at least to be significantly lower than the original MCP value (0.93).

\clearpage

\section{Conclusion}
\label{chap3:sec:conclusion}

In this chapter, we defined a new confidence criterion, TCP, which enjoys simple guarantees and empirical evidence of improving the confidence estimation for classifiers with a reject option. We proposed a specific method to learn this criterion with an auxiliary neural network built upon the encoder of the model that is monitored. Applied to failure prediction, this learning scheme consists in training the auxiliary network and then enabling the fine-tuning of its encoder (the one of the monitored classifier remains frozen). In each image classification experiment, we were able to improve the capacity of the model to distinguish correct from erroneous samples and to achieve better selective classification. Besides failure prediction, other applications can benefit from this improved confidence estimation. In the next chapter, we propose a new application of our learned confidence approach related to the task of unsupervised domain adaptation for semantic segmentation using self-training.

%% file: chapter04/chapter04.tex
\begin{center}
  \textsc{Chapter Abstract}
\end{center}
\begin{quote}
\noindent \textit{Semantic segmentation is a key component for scene understanding with application in self-driving cars and robotics. But collecting and manually annotating urban street scenes with dense pixel-level labels is extremely costly due to the large amount of human effort required. On the other hand, recent advances in computer graphics make it possible to train models on photo-realistic synthetic images with computer-generated annotations. Unsupervised domain adaptation (UDA) aims at learning only from source supervision a well-performing model on target-domain samples.}

\textit{Self-training has recently proven a potent strategy to improve the effectiveness of UDA in semantic segmentation. This line of work mostly relies on the generation of pseudo-labels over the unannotated target domain to incorporate target-domain images and learn a better segmentation adaptation model. A crucial issue in this context is to base the pseudo-label selection on reliable confidence measures.}

\textit{In this chapter, we propose to adapt our learning confidence approach with an auxiliary model to estimate the confidence of the segmentation network in its predictions and to use these confidence estimates as a criterion for pseudo-label extraction. We further enforce confidence distribution alignment between source and target domains using adversarial training, and we equip the architecture of the confidence network with multi-scale prediction suitable for semantic segmentation. We show that this strategy produces more accurate pseudo-labels and outperforms strong state-of-the-art baselines at the time of publication on three challenging UDA segmentation benchmarks.}

\textit{The work described in this chapter is based on the following publication:}
\begin{itemize}
    \item \textit{Charles Corbière, Nicolas Thome, Antoine Saporta, Tuan-Hung Vu, Matthieu Cord, Patrick Pérez. ``Confidence Estimation via Auxiliary Models". In IEEE Transactions on Pattern Analysis and Machine Intelligence, 2021.} 
\end{itemize}
\end{quote}

\clearpage
\minitoc[tight]

\section{Context}
\label{chap4:sec:context}

Perception systems in autonomous cars require in-depth understanding of scenes in which they operate. For this reason, semantic segmentation modules are often incorporated to obtain class-label predictions for every scene pixel. While recent advances in deep convolutional networks have significantly improved segmentation performance, their efficacy depends on large quantities of accurately labeled training data. But the labeling process usually requires experts’ efforts and the annotation cost limits the operational domains of such systems. On the other hand, a lot of driving scenes data are synthesized by game engines such as GTA5 \cite{richter-eecv2016}. Consequently, recent works try to leverage this cheap alternative supervision by training models on these image sources and predicting on real images. But direct transfer is not effective as we observe a drop in performance when evaluating on real images, due to a `domain gap`.

Unsupervised Domain Adaptation (UDA) is the field of research that aims at reducing this domain gap between source and target domains. In the UDA context, annotated source samples along with some unlabeled target images are available at train time. Most works in this line of research aim at minimizing the distribution discrepancy between the source domain and the target domain, at the feature \cite{hoffman-arxiv2016} or prediction level~\cite{Tsai_adaptseg_2018,vu2018advent}, potentially combined with translation methods transforming source images to match the target domain `style'~\cite{Hoffman_cycada2017}. Recently, self-training~\cite{Li_2019_CVPR,Zou_2019_ICCV,zou2018unsupervised} proved its ability to boost adaptation performance significantly. The rationale behind these approaches is to label automatically the most confident target pixels according to current network prediction and to retrain the network accordingly. While this idea is appealing, the presence of noisy or incorrect pseudo-labels could be detrimental to the training of the neural network. As an example, using a ratio of 70\% of pseudo-labels in~\cite{Li_2019_CVPR} leads to a performance of around $48\%$ mIoU, which is better than $34\%$ with direct transfer (only trained on source domain), but still largely below $63\%$ obtained with the same amount of ground-truth labels. Therefore, defining good measures of confidence to select reliable predictions is of crucial importance towards the development of error-free self-training.

To improve self-training efficiency, we propose to adapt our learning confidence approach developed in the previous chapter for the particular context of unsupervised domain adaptation for semantic segmentation. Using an auxiliary model, we select a pool of pixels with high confidence scores to perform the pseudo-labeling (\cref{chap4:subsec:selecting}). We propose \emph{ConDA}, a new deep framework for UDA semantic segmentation with self-training. ConDA leverages the general idea of ConfidNet, but includes the following adaptations specifically designed for UDA:
\begin{itemize}
    \item an adversarial training scheme to prevent drifts in confidence distribution between source and target domains (\cref{chap4:subsec:advloss});
    \item an `atrous' pyramidal pooling architecture for the confidence network to perform multi-scale confidence estimation (\cref{chap4:subsec:aspp}).
\end{itemize}
In \cref{chap4:sec:experiments}, we empirically demonstrate that ConDA brings systematic improvements in performance over self-training based on the standard Maximum Class Probability (MCP), with experiments on various challenging UDA benchmarks with synthetic source datasets and real target datasets and using multiple UDA semantic segmentation methods~\cite{Li_2019_CVPR,vu2018advent,vu-iccv19}, some of them including multiple modalities (\eg depth~\cite{vu-iccv19}).

\section{Unsupervised domain adaptation}
\label{chap4:sec:uda}

Formally, let us consider the annotated source-domain training set $\cD_{\so} = \{ (\xson, \y_{\so,n}) \}_{n=1}^{N_{\so}}$, where $\xson$ is a color image of size $(H,W)$ and $\y_{\so,n} \in \cY^{H\times W}$ its associated ground-truth segmentation map. A segmentation network $F$ with parameters $\btheta$ takes as input an image $\x$ and returns a predicted \textit{soft}-segmentation map $F(\x;\btheta) = \Pmap  \in [0,1]^{H\times W\times K}$, where  $\Pmap[h,w,:] = P(Y[h,w] \,\vert\, \x;\btheta)\in\Delta^{K-1}$, with $\Delta^{K-1}$ the probability (K-1)-simplex. 
The final prediction of the network is the segmentation map $f(\x)$ defined pixel-wise as $f(\x)[h,w] = \arg\!\max_{k\in\cY} \Pmap[h,w,k]$.
This network is learned over the source domain samples $(\bm{x}_s,\bm{y}_s)$ using the cross-entropy segmentation loss:
\begin{equation}
    \mathcal{L}_\text{seg} (\bm{x}_s, \bm{y}_s) = - \sum\limits_{h=1}^H\sum\limits_{w=1}^W\log \bm{P}_{\bm{x}_s}^{(h,w,k^*)},
    \label{eq:l_seg}
\end{equation}
which is minimized over the parameters $\theta_F$ of the network and where $k^*$ is the ground-truth segmentation class for pixel $(h,w)$\footnote{We omit the location dependence $(h,w)$ on $k^*$ for conciseness.}.

In UDA, the main challenge is to use the unlabeled target set $\cD_{\tg} = \{\bm{x}_{\tg,n}\}_{n=1}^{N_{\text{t}}}$ available during training to learn domain-invariant features on which the segmentation model would behave similarly in both domains. Common strategies to perform this task are to minimize the maximum mean discrepancy (MMD)~\cite{long2015learning}, to align the second-order statistics of the distributions (CORAL)~\cite{sun2016deep} or to adopt an adversarial training approach to produce indistinguishable source-target distributions in feature space~\cite{hoffman-arxiv2016} or output space~\cite{Tsai_adaptseg_2018}. For the semantic segmentation task, most recent progresses have been found around the latter. To cite a few methods: CyCADA~\cite{Hoffman_cycada2017} first stylizes the source images as target-domain images before aligning source and target in the feature space; AdaptSegNet~\cite{Tsai_adaptseg_2018} constructs a multi-level adversarial network to perform domain adaptation at different feature levels; AdvEnt~\cite{vu2018advent} aligns the entropy of the pixel-wise predictions with an adversarial loss; BDL~\cite{Li_2019_CVPR} learns alternatively an image translation model and  a segmentation model that promote each other; DISE~\cite{chang2019all} disentangles images into domain-invariant structure and domain-specific texture representations, enabling image translation across domains and label transfer to improve segmentation performance.

In the following, we denote $\cL_F$ as the objective function of the classifier $F$, regardless of the method used for UDA. For instance, with adversarial training methods, we would write $\cL_F = \cL_\text{seg} + \cL_\text{adv}$ where $\cL_\text{adv}$ is the adversarial term in the objective function.

\paragraph{Self-training} Within semi-supervised learning literature \cite{lee-icml2013,grandvalet-nips2005}, self-training with pseudo-labeling showed to be a simple but effective strategy that relies on picking up the current predictions on the unlabeled data and using them as if they were true labels for further training. It is shown in \cite{lee-icml2013} that the effect of pseudo-labeling is equivalent to entropy regularization~\cite{grandvalet-nips2005}. In a UDA setting, the idea is to collect pseudo-labels on the unlabeled target-domain samples in order to have an additional supervision loss in the target domain. To select only reliable pseudo-labels, such that the performance of the adapted semantic segmentation network effectively improves, BDL~\cite{Li_2019_CVPR} resorts to standard selection with MCP. ESL~\cite{saporta2020esl} uses instead the entropy of the prediction as confidence criterion for its pseudo-label selection. CBST~\cite{zou2018unsupervised} proposes an iterative self-training procedure where the pseudo-labels are generated based on a loss minimization. In~\cite{zou2018unsupervised}, the authors also propose a way to balance the classes in their pseudo-labels to avoid the dominance of large classes as well as a way to introduce spatial priors. More recently, the CRST framework~\cite{Zou_2019_ICCV} proposes multiple types of confidence regularization to limit the propagation of errors caused by noisy pseudo-labels.

\begin{figure}[t]
\begin{center}
\includegraphics[width=0.9\linewidth]{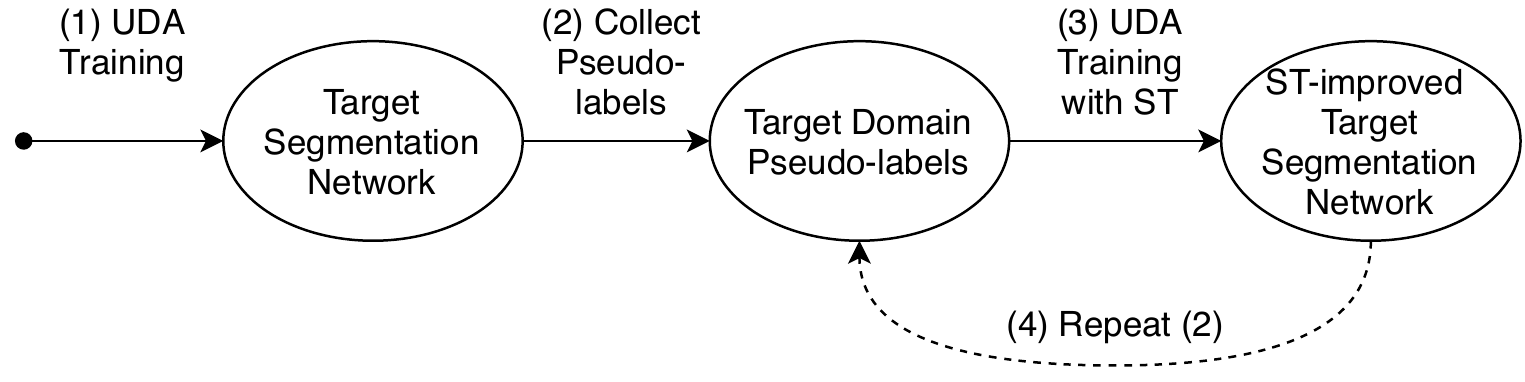}
\end{center}
   \caption[Layout of self-training for unsupervised domain adaptation]{\textbf{Self-training (ST) for UDA}. A segmentation model is first learned with UDA and used to collect pseudo-labels on target domain images. These automatically annotated data are used to subsequently retrain the model, 
   an operation that can be iterated.}
\label{chap4:fig:selftraining}
\end{figure}

By attaching pseudo-labels $\hat{\bm{y}}_t$ to the target-domain images $\bm{x}_t$, the objective function of $F$ with self-training can be written:
\begin{equation}
    \mathcal{L}_\text{F}^* = \mathcal{L}_\text{F} + \frac{\lambda_\text{ST}}{|\cD_t|}\sum\limits_{\bm{x}_t\in\cD_t}\mathcal{L}_\text{seg}(\bm{x}_t,\bm{\hat{y}}_t),
\end{equation}
with a weight $\lambda_\text{ST}$ to balance the self-training term.

Self-training in UDA leverages the domain alignment that has already been achieved by the UDA strategy, assuming that the predictions of the current segmentation network $F$ on target domain are relatively accurate. A high-level view of self-training for semantic segmentation with UDA is described in \cref{chap4:fig:selftraining}:
\begin{enumerate}
    \item Train a segmentation network for the target domain using a chosen UDA technique;\label{step1}
    \item Collect pseudo-labels among the predictions that this network makes on the target-domain training images;\label{step2}
    \item Train a new semantic-segmentation network from scratch using the chosen UDA technique in combination with supervised training on target-domain data with pseudo-labels;
    \item Possibly, repeat from step~\ref{step2} by collecting better pseudo-labels after each iteration.
\end{enumerate}

While the general idea of self-training is simple and intuitive, collecting good pseudo-labels is quite tricky. If too many of them correspond to erroneous predictions of the current segmentation network, the performance of the whole UDA can deteriorate. Thus, a measure of confidence should be used in order to only gather reliable predictions as pseudo-labels and to reject the others.

\begin{figure}
\begin{center}
\includegraphics[width=\linewidth]{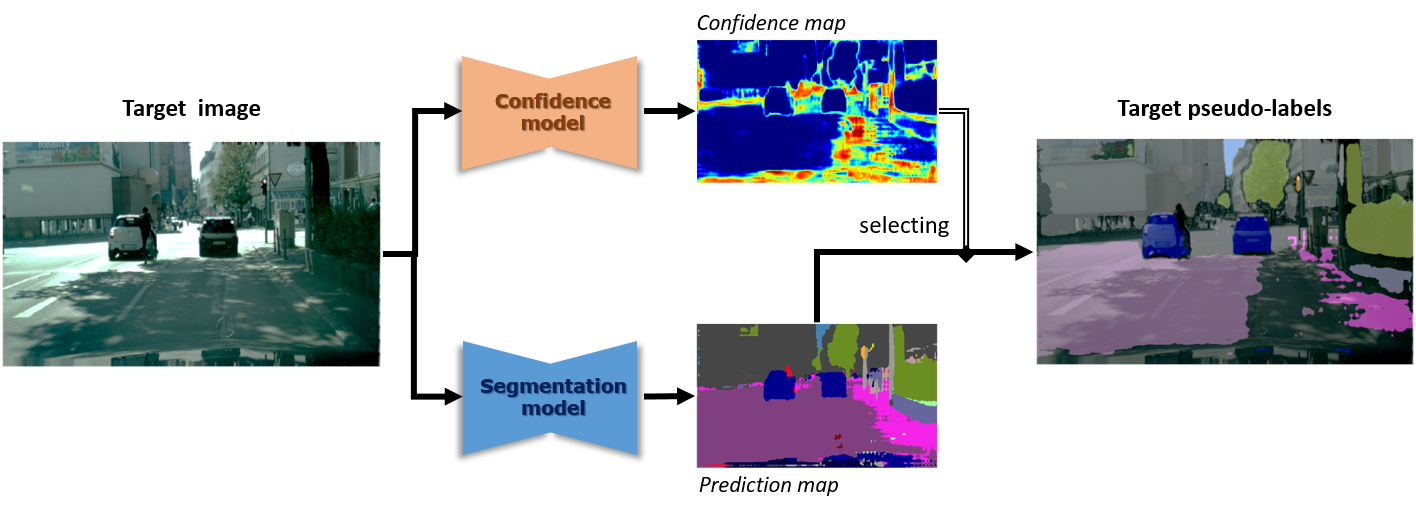}
\end{center}
   \caption[Proposed self-training with learned confidence]{\textbf{Proposed self-training with learned confidence}. Instead of relying only on the segmentation model to generate pseudo-label maps for target images, we propose to use a confidence model specifically trained to this end. This model outputs a reliable confidence map which helps to improve the quality of the final pseudo-label map.}
\label{chap4:fig:selecting}
\end{figure}

\section{ConDA: Confidence learning in domain adaptation}
\label{chap4:sec:conda}

\begin{figure}[t]
\centering
\includegraphics[width=\linewidth]{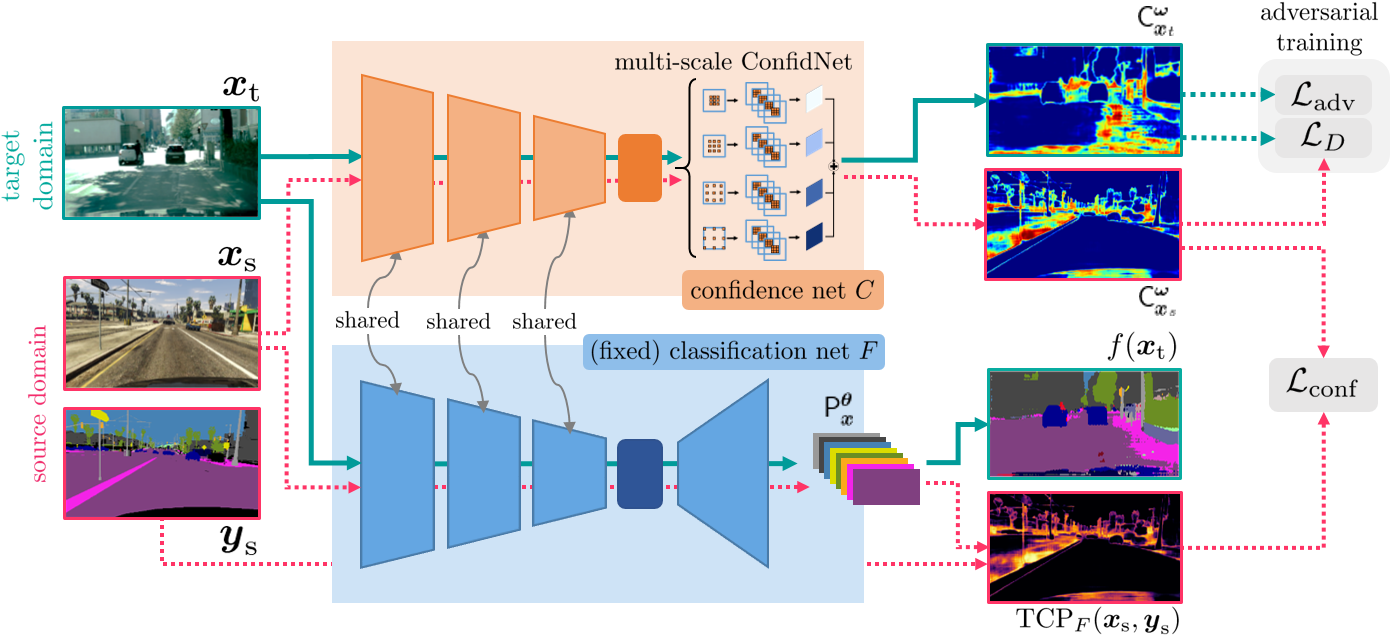}
   \caption[Overview of proposed confidence learning for domain adaptation (ConDA) in semantic segmentation]{\textbf{Overview of proposed confidence learning for domain adaptation (ConDA) in semantic segmentation}. Given images in source and target domains, we pass them to the encoder part of the segmentation network $F$ to obtain their feature maps. This network $F$ is fixed during this phase and its weights are not updated. The confidence maps are obtained by feeding these feature maps to the trainable head of the confidence network $C$, which includes a multi-scale ConfidNet module. For source-domain images, a regression loss $\cL_{\text{conf}}$ (\cref{chap4:eq:loss-conf-conda}) is computed to minimize the distance between $\Cmaps$ and the fixed true-class-probability map $\text{TCP}_F(\xso, \y_{\so})$. An adversarial training scheme -- based on discriminator's loss $\cL_D(\bpsi)$ (\cref{chap4:eq:l_Dconf}) and adversarial part $\cL_{\text{adv}}(\bomega)$ of confidence net's loss (\cref{chap4:eq:l_C}) --, is also added to enforce the consistency between the $\Cmaps$'s and $\Cmapt$'s. Dashed arrows stand for paths that are used only at train time.
   }
\label{chap4:fig:confidence_training}
\end{figure}

Leveraging automatic pseudo-labeling of target-domain training examples is in particular a simple, yet powerful way to further improve UDA performance with self-training. One key ingredient of such an approach being the selection of the most promising pseudo-labels, the proposed auxiliary confidence-prediction model lends itself particularly well to this task.  In this section, we detail how the proposed approach to confidence prediction can be adapted to semantic segmentation, with application to domain adaptation through self-training.

\subsection{Selecting pseudo-labels with a confidence model}
\label{chap4:subsec:selecting}

Following the self-training framework previously described, a confidence network $C$ is learned at step~(\ref{step2}) to predict the confidence of the UDA-trained semantic segmentation network $F$ and used to select only trustworthy pseudo-labels on target-domain images as illustrated in \cref{chap4:fig:selecting}.

To this end, the framework proposed in \cref{chap3:sec:confidnet} in an image classification setup, and applied to predicting erroneous image classification, needs here to be adapted to the structured output of semantic segmentation, which can be seen as a pixel-wise classification problem. Given a target-domain image $\xtg$, we want to predict both its soft semantic map $F(\xtg;\btheta)$ and, using an auxiliary model with trainable parameters $\bomega$, its confidence map:
\begin{equation}
    C(\xtg;\bomega) = \Cmapt \in [0,1]^{H\times W}.
\end{equation}
Given a pixel $(h,w)$, if its confidence $\Cmapt[h,w]$ is above a chosen threshold $\delta$, we label it with its predicted class $f(\xtg)[h,w] = \arg\!\max_{k\in\cY} \Pmapt[h,w,k]$, otherwise it is masked out. Computed over all images in $\cD_{\tg}$, these incomplete segmentation maps constitute target pseudo-labels that are used to train a new semantic-segmentation network. Optionally, we may repeat from step~(\ref{step2}) and learn alternately a confidence model to collect pseudo-labels and a segmentation network using this self-training.

\subsection{Confidence training with adversarial loss}
\label{chap4:subsec:advloss}

To train the confidence network $C$, we propose to jointly optimize two objectives. Following the approach proposed in \cref{chap3:sec:confidnet}, the first one supervises the confidence prediction on annotated source-domain examples using the known true class probabilities for the predictions from $F$. Specific to semantic segmentation with UDA, the second one is an adversarial loss that aims at reducing the domain gap between source and target. A complete overview of the approach is provided in \cref{chap4:fig:confidence_training}.  

\paragraph{Confidence loss.} 
The first objective is a pixel-wise version of the confidence loss \cref{chap3:eq:loss-conf}. On annotated source-domain images, it requires the confidence network $C$ to predict at each pixel the score assigned by the classifier $F$ to the (known) true class:
\begin{equation} 
\cL_{\text{conf}}(\bomega;\cD_{\so}) = 
\frac{1}{N_{\so}} \sum_{n=1}^{N_{\so}} 
\big\| 
\Cmapsn - \text{TCP}_F(\xson,\y_{\so,n}) \big\|^2_{\text{F}}, 
\label{chap4:eq:loss-conf-conda}
\end{equation}
where $\|\cdot\|_{\text{F}}$ denotes the Frobenius norm and, for an image $\x$ with true segmentation map $\y$ and predicted soft one-hot $F(\x;\hat{\btheta})$, we note
\begin{equation}
    \text{TCP}_F(\x,\y)[h,w] = F(\x;\hat{\btheta})\Big[h,w,\y[h,w]\Big]
\end{equation}
at location $(h,w)$. On a new input image, $C$ should predict at each pixel the score that $F$ will assign to the unknown true class, which will serve as a confidence measure.

However, compared to the application in the previous chapter, we have here the additional problem of the gap between source and target domains, an issue that might affect the training of the confidence model as in the training of the segmentation model.

\paragraph{Adversarial loss.} 
The second objective concerns the domain gap. While the confidence network $C$ learns to estimate TCP on source-domain images, its confidence estimation on target-domain images may suffer dramatically from this domain shift. As classically done in UDA, we propose an adversarial learning of our auxiliary model in order to address this problem. More precisely, we want the confidence maps produced by $C$ in the source domain to resemble those obtained in the target domain.

A discriminator $D:[0,1]^{H \times W} \rightarrow \{0,1\}$, with parameters $\bpsi$, is trained concurrently with $C$ with the aim to recognize the domain (1 for source, 0 for target) of an image given its confidence map. The following loss is minimized w.r.t. $\bpsi$:
\begin{equation}
    \cL_D(\bpsi;\cD_{\so}\cup\cD_{\tg}) = 
    \frac{1}{N_{\so}}\sum\limits_{n=1}^{N_{\so}} \cL_\text{adv}(\xson,1) +     \frac{1}{N_{\tg}}\sum\limits_{n=1}^{N_{\tg}} \cL_\text{adv}(\xtgn,0),
    \label{chap4:eq:l_Dconf}
\end{equation}
where $\cL_\text{adv}$ denotes the cross-entropy loss of the discriminator based on confidence maps:
\begin{equation}
    \cL_\text{adv}(\x,\lambda) = -\lambda\log\big(D(\Cmap;\bpsi)\big) - (1-\lambda)\log(1-D\big(\Cmap;\bpsi)\big),
\end{equation}
for $\lambda \in \{0,1\}$, which is a function of both $\bpsi$ and $\bomega$. In alternation with the training of the discriminator using \cref{chap4:eq:l_Dconf}, the adversarial training of the confidence net is conducted by minimizing, w.r.t. $\bomega$, the following loss:
\begin{equation}
    \cL_C(\bomega;\cD_{\so}\cup\cD_{\tg}) = \cL_\text{conf}(\bomega; \cD_{\so}) + \frac{\lambda_\text{adv}}{N_{\tg}}\sum\limits_{n=1}^{N_{\tg}}\cL_\text{adv}(\xtg,1),
    \label{chap4:eq:l_C}
\end{equation}
where the second term, weighted by $\lambda_\text{adv} > 0$, encourages $C$ to produce maps in the target domain that will confuse the discriminator.

This proposed adversarial confidence learning scheme also acts as a regularizer during training, improving robustness of the unknown TCP target confidence. As the training of the confidence model may actually be unstable, adversarial training provides additional information signal, in particular imposing that confidence estimation should be invariant to domain shifts. We empirically observe that this adversarial confidence learning provides better confidence estimates and improves convergence and stability of the training scheme.

\subsection{Multi-scale ConfidNet architecture}
\label{chap4:subsec:aspp}

\begin{figure}[t]
\centering
\includegraphics[width=0.8\linewidth]{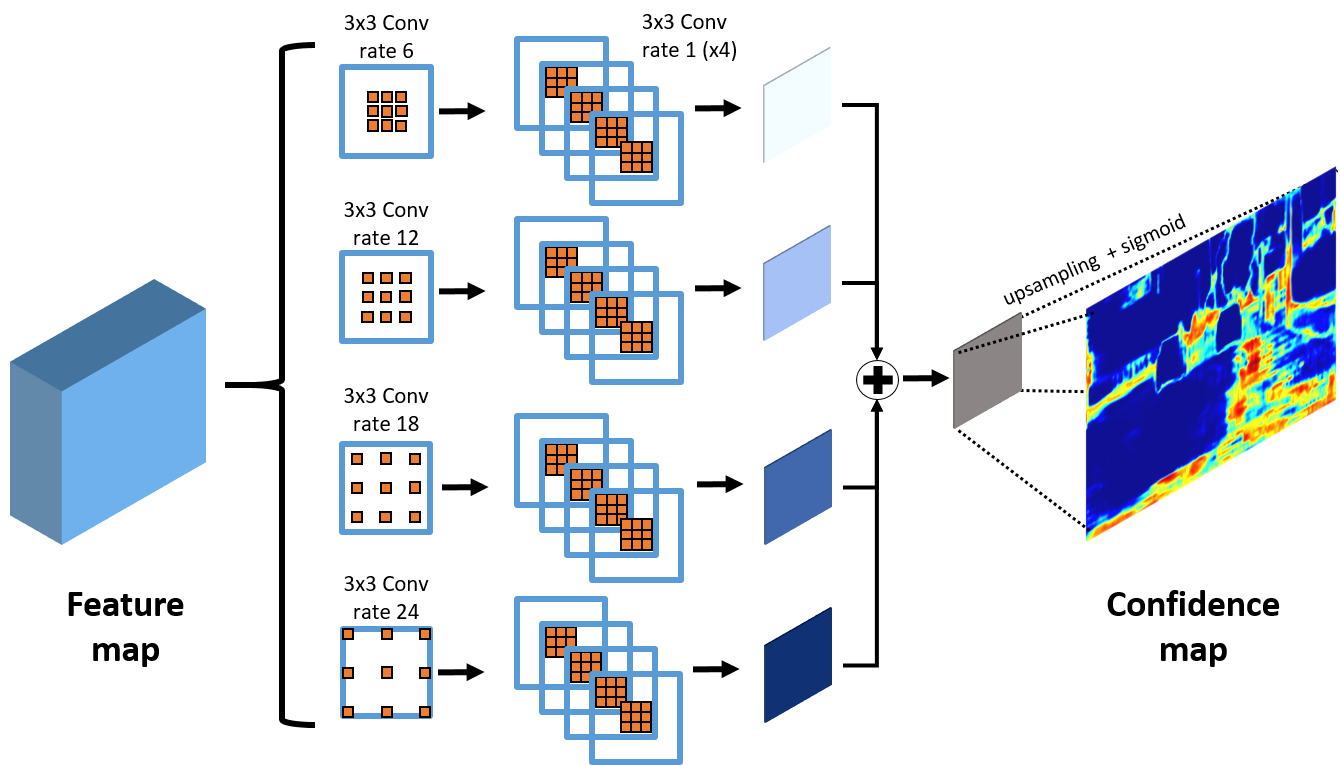}
\caption[Multi-scale architecture for confidence learning]{\textbf{Multi-scale architecture for confidence learning}. Four atrous convolutional layers are applied to a feature map in parallel, each of them is followed by a series of four standard convolutional layers. Confidence maps are obtained by summing the resulting features and upsampling to original image size.}
\label{chap4:fig:multi_scale_confidence}
\end{figure}

In semantic segmentation, models consist of fully convolutional networks where hidden representations are 2D feature maps. This is in contrast with the architecture of classification models considered in \cref{chap3}. Consequently, we replace fully-connected layers in the ConfidNet module by $1\!\times\!1$ convolutional layers with the adequate number of channels.

In many segmentation datasets, the existence of objects at multiple scales may complicate confidence estimation.
As in recent works dealing with varying object sizes~\cite{ChenPK0Y16}, we further improve our confidence network $C$ by adding a multi-scale architecture based on spatial pyramid pooling. It consists of a computationally efficient scheme to re-sample a feature map at different scales, and then to aggregate the confidence maps. We illustrate the multi-scale architecture for a confidence network in \cref{chap4:fig:multi_scale_confidence}.  

From a feature map, we apply parallel atrous convolutional layers with $3\!\times\!3$ kernel size and different sampling rates, each of them followed by a series of 4 standard convolutional layers with $3\!\times\!3$ kernel size. In contrast with convolutional layers with large kernels, atrous convolution layers enlarge the field of view of filters and help to incorporate a larger context without increasing the number of parameters and the computation time. Resulting features are then summed before upsampling to the original image size of $H\times W$. We apply a final sigmoid activation to output a confidence map with values between 0 and 1.

The whole architecture of the confidence model $C$ is represented in the orange block of \cref{chap4:fig:confidence_training}, along with its training given a fixed segmentation model $F$ (blue block) with which it shares the encoder. Such as in the previous section, fine-tuning the encoder within $C$ is also possible, although we did not explore the option in this semantic segmentation context due to the excessive memory overhead it implies.

\section{Experiments}
\label{chap4:sec:experiments}

In this section, we analyse on several semantic segmentation benchmarks the performance of ConDA, our approach to domain adaptation with confidence-based self-training. We report comparisons with state-of-the-art methods at the time of publication on each benchmark at the time of publication. We also analyse further the quality of ConDA's pseudo-labelling and demonstrate via an ablation study the importance of each of its components.

\subsection{Experimental setup}
\label{chap4:subsec:benchmark}

\paragraph{Datasets}
As in many domain adaptation works for semantic segmentation, we consider the specific task of adapting from synthetic to real data in urban scenes. We experiment with two synthetic source datasets -- SYNTHIA~\cite{ros-cvpr2016} and GTA5~\cite{richter-eecv2016} -- and two real target datasets -- Cityscapes~\cite{cordts-cvpr2016} and Mapillary Vistas~\cite{neuhold-iccv2017}. More specifically, we use the SYNTHIA-RAND-CITYSCAPES split for SYNTHIA~\cite{ros-cvpr2016}, composed of 9,400 color images generated in a simulator, of dimension $1280 \times 760$ and annotated for semantic segmentation with 16 classes in common with Cityscapes~\cite{cordts-cvpr2016}. As for GTA5~\cite{richter-eecv2016}, the dataset is composed of 24,966 images extracted from the eponymous game, of dimension $1914 \times 1052$, with semantic segmentation annotation with 19 classes in common with Cityscapes~\cite{cordts-cvpr2016}. On the other hand, Cityscapes~\cite{cordts-cvpr2016} is a dataset of real street-level images. It is split in a training set, a validation set and a test set. For domain adaptation, we use the training set as the target dataset during training. It is composed of 2,975 images of dimension $2048 \times 1024$. Since the ground-truth segmentation maps are missing from the testing dataset, we exploit the validation set composed of 500 images for testing purposes. We also validate the approach on Mapillary Vistas~\cite{neuhold-iccv2017}, another dataset of street-level images. As Cityscapes~\cite{cordts-cvpr2016}, it is split in a train set, a validation set and a test set, and the ground-truth maps are missing from the testing dataset. For domain adaptation, we use the 18,000 images from the training set as target and the 2,000 images from the validation set for testing. On Mapillary Vistas~\cite{neuhold-iccv2017} experiments, we consider 7 `super classes' that include the 19 and 16 classes used in Cityscapes~\cite{cordts-cvpr2016} experiments with GTA5~\cite{richter-eecv2016} and SYNTHIA~\cite{ros-cvpr2016}, respectively. All results are reported in terms of the mean intersection over union (mIoU) metric. The higher this percentage, the better. 

\begin{figure}[t]
\centering
\begin{minipage}[c]{0.45\linewidth}
\centering
    \includegraphics[width=\linewidth]{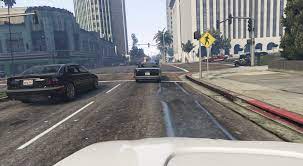}
    \subcaption{GTA5}
    \label{chap5:fig:gta5}
\end{minipage}%
\hspace{0.1cm}
\begin{minipage}[c]{0.52\linewidth}
\centering
    \includegraphics[width=\linewidth]{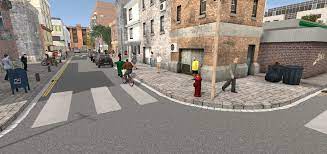}
    \subcaption{SYNTHIA}
    \label{chap5:fig:synthia}
\end{minipage}%
\vspace{0.2cm}
\begin{minipage}[c]{0.58\linewidth}
\centering
    \includegraphics[width=\linewidth]{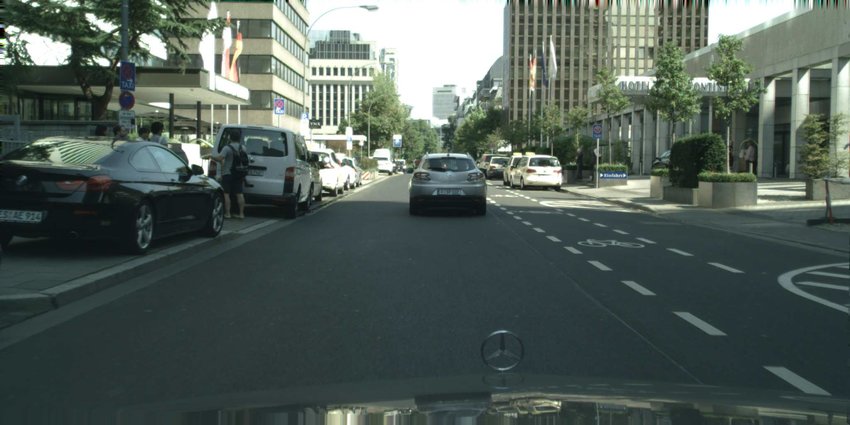}
    \subcaption{Cityscapes}
    \label{chap5:fig:cityscapes}
\end{minipage}%
\hspace{0.1cm}
\begin{minipage}{0.39\linewidth}
\centering
    \includegraphics[width=\linewidth]{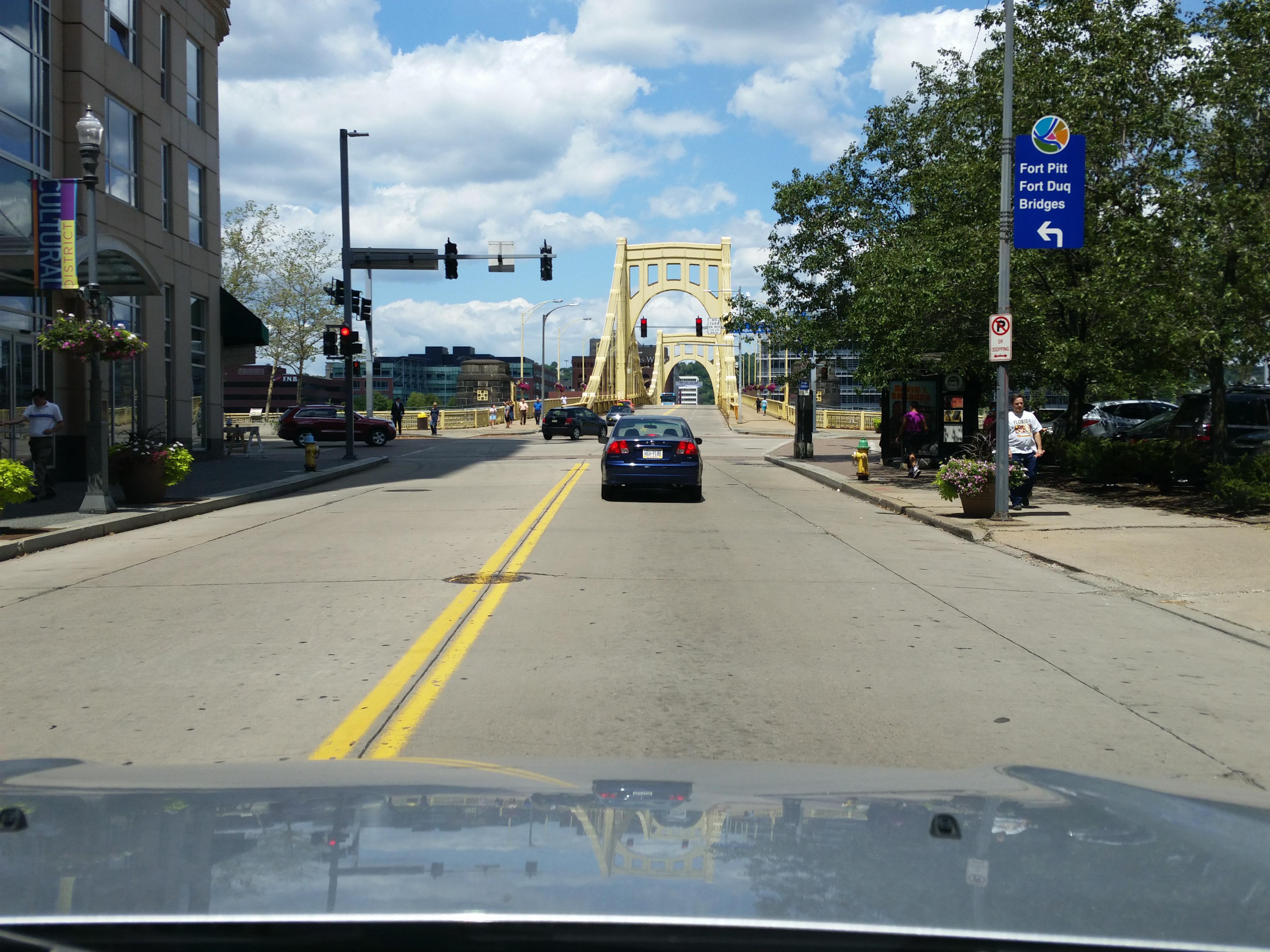}
    \subcaption{Mapillary Vistas}
    \label{chap5:fig:vistas}
\end{minipage}
\caption[Image samples from autonomous driving datasets used in UDA experiments]{\textbf{Images sample from datasets used in UDA experiments.} Synthetic datasets used as source domains are GTA5-dataset (a) and SYNTHIA (b); real datasets used as target domains are Cityscapes (c) and Mapillary Vistas (d).}
\label{chap5:fig:UDA_datasets}
\end{figure}

\paragraph{Network architectures.}
We evaluate the proposed self-training method on three state-of-the-art domain adaptation architectures at the time of publication. They all are based on DeepLabV2~\cite{ChenPK0Y16}, a standard semantic segmentation network.  
The domain alignment modules are nevertheless different:
\begin{itemize}
    \item \textit{AdaptSegNet}~\cite{Tsai_adaptseg_2018} performs adversarial domain adaptation on the output of the semantic segmentation network to align directly the segmentation maps between source and target domains.
    \item \textit{AdvEnt}~\cite{vu2018advent} proposes another adversarial learning framework for domain adaptation: instead of the softmax output prediction, 
    AdvEnt aligns the entropy of the pixel-wise predictions.
    \item \textit{DADA}~\cite{vu-iccv19} uses depth information on source images as privileged information during segmentation training.
\end{itemize}

\paragraph{Implementation details}
The semantic segmentation models are initialized with backbones pretrained on ImageNet~\cite{krizhevsky2012imagenet}. The segmentation network is optimized by Stochastic Gradient Descent with learning rate $2.5\times 10^{-4}$, momentum $0.9$ and weight decay $10^{-4}$. As for the discriminator, it is optimized by Adam~\cite{kingma-iclr2015} with learning rate $10^{-4}$. The hyperparameters $\lambda_\text{adv}$ and $\lambda_\text{ST}$ are fixed at $10^{-3}$ and $1$, respectively. For each baseline model, we start our self-training procedure from the pre-trained weights given on the author's GitHub\footnote{\url{https://github.com/valeoai/ADVENT}, \url{https://github.com/liyunsheng13/BDL}}.  
In some experiments, we use translated source images into the target domain. Those translated images are pre-computed using a CycleGAN~\cite{zhu-iccv2017}, as provided by~\cite{Li_2019_CVPR}. Regarding confidence training, we use the same training dataset than in segmentation training. As in \cref{chap3}, hyperparameters are chosen based on validation-set AUPR.

\begin{table}[t]
	\centering
    \caption[Comparison on mIoU of MCP-based vs. ConDA-based self-training]{\textbf{Comparison on mIoU of MCP-based vs. ConDA-based self-training} over multiple architectures and benchmarks. For DADA* architecture, segmentation models are trained using depth as privileged information.}
	\resizebox{\linewidth}{!}{%
		\begin{tabular}{l | c c | c | c c}
		    \toprule
			 & \multicolumn{2}{c|}{GTA5\uda\,Cityscapes} & \multicolumn{1}{c|}{SYNTHIA\uda\,Cityscapes} & \multicolumn{2}{c}{SYNTHIA\uda\,Mapillary } \\
			Self-Training~~ & AdaptSegNet & AdvEnt & AdvEnt & AdvEnt & DADA* \\
			\midrule
			MCP &  48.3 & 49.0 & 45.5 & 65.1 & 70.9 \\
			ConDA &  \textbf{49.3} & \textbf{49.9} & \textbf{46.0} & \textbf{66.4} & \textbf{72.0} \\
			\bottomrule
		\end{tabular}
	}
	\label{chap4:tab:mcpconda}
\end{table}

\subsection{ConDA vs. MCP self-training}
\label{chap4:subsec:conda_vs_mcp}

We compare the adaptation results using MCP to collect pseudo-labels and using ConDA instead, on two different methods: AdaptSegNet~\cite{Tsai_adaptseg_2018} and AdvEnt~\cite{vu2018advent}. Results are available in \cref{chap4:tab:mcpconda}. On GTA5\uda\,Cityscapes benchmark, we can see that, for both method, ConDA improves over MCP in self-training framework for domain adaptation by adding respectively $+1.0$ point mIoU improvement on AdaptSegNet and $+0.9$ point mIoU improvement on AdvEnt, which is our best result on this benchmark. We also compare the best adaptation method (Advent) on two other datasets: as shown in \cref{chap4:tab:mcpconda}, we observe a systematic improvement over MCP: $+0.5$ point mIoU on SYNTHIA\uda\,Cityscapes and $+1.3$ point mIoU on `SYNTHIA\uda\,Mapillary'. Finally, we also extend to methods dealing with multiple modalities (\eg depth) such as DADA~\cite{vu-iccv19} and we observe similarly that ConDA outperforms MCP by $+1.1$ point. These results demonstrate the relevance of our method by selecting better pseudo-labels to improve adaptation, regardless of the segmentation method or the UDA segmentation benchmark used.

\subsection{Comparison with UDA baselines}
\label{chap4:subsec:results}

In this section, ConDA results correspond to applying our self-training approach on AdvEnt domain adaptation method.

\paragraph{GTA5\uda\,Cityscapes.}
Results for semantic segmentation on the Cityscapes validation set using GTA5 as source domain are available in \cref{chap4:tab:gta2cityscapes} in the following page. We first notice that self-training based methods from the literature improved performance on this benchmark up to $48.5\%$ mIoU with BDL~\cite{Li_2019_CVPR}. ConDA outperforms all those methods on this framework by reaching $49.9\%$ mIoU. Note also that combining AdvEnt with MCP self-training already achieved $49.0\%$ mIoU.

\paragraph{SYNTHIA\uda\,Cityscapes.}
We extend experiments by using another source domain dataset. We report in a consistent way adaptation results for the task `SYNTHIA\uda\,Cityscapes' in \cref{chap4:tab:synthia2cityscapes}. Following relevant literature on this dataset, mIoU results for 16 categories and for 13 categories are available. Again, ConDA achieves state-of-the-art performance on this benchmark at the time of publication with $46.0\%$ mIoU .

\paragraph{SYNTHIA\uda\,Mapillary.}
Along with results on Cityscapes, we further study domain adaptation on another target dataset, namely Mapillary Vistas. \cref{chap4:tab:mapillary} presents semantic segmentation performance using SYNTHIA as source dataset. This benchmark has also been used in other recent works, such as in AdvEnt~\cite{vu2018advent} and DADA~\cite{vu-iccv19}. ConDA outperforms the baseline method with $66.4\%$ mIoU compared to $65.2\%$ mIoU in AdvEnt. When using depth from SYNTHIA as privileged information such as in  DADA, proposed method (ConDA*) further increases performance from $67.6\%$ mIoU to $72.0\%$ mIoU.

\begin{sidewaystable}[t]
	\centering
	\caption[Comparative performance on semantic segmentation with synth-to-real unsupervised domain adaptation]{\textbf{Comparative performance on semantic segmentation with synth-to-real unsupervised domain adaptation.} Results in per-class IoU and class-averaged mIoU on GTA5\uda\,Cityscapes. All methods are based on a DeepLabv2 backbone.}
	\resizebox{\linewidth}{!}{%
		\begin{tabular}{l | c | c c c c c c c c c c c c c c c c c c c|c}
			\toprule
			\multicolumn{22}{c}{GTA5\uda\,Cityscapes}\\
			\toprule
			Method & \rotatebox{90}{Self-Train.} & \rotatebox{90}{road} & \rotatebox{90}{sidewalk} & \rotatebox{90}{building} & \rotatebox{90}{wall} & \rotatebox{90}{fence} & \rotatebox{90}{pole} & \rotatebox{90}{light} & \rotatebox{90}{sign} & \rotatebox{90}{veg} & \rotatebox{90}{terrain} & \rotatebox{90}{sky} & \rotatebox{90}{person} & \rotatebox{90}{rider} & \rotatebox{90}{car} & \rotatebox{90}{truck} & \rotatebox{90}{bus} & \rotatebox{90}{train} & \rotatebox{90}{mbike} & \rotatebox{90}{bike} & mIoU \\
			\midrule
			AdaptSegNet~\cite{Tsai_adaptseg_2018} & & 86.5 & 25.9 & 79.8 & 22.1 & 20.0 & 23.6 & 33.1 & 21.8 & 81.8 & 25.9 & 75.9 & 57.3 & 26.2 & 76.3 & 29.8 & 32.1 & 7.2 & \textbf{29.5} & 32.5 & 41.4 \\
			CyCADA~\cite{Hoffman_cycada2017} & & 86.7 & 35.6 & 80.1 & 19.8 & 17.5 & \textbf{38.0} & \textbf{39.9} & \textbf{41.5} & 82.7 & 27.9 & 73.6 & \textbf{64.9} & 19.0 & 65.0 & 12.0 & 28.6 & 4.5 & 31.1 & 42.0 & 42.7 \\
			DISE~\cite{chang2019all} & & 91.5 & 47.5 & 82.5 & 31.3 & 25.6 & 33.0 & 33.7 & 25.8 & 82.7 & 28.8 & 82.7 & 62.4 & 30.8 & 85.2 & 27.7 & 34.5 & 6.4 & 25.2 & 24.4 & 45.4 \\
			AdvEnt~\cite{vu2018advent} & & 89.4 & 33.1 & 81.0 & 26.6 & 26.8 & 27.2 & 33.5 & 24.7 & 83.9 & 36.7 & 78.8 & 58.7 & 30.5 & 84.8 & 38.5 & 44.5 & 1.7 & 31.6 & 32.4 & 45.5 \\	
			\midrule
		    CBST~\cite{zou2018unsupervised} & \checkmark & 91.8 & 53.5 & 80.5 & 32.7 & 21.0 & 34.0 & 28.9 & 20.4 & 83.9 & 34.2 & 80.9 & 53.1 & 24.0 & 82.7 & 30.3 & 35.9 & 16.0 & 25.9 & \textbf{42.8} & 45.9 \\
			MRKLD~\cite{Zou_2019_ICCV} & \checkmark & 91.0 & 55.4 & 80.0 & 33.7 & 21.4 & 37.3 & 32.9 & 24.5 & 85.0 & 34.1 & 80.8 & 57.7 & 24.6 & 84.1 & 27.8 & 30.1 & \textbf{26.9} & 26.0 & 42.3 & 47.1 \\
			BDL~\cite{Li_2019_CVPR} & \checkmark & 91.0 & 44.7 & 84.2 & 34.6 & \textbf{27.5} & 30.2 & 36.0 & 36.0 & 85.0 & \textbf{43.6} & 83.0 & 58.6 & \textbf{31.6} & 83.3 & 35.3 & 49.7 & 3.3 & 28.8 & 35.6 & 48.5 \\
			ESL~\cite{saporta2020esl} & \checkmark & 90.2 & 43.9 & 84.7 & 35.9 & 28.5 & 31.2 & 37.9 & 34.0 & 84.5 & 42.2 & 83.9 & 59.0 & 32.2 & 81.8 & 36.7 & 49.4 & 1.8 & 30.6 & 34.1 & 48.6 \\
			\rowcolor{Gray}
			ConDA & \checkmark & \textbf{93.5} & \textbf{56.9} & \textbf{85.3} & \textbf{38.6} & 26.1 & 34.3 & 36.9 & 29.9 & \textbf{85.3} & 40.6 & \textbf{88.3} & 58.1 & 30.3 & \textbf{85.8} & \textbf{39.8} & \textbf{51.0} & 0.0 & 28.9 & 37.8 & \textbf{49.9} \\
			\bottomrule
		\end{tabular}
	}
	\label{chap4:tab:gta2cityscapes}
\end{sidewaystable}

\begin{sidewaystable}[t]
	\centering
	\caption[Comparison in mIoU for SYNTHIA\uda\,Cityscapes experiments]{\textbf{Comparison in mIoU for SYNTHIA\uda\,Cityscapes experiments}. mIoU* is the 13-class setup (excluding the classes `wall', `fence' and `pole') as used in earlier works.  All methods reported are based on a DeepLabv2 backbone. 
	}
	\resizebox{\linewidth}{!}{%
		\begin{tabular}{l| c | c c c c c c c c c c c c c c c c | c |c}
			\toprule
			\multicolumn{20}{c}{SYNTHIA\uda\,Cityscapes}\\
			\toprule
			Method & \rotatebox{90}{Self-Train.} & \rotatebox{90}{road} & \rotatebox{90}{sidewalk} & \rotatebox{90}{building} & \rotatebox{90}{wall} & \rotatebox{90}{fence} & \rotatebox{90}{pole} & \rotatebox{90}{light} & \rotatebox{90}{sign} & \rotatebox{90}{veg} & \rotatebox{90}{sky} & \rotatebox{90}{person} & \rotatebox{90}{rider} & \rotatebox{90}{car} & \rotatebox{90}{bus} & \rotatebox{90}{mbike} & \rotatebox{90}{bike} & mIoU & mIoU*\\
			\midrule
			AdaptSegNet~\cite{Tsai_adaptseg_2018} & & 84.3 & 42.7 & 77.5 & - & - & - & 4.7 & 7.0 & 77.9 & 82.5 & 54.3 & 21.0 & 72.3 & 32.2 & 18.9 & 32.3 & - & 46.7 \\
			DISE~\cite{chang2019all} & & \textbf{91.7} & \textbf{53.5} & 77.1 & 2.5 & 0.2 & 27.1 & 6.2 & 7.6 & 78.4 & 81.2 & 55.8 & 19.2 & 82.3 & 30.3 & 17.1 & 34.3 & 41.5 & 48.8\\
			AdvEnt~\cite{vu2018advent} & & 85.6 & 42.2 & 79.7 & 8.7 & 0.4 & 25.9 & 5.4 & 8.1 & 80.4 & 84.1 & 57.9 & 23.8 & 73.3 & 36.4 & 14.2 & 33.0 & 41.2 & 48.0 \\
			\midrule
			CBST~\cite{zou2018unsupervised} & \checkmark & 68.0 & 29.9 & 76.3 & 10.8 & 1.4 & 33.9 & \textbf{22.8} & 29.5 & 77.6 & 78.3 & 60.6 & 28.3 & 81.6 & 23.5 & 18.8 & 39.8 & 42.6 & 48.9\\ 
			MRKLD~\cite{Zou_2019_ICCV} & \checkmark & 67.7 & 32.2 & 73.9 & 10.7 & \textbf{1.6} & \textbf{37.4} & 22.2 & \textbf{31.2} & 80.8 & 80.5 & \textbf{60.8} & \textbf{29.1} & \textbf{82.8} & 25.0 & 19.4 & 45.3 & 43.8 & 50.1\\  
			BDL~\cite{Li_2019_CVPR} & \checkmark & 83.9 & 43.7 & 80.2 & \textbf{12.9} & 0.5 & 30.1 & 18.0 & 17.3 & 79.7 & 83.5 & 52.2 & 25.8 & 72.5 & 35.5 & \textbf{25.8} & 45.4 & 44.2 & 51.0\\   
			ESL~\cite{saporta2020esl} & \checkmark & 84.3 & 39.7 & 79.0 & 9.4 & 0.7 & 27.7 & 16.0 & 14.3 & 78.3 & 83.8 & 59.1 & 26.6 & 72.7 & 35.8 & 23.6 & 45.8 & 43.5 & 50.7\\
			\rowcolor{Gray}
			ConDA (Ours) & \checkmark & 88.1 & 46.7 & \textbf{81.1} & 10.6 & 1.1 & 31.3 & 22.6 & 19.6 & \textbf{81.3} & \textbf{84.3} & 53.9 & 21.7 & 79.8 & \textbf{42.9} & 24.2 & \textbf{46.8} & \textbf{46.0} & \textbf{53.3}\\
			\bottomrule
		\end{tabular}
	}
	\label{chap4:tab:synthia2cityscapes}
\end{sidewaystable}

\clearpage

\begin{table}[t]
	\centering
    \caption[Comparison in mIoU for SYNTHIA\uda\,Mapillary experiments]{\textbf{Comparison in mIoU for SYNTHIA\uda\,Mapillary experiments}. DADA* and ConDA* are trained using depth as privileged information.}
	\resizebox{0.7\linewidth}{!}{%
		\begin{tabular}{l| c | c c c c c c c|c}
		    \toprule
			\multicolumn{10}{c}{SYNTHIA\uda\,Mapillary}\\
			\midrule
			Method & \rotatebox{90}{Self-Train.} & \rotatebox{90}{flat} & \rotatebox{90}{constr.} & \rotatebox{90}{object} & \rotatebox{90}{nature} & \rotatebox{90}{sky} & \rotatebox{90}{human} & \rotatebox{90}{vehicle} &
			mIoU\\
			\midrule
			AdvEnt~\cite{vu2018advent} & & 86.9 & 58.8 & 30.5 & 74.1 & 85.1 & 48.3 & 72.5 & 65.2 \\
			ESL~\cite{saporta2020esl} & \checkmark & 88.4 & 55.7 & \textbf{32.0} & \textbf{75.4} & 84.3 & 43.5 & \textbf{76.2} & 65.4 \\
			\rowcolor{Gray}
			ConDA (Ours) & \checkmark & \textbf{89.1} & \textbf{63.5} & 28.3 & 72.7 & \textbf{88.2} & \textbf{49.7} & 73.0 & \textbf{66.4}\\  
			\midrule  \midrule
			DADA*~\cite{vu-iccv19} & & 86.7 & 62.1 & 34.9 & 75.9 & 88.6 & 51.1 & 73.8 & 67.6 \\
			\rowcolor{Gray}
			ConDA* (Ours) & \checkmark & \textbf{87.8} & \textbf{67.5} & \textbf{40.5} & \textbf{76.8} & \textbf{92.3} & \textbf{60.7} & \textbf{78.5} & \textbf{72.0}\\  
			\bottomrule
		\end{tabular}
	}
	\label{chap4:tab:mapillary}
\end{table}

\subsection{Ablation study}
\label{chap4:subsec:ablation}

To study the effect of the adversarial training and of the multi-scale confidence architecture on the confidence model, we perform an ablation study on the GTA5\uda\,Cityscapes benchmark. The results on domain adaptation after re-training the segmentation network using collected pseudo-labels are reported in \cref{chap4:tab:ablationstudy}. In this table, ``ConfidNet'' refers to the simple network architecture defined in \cref{chap3:sec:confidnet} (adapted to segmentation by replacing the fully connected layers by $1\!\times\! 1$ convolutions of suitable width); ``Adv. ConfidNet'' denotes the same architecture but with the adversarial loss from \cref{chap4:subsec:advloss} added to its learning scheme; ``Multi-scale ConfidNet'' stands for the architecture introduced in \cref{chap4:subsec:aspp}; Finally, the full method, ``ConDA'' amounts to having both this architecture and the adversarial loss. We notice that adding the adversarial learning achieves significantly better performance, for both ConfidNet and multi-scale ConfidNet, with respectively $+1.4$ and $+0.8$ point increase. Multi-scale ConfidNet (resp. adv. multi-Scale ConfidNet) also improves performance up to $+0.9$ point (resp. $+0.3$) from their ConfidNet architecture counterpart. These results stress the importance of both components of the proposed confidence model.

\begin{table}[ht]
    \centering
	\resizebox{0.7\linewidth}{!}{%
	\begin{tabular}{l|cc|c}
		\toprule
		Model & Multi-Scale. & Adv & mIoU\\
		\midrule
		ConfidNet & & & 47.6 \\
		Multi-Scale ConfidNet  & \checkmark & & 48.5\\
		Adv. ConfidNet & & \checkmark & 49.0 \\
		ConDA (Adv. Multi-scale ConfidNet) & \checkmark & \checkmark & \ubold{49.9} \\
		\bottomrule
	\end{tabular}
	}
    \caption[Ablation study on semantic segmentation with pseudo-labelling-based adaptation]{\textbf{Ablation study on semantic segmentation with pseudo-labelling-based adaptation.} Full-fledged ConDA approach is compared on GTA5\uda\,Cityscapes to stripped-down variants (with/without multi-scale architecture in ConfidNet, with/without adversarial learning).}
    \label{chap4:tab:ablationstudy}
\end{table}

\subsection{Quality of pseudo-labels}
\label{chap4:subsec:quality}

We analyze the effectiveness of MCP and ConDA as confidence measures to select relevant pseudo-labels in the target domain. For a given fraction of retained pseudo-labels (coverage) on target-domain training images, we compare in \cref{chap4:fig:conda_analysis} the precision of each method. Here, precision means the ratio between the number of correct predictions and the total number of collected pseudo-labels, \ie accuracy. We vary the coverage between 70\% and 90\%\footnote{For instance, in our previous experiment with AdvEnt on `GTA5\uda\,Cityscapes', 80.5\% pixels were kept using MCP and 82.0\% using ConDA.}.

\begin{figure}[ht]
    \centering
    \includegraphics[width=0.75\linewidth]{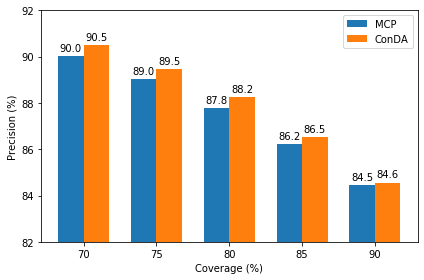}
    \caption[Comparative quality of selected pseudo-labels]{\textbf{Comparative quality of selected pseudo-labels}. Proportion of correct pseudo-labels (precision) for different coverages on GTA5\uda\,Cityscapes, for MCP and ConDA.}
    \label{chap4:fig:conda_analysis}
\end{figure}

ConDA outperforms MCP for all coverage levels, meaning it selects significantly fewer erroneous predictions for the next round of segmentation-model training. Along with the segmentation adaptation improvements presented earlier, these coverage results demonstrate that reducing the amount of noise in the pseudo-labels is key to learning a better segmentation adaptation model. \cref{chap4:fig:qualitative_results} presents qualitative results of those pseudo-labels methods. We find again that MCP and ConDA seem to select around the same amount of correct predictions in their pseudo-labels, but with ConDA picking out a lot fewer erroneous ones.

\begin{figure}[ht]
\centering
\includegraphics[width=\linewidth]{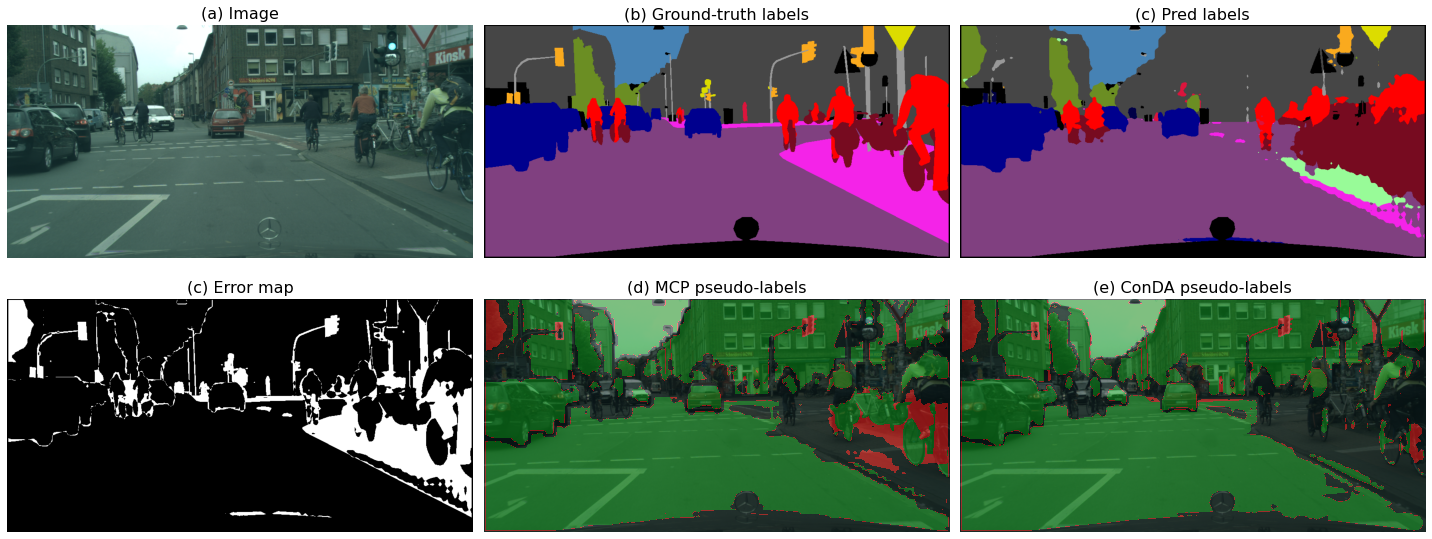}
\caption[Visualisations of pseudo-label selection on GTA5\uda\,Cityscapes benchmark] {\textbf{Qualitative result of pseudo-label selection for semantic-segmentation adaptation.} The three first panels present a target-domain image of the GTA5\uda\,Cityscapes benchmark (a) along with its ground-truth segmentation maps (b) and the predicted map before self-training (c). The error map associated with the predicted segmentation is shown in (d), with erroneous predictions flagged in white. We compare pseudo-labels collected with MCP (e) and with ConDA (f). Green (resp. red) pixels are correct (resp. erroneous) predictions selected by the method and black pixels are discarded predictions. ConDA retains fewer errors while preserving approximately the same amount of correct predictions.}
\label{chap4:fig:qualitative_results}
\end{figure}

\paragraph{Computational Cost.} Composed of only four atrous convolutional layers in parallel, our multi-scale confidence network remains light. When collecting pseudo-labels, the overhead cost induced by our method is minor when estimating confidence, only adding roughly 10\% time increase compared to MCP. In our setting, a forward pass using ConDA takes 43ms on average. Note that segmentation training and inference are not changed, which makes ConDA suitable for real-time purposes.

\section{Conclusion}
\label{chap4:sec:conclusion}

In this chapter, we show that applied to self-training with pseudo-labels, using an auxiliary model dedicated to estimate the confidence of predictions help to better select relevant pixels for pseudo-labeling. Our learning approach brings systematic improvements in performance over self-training based on the standard MCP. We reach state-of-the-art results at the time of publication on three synthetic-to-real unsupervised-domain-adaptation benchmarks (GTA5\uda\, Cityscapes, SYNTHIA\uda Cityscapes and SYNTHIA\uda Mapillary Vistas). To achieve these results, we equipped the auxiliary model with a multi-scale confidence architecture and supplemented the confidence loss with an adversarial training scheme to enforce alignment between confidence maps in source and target domains. In particular, a clear benefit of our learning approach is to be compatible with any models which use domain adaptation, without adding substantial overhead cost (only 10\% time increase compared to MCP in our experiments).

Thus far, we focused on detecting errors to reject them or to alternatively select only correct predictions. However, in the wild, a ML system may also encounter data that is unlike its training data. In addition to in-distribution errors, we will also consider the detection of out-of-distribution samples in the next chapter to be robust to any kind of hazardous predictions.

%% file: chapter05/chapter05.tex
\begin{center}
  \textsc{Chapter Abstract}
\end{center}
\begin{quote}
\noindent \textit{A safe deployment of ML systems should include an accurate monitoring of errors but also unusual inputs where predicting might be hazardous. In this chapter, we address the task of jointly detecting errors and anomalies in classification tasks. Evidential models are a Bayesian approach which provides a sampling-free way of deriving second-order uncertainty measures on the simplex, \ie measures on the distribution over probabilities, that estimate different sources of uncertainty. In this chapter, we leverage the second-order representation provided by evidential models and we introduce \emph{KLoS}, a Kullback–Leibler divergence criterion defined on the class-probability simplex. By keeping the full distributional information, KLoS captures in-distribution and out-of-distribution (OOD) uncertainties in a single score. A crucial property of KLoS is to be a class-wise divergence measure built from in-distribution samples and to not require OOD training data, in contrast to current uncertainty measure used with evidential models. We further design an auxiliary neural network, \emph{KLoSNet}, to learn a refined criterion directly aligned with the evidential training objective. In the realistic context where no OOD data is available during training, our experiments show that KLoSNet outperforms every other uncertainty measures to simultaneously detect misclassifications and OOD samples. When training with OOD samples, we also observe that existing measures are brittle to the choice of the OOD dataset, whereas KLoS remains more robust.}

\textit{The work described in this chapter is based on the following publication:}
\begin{itemize}
    \item \textit{Charles Corbière, Marc Lafon, Nicolas Thome, Matthieu Cord, Patrick Pérez. ``Beyond First-Order Estimation with Evidential Models for Open-World Recognition". ICML 2021 Workshop on Uncertainty and Robustness in Deep Learning.}
\end{itemize}
\end{quote}

\clearpage
\minitoc[tight]

\section{Context}
\label{chap5:sec:context}

Machine learning models commonly rely on the closed-set assumption that source and target data are independent and identically distributed (\textit{i.i.d.}). Yet in practice, distribution shifts arise naturally in many real-world scenarios. For instance, self-driving cars struggle to perform well under conditions different to those of training, such as variations in weather \cite{volk2019}, light \cite{dai2018}, and object poses \cite{Alcorn_2019_CVPR}. Worse, models can be exposed to inputs from unseen classes which they will attempt to predict anyway. These failures may remain unnoticed as they do not result in explicit errors in the model. 

\begin{figure}[t]
\centering
    \includegraphics[width=0.95\linewidth]{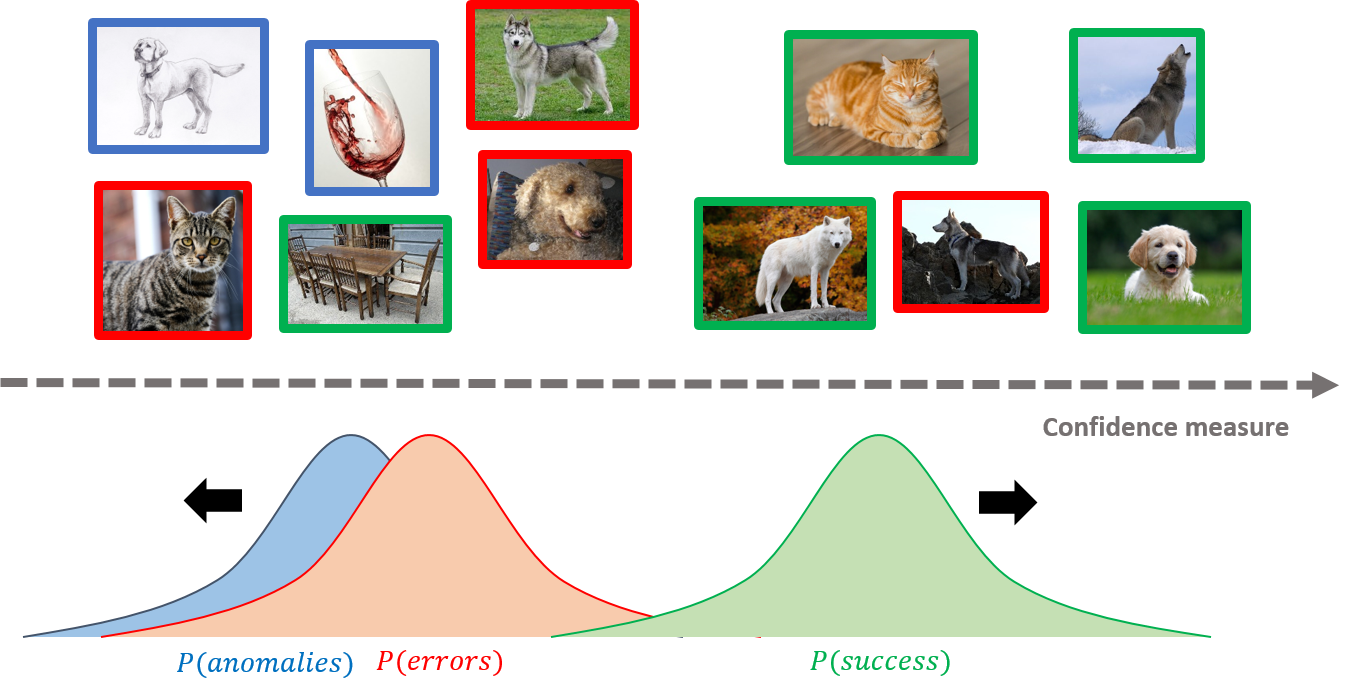}
    \caption[Simultaneous detection of misclassifications and OOD samples]{\textbf{Simultaneous detection of misclassifications and OOD samples.} When ranking samples according to their confidence value, correct predictions (in green) should have higher scores on average than misclassifications (in red) \emph{and} OOD samples (in blue) to enable the model to distinguish them.}
\label{chap5:fig:simultaneous}
\end{figure}

While previous works address separately misclassification detection and OOD detection, we argue it is necessary for a recognition system to be able to identify both in-distribution misclassifications and unknown/unseen inputs at test time for a safe deployment in open-world settings \cite{Bendale_2015_CVPR}. We illustrate this task in \cref{chap5:fig:simultaneous}. In particular, we find in \cref{chap5:sec:exp} that all previous approaches do not perform equally well on both detection tasks, which mitigates their ability on the joint detection task.

To address the task of simultaneous detection of misclassifications and OOD samples, an uncertainty measure should discriminate between correct predictions and erroneous predictions for in-distribution samples while increasing for inputs far from distribution. Consequently, it should capture both aleatoric and epistemic uncertainty. Bayesian approaches \cite{Gal2016,NIPS2019_9472} and ensembles \cite{deepensembles2017,ovadia2019} are principled methods which induce a more accurate estimation of epistemic uncertainty. These techniques produce a probability density over the predictive categorical distribution $p(y \vert \x, \cD)$ obtained from sampling as shown in \cref{chap2} (top row of \cref{chap2:fig:simplex}). But this comes at the expense of an increased computational cost.

A recent class of models, coined \emph{evidential} \cite{malinin2018,sensoy2018}, proposes instead to explicitly learn the concentration parameters of a Dirichlet distribution over probabilities. Based on the subjective logic framework \cite{josan2016sublogic}, evidential models enrich uncertainty representation with evidence information and enable a model to represent different sources of uncertainty (bottom row of \cref{chap2:fig:simplex}). \emph{Conflicting evidence}, \eg, noise or class confusion, is characterized by the expectation of the second-order Dirichlet distribution while the distribution spread on the simplex expresses the \emph{amount of evidence} in a prediction \cite{Shi2020MultifacetedUE}. These sources of uncertainty correspond respectively to aleatoric uncertainty and epistemic uncertainty in the machine learning literature \cite{malinin2019}. Evidential models have been shown to improve generalisation \cite{beingbayesian2020}, OOD detection \cite{maxgap2020} and adversarial attack detection \cite{malinin2019}.

In this chapter, we leverage the second-order uncertainty representation that evidential models provide and we introduce \emph{KLoS}, a measure that accounts for both in-distribution and out-of-distribution sources of uncertainty and that is effective even without having access to auxiliary OOD data at train time. KLoS computes the Kullback–Leibler (KL) divergence between the model's predicted Dirichlet distribution and a specifically designed class-wise prototype Dirichlet distribution. Prototype distributions are designed with concentration parameters shared with in-distribution training data, which enables a model to detect OOD samples without assuming any restrictive behavior, \eg, having low precision $\alpha_0$. KLoS naturally reflects the training objective used in evidential models and we propose to learn an auxiliary model, named \emph{KLoSNet}, to regress the values of a refined objective for training samples and to improve uncertainty estimation. To assess the quality of uncertainty estimates in open-world recognition, we design the new task of simultaneous detection of misclassifications and OOD samples. Extensive experiments show the benefits of KLoSNet on various image datasets and model architectures. In presence of OOD training data, we also found that our proposed measure is more robust to the choice of OOD samples while previous measures may perform poorly. Finally, we show that KLoS can be successfully combined with ensembling to improve performance.

\section{Evidential neural networks}
\label{chap5:sec:enn}

\sloppy
In this section, we partially remind the reader the framework of evidential models introduced in \cref{chap2} and focus on the derivation of the training objective to put in perspective the link with our proposed uncertainty measure afterwards. The training dataset $\cD$ consists of $N$ \textit{i.i.d.} samples $(\x,y)$ drawn from an unknown joint distribution $P(X,Y)$. We denote $\bpi =(\pi_1, \cdots, \pi_K)$ the random variable over categorical probabilities, where $\sum_{k=1}^K \pi_k$ = 1, and which lives on the (K-1)-dimensional simplex $\triangle^{K-1}$.

Evidential Neural Networks (ENNs) propose to model explicitly the posterior distribution over categorical probabilities $p(\bpi \vert \x, y)$ by a variational Dirichlet distribution,
\begin{equation}
    q_{\btheta}(\bpi \vert \x) \!=\! \text{Dir} \big ( \bpi \vert \balpha(\x, \btheta) \big )
    \!=\! 
    \frac{\Gamma(\alpha_0 (\x, \btheta))}
    {\prod_{k=1}^K \Gamma(\alpha_k(\x, \btheta)) }\prod_{k=1}^K \pi_k^{\alpha_k(\x, \btheta) - 1},
\end{equation}
whose concentration parameters $\balpha(\x, \btheta) ~{=}~ \exp f(\x,\btheta)$ are output by a neural network $f$ with parameters $\btheta$; $\Gamma$ is the Gamma function and $\alpha_0(\x, \btheta) ~{=}~ \sum_{k=1}^K \alpha_k(\x, \btheta)$ with $\alpha_k = \exp f_k(\x, \btheta)$ indexing the $k^{\text{th}}$ element of the vector of all $K$ concentration parameters $\balpha$. Precision $\alpha_0$ controls the sharpness of the density with more mass concentrating around the mean as $\alpha_0$ grows. By conjugate property, the predictive distribution for a new point $\x^*$ is 
\begin{equation}
P(Y~{=}~k ~\vert~ \x^*, \cD)~
     {\approx}~\mathbb{E}_{q_{\btheta}(\bpi \vert \x^*)} [\pi_k] ~{=}~ 
     \frac{\exp f_k(\x^*,\btheta)}{\sum_{j=1}^K \exp f_j(\x^*,\btheta)},
\end{equation}
which is the usual output of a network $f$ with softmax activation.

The concentration parameters $\balpha$ can be interpreted as pseudo-counts representing the amount of evidence in each class. For instance, in \cref{chap5:fig:intro_a}, the $\balpha$'s output by the ENN indicate that the image is almost equally likely to be classified as \textit{wolf} or as \textit{dog}. More interestingly, it also distinguishes these in-distribution images from the OOD sample in \cref{chap5:fig:intro_b} via the total amount of evidence $\alpha_0$.

\paragraph{\textbf{Training Objective}} The ENN training is formulated as a variational approximation to minimize the KL divergence between the distribution $q_{\btheta}(\bpi \vert \x) $ and the true posterior distribution $p(\bpi \vert \x, y)$:
\begin{align}
    \cL_{\text{var}}(\btheta;\cD) &= \mathbb{E}_{(\x,y) \sim P(X, Y)} \big [ \mathbb{KL} \big ( q_{\btheta}(\bpi \vert \x)~\|~p(\bpi \vert \x, y) \big ) \big ] \\
    &= \frac{1}{N} \sum_{(\x, y) \in \cD} \Big [ \int q_{\boldsymbol{\theta}}(\bpi \vert \x) \log \frac{q_{\boldsymbol{\theta}}(\bpi \vert \x)}{p(\bpi \vert \x, y)} \Big ] \\
    &= \frac{1}{N} \sum_{(\x, y) \in \cD} \Big [ \int q_{\boldsymbol{\theta}}(\bpi \vert \x) \log \frac{q_{\boldsymbol{\theta}}(\bpi \vert \x)p(y \vert \x)}{p(y \vert \bpi, \x)p(\bpi \vert \x)} \Big ] \\
    &= \frac{1}{N} \sum_{(\x, y) \in \cD} \Big [ \mathbb{E}_{q_{\boldsymbol{\theta}}(\bpi \vert \x)} \big [ - \log p(y \vert \bpi, \x) \big ] + \mathbb{KL} \big (q_{\boldsymbol{\theta}}(\bpi \vert \x)~\|~p(\bpi \vert \x) \big ) + \log p(y \vert \x) \Big ],
\end{align}
where $N=\mathrm{card}(\cD)$. 
As the log-likelihood $\log p(y \vert \x)$ does not depend on parameters $\btheta$,
\begin{equation}
    \min_{\theta} \cL_{\text{var}}(\btheta;\cD) = \min_{\theta}  \frac{1}{N} \sum_{(\x, y) \in \cD} \Big [ \mathbb{E}_{q_{\boldsymbol{\theta}}(\bpi \vert \x)} \big [ - \log p(y \vert \bpi, \x) \big ] + \mathbb{KL} \big (q_{\boldsymbol{\theta}}(\bpi \vert \x)~\|~p(\bpi \vert \x) \big ) \Big ].
\end{equation}
For conciseness, we denote $\alpha_k = \alpha_k(\x, \btheta), \forall k \in \cY$ hereafter. For a sample $(\x, y)$, the reverse cross-entropy term amounts to $\mathbb{E}_{\bpi \sim q_{\btheta}(\bpi \vert \x)} \big [\!-\!\log p(y \vert \bpi, \x) \big ] = - \big ( \psi(\alpha_y) \!-\! \psi(\alpha_0) \big )$ where $\psi$ is the digamma function. Hence, the optimization objective is written as:
\begin{equation}
    \min_{\theta} \cL_{\text{var}}(\btheta;\cD) = \frac{1}{N} \sum_{(\x, y) \in \cD} - \big ( \psi(\alpha_y) \!-\! \psi(\alpha_0) \big ) + \mathbb{KL} ( q_{\btheta}(\bpi \vert \x)~\|~p(\bpi \vert \x) ).
    \label{chap5:eq:dev_loss_vi}
\end{equation}
Considering that most of the training inputs $\x$ are associated with only one observation $y$ in $\mathcal{D}$, we should choose small concentration parameters $\boldsymbol{\beta}$ for the prior $p(\bpi \vert \x)= \text{Dir} (\bpi \vert \boldsymbol{\beta})$ to prevent the resulting posterior distribution $p(\bpi \vert \x) = \text{Dir} (\bpi \vert \beta_1,...,\beta_y + 1,...,\beta_K)$ from being dominated by the prior. However, this causes gradients to be very large in small-value regimes due to the digamma function, \eg $\psi'(0.01) > 10^{-4}$.

\begin{figure}[t]
\centering
\captionsetup[subfigure]{justification=centering}
\begin{minipage}[c]{0.45\linewidth}
\centering
    \includegraphics[width=\linewidth]{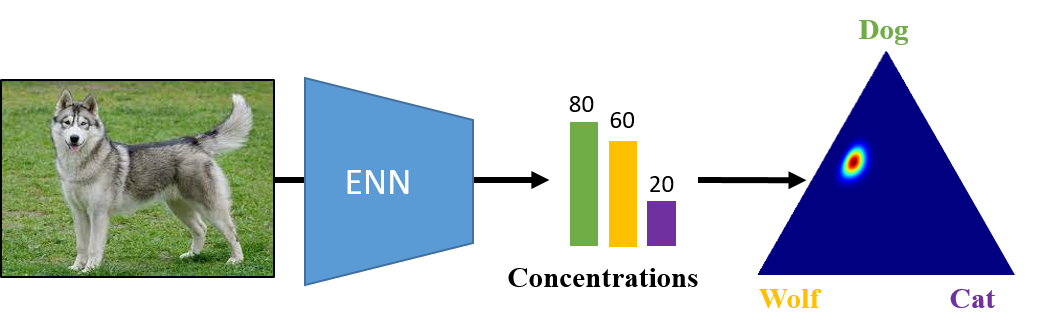}
    \subcaption{
    In-distribution image \\
    $\textrm{MCP}=0.50~$, entropy $= 0.97$, 
    $\textrm{KLoS}=97.85$}
    \label{chap5:fig:intro_a}
\end{minipage}%
\hspace{0.3cm}
\begin{minipage}{0.45\linewidth}
\centering
    \includegraphics[width=\linewidth]{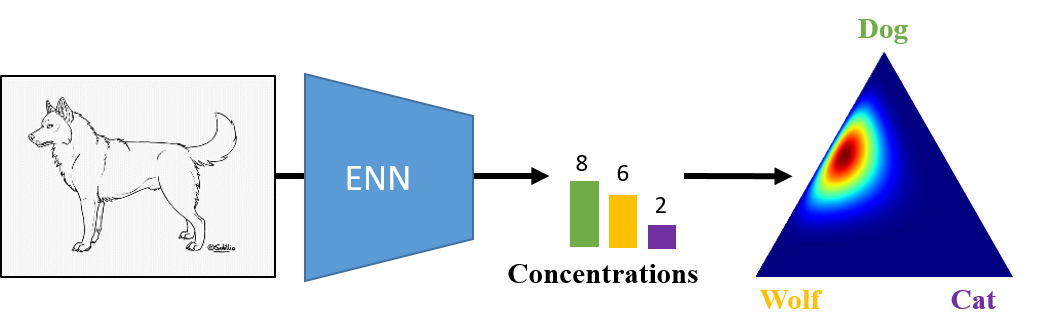}
    \subcaption{
    Outlier with same class confusion \\
     $\textrm{MCP}=0.50~$, entropy $= 0.97$,
    $\textrm{KLoS}=104.71$}
    \label{chap5:fig:intro_b}
\end{minipage}
\caption[Limitations of first-order uncertainty measures and their handling with KLoS]{\textbf{Limitations of first-order uncertainty measures and their handling with KLoS}. (a) An in-distribution image with conflicting evidence between \textit{dog} and \textit{wolf}. (b) An outlier with the same class confusion but a lower amount of evidence. An evidential neural network (ENN) outputs class-wise evidence information as concentration parameters of a Dirichlet density (visualized on the simplex) over 3-class distributions. Although this density is flatter for the second input, the predictive entropy and MCP, only based on first-order statistics, are equal for both inputs. In contrast, the proposed measure, KLoS, captures both class confusion and lack of evidence, hence correctly reflecting the larger uncertainty for the latter sample.}
\label{chap5:fig:intro}
\end{figure}

To stabilize the optimization, we follow \cite{beingbayesian2020} and use the non-informative uniform prior distribution $p(\bpi \vert \x) = \text{Dir} \big ( \bpi \vert \boldsymbol{1} \big )$ where $\boldsymbol{1}$ is the all-one uniform vector, and we weight the KL divergence term with $\lambda>0$:
\begin{equation}
    \mathcal{L}^{\lambda}_{\text{var}}(\btheta;\cD) =
    \frac{1}{N} \sum_{(\x, y) \in \cD} - \big ( \psi(\alpha_y) - \psi(\alpha_0) \big ) + \lambda \mathbb{KL} \big (\text{Dir}( \bpi \vert \boldsymbol{\alpha}(\boldsymbol{x}, \boldsymbol{\theta})) \vert\vert \text{Dir}( \bpi \vert \mathbf{1} ) \big ).
    \label{chap5:eq:loss_vi_lbda}
\end{equation}
In particular, minimizing loss \cref{chap5:eq:loss_vi_lbda} enforces training sample's precision $\alpha_0$ to remain close to the value $K+1/\lambda$ \cite{malinin2019}.

While $\mathcal{L}^{\lambda}_{\text{var}}(\btheta;\cD)$ slightly differs from $\cL_{\text{var}}(\btheta;\cD)$, both functions lead to the same optima. Indeed, by considering their gradient, we can show that a local optimum of $\cL_{\text{var}}(\btheta;\cD)$ is achieved for a sample $\x$ when $\balpha^* = (\beta_1,...,\beta_y + 1,...,\beta_K)$ and a local optimum of $\cL^{\lambda}_{\text{var}}(\btheta;\cD)$ is  $\balpha^{\bullet} = (1,\cdots,1 + 1/\lambda,\cdots,1)$. Hence, their ratio between each element is equal:
\begin{equation}
    \forall i,j \in \llbracket 1,K\rrbracket, \frac{\alpha^*_i}{\alpha^*_j} = \frac{\alpha^{\bullet}_i}{\alpha^{\bullet}_j}.
\end{equation}

\section{Capturing in-distribution and out-of-distribution uncertainties}
\label{chap5:sec:method}

In this section, we present the limits of current uncertainty measures used in evidential models (\cref{chap5:subsec:limits}) and we introduce our measure to effectively capture class confusion and lack of evidence with evidential models (\cref{chap5:subsec:klos}). We further propose a confidence learning approach to enhance in-distribution uncertainty estimation in \cref{chap5:subsec:klosnet}.

\subsection{Limits of current uncertainty measures with evidential models}
\label{chap5:subsec:limits}

For open-world recognition, a model should be equipped with an uncertainty measure which accounts both for first-order and second-order uncertainties to detect misclassifications and out-of-distribution samples. However, current uncertainty measures do not leverage the distribution over output probabilities on the simplex to derive such a joint measure of the two sources of uncertainty. The predictive entropy $\bbH[Y \vert \x, \cD]$ and the \textit{maximum class probability} (MCP) targeting total uncertainty actually reduce distributions of probabilities on the simplex to their expected value and compute first-order uncertainty measure from point-wise probabilities \cite{sensoy2018, beingbayesian2020}. This significantly limits the expressiveness of the resulting measures. In particular, they are invariant to the spread of the distribution over probabilities. This causes a significant loss of information. Given similar conflicting evidence, first-order uncertainty measures assign the same value to in-distribution and out-of-distribution samples as shown in \cref{chap5:fig:intro}, which is undesirable. With the goal of obtaining accurate estimates, measures should allow uncertainty caused by class confusion and lack of evidence to be cumulative, a property naturally fulfilled by the predictive variance in Bayesian regression \cite{murphy2012machine}.

Other uncertainty measures used with evidential models include second-order uncertainty measures which quantify the distribution dispersion on the simplex, \eg, precision $\alpha_0$ or mutual information. While adequate to detect OOD samples, these measures are not suited to estimate aleatoric uncertainty, which is characterized by the expectation of the Dirichlet distribution.

\begin{wrapfigure}{r}{0.35\textwidth}
    \centering
    \captionsetup{justification=centering}
    \vspace{-0.5cm}
    \includegraphics[width=0.95\linewidth]{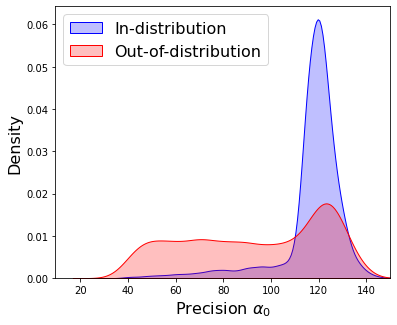}
    \caption{Precision densities for ID (CIFAR-10) and OOD (TinyImageNet) samples when no OOD training data is used.}
    \vspace{-0.5cm}
    \label{chap5:fig:intro_density_plot}
\end{wrapfigure}
In addition, the success of these measures rely on the assumption that the Dirichlet distribution spread is larger for OOD than for in-distribution (ID) samples. Consequently, some previous works \cite{malinin2018, malinin2019,maxgap2020} propose to use auxiliary OOD data during training to enforce higher distribution spread on OOD inputs, which may be unrealistic in many applications. This has been debated recently by \cite{charpentier2020}. Finally, this assumption is not always fulfilled in absence of OOD training data \cite{charpentier2020,sensoy2020}. As shown in \cref{chap5:fig:intro_density_plot} for a model trained on CIFAR-10, $\alpha_0$ values largely overlap between IDs and OODs when no OOD training data is used, limiting the effectiveness of existing second-order uncertainty measures. Consequently, neither current first-order nor second-order uncertainty measures appear to be suited for open-world settings.

\subsection{KLoS: a Kullback-Leibler divergence measure on the simplex}
\label{chap5:subsec:klos}

By explicitly learning a distribution of the categorical probabilities $\bpi$, evidential models provide a second-order uncertainty representation where the expectation of the Dirichlet distribution relates to class confusion and its spread to the amount of evidence. While originally used to measure the total uncertainty, the predictive entropy $\bbH[y \vert \x, \btheta]$ and the maximum class probability $\text{MCP}(\x, \btheta) = \max_k P(\mathbf{y}= k \vert \x, \btheta)$ only account for the position on the simplex. These measures are invariant to the dispersion of the Dirichlet distribution that generates the categorical probabilities. This can be problematic, as illustrated in \cref{chap5:fig:intro}. To capture uncertainties due to class confusion \emph{and} lack of evidence, an effective measure should account for the sharpness of the Dirichlet distribution and its location on the simplex.

We introduce a novel measure, named \emph{KLoS} for ``KL on Simplex'', that computes the KL divergence between the model's output and a class-wise prototype Dirichlet distribution with concentrations $\boldsymbol{\gamma}_{\hat{y}}$ focused on the \textit{predicted} class $\hat{y}$.

\begin{definition}[KLoS]
\sloppy
For any admissible input $\x\in\cX$, its KLoS measure is defined as:
\begin{equation}
    \textrm{KLoS}(\boldsymbol{x}) \triangleq \mathbb{KL} \Big ( \textrm{Dir} \big (\bpi \vert \balpha(\boldsymbol{x}, \btheta) \big ) ~\|~ \textrm{Dir} \big ( \bpi \vert \boldsymbol{\gamma}_{\hat{y}} \big ) \Big ) ,
    \label{chap5:eq:klos}
\end{equation}
where $\balpha(\boldsymbol{x}, \btheta) = \exp f(\boldsymbol{x}, \btheta)$ is the model's output and $\boldsymbol{\gamma}_{\hat{y}} = (1,\ldots,1,\tau,1,\ldots,1)$ are uniform concentration parameters except for the predicted class with concentration $\tau$.
\end{definition}

The lower KLoS is, the more certain the prediction is. Correct predictions will have Dirichlet distributions similar to the prototype Dirichlet distribution $\boldsymbol{\gamma}_{\hat{y}}$ and will thus be associated with a low uncertainty score (\cref{chap5:fig:simplex_behavior}a). Samples with high class confusion will present an expected probability distributions closer to simplex's center than the expected class-wise prototype
$p^*_{\hat{y}} = (\frac{1}{K-1+\tau},\cdots,\frac{\tau}{K-1+\tau},\ldots,\frac{1}{K-1+\tau})$, resulting in a higher KLoS score (\cref{chap5:fig:simplex_behavior}b). Similarly, KLoS also penalizes samples having a different precision $\alpha_0$ than the precision $\alpha_0^*=\tau +K-1$ of the prototype $\boldsymbol{\gamma}_{\hat{y}}$. Samples with smaller (\cref{chap5:fig:simplex_behavior}c) and higher (\cref{chap5:fig:simplex_behavior}d) amount of evidence than $\alpha_0^*$ receive a larger KLoS score.

\begin{figure}[t]
    \centering
    \includegraphics[width=\linewidth]{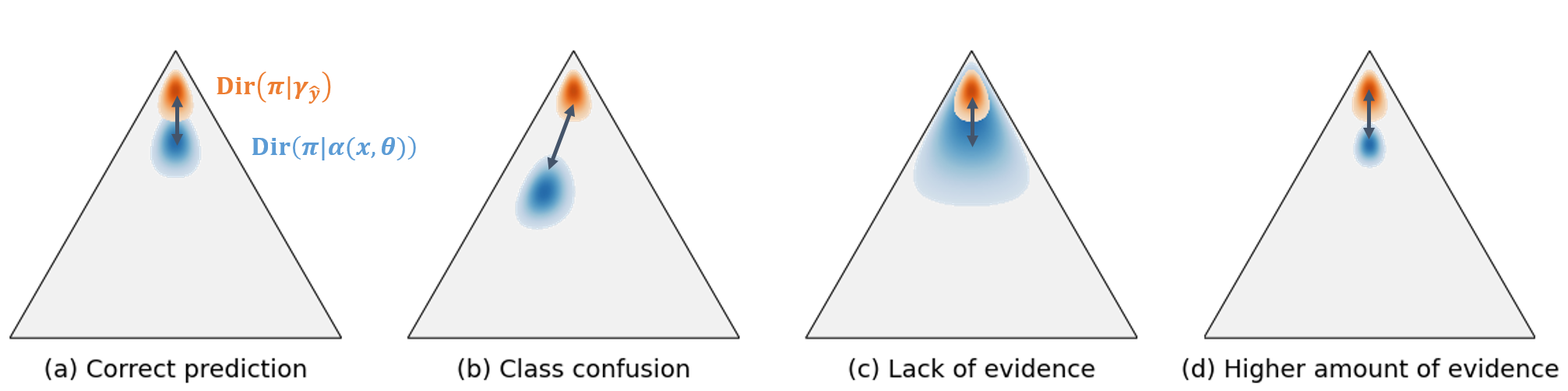}
    \caption[KLoS on the probability simplex]{\textbf{KLoS on the probability simplex}. Given the input sample, the blue region represents the distribution predicted by the evidential model and the orange region represents the prototype Dirichlet distribution with parameters $\boldsymbol{\gamma}_{\hat{y}} = (1,\cdots,1,\tau,1,\ldots,1)$ focused on the predicted class $\hat{y}$. Illustration of the behavior of KLoS in absence of uncertainty (a), in case of class confusion (b) and in case of a different amount of evidence, either lower (c) or higher (d).}
    \label{chap5:fig:simplex_behavior}
\end{figure}

\paragraph{Effective measure without OOD training data.} 
Since in-distribution samples are enforced to have precision close to $\alpha_0^*$ during training, the class-wise prototypes are fine estimates of the concentration parameters of training data for each class. Hence, KLoS is a divergence-based metric, which only needs in-distribution data during training to compute its prototypes. This behavior is illustrated in \cref{chap5:subsec:synth_exp}. The proposed measure will be effective to detect various types of OOD samples whose precision is far from $\alpha_0^*$. In contrast, second-order uncertainty measures, \eg, mutual information, assume that OOD samples have smaller $\alpha_0$, a property difficult to fulfill for models trained only with in-distribution samples (see \cref{chap5:fig:intro_density_plot}). In \cref{chap5:sec:ood_training}, we explore more in-depth the impact of the choice of OOD training data on the actual $\alpha_0$ values for OOD samples.

\paragraph{Decomposition of KLoS} 
Even though our method targets the simultaneous detection of misclassifications and OOD samples, one can detect the source of uncertainty in KLoS scores by using the following decomposition:

\begin{proposition}
By approximating the digamma function $\psi$, KLoS can be decomposed as:
\begin{equation}
    \textrm{KLoS}(\x) \approx - (\tau-1) \log \big ( \frac{\alpha_{\hat{y}}}{\alpha_0} \big ) + \Big (- (\tau-1)(\frac{1}{2\alpha_0} -\frac{1}{2\alpha_{\hat{y}}}) + \mathbb{KL} \big (\text{Dir}( \bpi \vert \balpha)~\|~\text{Dir}( \bpi \vert \mathbf{1} ) \big ) \Big ) + r,
    \label{chap5:eq:decomposition}
\end{equation}
where $r = - \big (\log \Gamma(\tau) - \log \Gamma(K-1+\tau) - \log \Gamma(K) \big )$ does not depend on the model parameters $\btheta$ nor on the input $\x$.
\end{proposition}

\begin{proof}
The KL divergence between two Dirichlet distributions can be obtained in closed form and KLoS can be calculated as:
\begin{align}
    \textrm{KLoS}(\boldsymbol{x}) &= \mathbb{KL} \Big ( \textrm{Dir} \big (\bpi \vert \boldsymbol{\alpha} \big ) ~\|~ \textrm{Dir} \big ( \bpi \vert \boldsymbol{\gamma}_{\hat{y}} \big ) \Big ) \\
    &= \log \Gamma(\alpha_0) - \log \Gamma(K-1+\tau) + \log \Gamma(\tau) - \sum_{k=1}^K \log \Gamma(\alpha_k) \nonumber \\
    &~~~~~~+ \sum_{k\neq y} \big ( \alpha_k -1 \big )  \big ( \psi(\alpha_k) - \psi(\alpha_0) \big ) + \big (\alpha_{\hat{y}} - \tau \big)\big ( \psi(\alpha_{\hat{y}}) - \psi(\alpha_0) \big ).
\end{align}
On the other hand, the KL divergence between the model's output and an uniform Dirichlet distribution $\text{Dir} \big ( \bpi \vert \boldsymbol{1} \big )$ reads:
\begin{equation}
   \textrm{KL} \Big ( \textrm{Dir} \big (\bpi \vert \boldsymbol{\alpha} \big ) ~\|~ \textrm{Dir} \big ( \bpi \vert \boldsymbol{1} \big ) \Big ) = \log \Gamma(\alpha_0) - \log \Gamma(K) - \sum_{k=1}^K \log \Gamma(\alpha_k) + \sum_{k = 1}^K \big ( \alpha_k -1 \big )  \big ( \psi(\alpha_k) - \psi(\alpha_0) \big ).
\end{equation}
Hence, KLoS can be written as:
\begin{align}
    \textrm{KLoS}(\boldsymbol{x}) &= - (\tau - 1) \big ( \psi(\alpha_{\hat{y}}) - \psi(\alpha_0) \big ) + \mathbb{KL} \big (\text{Dir}( \bpi \vert \boldsymbol{\alpha})~\|~\text{Dir}( \bpi \vert \mathbf{1} ) \big ) \nonumber \\
    &~~~~~~~~+ \big ( \log \Gamma(\tau) - \log \Gamma(K-1+\tau) - \log \Gamma(K) \big ).
    \label{chap5:eq:proof_decompo}
\end{align}

By considering the asymptotic series approximation to the digamma function, $\psi(x) = \log x - \frac{1}{2x} + \mathcal{O}(\frac{1}{x^2})$, the previous expression can be approximated by:
\begin{equation}
    \textrm{KLoS}(\boldsymbol{x}) \approx - (\tau - 1) \log (\frac{\alpha_{\hat{y}}}{\alpha_0}) + \Big (- (\tau - 1) (\frac{1}{2\alpha_0} - \frac{1}{2\alpha_{\hat{y}}})) + \mathbb{KL} \big (\text{Dir}( \bpi \vert \boldsymbol{\alpha})~\|~\text{Dir}( \bpi \vert \mathbf{1} ) \big ) \Big ) + r,
\end{equation}
where $r = - \big (\log \Gamma(\tau) - \log \Gamma(K-1+\tau) - \log \Gamma(K) \big )$.
\end{proof}

The first term is the standard log-likelihood and relates only to expected probabilities, hence to the class confusion. The ratio $\alpha_{\hat{y}} / \alpha_0$ makes it invariant to any scaling of the concentration parameters vector $\balpha$. The second term takes into account the spread of the distribution by measuring how close $\alpha_0$ is to $(\tau+K-1)$, and measures the amount of evidence.

\subsection{Improving uncertainty estimation with confidence learning}
\label{chap5:subsec:klosnet}

\begin{wrapfigure}{r}{0.45\textwidth}
\centering
    \captionsetup{justification=centering}

    \includegraphics[width=\linewidth]{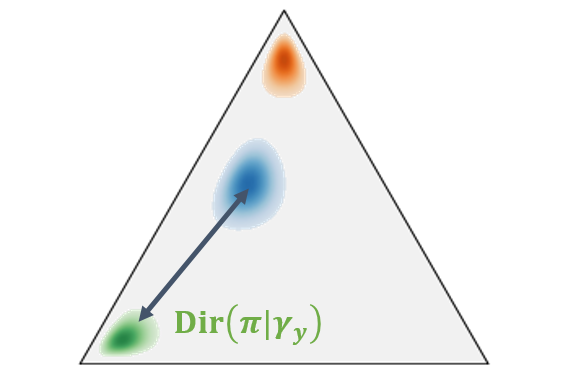}
    \vspace{-0.5cm}
    \caption[Behavior of KLoS{$^*$}]{KLoS{$^*$} measures the distance to the prototype Dirichlet distribution centered on the true class $y$ (green).}
    \label{chap5:fig:klos*}
\end{wrapfigure}
When the model misclassifies an example, \ie, the predicted class $\hat{y}$ differs from the ground truth $y$, KLoS measures the distance between the ENN's output and the wrongly estimated posterior $p(\bpi \vert \x, \hat{y})$. This may result in an arbitrarily high confidence / low KL divergence value. Measuring instead the distance to the true posterior distribution $p(\pi \vert \x, y)$ (\cref{chap5:fig:klos*}) more likely yield a greater value, reflecting the fact that the classifier made an error. 

Thus, a better measure for misclassification detection would be: 
\begin{equation}
    \textrm{KLoS}^*(\x, y) \triangleq \mathbb{KL} \Big ( \textrm{Dir} \big (\bpi \vert \balpha(\x, \btheta) \big ) ~\|~ \textrm{Dir} \big ( \bpi \vert \boldsymbol{\gamma}_{y} \big ) \Big ),
    \label{chap5:eq:klos*}
\end{equation}
where $\boldsymbol{\gamma}_{y}$ corresponds to the uniform concentrations except for the \emph{true} class $y$ with $\tau ~{=}~1+\lambda^{-1}$.

\paragraph{Connecting $\textrm{KLoS}^*$ with Evidential Training Objective} Interestingly, the following proposition that choosing such value for $\tau$ results in $\textrm{KLoS}^*$ matching the objective function in \cref{chap5:eq:loss_vi_lbda}.

\begin{proposition}
    If $\tau ~{=}~1+\lambda^{-1}$, then minimizing the evidential training objective $\mathcal{L}^{\lambda}_{\text{var}}(\btheta;\cD)$ is equivalent to minimizing the $\textrm{KLoS}^*$
value of each training point $\x$.
\end{proposition}

\begin{proof}
    Let us decompose $\mathcal{L}^{\lambda}_{\text{var}}(\btheta;\cD) = \frac{1}{N} \sum_{(\x, y) \in \cD} l^{\lambda}_{\text{var}}(\x, y, \btheta)$.
    
    By deriving $\textrm{KLoS}^*$ in a similar way than \cref{chap5:eq:proof_decompo}, we can observe that:
\begin{equation}
     \textrm{KLoS}^*(\boldsymbol{x}) = l^{\lambda}_{\text{var}}(\x, y, \btheta) + r,
\end{equation}
where $r = - \big (\log \Gamma(1+1/\lambda) - \log \Gamma(K-1+1/\lambda) - \log \Gamma(K) \big )$ does not depend on the model parameters $\btheta$.

Hence, minimizing the evidential training objective $\mathcal{L}^{\lambda}_{\text{var}}(\btheta;\cD)$ is equivalent to minimizing the $\textrm{KLoS}^*$ value of each training point $\x$.
\end{proof}

This means that $\textrm{KLoS}^*$ is explicitly minimized by the evidential model during training for in-distribution samples. By mimicking the evidential training objective, we reflect the fact that the model is confident about its prediction if $\textrm{KLoS}^*$ is close to zero. In addition, minimizing the KL divergence between the variational distribution $q_{\btheta}(\bpi \vert \x) $ and the posterior $p(\bpi \vert \x, y)$ is equivalent to maximizing the evidence lower bound (ELBO) \cite{murphy2012machine}. Hence, a small $\textrm{KLoS}^*$ value corresponds to a high ELBO, which
is coherent with the common assumption in variational inference that higher ELBO corresponds to ``better'' models \cite{Gal2016PhD}.

Obviously, the true class of an output is not available when estimating confidence on test samples. We propose to learn $\textrm{KLoS}^*$ by introducing an auxiliary confidence neural network, \emph{KLoSNet}, with parameters $\bomega$, which outputs a confidence prediction $C(\x, \bomega)$. KLoSNet consists in a small decoder, composed of several dense layers attached to the penultimate layer of the original classification network. During training, we seek $\bomega$ such that $C(\x, \bomega)$ is close to $\text{KLoS}^*(\x,y)$, by minimizing
\begin{equation} 
\cL_{\text{KLoSNet}}(\bomega; \cD) = \frac{1}{N} \sum_{(\x,y)\in\mathcal{D}} \big\|  C(\x,\bomega) - \textrm{KLoS}^*(\x, y) \big\|^2.
    \label{chap5:eq:loss-conf}
\end{equation}
KLoSNet can be further improved by endowing it with its own feature extractor. Initialized with the encoder of the classification network, which must remain untouched for not affecting its performance, the encoder of KLoSNet can be fine-tuned along with its regression head. This amounts to minimizing \cref{chap5:eq:loss-conf} with respect to both sets of parameters.

The training set for confidence learning is the one used for classification training. In the experiments, we observe a slight performance drop when using a validation set instead. Indeed, when dealing with models with high predictive performance and small validation sets, we end up with fewer misclassification examples than in the train set. At test time, we now directly use KLoSNet's scalar output $C(\x, \bomega')$ as our uncertainty estimate. As previously, the lower the output value, the more confident the prediction.

\section{Related work}
\label{chap5:sec:related_work}

We detail here related work on OOD detection used in the following experiments. For misclassification detection, we refer to the related work in \cref{chap3:sec:related_work} and the presentation of the task in the background chapter (\cref{chap2:subsec:misclassif}).

In the literature, a range of methods aim to detect anomalies in the form of out-of-distribution (OOD) samples. Applied on a pre-trained model, ODIN \cite{odin2018} mitigates over-confidence by post-processing logits with temperature scaling and by adding inverse adversarial perturbations. \cite{mahalanobis2018} proposes a confidence score based on the class-conditional Mahalanobis distance, with the assumption of tied covariance. Although effective, both approaches need OOD data to tune hyperparameters, which might not generalize to other OOD datasets \cite{Shafaei2019ALB}. Finally, Liu \textit{et al.} \cite{NEURIPS2020_f5496252} interpret a pre-trained NN as an energy-based model and compute the energy score to detect OOD samples. Interestingly, this score corresponds to the log precision $\log \alpha_0$, which is similar to the EPKL measure \cite{malinin2019} used in ENNs.

\section{Experiments}
\label{chap5:sec:exp}

We evaluate our approach against: first-order uncertainty metrics (Maximum Class probability (\textit{MCP}) and predictive entropy (\textit{Entropy})), second-order metrics (mutual information (\textit{Mut. Inf.}), differential entropy (\textit{Diff. Ent.}), expected pairwise KL divergence (\textit{EPKL}) and \textit{dissonance}), post-training methods for OOD detection (\textit{ODIN} and \textit{Mahalanobis}) and for misclassification detection (\textit{ConfidNet}). Except in \cref{chap5:sec:ood_training}, we consider setups where no OOD data is available for training. Consequently, the results reported for ODIN and Mahalanobis are obtained without adversarial perturbations, which is also the best configuration for the considered tasks. We indeed show in \cref{appxB:subsec:adv_perturb} that these perturbations degrade misclassification detection.

\subsection{Synthetic experiment}
\label{chap5:subsec:synth_exp}

\begin{figure}[t]
\centering
\begin{minipage}[c]{0.32\linewidth}
\centering
    \includegraphics[width=\linewidth]{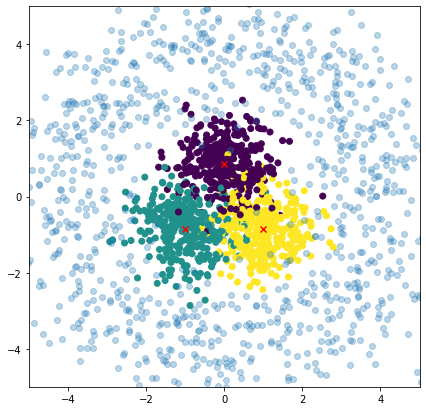}
    \subcaption{Toy dataset}
    \label{chap5:fig:toy_dataset}
\end{minipage}%
\hspace{0.1cm}
\begin{minipage}[c]{0.32\linewidth}
\centering
    \includegraphics[width=\linewidth]{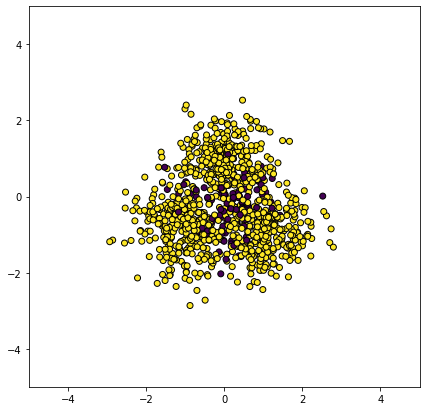}
    \subcaption{Corr./Err.}
    \label{chap5:fig:toy_errors}
\end{minipage}%
\hspace{0.1cm}
\begin{minipage}[c]{0.33\linewidth}
\centering
    \includegraphics[width=\linewidth]{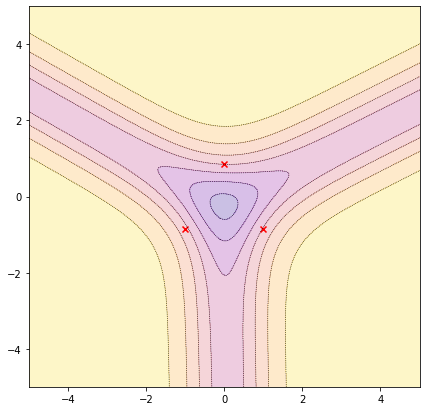}
    \subcaption{Entropy}
    \label{chap5:fig:toy_entropy}
\end{minipage}%
\hspace{0.1cm}
\begin{minipage}{0.32\linewidth}
\centering
    \includegraphics[width=\linewidth]{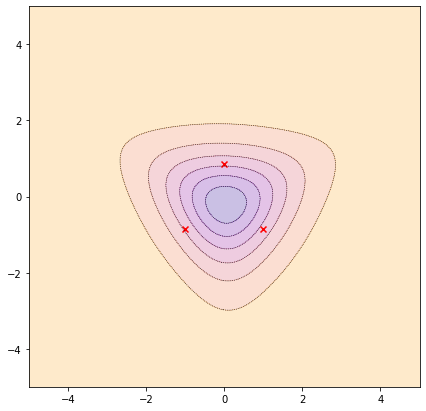}
    \subcaption{Mut. Inf.}
    \label{chap5:fig:toy_mutinf}
\end{minipage}
\hspace{0.1cm}
\begin{minipage}{0.32\linewidth}
\centering
    \includegraphics[width=\linewidth]{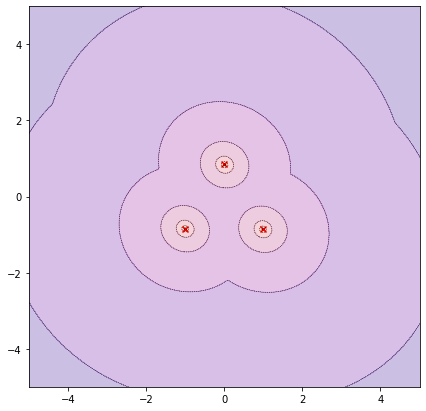}
    \subcaption{Mahalanobis}
    \label{chap5:fig:toy_mahalanobis}
\end{minipage}
\begin{minipage}{0.32\linewidth}
\centering
    \includegraphics[width=\linewidth]{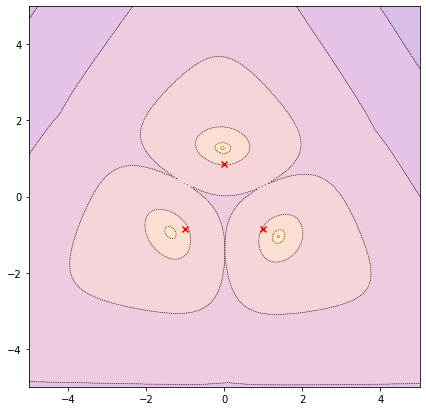}
    \subcaption{KLoS}
    \label{chap5:fig:toy_klos}
\end{minipage}
\caption[Comparison of various uncertainty measures for a given evidential classifier on a toy dataset]{\textbf{Comparison of various uncertainty measures for a given evidential classifier on a toy dataset.} (a) Training samples from 3 input Gaussian distributions with large overlap (hence class confusion) and OOD test samples (blue); (b) Correct (yellow) and erroneous (red) class predictions on in-domain test samples; (c-f) Visualisation of different uncertainty measures derived from the evidential model trained on the toy dataset. Yellow (resp. purple) indicates high (resp. low) certainty.}
\label{chap5:fig:vizu_toy}
\end{figure}

We analyse the behavior of the KLoS measure and the limitations of existing first- and second-order uncertainty metrics on a 2D synthetic dataset composed of three Gaussian-distributed classes with equidistant means and identical isotropic variance (\cref{chap5:fig:vizu_toy}):
\begin{equation}
    p(X = \boldsymbol{x},Y=y) = \frac{1}{3} \cdot \mathcal{N}(X = \boldsymbol{x} ~\vert~ \boldsymbol{\mu}_y, \sigma^2 \mathrm{I}_{2\!\times\!2}),
\end{equation}
where $\boldsymbol{\mu}_1 = (0, \sqrt{3}/2)$, $\boldsymbol{\mu}_2 = (-1, -\sqrt{3}/2)$ , $\boldsymbol{\mu}_3 = (1, -\sqrt{3}/2)$ and $\sigma=4$.
The marginal distribution of $\mathbf{x}$ is a Gaussian mixture with three equally weighted components having equidistant centers and equal spherical covariance matrices. The test dataset consists of 1,000 other samples from this distribution. Finally, we construct an out-of-distribution (OOD) dataset following \cite{malinin2019}, by sampling 100 points in $\mathbb{R}^2$ such that they form a `ring' with large noise around the training points. Some OOD samples will be close to the in-distribution while others will be very far (see \cref{chap5:fig:vizu_toy}). The number of OOD samples has been chosen so that it amounts approximately to the number of test points misclassified by the classifier. The classification is performed by a simple logistic regression. A set of five models is trained for 200 epochs using the evidential training objective with regularization parameter $\lambda=5\text{e-}2$ and Adam optimizer with learning rate $0.02$. Uncertainty metrics -- MCP, Entropy, Mut.\,Inf., Malahanobis and KLoS -- are computed from these models. This constitutes a scenario with high first-order uncertainty due to class overlap. OOD samples are drawn from a ring around the in-distribution dataset and are only used for evaluation.

\cref{chap5:fig:toy_entropy} shows that Entropy correctly assigns large uncertainty along decision boundaries, which is convenient to detect misclassifications, but yields low uncertainty for points far from the distribution. Mut.\,Inf. (\cref{chap5:fig:toy_mutinf}) has the opposite behavior than desired by decreasing when moving away from the training data. This is due to the linear nature of the toy dataset where models assign higher concentration parameters far from decision boundaries, hence smaller spread on the simplex, as also noted in \cite{charpentier2020}. Additionally, Mut.\,Inf. does not reflect the uncertainty caused by class confusion along decision boundaries. Neither Entropy nor Mut.\,Inf. is suitable to detect OOD samples in this synthetic experiment. In contrast, KLoS allows discriminating both misclassifications and OOD samples from correct predictions as uncertainty increases far from in-distribution samples for each class (\cref{chap5:fig:toy_klos}). KLoS measures a distance between the model's output and a class-wise prototype distribution. Here, we can observe that it acts as a divergence-based measure for each class.

We extend the comparison to include Mahalanobis (\cref{chap5:fig:toy_mahalanobis}), which is a distance-based measure by assuming Gaussian class conditionals on latent representations, here in the input space. However, Mahalanobis does not discriminate points close to the decision boundaries from points with a similar distance to the origin. Hence, it may be less suited to detect misclassifications than KLoS. Additionally, KLoS does not assume Gaussian distributions in the latent space nor tied covariance, which may be a strong assumption when dealing with a high-dimension latent space. In \cref{appxB:subsec:detail_results_synth}, a complementary quantitative evaluation on this toy problem confirms our findings regarding the inadequacy of first-order uncertainty measures such as MCP and Entropy, and the improvement provided by KLoS over Mahalanobis on misclassification detection.

\paragraph{Decomposition of KLoS.} To gain further intuition about the decomposition, we provide illustrations of the first term (negative log-likelihood, NLL) and the second term in \cref{chap5:fig:toy_illu}. We observe that the NLL term, which is equivalent to MCP measure, helps to detect misclassifications while the second term denotes increasing uncertainty as we move far away from training data. Hence, by using either the NLL term or the second term, one could distinguish the source of uncertainty if needed. 

\begin{figure}[h]
\centering
\begin{minipage}{0.32\linewidth}
\centering
    \includegraphics[width=\linewidth]{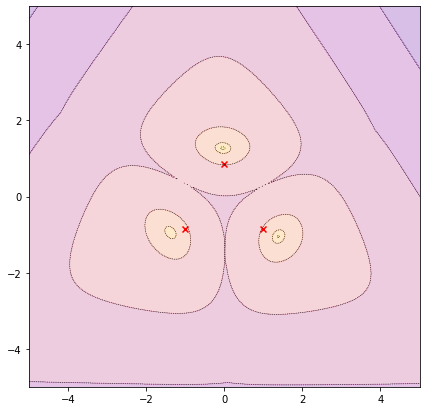}
    \subcaption{KLoS}
\end{minipage}
\hspace{0.1cm}
\begin{minipage}{0.32\linewidth}
\centering
    \includegraphics[width=\linewidth]{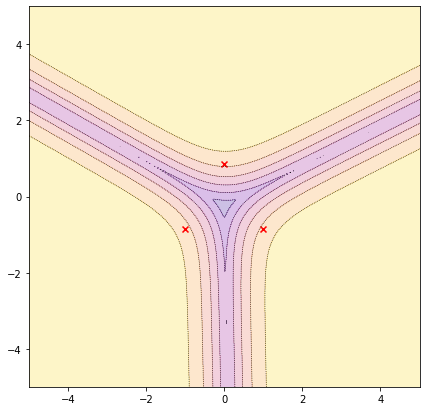}
    \subcaption{NLL term}
\end{minipage}
\hspace{0.1cm}
\begin{minipage}{0.32\linewidth}
\centering
    \includegraphics[width=\linewidth]{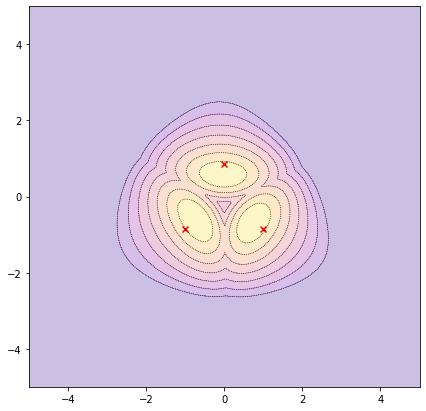}
    \subcaption{Second term}
\end{minipage}
\caption[Visualisation of the decomposition of KLoS on the toy dataset]{\textbf{Visualisation of the decomposition of KLoS on the toy dataset.}}
\label{chap5:fig:toy_illu}
\end{figure}

\subsection[Simultaneous detection of errors and OOD samples]{Simultaneous detection of errors and out-of-distribution samples}
\label{chap5:subsec:simultaneous_exp}

\begin{figure}[t]
\centering
\begin{minipage}[c]{0.24\linewidth}
\centering
    \includegraphics[width=\linewidth]{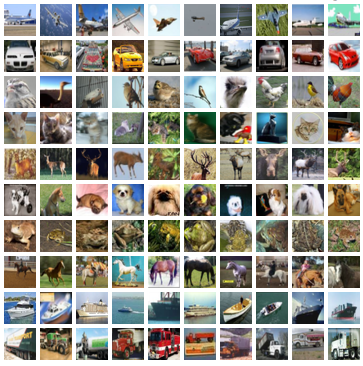}
    \subcaption{CIFAR-10}
    \label{chap5:fig:cifar10}
\end{minipage}%
\hspace{0.1cm}
\begin{minipage}[c]{0.24\linewidth}
\centering
    \includegraphics[width=\linewidth]{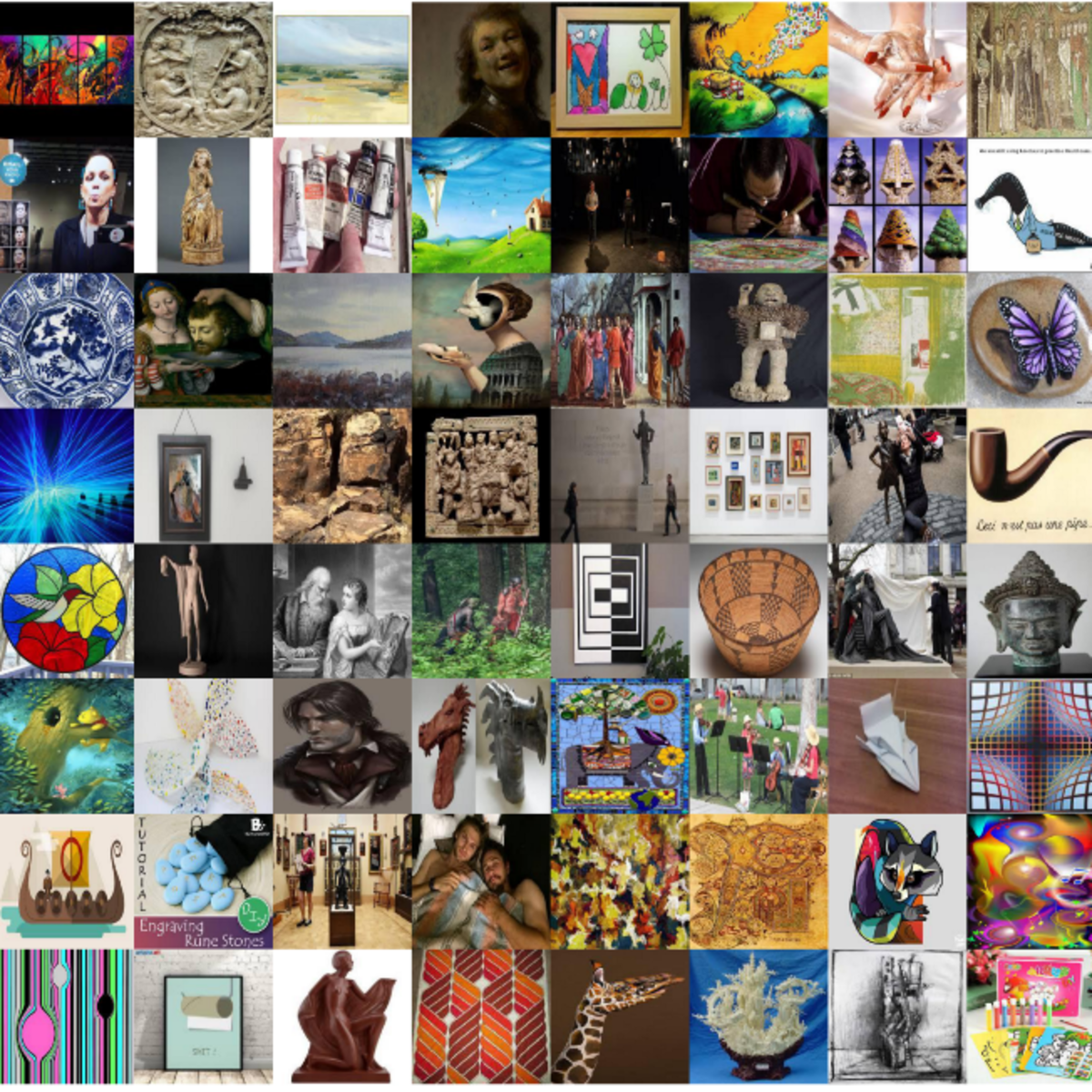}
    \subcaption{TinyImageNet}
    \label{chap5:fig:tinyimagenet}
\end{minipage}%
\hspace{0.1cm}
\begin{minipage}[c]{0.24\linewidth}
\centering
    \includegraphics[width=\linewidth]{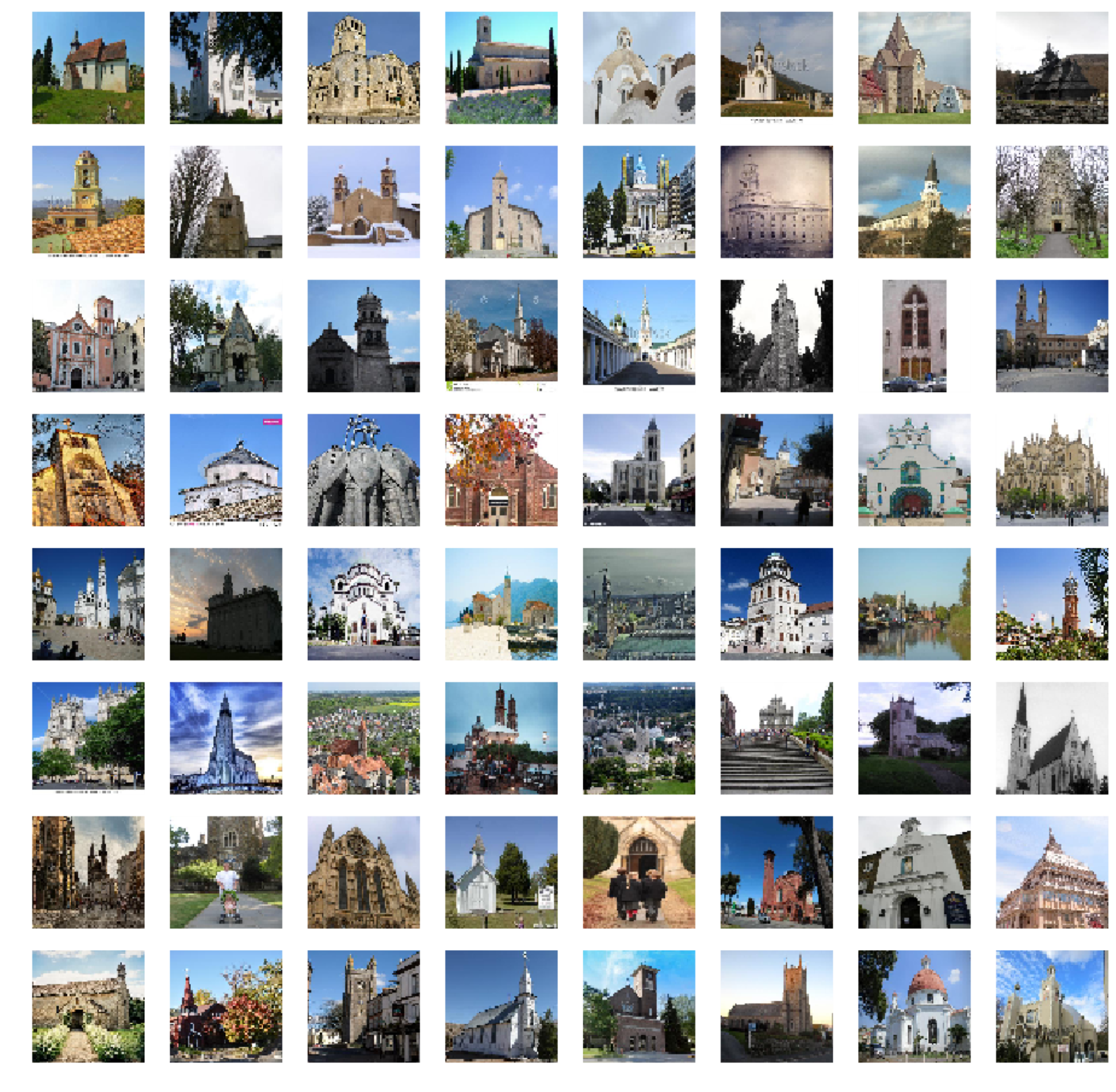}
    \subcaption{LSUN}
    \label{chap5:fig:lsun}
\end{minipage}%
\hspace{0.1cm}
\begin{minipage}{0.24\linewidth}
\centering
    \includegraphics[width=\linewidth]{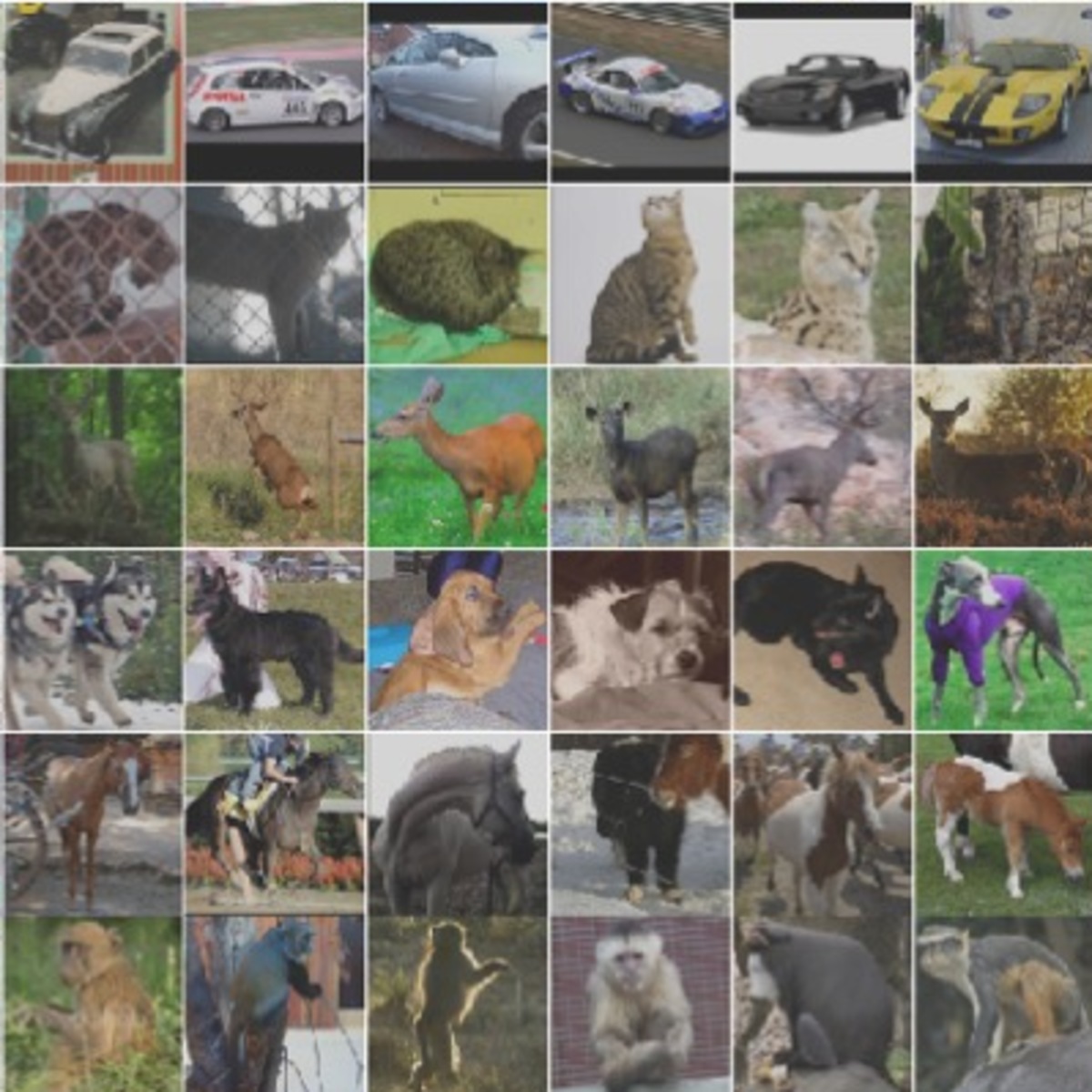}
    \subcaption{STL10}
    \label{chap5:fig:stl10}
\end{minipage}
\caption[Image samples from OOD datasets used in experiments]{\textbf{Images samples from OOD datasets} (b,c,d) used in experiments and compared to in-distribution CIFAR-10 (a).}
\label{chap5:fig:OOD_datasets}
\end{figure}

The task of detecting both in-distribution misclassifications and OOD samples gives the opportunity to jointly evaluate in-distribution and out-of-distribution uncertainty representations of a method. In this binary classification problem, correct predictions are considered as positive samples while misclassified inputs and OOD examples constitute negative samples. Following standard practices \cite{hendrycks17baseline}, we use the area under the ROC curve (AUROC) to evaluate threshold-independent performance.

The models used in the experiments present high predictive performances. Most often, there are much fewer misclassifications in the test set than considered OOD samples. Hence, joint detection performances might be dominated by the evaluation of the quality of OOD detection. To mitigate this unbalance, we propose to consider the following scheme based on oversampling. Let $\mathcal{A}_{\mathrm{M}}$ be the subset of in-distribution test examples that are misclassified by the observed model and $\mathcal{A}_{\mathrm{O}}$ the set of OOD test samples. We randomly sample $\kappa|\mathcal{A}_{\mathrm{O}}|$ points in $\mathcal{A}_{\mathrm{M}}$, with $\kappa$ = 1. Supposing $|\mathcal{A}_{\mathrm{O}}| \geq |\mathcal{A}_{\mathrm{M}}|$, this corresponds to oversampling the set of misclassifications. This over-sampled set is then added to the OOD set to form the negative examples for detection training. The set of correct predictions remains the same. We observed that the variance in AUROC due to this sampling is negligible and we report only the mean hereafter.

In the following, uncertainty measures are derived from an evidential model (\cref{chap5:eq:loss_vi_lbda}) with $\lambda =10^{-2}$, except for second-order metrics where we found that setting $\lambda=10^{-3}$ improves performance. We rely on the learned classifier to train our auxiliary confidence model KLoSNet, using the same training set and following loss \cref{chap5:eq:loss-conf}. Experiments are conducted with VGG-16 \cite{Simonyan15} and ResNet-18 \cite{resnet2015} architectures on CIFAR-10 (\cref{chap5:fig:cifar10}) and CIFAR-100 datasets \cite{Krizhevsky09}. The OOD datasets used for evaluation are presented in \cref{chap5:fig:OOD_datasets}: TinyImageNet\footnote{https://tiny-imagenet.herokuapp.com/} -- a subset of ImageNet (10,000 test images with 200 classes) --, LSUN \cite{yu15lsun} -- a scene classification dataset (10,000 test images of 10 scenes) --, STL-10 -- a dataset similar to CIFAR-10 but with different classes, and SVHN \cite{svhn-dataset} -- an RGB dataset of $28\!\times\!28$ house-number images (73,257 training and 26,032 test images with 10 digits) --. We downsample each image of TinyImageNet, LSUN and STL-10 to size $32\!\times\!32$. All details about architectures, training algorithms and datasets are available in \cref{appxB}.

Along with simultaneous detection results, we provide separate results for misclassifications detection and OOD detection respectively in \cref{chap5:tab:comparative_experiment}. On OOD detection, Mahalanobis and KLoSNet outperform other methods, including second-order measures. ODIN also fails to deliver here as logits are small due to regularization in the evidential training objective. Mut. Inf. and other spread-based second-order uncertainty measures fall short to detect correctly OOD. Indeed, for settings where OOD training data is not available, there is no guarantee that every OOD sample will result in lower predicted concentration parameters as previously shown by the density plot of precision $\alpha_0$ in \cref{chap5:fig:intro_density_plot}. This stresses the importance of class-wise divergence-based measure.

While Mahalanobis may sometimes be slightly better than KLoSNet for OOD detection, it performs significantly less well in misclassification detection, in line with the behavior shown in synthetic experiments. As a result, KLoSNet appears to be the best measure in every simultaneous detection benchmark. For instance, for CIFAR-10/STL-10 with VGG-16, KLoSNet achieves $81.8\%$ AUROC while the second best, Mahalanobis, scores $78.8\%$.

\begin{table}[ht]
    \centering
    \caption[Results of simultaneous detection on CIFAR-10 and CIFAR-100]{\textbf{Comparative experiments on CIFAR-10 and CIFAR-100.} Misclassification (Mis.), out-of-distribution (OOD) and simultaneous (Mis+OOD) detection results (mean \% AUROC and std. over 5 runs). Bold type indicates significantly best performance ($p{<}0.05$) according to paired t-test.}
	\resizebox{0.95\linewidth}{!}{%
    	\begin{tabular}{c l|c|cc|cc|cc}
    		\toprule
   		    & & & \multicolumn{2}{c|}{\textbf{LSUN}} & \multicolumn{2}{c|}{\textbf{TinyImageNet}} & \multicolumn{2}{c}{\textbf{STL-10}} \\
    		& \textbf{Method} & Mis. & OOD  & Mis+OOD & OOD  & Mis+OOD & OOD  & Mis+OOD \\
    		\midrule
    		\multirow{12}{*}{\rotatebox{90}{\shortstack[c]{\ubold{CIFAR-10} \\ VGG-16}}} & MCP & 87.6 {\small $\pm 1.6$} & 79.7 {\small $\pm 1.1$} & 84.9 {\small $\pm 1.1$} & 80.3 {\small $\pm 1.5$} & 85.2 {\small $\pm 1.5$} & 60.3 {\small $\pm 1.2$} & 75.2 {\small $\pm 1.4$} \\
    		& Entropy & 83.5 {\small $\pm 2.4$} & 83.8 {\small $\pm 0.3$} & 87.9 {\small $\pm 0.2$} & 82.3 {\small $\pm 0.4$} & 87.2 {\small $\pm 0.4$} & 60.1 {\small $\pm 1.2$} & 75.0 {\small $\pm 1.4$} \\
    		& ConfidNet & 90.2 {\small $\pm 0.8$} & 82.1 {\small $\pm 1.5$} & 87.6 {\small $\pm 1.1$} & 83.5 {\small $\pm 0.6$} & 88.3 {\small $\pm 0.7$} & 61.5 {\small $\pm 1.6$} & 77.2 {\small $\pm 1.1$} \\
    		& Dissonance & 91.9 {\small $\pm 0.2$} & 84.8 {\small $\pm 0.3$} & 90.1 {\small $\pm 0.1$} & 84.2 {\small $\pm 0.2$} & 89.7 {\small $\pm 0.1$} & 64.1 {\small $\pm 0.1$} & 79.6 {\small $\pm 0.1$}\\
    		& Mut. Inf. & 84.1 {\small $\pm 1.5$} & 84.6 {\small $\pm 0.6$} & 85.1 {\small $\pm 1.0$} & 80.6 {\small $\pm 0.8$} & 83.4 {\small $\pm 1.1$} & 61.3 {\small $\pm 0.8$} & 65.0 {\small $\pm 2.5$} \\
    		& Diff. Ent. & 86.8 {\small $\pm 1.0$} & 85.6 {\small $\pm 0.5$} & 87.2 {\small $\pm 0.7$} & 82.7 {\small $\pm 0.7$} & 85.8 {\small $\pm 0.8$} & 62.0 {\small $\pm 1.0$} & 75.4 {\small $\pm 1.3$}\\
    		& EPKL & 83.9 {\small $\pm 1.5$} & 84.5 {\small $\pm 0.7$} & 85.1 {\small $\pm 1.0$} & 80.4 {\small $\pm 0.8$} & 83.2 {\small $\pm 1.2$} & 61.3 {\small $\pm 0.8$} & 73.8 {\small $\pm 1.1$}\\
    		& Diss.+Mut. Inf. & 92.0 {\small $\pm 0.2$} & 86.5 {\small $\pm 0.3$} & 89.8 {\small $\pm 0.2$} & 83.6 {\small $\pm 0.3$} & 89.5 {\small $\pm 0.3$} & 63.6 {\small $\pm 0.5$} & 79.4 {\small $\pm 0.4$} \\
    		& ODIN & 86.0 {\small $\pm 2.0$} & 79.5 {\small $\pm 1.2$} & 83.8 {\small $\pm 1.5$} & 79.6 {\small $\pm 1.9$} & 84.0 {\small $\pm 2.0$} & 54.7 {\small $\pm 1.5$} & 65.0 {\small $\pm 2.6$} \\
    		& Mahalanobis & 91.2 {\small $\pm 0.3$} & \ubold{88.9} {\small $\pm 0.2$} & \ubold{91.3} {\small $\pm 0.1$} & \ubold{86.4} {\small $\pm 0.2$} & 90.2 {\small $\pm 0.1$} & 63.4 {\small $\pm 0.2$} & 78.8 {\small $\pm 0.3$} \\
    		&\cellcolor[gray]{.92}KLoSNet (Ours) &\cellcolor[gray]{.92}\ubold{92.5} {\small $\pm 0.6$} &\cellcolor[gray]{.92}87.6 {\small $\pm 0.9$} &\cellcolor[gray]{.92}\ubold{91.7} {\small $\pm 0.9$} &\cellcolor[gray]{.92}\ubold{86.6} {\small $\pm 0.9$} &\cellcolor[gray]{.92}\ubold{91.2} {\small $\pm 0.8$} &\cellcolor[gray]{.92}\ubold{67.7} {\small $\pm 1.4$} &\cellcolor[gray]{.92}\ubold{81.8} {\small $\pm 0.9$} \\

    		\midrule

    		\multirow{12}{*}{\rotatebox{90}{\shortstack[c]{\ubold{CIFAR-10} \\ ResNet-18}}} & MCP & 84.9 {\small $\pm 0.8$} & 79.6 {\small $\pm 1.0$} & 83.0 {\small $\pm 0.9$} & 77.2 {\small $\pm 0.7$} & 81.8 {\small $\pm 0.7$} & 58.5 {\small $\pm 1.2$} & 72.5 {\small $\pm 0.4$} \\
    		& Entropy & 84.6 {\small $\pm 0.8$} & 79.6 {\small $\pm 1.1$} & 82.8 {\small $\pm 0.9$} & 77.2 {\small $\pm 0.7$} & 81.6 {\small $\pm 0.7$} & 58.4 {\small $\pm 1.2$} & 72.2 {\small $\pm 0.4$} \\
    		& ConfidNet & 90.7 {\small $\pm 0.4$} & 84.6 {\small $\pm 1.1$} & 88.6 {\small $\pm 0.6$} & 83.5 {\small $\pm 1.1$} & 88.0 {\small $\pm 0.6$} & 63.2 {\small $\pm 1.2$} & 77.9 {\small $\pm 0.5$} \\
    		& Dissonance & 92.9 {\small $\pm 0.4$} & 90.3 {\small $\pm 0.4$} & 92.7 {\small $\pm 0.4$} & 87.7 {\small $\pm 0.3$} & 91.4 {\small $\pm 0.3$} & 67.3 {\small $\pm 0.5$} & 81.2 {\small $\pm 0.4$} \\
    		& Mut. Inf & 80.6 {\small $\pm 0.6$} & 77.0 {\small $\pm 1.2$} & 79.4 {\small $\pm 0.9$} & 74.3 {\small $\pm 0.8$} & 78.0 {\small $\pm 0.7$} & 56.4 {\small $\pm 1.0$} & 69.1 {\small $\pm 0.2$} \\
    		& Diff. Ent & 82.7 {\small $\pm 0.6$} & 78.3 {\small $\pm 1.2$} & 81.1 {\small $\pm 0.9$} & 75.9 {\small $\pm 0.8$} & 79.9 {\small $\pm 0.7$} & 57.5 {\small $\pm 1.1$} & 70.8 {\small $\pm 0.3$} \\
    		& EPKL & 80.2 {\small $\pm 0.6$} & 76.8 {\small $\pm 1.3$} & 79.0 {\small $\pm 0.9$} & 74.1 {\small $\pm 0.8$} & 77.7 {\small $\pm 0.7$} & 56.2 {\small $\pm 1.0$} & 68.9 {\small $\pm 0.3$} \\
    		& Diss.+Mut. Inf. & 92.4 {\small $\pm 0.5$} & 86.7 {\small $\pm 1.0$} & 90.1 {\small $\pm 0.8$} & 84.3 {\small $\pm 0.5$} & 88.8 {\small $\pm 0.6$} & 65.2 {\small $\pm 0.7$} & 80.3 {\small $\pm 0.4$} \\
    		& ODIN & 83.7 {\small $\pm 0.7$} & 78.9 {\small $\pm 1.0$} & 81.9 {\small $\pm 0.9$} & 76.5 {\small $\pm 0.7$} & 80.7 {\small $\pm 0.7$} & 57.9 {\small $\pm 1.2$} & 71.5 {\small $\pm 0.4$} \\
    		& Mahalanobis & 91.2 {\small $\pm 0.4$} & 90.7 {\small $\pm 0.4$} & 91.8 {\small $\pm 0.3$} & 87.6 {\small $\pm 0.4$} & 90.3 {\small $\pm 0.4$} & 66.8 {\small $\pm 0.5$} & 80.0 {\small $\pm 0.3$} \\
    		&\cellcolor[gray]{.92}KLoSNet (Ours) &\cellcolor[gray]{.92}\ubold{93.9} {\small $\pm 0.4$} &\cellcolor[gray]{.92}\ubold{93.1} {\small $\pm 1.1$} &\cellcolor[gray]{.92}\ubold{94.4} {\small $\pm 0.3$} &\cellcolor[gray]{.92}\ubold{90.6} {\small $\pm 0.6$} &\cellcolor[gray]{.92}\ubold{93.2} {\small $\pm 0.2$} &\cellcolor[gray]{.92}\ubold{68.5} {\small $\pm 0.3$} &\cellcolor[gray]{.92}\ubold{82.3} {\small $\pm 0.2$} \\
    		
    		\midrule
    		
    		\multirow{12}{*}{\rotatebox{90}{\shortstack[c]{\ubold{CIFAR-100} \\ VGG-16}}} & MCP & 82.9 {\small $\pm 0.8$} & 62.8 {\small $\pm 1.3$} & 77.6 {\small $\pm 0.9$} & 72.0 {\small $\pm 0.5$} & 81.8 {\small $\pm 0.7$} & 69.7 {\small $\pm 0.7$} & 80.9 {\small $\pm 0.7$} \\
    		& Entropy & 82.2 {\small $\pm 0.8$} & 63.2 {\small $\pm 1.4$} & 77.2 {\small $\pm 1.0$} & 72.5 {\small $\pm 0.6$} & 81.5 {\small $\pm 0.8$} & 70.1 {\small $\pm 0.8$} & 80.6 {\small $\pm 0.7$} \\
    		& ConfidNet & 84.4 {\small $\pm 0.6$} & 65.3 {\small $\pm 2.0$} & 80.0 {\small $\pm 1.3$} & 73.8 {\small $\pm 0.6$} & 83.7 {\small $\pm 0.7$} & 71.5 {\small $\pm 0.6$} & 82.7 {\small $\pm 0.3$} \\
    		& Dissonance & 84.1 {\small $\pm 0.4$} & 62.5 {\small $\pm 1.4$} & 78.7 {\small $\pm 0.8$} & 70.3 {\small $\pm 0.4$} & 82.5 {\small $\pm 0.4$} & 69.3 {\small $\pm 0.4$} & 82.2 {\small $\pm 0.4$} \\
    		& Mut. Inf. & 78.9 {\small $\pm 0.8$} & 65.6 {\small $\pm 0.7$} & 76.2 {\small $\pm 0.9$} & 71.8 {\small $\pm 0.2$} & 79.1 {\small $\pm 0.4$} & 70.1 {\small $\pm 0.6$} & 78.5 {\small $\pm 0.6$} \\
    		& Diff. Ent. & 80.2 {\small $\pm 0.8$} & 65.6 {\small $\pm 0.9$} & 77.2 {\small $\pm 0.8$} & 72.7 {\small $\pm 0.3$} & 80.4 {\small $\pm 0.4$} & 71.0 {\small $\pm 0.5$} & 79.7 {\small $\pm 0.5$}\\
    		& EPKL & 78.8 {\small $\pm 0.8$} & 65.2 {\small $\pm 1.0$} & 76.1 {\small $\pm 0.9$} & 71.6 {\small $\pm 0.2$} & 78.9 {\small $\pm 0.4$} & 70.0 {\small $\pm 0.6$} & 78.3 {\small $\pm 0.6$}\\
    		& Diss.+Mut. Inf. & 84.2 {\small $\pm 0.6$} & 65.1 {\small $\pm 0.3$} & 80.1 {\small $\pm 0.4$} & 70.1 {\small $\pm 0.3$} & 82.5 {\small $\pm 0.5$} & 69.5 {\small $\pm 0.3$} & 82.3 {\small $\pm 0.5$} \\
    		& ODIN & 82.1 {\small $\pm 0.8$} & 62.9 {\small $\pm 1.4$} &77.1 {\small $\pm 1.0$} & 71.9 {\small $\pm 0.6$} & 81.3 {\small $\pm 0.8$} & 69.6 {\small $\pm 0.8$} & 80.3 {\small $\pm 0.7$} \\
    		& Mahalanobis & 84.0 {\small $\pm 0.2$} & \ubold{71.1} {\small $\pm 1.0$} & 82.4 {\small $\pm 0.5$} & \ubold{77.0} {\small $\pm 0.5$} & 84.9 {\small $\pm 0.3$} & \ubold{75.4} {\small $\pm 0.3$} & 84.3 {\small $\pm 0.5$} \\ 
    		&\cellcolor[gray]{.92}KLoSNet (Ours) & \cellcolor[gray]{.92}\ubold{86.7} {\small $\pm 0.4$} &\cellcolor[gray]{.92}68.4 {\small $\pm 1.1$} &\cellcolor[gray]{.92}\ubold{83.0} {\small $\pm 0.6$} &\cellcolor[gray]{.92}76.4 {\small $\pm 0.4$} &\cellcolor[gray]{.92}\ubold{86.4} {\small $\pm 0.4$} &\cellcolor[gray]{.92}\ubold{75.0} {\small $\pm 0.5$} &\cellcolor[gray]{.92}\ubold{86.0} {\small $\pm 0.4$} \\
    		
    		\midrule

            \multirow{12}{*}{\rotatebox{90}{\shortstack[c]{\ubold{CIFAR-100} \\ ResNet-18}}} & MCP & 84.0 {\small $\pm 0.4$} & 70.4 {\small $\pm 0.9$} & 81.0 {\small $\pm 0.3$} & 76.6 {\small $\pm 0.5$} & 83.6 {\small $\pm 0.4$} & 75.4 {\small $\pm 0.5$} & 83.1 {\small $\pm 0.2$} \\
    		& Entropy & 83.7 {\small $\pm 0.4$} & 70.4 {\small $\pm 0.9$} & 80.8 {\small $\pm 0.3$} & 76.9 {\small $\pm 0.5$} & 83.5 {\small $\pm 0.3$} & 75.7 {\small $\pm 0.5$} & 83.0 {\small $\pm 0.3$} \\
    		& ConfidNet & \ubold{87.1} {\small $\pm 0.2$} & 73.0 {\small $\pm 1.4$} & \ubold{84.5} {\small $\pm 0.6$} & 79.1 {\small $\pm 0.3$} & 86.8 {\small $\pm 0.3$} & \ubold{78.5} {\small $\pm 0.8$} & \ubold{86.6} {\small $\pm 0.5$} \\
    		& Dissonance & 86.7 {\small $\pm 0.4$} & 72.3 {\small $\pm 0.4$} & 84.0 {\small $\pm 0.2$} & 75.0 {\small $\pm 0.4$} & 85.3 {\small $\pm 0.4$} & 74.7 {\small $\pm 0.3$} & 85.2 {\small $\pm 0.2$} \\
    		& Mut. Inf & 82.6 {\small $\pm 0.4$} & 70.2 {\small $\pm 1.1$} & 80.0 {\small $\pm 0.4$} & 76.4 {\small $\pm 0.6$} & 82.6 {\small $\pm 0.3$} & 75.1 {\small $\pm 0.5$} & 82.1 {\small $\pm 0.3$} \\
    		& Diff. Ent & 83.0 {\small $\pm 0.4$} & 70.1 {\small $\pm 1.1$} & 80.2 {\small $\pm 0.4$} & 76.8 {\small $\pm 0.5$} & 83.0 {\small $\pm 0.3$} & 75.6 {\small $\pm 0.5$} & 82.5 {\small $\pm 0.3$} \\
    		& EPKL & 82.5 {\small $\pm 0.4$} & 70.2 {\small $\pm 1.1$} & 80.0 {\small $\pm 0.4$} & 76.3 {\small $\pm 0.6$} & 82.5 {\small $\pm 0.3$} & 75.0 {\small $\pm 0.5$} & 82.0 {\small $\pm 0.2$} \\
    		& Diss.+Mut. Inf. & 86.5 {\small $\pm 0.4$} & 71.8 {\small $\pm 0.8$} & 83.6 {\small $\pm 0.5$} & 76.1 {\small $\pm 0.3$} & 84.7 {\small $\pm 0.4$} & 75.2 {\small $\pm 0.5$} & 84.6 {\small $\pm 0.3$} \\
    		& ODIN & 83.7 {\small $\pm 0.4$} & 70.3 {\small $\pm 0.9$} & 80.8 {\small $\pm 0.3$} & 76.6 {\small $\pm 0.5$} & 83.5 {\small $\pm 0.3$} & 75.4 {\small $\pm 0.5$} & 83.0 {\small $\pm 0.3$} \\
    		& Mahalanobis & 85.9 {\small $\pm 0.4$} & \ubold{75.2} {\small $\pm 0.6$} & \ubold{84.5} {\small $\pm 0.1$} & 78.4 {\small $\pm 0.5$} & 85.9 {\small $\pm 0.3$} & 77.5 {\small $\pm 0.4$} & 85.6 {\small $\pm 0.3$} \\
    		&\cellcolor[gray]{.92}KLoSNet (Ours) &\cellcolor[gray]{.92}86.9 {\small $\pm 0.3$} &\cellcolor[gray]{.92}73.1 {\small $\pm 0.4$} &\cellcolor[gray]{.92}\ubold{84.4} {\small $\pm 0.1$} &\cellcolor[gray]{.92}\ubold{80.8} {\small $\pm 0.2$} &\cellcolor[gray]{.92}\ubold{87.3} {\small $\pm 0.2$} &\cellcolor[gray]{.92}\ubold{79.0} {\small $\pm 0.2$} &\cellcolor[gray]{.92}\ubold{86.7} {\small $\pm 0.3$} \\
    		
    		\bottomrule
    	\end{tabular}
    }
    \label{chap5:tab:comparative_experiment}
\end{table}

\clearpage

We also observe that KLoSNet significantly improves misclassification detection, even compared to dedicated methods such as ConfidNet or second-order measures related to class confusion, \eg, dissonance. Another baseline could be to combine two measures specialized respectively for class confusion and lack of evidence, such as Dissonance+Mut.Inf. But it still performs less well than KLoSNet.

\paragraph{Impact of Confidence Learning.} To evaluate the effect of the uncertainty measure KLoS and of the auxiliary confidence network KLoSNet, we report a detailed ablation study in \cref{chap5:tab:ablation_study}. We can notice that KLoSNet improves misclassification over KLoS but also OOD detection in every benchmark. We intuit that learning to improve misclassification detection also helps to spot some OOD inputs that share similar characteristics.

\begin{table}[t]
    \centering
    \caption[Impact of confidence learning]{\textbf{Impact of confidence learning.} Comparison of detection performances (\% AUROC) between KLoS and KLoSNet for CIFAR-10 and CIFAR-100 experiments with VGG-16 architecture.}
	\resizebox{\linewidth}{!}{%
    	\begin{tabular}{cl|c|cc|cc|cc}
    		\toprule
    		& & & \multicolumn{2}{c|}{\textbf{LSUN}} & \multicolumn{2}{c|}{\textbf{TinyImageNet}} & \multicolumn{2}{c}{\textbf{STL-10}}\\
    		& \textbf{Method} & Mis. & OOD & Mis+OOD & OOD & Mis+OOD & OOD & Mis+OOD \\
    		\midrule
    		\multirow{2}{*}{\shortstack[c]{\ubold{CIFAR-10} \\ VGG-16}}
    		& KLoS & 92.1 {\small $\pm 0.3$} & 86.5 {\small $\pm 0.3$} & 91.2 {\small $\pm 0.2$} & 85.4 {\small $\pm 0.3$} & 90.4 {\small $\pm 0.2$} & 64.1 {\small $\pm 0.3$} & 79.6 {\small $\pm 0.3$} \\
    		& KLoSNet & \ubold{92.5} {\small $\pm 0.6$} & \ubold{87.6} {\small $\pm 0.9$} & \ubold{91.7} {\small $\pm 0.9$} & \ubold{86.6} {\small $\pm 0.9$} & \ubold{91.2} {\small $\pm 0.8$} & \ubold{67.7} {\small $\pm 1.4$} & \ubold{81.8} {\small $\pm 0.9$} \\
    		\midrule
    		\multirow{2}{*}{\shortstack[c]{\ubold{CIFAR-100} \\ VGG-16}}
    		& KLoS & 85.4 {\small $\pm 0.2$} & 65.1 {\small $\pm 1.1$} & 81.3 {\small $\pm 0.6$} & 74.5 {\small $\pm 0.4$} & 85.4 {\small $\pm 0.4$} & 72.7 {\small $\pm 0.3$} & 84.8 {\small $\pm 0.4$} \\
    		& KLoSNet & \ubold{86.7} {\small $\pm 0.4$} & \ubold{68.4} {\small $\pm 1.1$} & \ubold{83.0} {\small $\pm 0.6$} & \ubold{76.4} {\small $\pm 0.4$} & \ubold{86.4} {\small $\pm 0.4$} & \ubold{75.0} {\small $\pm 0.5$} & \ubold{86.0} {\small $\pm 0.4$} \\
    		\bottomrule
    	\end{tabular}
    }
    \label{chap5:tab:ablation_study}
\end{table}

\begin{wrapfigure}{r}{0.55\textwidth}
\centering
    \captionsetup{justification=centering}

    \vspace{-0.6cm}
    \includegraphics[width=\linewidth]{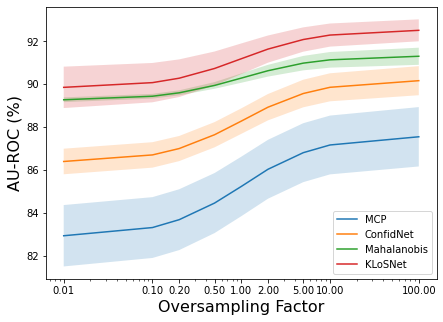}
    \caption[Study of the impact of the oversampling factor]{\textbf{Impact of the oversampling factor $\kappa$} (CIFAR-10/TinyImageNet).}
    \vspace{-0.6cm}
    \label{chap5:fig:varying_oversampling}
\end{wrapfigure}
\paragraph{Oversampling Factor.} When deploying a model in the wild, it is difficult to know beforehand the proportions of misclassifications and OOD samples the model will have to handle. Until now, we assumed an equal proportion in order to evaluate equally the capacity to detect both kinds of inputs. In \cref{chap5:fig:varying_oversampling}, we vary the oversampling factor $\kappa$ in $[0.01, 100]$ for CIFAR10/TinyImageNet to assess the robustness of tested methods and our approach. The higher the oversampling factor is, the more misclassifications will be sampled in the test set, hence giving importance to misclassification detection, and vice versa. Results show that regardless of the value chosen for oversampling, KLoSNet consistently outperforms all other measures, with a larger gain when $\kappa$ increases.

\paragraph{Combining KLoS with Ensembling.} Aggregating predictions from an ensemble of neural networks not only improves generalization \cite{deepensembles2017,rame2021dice} but also uncertainty estimation \cite{ovadia2019}. We train an ensemble of ten evidential models on CIFAR-10 and evaluate the performance of various uncertainty measures -- MCP, predictive entropy, precision, mutual information, and KLoS -- obtained from averaged concentration parameters. Results for each detection task with CIFAR-10/TinyImageNet benchmark are available in \cref{chap5:fig:plot_ensembling}. As expected, every method obtains improved performance when computed from an ensemble of models. The gain is particularly pronounced for OOD detection: for instance performance with precision $\alpha_0$ is improved by $+8.0$ and with mutual information by $+7.8$ points. These gains are due to the diversity in predictions provided by ensembling which helps to better capture epistemic uncertainty, as explained in \cref{chap2}. While improvements with KLoS are less significant, KLoS remains the best measure in each detection task with respectively $93.8\%$ AU-ROC in misclassification detection, $88.7\%$ in OOD detection and $92.3\%$ in the joint detection. One possible explanation is that KLoS was already capturing effectively epistemic uncertainty and the improvement with ensembling may consequently be less significant.

\begin{figure}[t]
\centering
\begin{minipage}[c]{0.33\linewidth}
\centering
    \includegraphics[width=\linewidth]{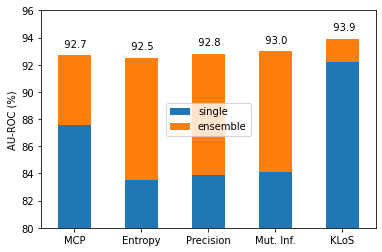}
    \subcaption{Mis. detection}
    \label{chap5:fig:ensembling_mis}
\end{minipage}%
\begin{minipage}[c]{0.33\linewidth}
\centering
    \includegraphics[width=\linewidth]{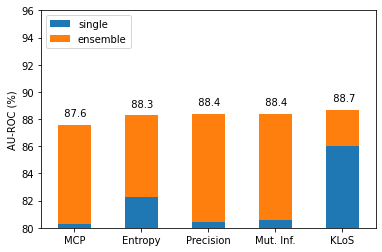}
    \subcaption{OOD detection}
    \label{chap5:fig:ensembling_ood}
\end{minipage}%
\begin{minipage}[c]{0.33\linewidth}
\centering
    \includegraphics[width=\linewidth]{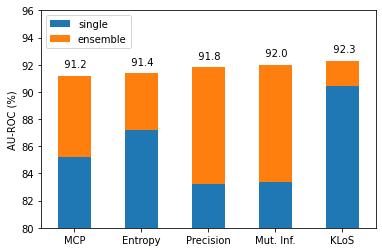}
    \subcaption{Mis.+OOD detection}
    \label{chap5:fig:ensembling_misood}
\end{minipage}%
    \caption[Comparative gain with ensembling for each detection task on CIFAR10 vs. TinyImageNet]{\textbf{Comparative gain with ensembling for each detection task on CIFAR10 vs. TinyImageNet.} Ensembling improves performances with every tested method, in particular in OOD detection (\cref{chap5:fig:ensembling_ood}). KLoS remains the best method when combined with ensembling.}
\label{chap5:fig:plot_ensembling}
\end{figure}

\subsection{Selective classification in presence of distribution shifts}
\label{chap5:subsec:detail_results_selective}

\begin{figure}[t]
\begin{minipage}{\linewidth}
\centering
    \centering
    \includegraphics[width=0.9\linewidth]{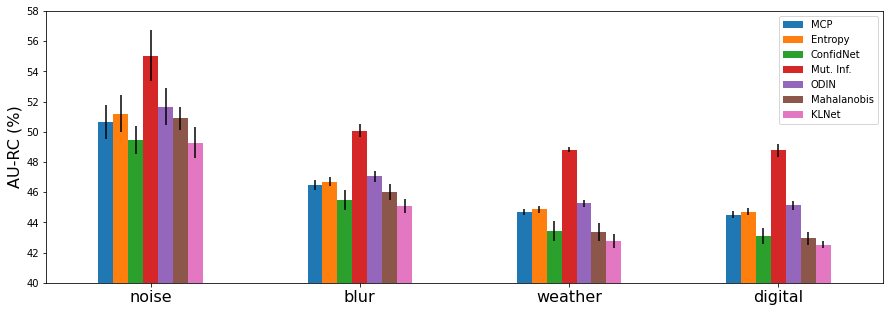}
    \subcaption{CIFAR-10-C}
    \label{chap5:fig:results_cifar10c}
\end{minipage}
\begin{minipage}{\linewidth}
\centering
    \centering
    \includegraphics[width=0.9\linewidth]{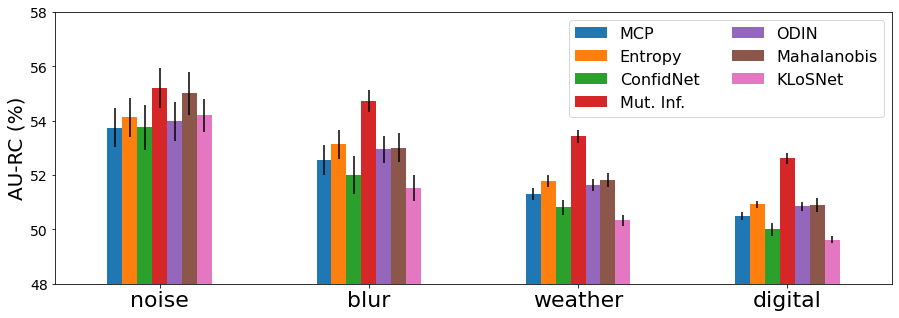}
    \subcaption{CIFAR-100-C}
    \label{chap5:fig:results_cifar100c}
\end{minipage}
\label{chap5:fig:exp_select}
    \caption[Aggregated results for selective classification on CIFAR-10-C and CIFAR-100-C]{\textbf{Aggregated results for selective classification on CIFAR-10-C and CIFAR-100-C}. Comparative performance in AURC (\%) of classification with the option to reject misclassified test samples and samples from shifted distributions. Results are averaged on 5 runs (mean $\pm$ std.).}
\end{figure}

Classification with a reject option, also known as \emph{selective classification} \cite{elyaniv10a}, consists in a scenario where a classifier can abstain on samples where its confidence is below a certain threshold. This is appropriate for applications where the system can hand over to human experts or users. Performance can be measured on \textit{risk-coverage} curves. We recall evaluation metrics in the following and refer the reader to \cref{chap2:sec:evaluation} for a more detail description. The \textit{coverage} is the probability mass of the non-rejected region in $\cX$ and can be empirically estimated by the percentage of the non-rejected samples. The \textit{risk} of a selective classifier is the average loss on the accepted samples.  Given a chosen coverage, good selective classifiers correlate with low risk. Averaged performances are evaluated on risk-coverage curves with a threshold-independent area-under-curve metric, denoted here AURC. The lower the AURC, the better the selective classifier.

Previous works evaluate the performance on in-distribution data. However, a classifier may encounter data drawn from a different distribution when deployed in the wild. Following \cite{wilds2020}, we extend selective classification by penalizing non-rejected OOD samples. If a sample is drawn from the in-distribution, we compute the 0/1 cost function as usual. For OOD samples, we apply the maximum cost of 1, whatever the prediction. As for simultaneous detection, we rely on oversampling to mitigate the unbalance between misclassifications and OOD samples.

Experiments are conducted with previously trained VGG-16 networks on CIFAR-100.  We measure their selective classification when subject to distribution shifts by considering CIFAR-100C \cite{hendrycks2018benchmarking} as OOD dataset. This dataset is constructed by corrupting the original CIFAR-100 test set. There is a total of 15 types of corruptions, which can be grouped into four families, namely \textit{noise}, \textit{blur}, \textit{weather} and \textit{digital}. Each corruption comes with five different levels of severity. While this dataset is commonly used to measure robustness to distribution shift, we focus here on models' ability to reject these samples along with misclassifications made on the original CIFAR-100 test set.

The results are reported by corruption families (noise, blur, weather and digital) in \cref{chap5:fig:exp_select} and further detailed in \cref{appxB:tab:CIFAR-10_exp_select}. One common observation regardless of the criterion is that selective classification is harder when subject to noise perturbations than other types of perturbation. In each case, KLoSNet and ConfidNet obtain the best performances. For instance, for weather perturbations on CIFAR-10-C, KloSNet achieves 42.7\% AURC and ConfidNet 43.4\% AURC. In particular, KloSNet outperforms every other method for blur, weather and digital perturbations of CIFAR-100-C. Hence, when subject to an unforeseen distribution shift, a model equipped with KLoSNet provides more accurate uncertainty estimates without sacrificing predictive performances. Note that for noise corruptions, the results depend widely on the run, which makes interpretation more difficult.

\subsection{Impact of Adversarial Perturbations}
\label{appxB:subsec:adv_perturb}

In the original papers, ODIN and Mahalanobis preprocess inputs by adding small inverse adversarial perturbations to reinforce networks in their prediction; this has also the observed benefit to make in-distribution and out-of-distribution samples more separable. The tuning of the adversarial noise's magnitude depends on the evaluated OOD data.

In \cref{appxB:fig:adv_perturb_svhn}, we plot the AUC of each detection task with different values of perturbation magnitude $\varepsilon$ with ODIN, Mahalanobis and KLoS, using SVHN as OOD dataset. Even though there exists a particular noise value for improved OOD detection (\cref{appxB:fig:adv_perturb_svhn}, middle), increasing noise magnitude deteriorates performances in misclassification detection (\cref{appxB:fig:adv_perturb_svhn}, left) for each method. The best results on the simultaneous detection task (\cref{appxB:fig:adv_perturb_svhn}, right) correspond to $\varepsilon=0$, as done in experiments presented in previous sections.

Except with SVHN, adversarial perturbations are detrimental even to OOD detection. We report the AUC results of varying adversarial perturbations on CIFAR-10 dataset when using LSUN (\cref{appxB:fig:adv_perturb_lsun}), TinyImageNet (\cref{appxB:fig:adv_perturb_tinyin}) and STL-10 (\cref{appxB:fig:adv_perturb_st10}) as OOD datasets. The best results on each considered task correspond to $\varepsilon=0$ and KLoS outperforms both Mahalanobis and ODIN. As opposed to results with SVHN as OOD dataset, we did not observe improvements on any method (ODIN, Mahalanobis and KLoS) when using inverse adversarial perturbations for OOD detection with LSUN, TinyImageNet and STL-10 datasets. Similar results are observed in \cite{odin2018} (Appendix B, Fig.\,8) when using WideResNet architectures.

\subsection{Effect of training with out-of-distribution samples}
\label{chap5:sec:ood_training}

Previous results demonstrate that simultaneous detection of misclassifications and OOD samples can be significantly improved by KLoSNet. We now investigate settings where OOD samples are available. We train an evidential model to minimize the reverse KL divergence \cite{malinin2019} between the model output and a sharp Dirichlet distribution focused on the predicted class for in-distribution samples, and between the model output and a uniform Dirichlet distribution for OOD samples. This loss induces low concentration parameters for OOD data and improves second-order uncertainty measures such as Mut.\,Inf

The literature on evidential models only deals with an OOD training set somewhat related to the in-distribution dataset, \eg~{CIFAR-100 for models trained on CIFAR-10}. Despite semantic differences, CIFAR-10 and CIFAR-100 images were collected the same way, which might explain the generalisation to other OOD samples in evaluation.

\begin{figure}[ht]
\centering
\begin{minipage}{0.85\linewidth}
    \centering
    \includegraphics[width=0.90\linewidth]{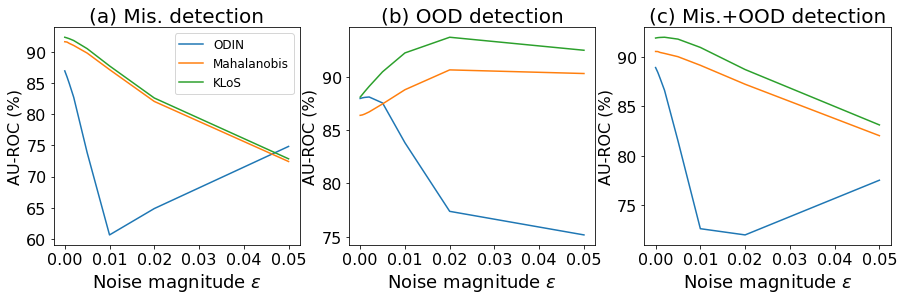}
    \subcaption{CIFAR-10 / SVHN}
    \label{appxB:fig:adv_perturb_svhn}
\end{minipage}\hfill
\begin{minipage}{0.85\linewidth}
    \centering
    \includegraphics[width=0.90\linewidth]{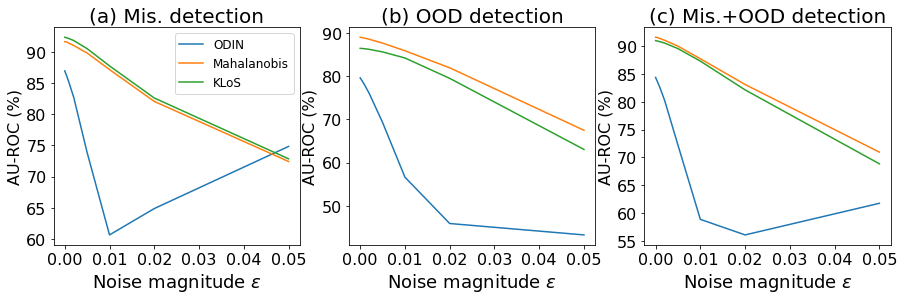}
    \subcaption{CIFAR-10 / LSUN}
    \label{appxB:fig:adv_perturb_lsun}
\end{minipage}\hfill
\begin{minipage}{0.85\linewidth}
    \centering
    \includegraphics[width=0.90\linewidth]{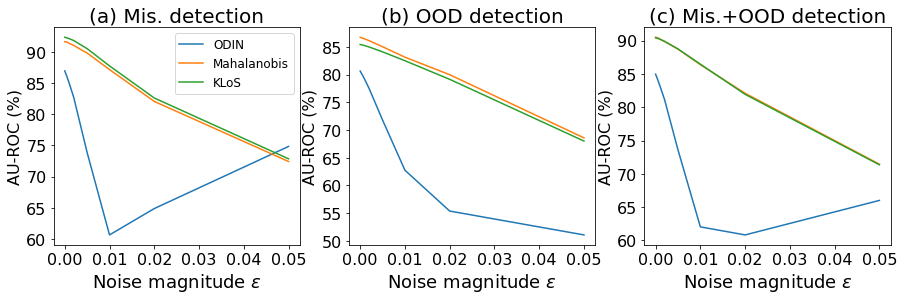}
    \subcaption{CIFAR-10 / TinyImageNet}
    \label{appxB:fig:adv_perturb_tinyin}
\end{minipage}\hfill
\begin{minipage}{0.85\linewidth}
    \centering
    \includegraphics[width=0.90\linewidth]{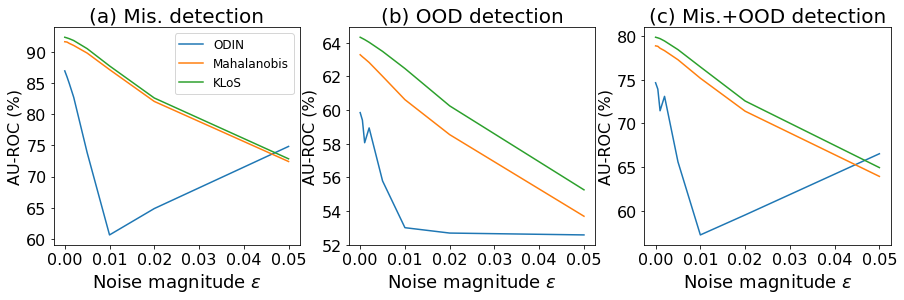}
    \subcaption{CIFAR-10 / STL-10}
    \label{appxB:fig:adv_perturb_st10}
\end{minipage}
    \caption[Effect of inverse adversarial perturbations on OOD-designed measures and KLoS]{Effect of inverse adversarial perturbations on OOD-designed measures and KLoS for misclassification detection, OOD detection and simultaneous detection with VGG-16 architecture.}
\label{appxB:fig:adv_perturb_others}
\end{figure}

\clearpage

Contrarily, CIFAR-10 objects and SVHN street-view numbers differ more for instance. In \cref{chap5:fig:compa_OOD_training}, we vary the OOD training set and compare the uncertainty metrics taken from the resulting models.

\begin{figure}[t]
\centering
\begin{minipage}[c]{0.25\linewidth}
\centering
    \includegraphics[width=\linewidth]{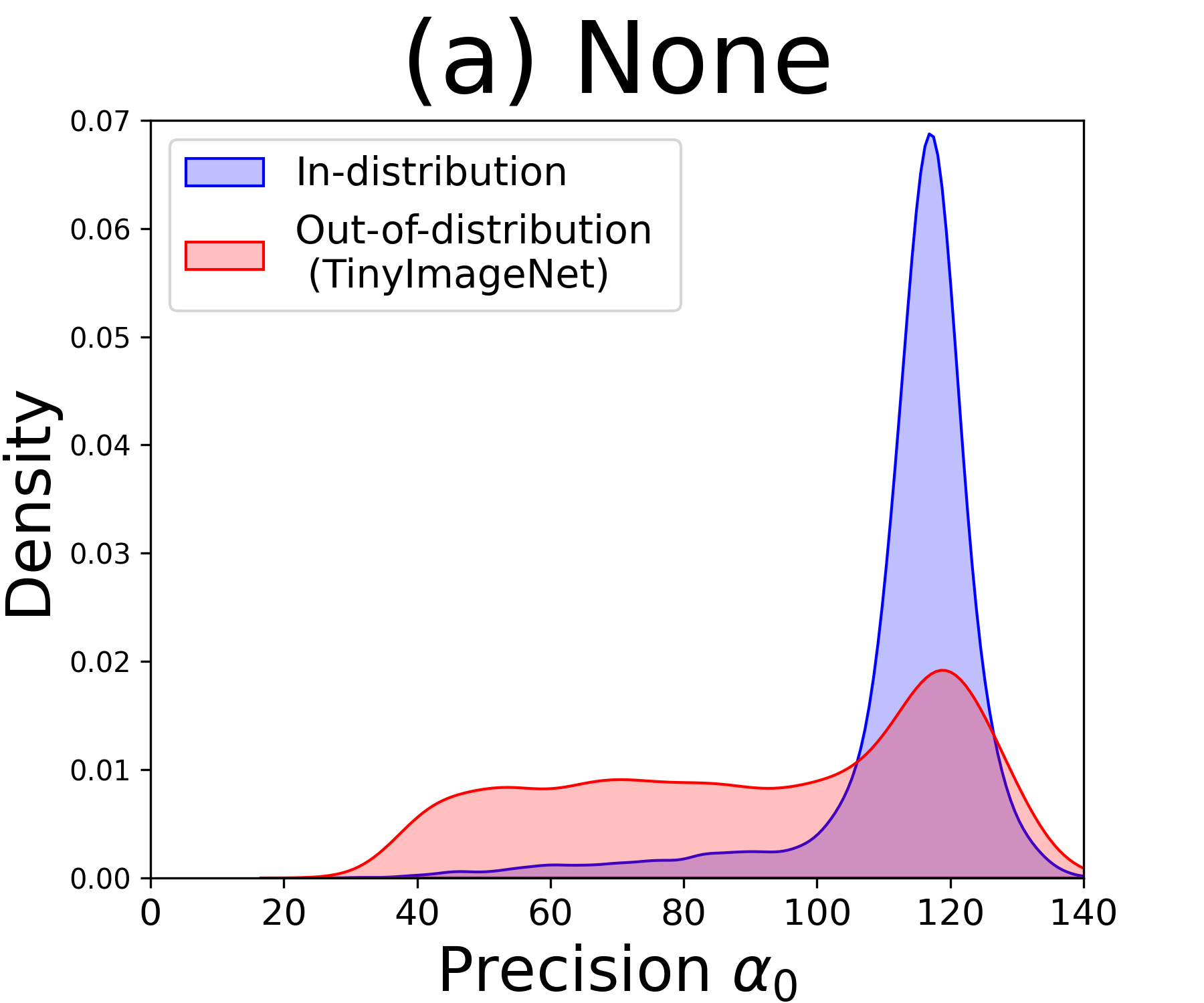}
\end{minipage}%
\begin{minipage}[c]{0.25\linewidth}
\centering
    \includegraphics[width=\linewidth]{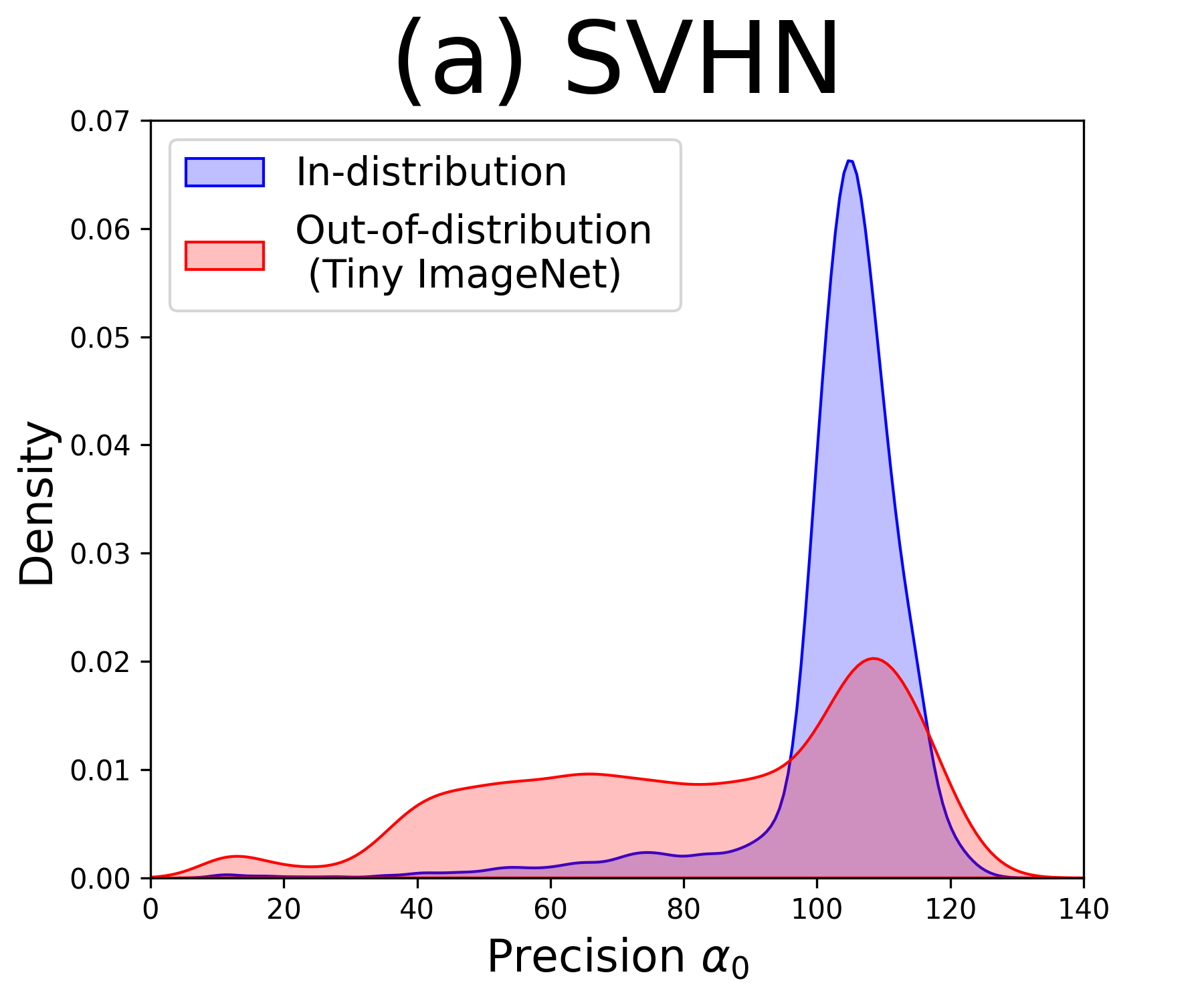}
\end{minipage}%
\begin{minipage}[c]{0.25\linewidth}
\centering
    \includegraphics[width=\linewidth]{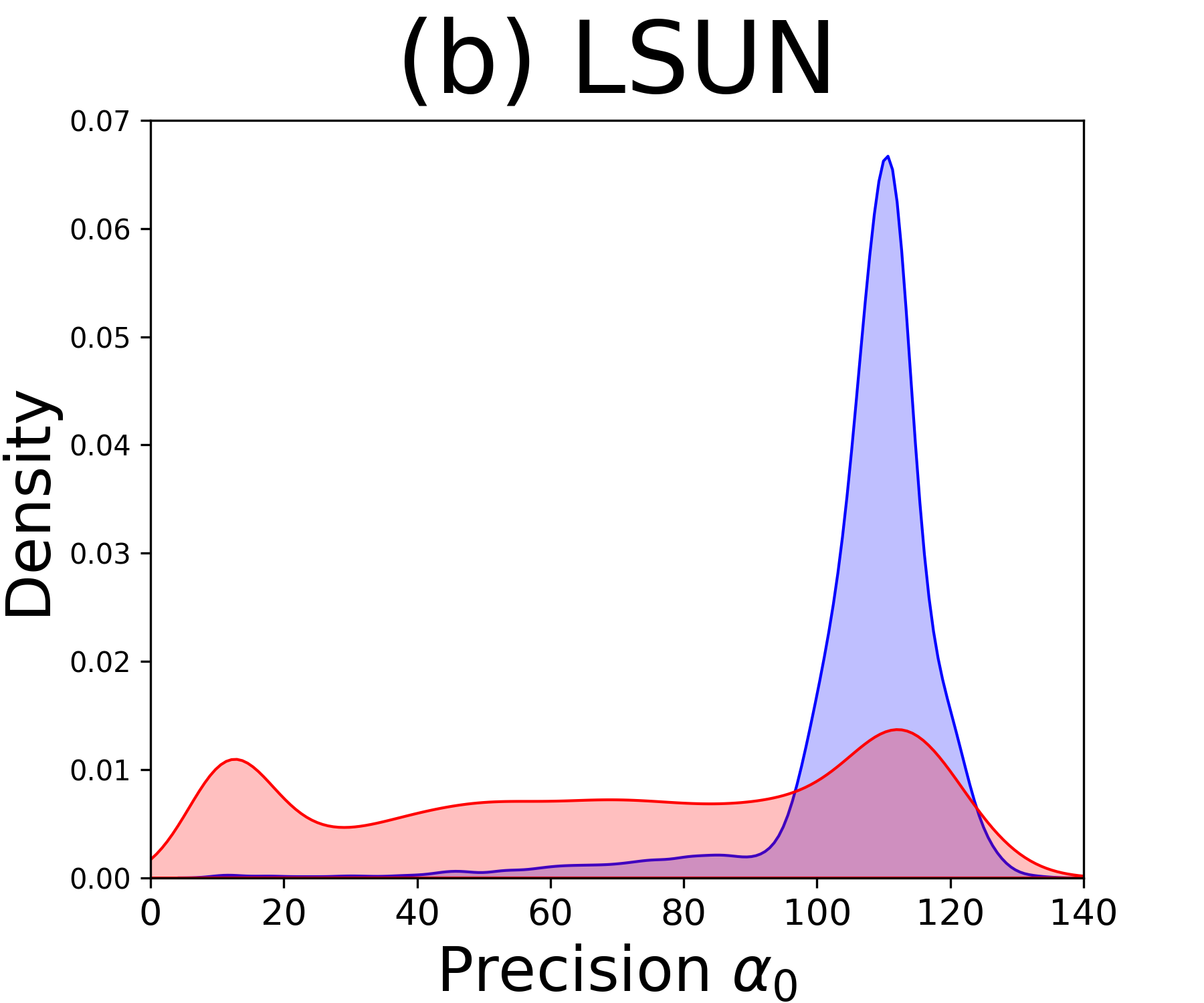}
\end{minipage}%
\begin{minipage}[c]{0.25\linewidth}
\centering
    \includegraphics[width=\linewidth]{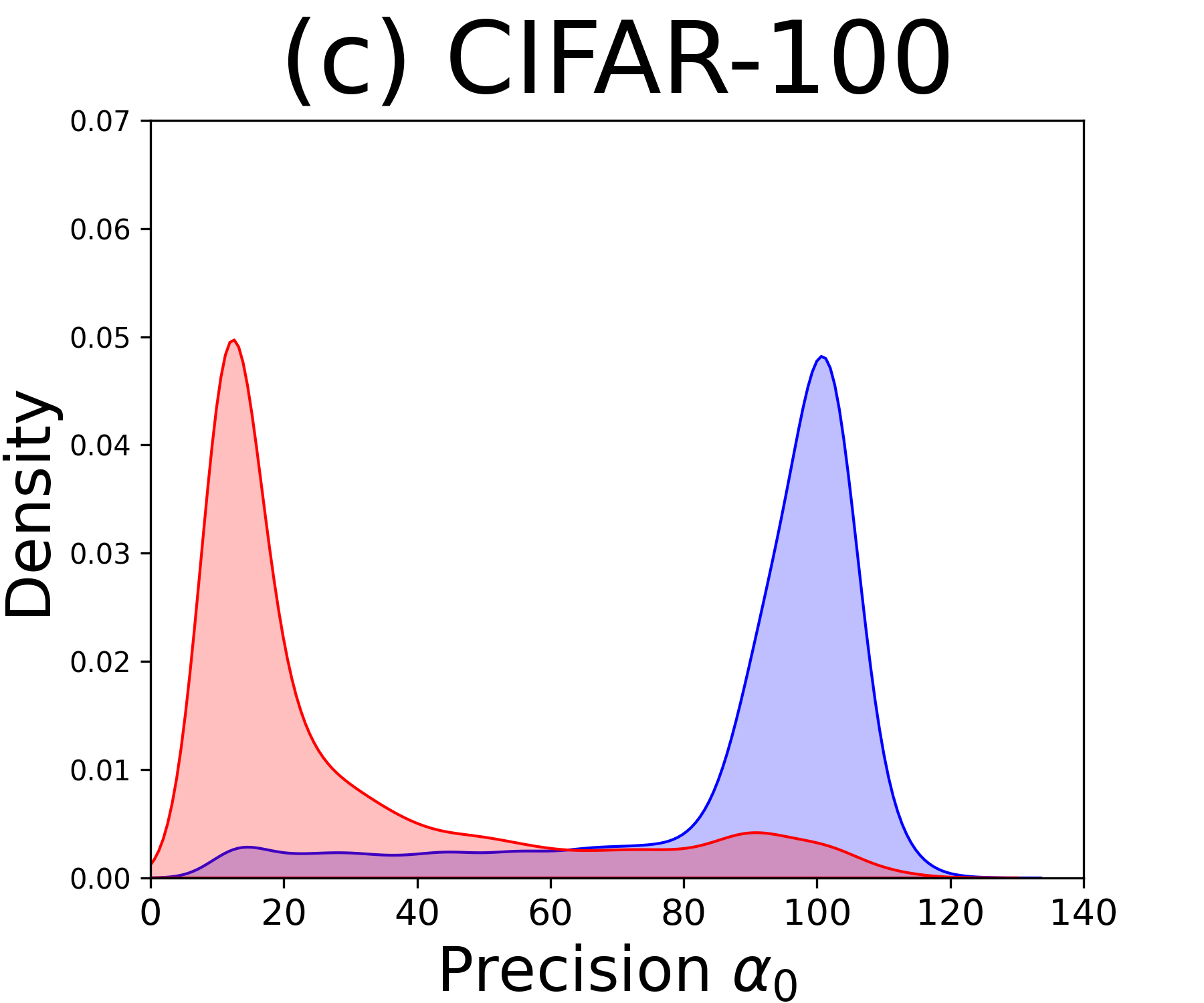}
\end{minipage}%
\caption[Visualisation of the effect of OOD training data on precision $\alpha_0$ for CIFAR-10 and CIFAR-100 datasets]{\textbf{Effect of OOD training data on precision $\alpha_0$.} Density plots for CIFAR-10/TinyImageNet benchmark: (a) without OOD training data, (b,c) with inappropriate OOD samples (SVHN, LSUN); (d) with close OOD samples (CIFAR-100).} 
\label{chap5:fig:density_plot_alpha0}
\end{figure}

\begin{figure}[ht]
\centering
\hspace{-1.0cm}
\begin{minipage}[c]{0.33\linewidth}
\centering
    \includegraphics[width=1.15\linewidth]{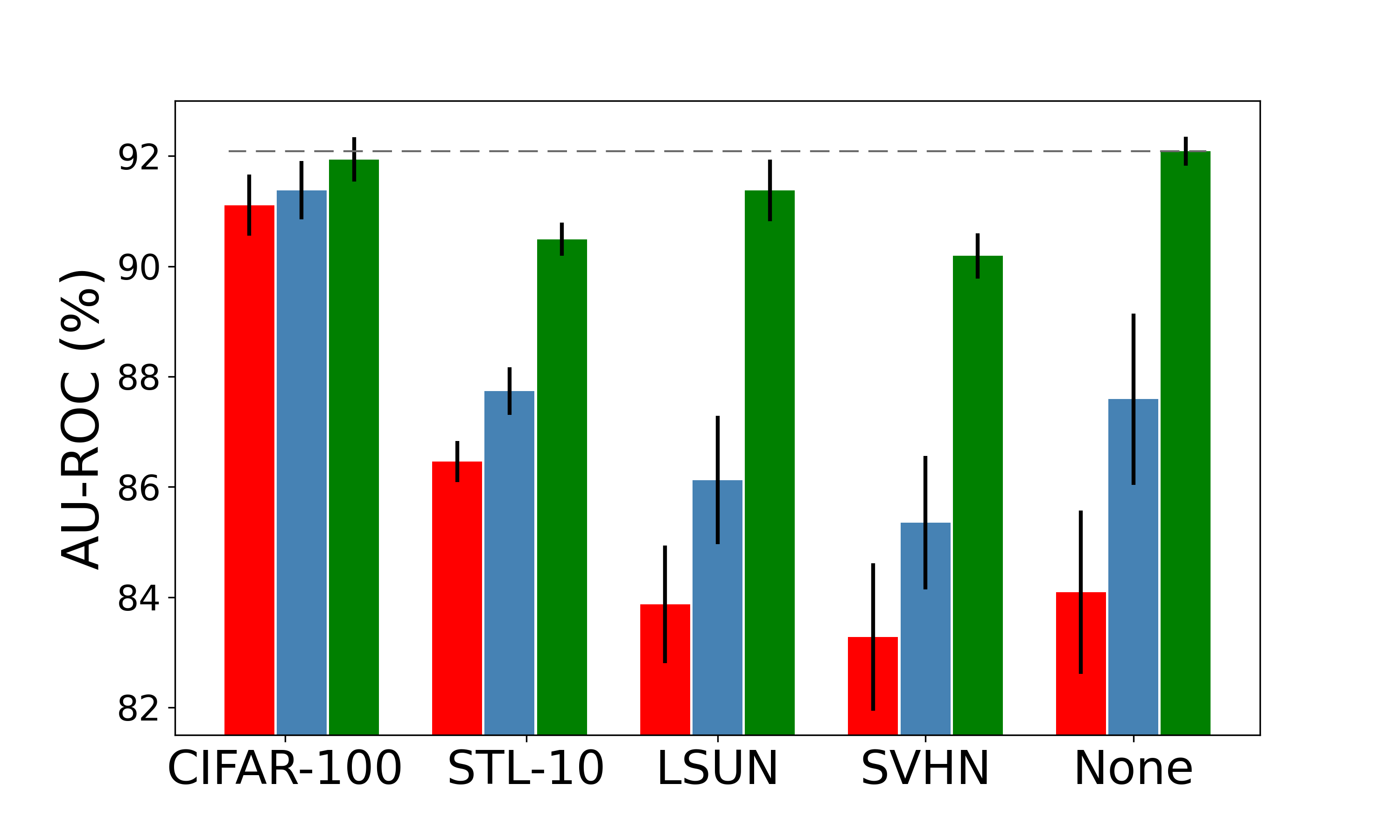}
    \subcaption{Mis. detection}
    \label{chap5:fig:compa_OOD_training_mis}
\end{minipage}%
\begin{minipage}[c]{0.33\linewidth}
\centering
    \includegraphics[width=1.15\linewidth]{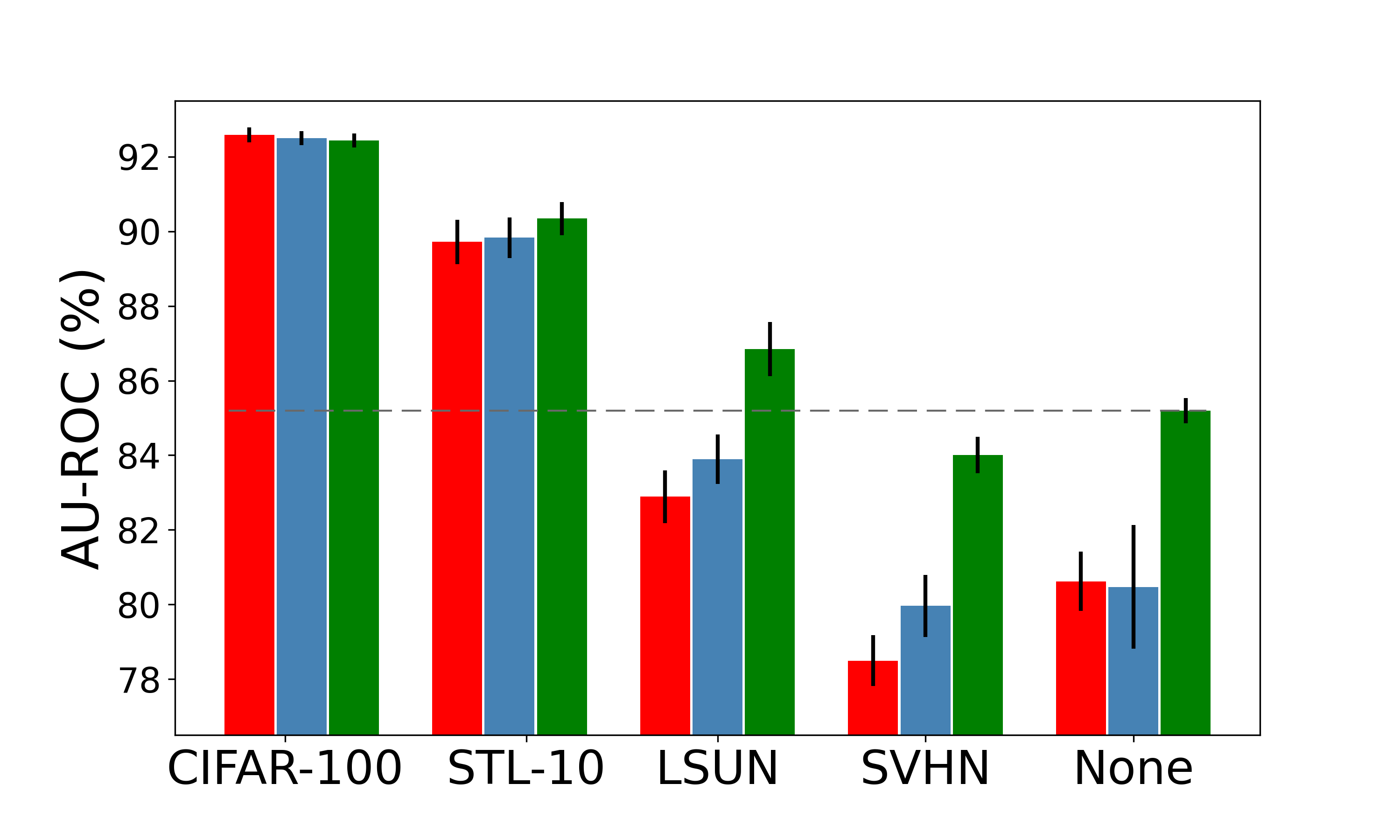}
    \subcaption{OOD detection}
    \label{chap5:fig:compa_OOD_training_ood}
\end{minipage}%
\begin{minipage}[c]{0.33\linewidth}
\centering
    \includegraphics[width=1.15\linewidth]{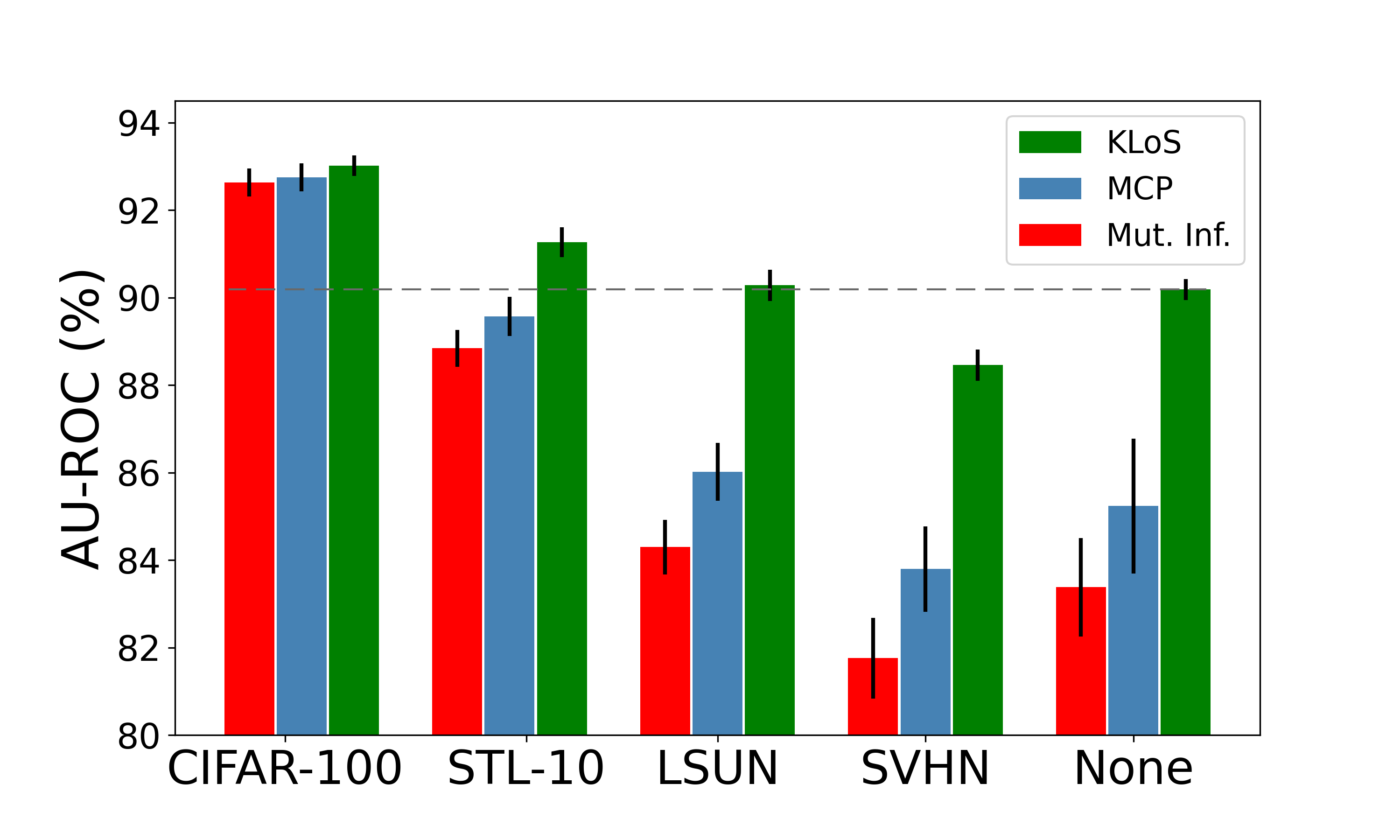}
    \subcaption{Mis.+OOD detection}
    \label{chap5:fig:compa_OOD_training_misood}
\end{minipage}%
    \caption[Comparative detection results with different OOD training datasets]{\textbf{Comparative detection results with different OOD training datasets.} While using OOD samples in training improves performance in general, the gain varies widely, sometimes even being negative for inappropriate OOD samples \eg,~SVHN. KLoS remains the best measure in every setting. Experiment with VGG-16 architecture on CIFAR-10 dataset.}
\label{chap5:fig:compa_OOD_training}
\end{figure}

As expected, using CIFAR-100 as training OOD data improves performance for every measure (MCP, Mut.\,Inf. and KLoS). However, the boost provided by training with OOD samples depends highly on the chosen dataset: The performance of Mut.\,Inf. decreases from 92.6\% AUC with CIFAR-100 to 82.9\% when switching to LSUN, and even becomes worse with SVHN (78.5\%) compared to using no OOD data (80.6\%). Indeed, \cref{chap5:fig:density_plot_alpha0} shows that only the CIFAR-100 dataset seems to be effective to enforce low $\alpha_0$ on unseen OOD samples.

We also note that KLoS outperforms or is on par with MCP and Mut.\,Inf. in every setting. These results confirm the adequateness of KLoS for simultaneous detection and extend our findings to settings where OOD data is available at train time. Most importantly, using KLoS on models without OOD training data yields better detection performance than other measures taken from models trained with inappropriate OOD samples, here being every OOD dataset other than CIFAR-100.

\section{Conclusion}
\label{chap5:sec:conclusion}

 Based on evidential models, we define \emph{KLoS}, a Kullback–Leibler divergence criterion defined on the class-probability simplex. By design, KLoS encompasses both class confusion and evidence information, which is necessary for open-world recognition. We adapt our learning confidence approach to evidential models and proposed KLoSNet, an auxiliary model to estimate the uncertainty of a classifier for both in-domain and out-of-domain inputs. KLoSNet is trained to predict the KLoS$^*$ value of a prediction. Our experiments extensively demonstrate its effectiveness across various architectures, datasets and configurations, and reveal its class-wise divergence-based behavior. We also show that, far from being the panacea, using training OOD samples depends critically on the choice of these samples for existing uncertainty measures. KLoS, on the other hand, is more robust to this choice and can alleviate their use altogether.

%% file: chapter06/chapter06.tex
We first summarize the contributions that we proposed in this thesis before
discussing interesting directions for future work.

\section{Summary of contributions}
\label{chap6:sec:summary}

The main contribution of this thesis is to use an auxiliary confidence model to learn prediction confidence from deep neural networks in classification. Given a trained classification model, the confidence model learns from data to estimate an adequate criterion derived from the classifier, such as the true class probability for standard neural networks and KLoS for evidential models. At test time, we directly use the confidence model's output as our uncertainty estimate. One major benefit of this method is to be architecture-agnostic: in our experiments, we successfully improve uncertainty estimation for classification models with different deep learning architectures (MLP, LeNet, VGGs, ResNets). We applied our approach on three tasks: failure prediction (\cref{chap3}), unsupervised domain adaptation for semantic segmentation (\cref{chap4}), and simultaneous detection of in-distribution errors and out-of-distribution samples (\cref{chap5}). For each task, there are two main challenges to address: (1) which criterion should we use, and (2) how to efficiently train the confidence model. 

\paragraph{Failure prediction with learned confidence.} \cref{chap3} starts by detailing the fundamental limit of maximum class probability (MCP), which yields over-confident uncertainty estimates for misclassifications. We define the true class probability (TCP) as an alternative measure which provides a better ranking between correct predictions and misclassifications than MCP. As the true class is unknown at test time, we introduce \emph{ConfidNet}, an auxiliary confidence neural network trained to learn TCP from data. ConfidNet consists in a small decoder neural network composed of several dense layers and initially sharing the same ConvNet encoder as the classification model. ConfidNet's learning scheme consists in training the auxiliary network to regress TCP values and then enabling the fine-tuning of its encoder by decoupling it from the classification' encoder. We were able to improve the capacity of the model to distinguish correct from erroneous samples and to achieve better selective classification with many different architectures and for each image classification experiment. In the long history of classifiers with a reject option, our contribution on learning a model's confidence can be seen as a specific case of selective classification, where the selection function is based on an independent neural network to define the underpinning confidence-rate.

\paragraph{Selection of confident pseudo-labels for domain adaptation.} \cref{chap4} shows that reliable confidence estimates are key to improve self-training approaches in domain adaptation. We transpose the idea of learning confidence via an auxiliary model and we select relevant pixels for pseudo-labeling based on confidence estimates output by this auxiliary model. In a manner analogous to ConfidNet, we learn to regress to TCP from training data. The proposed adaptation of our original approach to this new context, termed \emph{ConDA}, involves an ‘atrous’ pyramidal pooling architecture with structured output to perform multi-scale confidence estimation and we adopt an adversarial learning scheme which enforces alignment between confidence maps in source and target domains. Results showed significant improvements from strong baselines in each benchmark.

\paragraph{Detecting errors and out-of-distribution samples with evidential models.} Finally, we extend our learning confidence via auxiliary models to the context of simultaneous detection of in-distribution errors and out-of-distribution samples. It first requires defining a criterion that captures aleatoric and epistemic uncertainty in a single score. As a Bayesian approach, evidential models enrich uncertainty representation with evidence information and allows one to fulfill the previous requirement by deriving second-order measures on the class-probability simplex. Consequently, we defined \emph{KLoS}, a KL divergence criterion between a model's output and a class-wise prototype Dirichlet distribution focused on the predicted class. By  design, KLoS encompasses both class confusion and evidence information, thanks to its class-wise divergence-based behavior. An auxiliary model, \emph{KLoSNet}, is then trained to predict a refined criterion, KLoS*, measuring KL divergence with a prototype based on the true class of an input. Across various architectures, datasets and configurations, KLoSNet improves performance on the joint detection and reveals itself to be more robust to the type of OOD samples in scenarios allowing this type of auxiliary training data.

\section{Perspectives for future work}
\label{chap6:sec:perspectives}

Let us now discuss interesting directions that could be addressed in future
work in relation to our contributions.

\subsection{Error data generation to ease confidence learning}
\label{chap6:subsec:generate_data}

Confidence learning showed significant improvements over strong baselines in uncertainty estimation for each considered work. Nevertheless, the training of the auxiliary model depends on the quality of the dataset, \ie the number of errors available. Modern neural networks are over-parameterized and tend to over-fit training data, hence achieving high accuracy on training sets and leaving only a small fraction of misclassified samples. We believe this data imbalance issue mitigates performance in confidence learning. In \cref{chap3}, we experimented training ConfidNet on a hold-out dataset. We observed a general performance drop when using a validation set for training TCP confidence. The drop is especially pronounced for small datasets (MNIST). Consequently, with a high accuracy and a small validation set, we do not get a larger absolute number of errors using the validation set compared to the train set. In preliminary experiments, we also tried a weighted MSE loss (and a weighted binary cross-entropy) where the cost of wrong TCP estimates were higher for misclassifications than for correct predictions. But it did not result in improved performance either.

To mitigate the imbalance issue, another solution would be to artificially generate errors. Adversarial perturbations \cite{goodfellow2015,madry2018towards} are small perturbations to the input that are almost imperceptible to humans but which fool a neural network, hence switching a correct prediction to an erroneous one. A combined set of genuine and adversarial inputs would help to re-balanced training. Mix-up \cite{zhang2018mixup}, and more generally aggressive data augmentation techniques such as AugMix \cite{hendrycks2020augmix} and CutMix \cite{yun2019cutmix} have been shown to improve robustness and could be also applied here to generate samples with mixed probabilities, hence providing a larger range of TCP values for confidence training.

\subsection{Confidence learning of an ensemble}
\label{chap6:subsec:ensemble}

As a simple alternative to fully Bayesian methods, ensembles have been a popular research topic within probabilistic methods \cite{Dietterich00ensemblemethods,Rokach2010,deepensembles2017}. With deep neural networks, not only they improved generalization but they also outperformed other Bayesian methods and a single model in uncertainty estimation \cite{ovadia2019}. In particular, diversity between individual NN's predictions allows an ensemble to better capture epistemic uncertainty: the more diverse predictions are, the more the model is uncertain about this input. Measures such as mutual information can be derived to evaluate this diversity. Yet, when it comes to failure prediction, previous methods rely on averaging the predictions into a single probability vector and deriving usual measures such as MCP and entropy.

\cref{chap3} proposed to improve failure prediction only on a single model as the auxiliary confidence model should be initially attached to an intermediate representation of the classifier. An interesting direction should be to combine the idea of confidence learning to the context of ensembles. The straightforward solution would be to train an auxiliary model to regress the TCP value of the average probability vector. But preliminary experiments showed difficulties in converging to regress such averaged values as each individual model behaves differently and we cannot rely on initialized weights from one arbitrary model. One could try to attach an auxiliary model to each member of the ensemble and train them separately before averaging outputs of all confidence models. A clever approach would imply leveraging the diversity in predictions to refine the criterion to estimate during confidence learning.

\subsection{Generative models for out-of-distribution detection}
\label{chap6:subsec:gan}

In \cref{chap5}, we highlighted the class-wise density estimator behaviour of KLoS, which is a crucial property in the absence of OOD training data to improve the simultaneous detection of misclassifications and OOD samples. Along with this contribution, our experiments also revealed that while performance of existing uncertainty measures are considerably improved by using training OOD samples, it also critically depends on the choice of these samples (\cref{chap5:sec:ood_training}).

Among methods using OOD samples when training deep classifier, Hendrycks \textit{et al.} \cite{hendrycks2019oe} propose to learn to classify in-distribution samples while producing high predictive entropy for OOD samples:
\begin{equation}
    \cL_{\text{OE}}(\boldsymbol{\theta}, \cD) =  \mathbb{E}_{(\x, y) \sim p_{\text{in}}} \big [ \log p(y \vert \x, \boldsymbol{\theta}) \big ] + \lambda \mathbb{E}_{\tilde{\x} \sim p_{\text{out}}} \big [ \bbH[p(y \vert \tilde{\x}, \boldsymbol{\theta})] \big ].
\end{equation}
Accordingly, they use predictive entropy to discriminate between in-distribution samples and OOD samples. It relies on the availability of a large OOD dataset, for instance 80 Million Tiny Images\footnote{\url{https://groups.csail.mit.edu/vision/TinyImages/}} with CIFAR-10 or CIFAR-100 as in-distribution dataset. With evidential models, multiple works \cite{malinin2018,malinin2019,maxgap2020} proposed a similar approach by enforcing OOD samples to have low precision $\alpha_0$:
\begin{equation}
    \cL_{\text{RKL}}(\btheta;\cD) = \mathbb{E}_{(\x,y) \sim p_{\text{in}}} \big [ \text{KL} \big ( \textrm{Dir}(\bpi \vert \alpha(\x, \btheta)~\|~ \textrm{Dir}(\bpi \vert \beta_{\text{in}}) \big ) \big ] + \lambda \mathbb{E}_{\tilde{\x} \sim p_{\text{out}}} \big [ \text{KL} \big ( \textrm{Dir}(\bpi \vert \alpha(\x, \btheta)~\|~ \textrm{Dir}(\bpi \vert \boldsymbol{1}) \big ) \big ].
\end{equation}

\begin{wrapfigure}{r}{0.30\textwidth}
    \centering
    \captionsetup{justification=centering}
    \vspace{-0.3cm}
    \includegraphics[width=0.95\linewidth]{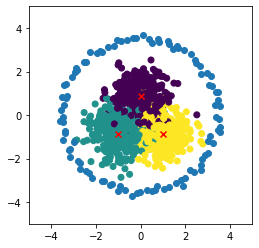}
    \caption{Illustration of ideal OOD training data on a toy dataset.}
    \vspace{-0.3cm}
    \label{chap6:fig:ideal_training_ood}
\end{wrapfigure}
The uncertainty measure used for OOD detection in that case is a dispersion measure, such as mutual information or precision.

Previous methods have been shown to be really effective to improve OOD detection. But while finding suitable OOD samples may be easy for some academic datasets, it may turn more problematic in real-world applications \cite{charpentier2020,sensoy2020}, with the risk of degrading performance with an inappropriate choice. 

Building a suitable OOD training set for real tasks is an open research perspective which could alleviate the need for real but hard-to-find OOD samples. To ensure good generalization to other types of anomalies, the main challenge is to produce samples which are close to the in-distribution, even at the boundary such as for the toy dataset shown in \cref{chap6:fig:ideal_training_ood}. Generative models such as proposed in \cite{lee2018training,sensoy2020,vernekar2019} could be an interesting solution to fulfill this problem.

\subsection{Further applications of confidence learning}
\label{chap6:subsec:applications}

In an analogous way to \cref{chap4}, the confidence learning approach could be applied to new contexts where the quality of uncertainty estimates is crucial.

\paragraph{Application to semi-supervised learning.} Unfortunately, the efficacy of deep neural networks depends on large quantities of accurately labeled training data. But the labeling process usually requires arduous and expensive efforts, which is one of the major limitations to train a fully-supervised deep neural network. If only a few labeled samples are available, it is challenging to build a successful ML system. In contrast, unlabeled data is usually abundant and can be easily or inexpensively obtained. Semi-supervised learning (SSL) \cite{SSLbook} is a learning paradigm that aims to improve learning performance from labeled data by using additional unlabeled instances. A family of approaches for SSL \cite{lee-icml2013,grandvalet-nips2005} proposed to infer pseudo-labels on unlabeled data and to re-train a network using these pseudo-labels. The key is to select the most confident labels. Obviously, due to the close connection between domain adaptation and semi-supervised learning methods, a natural extension to ConDA is to apply it in the latter context. 

\paragraph{Application to active learning.} An alternative to semi-supervised learning is active learning in which data are actively sampled to be labeled by human oracles with the goal of maximizing model performance while minimizing labeling costs. Various sampling strategies have been proposed for active learning over the years coming from different perspectives, e.g. uncertainty \cite{gal-active2017} and representativeness \cite{sener2018active}. Uncertainty-based methods are based on measures derived from the probability vector, such as entropy, MCP or margin sampling. Confidence learning could improve the selection of useful samples, such as done in a related work where the authors aim to learn the loss value \cite{Yoo_2019_CVPR}.

\paragraph{Application to multi-modal fusion.} The joint operation of several types of sensors is a key part of autonomous driving systems. Currently, the majority of systems are based on a late fusion due to safety reasons such as redundancy but also due to technical (vehicle network architecture) or commercial (use of several suppliers) constraints. As mentioned in \cref{chap1}, early multi-modal fusion could benefit from reliable uncertainty estimates. A system could rely more on a certain sensor or discard predictions from other sensors due to a low confidence estimate. On a related topic, Kendall \textit{et al.} \cite{kendall2017multi} showed that multi-task learning can be improved by weighting each task by a task-dependent uncertainty estimate. While this approach was developed for multi-output, it could be adapted in the scenario of multi-input and with confidence learning.

\subsection{From uncertainty estimation to robustness}
\label{chap6:subsec:robustness}

In this thesis, we focused on deriving reliable uncertainty estimates to ensure proper monitoring of deployed ML systems in real-world tasks. Alternatively, a whole part of AI safety literature aims to improve the robustness of deep neural networks to distribution shift. Mismatches in data distributions in test time compared to training time can cause a surprisingly large drop in predictive performance \cite{wilds2020}. Robustness to distribution shifts has recently been an increasingly popular topic among machine learning researchers \cite{Rusak2020a,hendrycks2020augmix,burns2021,robusttransformer2021}. For instance, ImageNet-trained models have been shown to lack robustness against common corruptions \cite{hendrycks2018benchmarking}, adversarial inputs \cite{goodfellow2015}, change in renditions \cite{Hendrycks_2021_ICCV} (\eg painting, embroidery, etc.). In contrast to the out-of-distribution samples studied in \cref{chap5}, all of these input manipulations do not change the semantic content of the input, and thus, machine learning models should not change their decision-making behavior in their presence. Consequently, the field is also sometimes referred to \emph{out-of-distribution generalization} \cite{Arjovsky2021PhD}.
A long-term perspective for this thesis is to leverage the tools developed for uncertainty estimation to improve out-of-distribution generalization.

%% file: synthese.tex
\section*{Introduction}
\label{synthese:sec:intro}

Depuis la victoire éclatante d'AlexNet \cite{krizhevsky2012imagenet}, une architecture de réseau de neurone convolutif, au Large Scale Visual Recognition Challenge (LSVRC) en 2012, l'apprentissage profond est omniprésent dans les domaines de la vision par ordinateur \cite{He2015, SSD2016,ChenPK0Y16}, du traitement du langage naturel \cite{MikolovKBCK10,DevlinCLT19}, de la reconnaissance vocale \cite{hinton2012deep,hannun2014speech} et de l'apprentissage par renforcement \cite{Silver_2016}. Les récentes percées en vision par ordinateur grâce à l'apprentissage profond expliquent aussi en grande partie le spectaculaire renouveau de la conduite autonome avec des acteurs technologiques majeurs comme Waymo, Tesla, Baidu et Yandex investissant dans des programmes de voitures autonomes. En tant qu'un des leaders mondiaux des capteurs automobiles, Valeo, qui finance cette thèse, se positionne au cœur de cette révolution actuelle, en développant des LiDAR de haute qualité avec leur technologie SCALA\textregistered. 

Si ces progrès sont indéniables, les robotaxis ne sont toujours pas déployés au moment de la rédaction de ce manuscrit et de nombreux défis doivent encore être résolus pour commercialiser la voiture autonome à grande échelle, notamment ceux liés à la sécurité. Les accidents survenant avec des voitures autonomes sont des exemples typiques où les répercussions peuvent être catastrophiques. Un exemple frappant est l'incident tragique qui s'est produit le 7 mai 2016 près de Williston (Floride, États-Unis) et qui a entraîné le premier décès causé par une voiture à assistance de conduite hautement automatisée \cite{NTSB2017}. Le constructeur automobile Tesla a déclaré qu'une des origines de l'accident était liée au système de vision qui avait incorrectement classé le côté blanc d'un camion-remorque comme un ciel éblouissant\footnote{\url{https://www.tesla.com/blog/tragic-loss}}. Le contrôle et la correct évaluation de la confiance du système dans ses prédictions semblent être plus que nécessaires pour déployer en toute sécurité des modèles d'apprentissage dans des environnements à fort enjeux \cite{mcallister2017}.

L'estimation de la confiance a une longue histoire en apprentissage automatique \cite{Chow1957AnOC,zaragoza1998confidence,Blatz2004,elyaniv10a}. Pourtant, une série de travaux récents ont montré que les réseaux de neurones (NNs) modernes souffrent de plusieurs inconvénients conceptuels qui les rendent peu fiables \cite{hendrycks17baseline,nguyen2015,Gal2016,kendall2015bayesian,Hein_2019_CVPR}. En classification, ils peinent à détecter les prédictions erronées \cite{hendrycks17baseline} et produisent des probabilités non calibrées \cite{guo2017}. Les NNs sont également connus pour être fragiles aux changements de distribution \cite{wilds2020}, leurs performances de prédiction diminuant sévèrement car ils ont tendance à s'appuyer sur des corrélations parasites \cite{Geirhos2020a}. Enfin, de nombreux travaux ont montré que les NNs fournissent des prédictions trop confiantes pour des échantillons loins des données d'entraînement \cite{Hein_2019_CVPR}, y compris des images trompeuses \cite{nguyen2015} ou adverses \cite{Szegedy2014}. 

Dans cette thèse, nous relevons le défi de fournir des estimations d'incertitude fiables pour les prédictions des réseaux de neurones profonds avec application en conduite autonome. En particulier, nous cherchons à améliorer la détection des prédictions erronées au moment du test en les distinguant des prédictions correctes. Les erreurs peuvent être de différentes natures et les contributions suivantes aborderont tout d'abord la tâche de détection d'erreur de classification. En plus de la détection de ces exemples au moment du test, nous élaborons également sur l'utilisation de l'approche proposée dans le cas de l'adaptation au domaine, où les approches d'auto-formation s'appuient sur les estimations d'incertitude pour sélectionner les échantillons dans la phase de ré-étiquetage. Enfin, nous considérons la présence d'anomalies et proposons de détecter simultanément les erreurs et les échantillons hors distribution à l'aide d'une seule mesure.

\section*{Apprentissage de confiance via un modèle auxiliaire}
\label{synthese:sec:confidnet}

La prédiction d'échec consiste à prédire à l'exécution si un modèle entraîné a pris une décision correcte ou non pour une entrée donnée. En détectant une prédiction erronée, un système peut décider de s'en tenir à la prédiction ou, au contraire, de la transmettre à un humain ou à un système de secours avec d'autres capteurs, ou simplement de déclencher une alarme. Étroitement liée à la prédiction d'échec, la classification avec option de rejet \cite{Chow1957AnOC}, également connue sous le nom de \emph{classification sélective}. \cite{elyaniv10a}, consiste en un scénario où le classifieur a la possibilité de rejeter une instance au lieu de prédire son étiquette. Ces deux tâches renvoient au même problème de \emph{classement ordinal}, qui vise à estimer les valeurs de confiance dont le classement des échantillons est efficace pour distinguer les prédictions correctes des prédictions incorrectes (voir \cref{chap3:fig:ordinal_ranking}).

Une méthode de référence largement utilisée avec les classifieurs de type réseaux de neurones consiste à prendre la valeur de la probabilité de la classe prédite, à savoir la \textit{probabilité de classe maximale} (MCP), donnée par la sortie de la couche softmax :
\begin{equation}
\mathrm{MCP}_F(\x) = 
\max_{k \in \cY} P(Y=k \vert \x, \hat{\btheta}) =  
\max_{k \in \cY} F(\x;\hat{\btheta})[k].
\end{equation} 
Cependant, en prenant la plus grande probabilité softmax comme estimation de confiance, MCP conduit à des valeurs de confiance élevées à la fois pour les prédictions correctes et erronées, ce qui rend difficile de leur distinction, comme le montre \cref{chap3:fig:density-plot-mcp}. Prendre l'entropie prédictive comme mesure d'incertitude n'est également pas adéquat. Dans \cref{chap3:fig:visu-entropy}, nous montrons côte à côte deux échantillons présentant une entropie similaire, issus d'un petit réseau convolutif entraîné sur SVHN, un jeu de données de numérotation urbaines\cite{svhn-dataset}. 

\paragraph{Probabilité de la Vrai Classe.}
Lorsque le modèle classifie mal un exemple, la probabilité associée à la vraie classe $y$ est inférieure à la probabilité maximale et risque d'être petite. Sur la base de cette simple observation, nous proposons de considérer plutôt cette \emph{probabilité de la vrai classe} comme une mesure de confiance appropriée.
\begin{equation}
    \mathrm{TCP}_F(\x,\,y) = P(Y=y \vert\,\x, \hat{\btheta}) = F(\x;\hat{\btheta}) [y].
\end{equation}

Dans \cref{chap3:fig:density-plot}, nous pouvons observer que TCP permet une bien meilleure séparation que MCP. En particulier, TCP offre les intéressantes garanties suivantes concernant sa capacité à caractériser les prédictions correctes ou erronées d'un modèle.
\begin{itemize}
    \item $\mathrm{TCP}_F(\bm{x},y) > 1/2$ $~\Rightarrow$ $f(\x) = y$, \ie l'exemple est correctement classé par le modèle ;
    \item $\mathrm{TCP}_F(\bm{x},y) < 1/K$ $\Rightarrow$ $f(\x) \neq y$, \ie l'exemple est mal classé par le modèle,
\end{itemize} 
Dans l'intervalle $[1/K, 1/2]$, rien ne garantit que les prédictions correctes et incorrectes ne se chevauchent pas en termes de TCP. Cependant, nous observons qu'avec des réseaux neuronaux profonds, la zone de chevauchement réelle est extrêmement petite en pratique, comme l'illustre \cref{chap3:fig:density-plot-tcp} sur le jeu de données CIFAR-10. Une explication possible vient du fait que les réseaux neuronaux profonds modernes produisent des prédictions trop confiantes et donc des probabilités non calibrées.

\paragraph{ConfidNet.}
L'utilisation du TCP comme mesure de confiance sur la sortie d'un modèle serait d'une grande aide lorsqu'il s'agit d'estimer de manière fiable sa confiance. Cependant, la vraie classe $y$ n'est évidemment pas disponible lors de l'estimation de la confiance sur les entrées de test. Nous proposons donc d'\emph{apprendre la confiance TCP à partir des données}. À cette fin, nous introduisons un \textit{modèle auxiliaire} $C$, avec des paramètres $\bomega$, qui est destiné à prédire $\text{TCP}_F$ et à agir comme une mesure de confiance pour la fonction de sélection $g$. 
Un aperçu de l'approche proposée est disponible dans \cref{chap3:fig:confidnet_network}. Ce modèle est entraîné de telle sorte qu'au moment de l'exécution, pour une entrée $\x \in \cX$ avec l'étiquette vraie (inconnue) $y$, nous avons : 
\begin{equation}
    C(\x;\bomega) \approx \text{TCP}_F(\x,y).
\end{equation}
En pratique, ce modèle auxiliaire $C$ est un réseau de neurone entraîné sous supervision complète sur $\cD$ pour produire cette estimation de confiance. Pour concevoir ce réseau, nous pouvons transférer les connaissances du réseau de classification déjà entraîné. Nous construisons ConfidNet sur une représentation intermédiaire tardive de $F$. ConfidNet est conçu comme un petit perceptron multicouche composé d'une succession de couches denses avec une activation sigmoïde finale qui produit $C(\bm{x};\bomega)\in[0,1]$. Comme nous voulons régresser un score entre $0$ et $1$, nous utilisons une fonction de perte d'erreur quadratique moyenne pour entraîner le modèle de confiance :
\begin{equation} 
\cL_{\text{conf}}(\bomega;\cD) = \frac{1}{N} \sum_{n=1}^N \big(C(\x_n;\bomega) - \text{TCP}_F(\x_n,y_n)\big)^2.
\label{synth:eq:loss-conf}
\end{equation}
Nous décomposons les paramètres du réseau de classification $F$ en $\btheta = (\btheta_{E}, \btheta_{\text{cls}})$, où $\btheta_{E}$ désigne les poids de son encodeur et $\btheta_{\text{cls}}$ les poids de ses dernières couches de classification. Comme dans l'apprentissage par transfert, l'apprentissage du réseau de confiance $C$ commence par fixer l'encodeur partagé et n'entraîner que les poids $\bphi$ de ConfidNet. Dans cette phase, la fonction de perte \cref{chap3:eq:loss-conf} est donc minimisée uniquement par rapport à $\bomega = \bphi$. Dans une deuxième phase, nous affinons le réseau complet $C$, y compris son encodeur qui est maintenant délié de l'encodeur de classification $E$ (le modèle de classification principal doit rester inchangé, par définition du problème traité). En désignant par $E'$ cet encodeur désormais indépendant, et par $\btheta_{E'}$ ses poids, cette seconde phase d'apprentissage optimise \cref{chap3:eq:loss-conf} en fonction de $\bomega = (\btheta_{E'}, \bphi)$ avec $\btheta_{E'}$ initialement fixé à $\btheta_{E}$. Nous désactivons également les couches de dropout dans cette dernière phase d'apprentissage et réduisons la vitesse d'entrainement afin d'atténuer les effets stochastiques qui pourraient amener le nouvel encodeur à trop s'écarter de l'encodeur original utilisé pour la classification. L'augmentation des données peut donc encore être utilisée. ConfidNet peut être entraîné en utilisant soit l'ensemble d'entraînement original, soit un ensemble de validation.

\paragraph{Expériences.}
Pour démontrer l'efficacité de notre méthode, nous avons implémenté et comparé des approches concurrentes d'estimation de la confiance et de l'incertitude, notamment la probabilité de classe maximale (MCP) comme méthode de référence \cite{hendrycks17baseline}, TrustScore \cite{NIPS2018_7798} et Monte-Carlo Dropout (MC-Dropout) \cite{Gal2016}. Les comparaisons sont effectués sur des jeux de données d'images d'échelle et de complexité variables : MNIST \cite{mnist}, SVHN \cite{svhn-dataset}, CIFAR-10 et CIFAR-100. Nous présentons également des expériences de segmentation sémantique sur CamVid \cite{camvid}, un jeu de données standard de scènes de conduite. Les architectures profondes de classification suivent les architectures standarement utilisées en classification d'images, telles que les architectures de perceptron multicouche (MLP), LeNet et VGG16. Pour CamVid, nous avons implémenté un modèle de segmentation sémantique SegNet, suivant \cite{kendall2015bayesian}. Enfin, nous mesurons la qualité de la prédiction des défaillances en suivant les métriques utilisées dans la littérature~\cite{hendrycks17baseline} : AUPR-Error, AUPR-Success, FPR at 95\% TPR et AUROC. Nous nous concentrerons principalement sur l'AUPR-Error, qui calcule l'aire sous la courbe Précision-Rappel en utilisant les erreurs comme classe positive. \\

Les résultats comparatifs montrent que notre approche surpasse les méthodes usuelles dans toutes les configurations, avec un écart significatif sur les petits modèles et jeux de données. Cela confirme à la fois que le TCP est un critère de confiance adéquat pour la prédiction des défaillances et que notre approche ConfidNet est capable de l'apprendre. Nous fournissons une illustration sur CamVid (\cref{chap3:fig:visu-camvid}) pour mieux comprendre notre approche pour la prédiction de défaillance. Par rapport à la méthode de base MCP, notre approche produit des scores de confiance plus élevés pour les prédictions de pixels corrects et plus faibles pour les pixels mal étiquetés, ce qui permet à l'utilisateur de mieux détecter les zones d'erreurs.

\section*{Auto-apprentissage avec confiance apprise pour l'adaptation de domaine}
\label{synthese:sec:conda}

Les systèmes de perception des voitures autonomes nécessitent une compréhension approfondie des scènes dans lesquelles ils évoluent. Pour cette raison, des modules de segmentation sémantique sont souvent incorporés afin d'obtenir des prédictions d'étiquettes de classe pour chaque pixel de la scène. Bien que les récents progrès des réseaux convolutifs profonds aient considérablement amélioré les performances de segmentation, leur efficacité dépend de grandes quantités de données d'entraînement étiquetées avec précision. Mais le processus d'étiquetage nécessite généralement l'intervention d'experts et le coût de l'annotation limite les domaines opérationnels de ces systèmes. D'un autre côté, de nombreuses données de scènes de conduite sont synthétisées par des moteurs de jeux tels que GTA5 \cite{richter-eecv2016}. Par conséquent, des travaux récents tentent d'exploiter cette supervision alternative bon marché en entraînant des modèles sur ces sources d'images et en prédisant sur des images réelles. Mais le transfert n'est pas directement efficace car on observe une baisse de performance lors de l'évaluation sur des images réelles, due à un gap entre les domaines.

L'adaptation de domaine non supervisée (UDA) est le domaine de recherche qui vise à réduire cet écart de domaine entre les domaines source et cible. Dans le contexte de l'UDA, des échantillons sources annotés et des images cibles non étiquetées sont disponibles au moment de l'entrainement. La plupart des travaux de cette ligne de recherche visent à minimiser l'écart de distribution entre le domaine source et le domaine cible, au niveau des features extraites ou de la prédiction~\cite{hoffman-arxiv2016}, potentiellement combiné à des méthodes de translation transformant les images sources pour qu'elles correspondent au `style'~\cite{Hoffman_cycada2017} du domaine cible. Récemment, l'auto-formation~\cite{Li_2019_CVPR,Zou_2019_ICCV,zou2018unsupervised} a prouvé sa capacité à augmenter les performances d'adaptation de manière significative. Le principe de ces approches est d'étiqueter automatiquement les pixels cibles les plus confiants selon la prédiction actuelle du réseau et de ré-entrainer le réseau en conséquence. Bien que cette idée soit séduisante, la présence de pseudo-étiquettes avec du bruit ou incorrectes pourrait nuire à l'entrainement du réseau de neurones. À titre d'exemple, l'utilisation d'un ratio de $70\%$ de pseudo-étiquettes dans~\cite{Li_2019_CVPR} conduit à une performance d'environ $48\%$ mIoU, ce qui est mieux que $34\%$ avec le transfert direct (uniquement entrainé sur le domaine source), mais toujours largement inférieur aux $63\%$ obtenus avec la même quantité d'étiquettes de terrain. Par conséquent, la définition de bonnes mesures de confiance pour sélectionner des prédictions fiables est d'une importance cruciale pour le développement d'un auto-apprentissage sans erreur.

Pour améliorer l'efficacité de l'auto-formation, nous proposons d'adapter notre approche d'apprentissage de confiance développée dans le chapitre précédent au contexte particulier de l'adaptation non supervisée de domaine pour la segmentation sémantique. Un réseau de confiance $C$ est appris pour prédire la confiance du réseau de segmentation sémantique $F$ entraîné par UDA et utilisé pour sélectionner uniquement les pseudo-étiquettes jugée confiantes sur les images du domaine cible, comme illustré dans \cref{chap4:fig:selecting}. À cette fin, le cadre proposé dans \cref{chap3:sec:confidnet} dans une configuration de classification d'images, et appliqué à la prédiction de classification d'images erronées, doit ici être adapté à la sortie structurée de la segmentation sémantique, qui peut être vue comme un problème de classification par pixels. Étant donné une image du domaine cible $\xtg$, nous voulons prédire à la fois sa carte sémantique $F(\xtg;\btheta)$ et, en utilisant un modèle auxiliaire avec des paramètres entraînables $\bomega$, sa carte de confiance :
\begin{equation}
    C(\xtg;\bomega) = \Cmapt \in [0,1]^{H\times W}.
\end{equation}
Étant donné un pixel $(h,w)$, si sa confiance $\Cmapt[h,w]$ est supérieure à un seuil choisi $\delta$, nous l'étiquetons avec sa classe prédite $f(\xtg)[h,w] = \arg\!\max_{k\in\cY} \Pmapt[h,w,k]$, sinon elle est masquée. Calculées sur toutes les images de $\cD_{\tg}$, ces cartes de segmentation incomplètes constituent des pseudo-étiquettes cibles qui sont utilisées pour entraîner un nouveau réseau de segmentation sémantique.

\paragraph{Entraînement.}
Pour entraîner le réseau de confiance $C$, nous proposons d'optimiser conjointement deux objectifs. Le premier objectif est une version pixel-à-pixel de la perte de confiance \cref{synth:eq:loss-conf}. Sur des images annotées du domaine source, il exige que le réseau de confiance $C$ prédise à chaque pixel le score attribué par le classifieur $F$ à la vraie classe (connue) :
\begin{equation} 
\cL_{\text{conf}}(\bomega;\cD_{\so}) = 
\frac{1}{N_{\so}} \sum_{n=1}^{N_{\so}} 
\big\| 
\Cmapsn - \text{TCP}_F(\xson,\y_{\so,n}) \big\|^2_{\text{F}}, 
\label{synth:eq:perte-conf-conda}
\end{equation}
où $\|\cdot\|_{\text{F}}$ désigne la norme de Frobenius et, pour une image $\x$ avec une carte de segmentation vraie $\y$ et une carte de segmentation prédite $F(\x;\hat{\btheta})$, on note
\begin{equation}
    \text{TCP}_F(\x,\y)[h,w] = F(\x;\hat{\btheta})\Big[h,w,\y[h,w]\Big]
\end{equation}
à l'emplacement $(h,w)$. Sur une nouvelle image d'entrée, $C$ doit prédire à chaque pixel le score que $F$ attribuera à la vraie classe inconnue, qui servira de mesure de confiance.

Cependant, par rapport à l'application du chapitre précédent, nous avons ici le problème supplémentaire du gap entre les domaines source et cible, un problème qui pourrait affecter l'entraînement du modèle de confiance comme dans l'entraînement du modèle de segmentation. Le deuxième objectif concerne donc le gap entre les domaines. Alors que le réseau de confiance $C$ apprend à estimer le TCP sur les images du domaine source, son estimation de la confiance sur les images du domaine cible peut souffrir considérablement de ce gap de domaine. Comme cela se fait classiquement dans l'UDA, nous proposons un apprentissage adversarial de notre modèle auxiliaire afin de résoudre ce problème. Plus précisément, nous voulons que les cartes de confiance produites par $C$ dans le domaine source ressemblent à celles obtenues dans le domaine cible.

Un discriminateur $D :[0,1]^{H \times W} \rightarrow \{0,1\}$, avec les paramètres $\bpsi$, est entraîné simultanément avec $C$ dans le but de reconnaître le domaine (1 pour la source, 0 pour la cible) d'une image étant donné sa carte de confiance. La fonction de perte suivante est minimisée par rapport à $\bpsi$ :
\begin{equation}
    \cL_D(\bpsi;\cD_{\so}\cup\cD_{\tg}) = 
    \frac{1}{N_{\so}}\sum\limits_{n=1}^{N_{\so}} \cL_\text{adv}(\xson,1) + \frac{1}{N_{\tg}}\sum\limits_{n=1}^{N_{\tg}} \cL_\text{adv}(\xtgn,0),
    \label{synth:eq:l_Dconf}
\end{equation}
où $\cL_\text{adv}$ désigne la perte d'entropie croisée du discriminateur basé sur les cartes de confiance :
\begin{equation}
    \cL_\text{adv}(\x,\lambda) = -\lambda\log\big(D(\Cmap;\bpsi)\big) - (1-\lambda)\log(1-D\big(\Cmap;\bpsi)\big),
\end{equation}
pour $\lambda \in \{0,1\}$, qui est fonction à la fois de $\bpsi$ et de $\bomega$. En alternance avec l'apprentissage du discriminateur à l'aide de \cref{synth:eq:l_Dconf}, l'apprentissage adversarial du réseau de confiance est effectué en minimisant, par rapport à $\bomega$, la fonction de perte suivante :
\begin{equation}
    \cL_C(\bomega;\cD_{\so}\cup\cD_{\tg}) = \cL_\text{conf}(\bomega; \cD_{\so}) + \frac{\lambda_\text{adv}}{N_{\tg}}\sum\limits_{n=1}^{N_{\tg}}\cL_\text{adv}(\xtg,1),
    \label{synth:eq:l_C}
\end{equation}
où le deuxième terme, pondéré par $\lambda_\text{adv} > 0$, encourage $C$ à produire des cartes dans le domaine cible qui confondront le discriminateur.

Le schéma d'apprentissage adversarial de confiance proposé agit également comme un régulateur pendant la formation, améliorant la robustesse de la confiance de la cible TCP inconnue. Comme l'apprentissage du modèle de confiance peut en fait être instable, l'apprentissage adversarial fournit un signal d'information supplémentaire, imposant en particulier que l'estimation de la confiance soit invariante aux changements de domaine. Nous observons empiriquement que cet apprentissage adversarial de la confiance fournit de meilleures estimations de la confiance et améliore la convergence et la stabilité du schéma d'apprentissage.

\paragraph{Architecture multi-échelle.}
Dans de nombreux jeux de données de segmentation, l'existence d'objets à des échelles différentes peut compliquer l'estimation de la confiance. Comme dans les travaux récents traitant d'échelles variables des objets~\cite{ChenPK0Y16}, nous améliorons encore notre réseau de confiance $C$ en ajoutant une architecture multi-échelle basée sur le regroupement spatial pyramidal. Cette architecture consiste en un schéma efficace en termes de calcul pour ré-échantillonner une carte de caractéristiques à différentes échelles, puis pour agréger les cartes de confiance. Nous illustrons l'architecture multi-échelle pour un réseau de confiance dans \cref{chap4:fig:multi_scale_confidence}. À partir d'une carte de features, nous appliquons en parallèle des couches convolutives à trous avec des noyaux de taille 3x3 et des taux d'échantillonnage différents, chacune d'entre elles étant suivie d'une série de quatre couches convolutives standardes avec des noyaux de taille 3x3. Contrairement aux couches convolutionnelles avec de grands noyaux, les couches de convolution à trous élargissent le champ de vision des filtres et permettent d'incorporer un contexte plus large sans augmenter le nombre de paramètres et le temps de calcul. Les caractéristiques résultantes sont ensuite additionnées avant d'être sur-échantillonnées à la taille de l'image originale de $H\times W$. Nous appliquons une activation sigmoïde finale pour obtenir une carte de confiance avec des valeurs comprises entre 0 et 1.

\paragraph{Expériences.}
Nous considérons la tâche spécifique d'adaptation de données synthétiques à des données réelles dans des scènes urbaines. Nous expérimentons avec deux jeux de données sources synthétiques -- SYNTHIA~\cite{ros-cvpr2016} et GTA5~\cite{richter-eecv2016}. -- et deux ensembles de données cibles réelles : Cityscapes~\cite{cordts-cvpr2016} et Mapillary Vistas~\cite{neuhold-iccv2017}.
Nous évaluons la méthode d'auto-apprentissage proposée sur trois architectures d'adaptation de domaine état de l'art (au moment du projet): \textit{AdaptSegNet}~\cite{Tsai_adaptseg_2018}, \textit{AdvEnt}~\cite{vu2018advent}, \textit{DADA}~\cite{vu-iccv19}. Elles sont toutes basées sur DeepLabv2~\cite{ChenPK0Y16}, un réseau de segmentation sémantique standard. ConDA apporte une amélioration systématique des performances par rapport à l'auto-formation basée sur la probabilité de classe maximale (MCP) standard. \cref{chap4:fig:qualitative_results} présente les résultats qualitatifs de ces méthodes de pseudo-étiquetage. En particulier, ConDA obtient des résultats état de l'art (au moment du projet) sur trois benchmarks de segmentation UDA (GTA5 $\rightarrow$ Cityscapes, SYNTHIA\uda Cityscapes et SYNTHIA\uda Mapillary Vistas) en utilisant le réseau de segmentation standard DeepLabv2~\cite{ChenPK0Y16} comme backbone.

\section*{Détection conjointe d'erreurs et d'anomalies avec les modèles évidentiels}
\label{synthese:sec:klos}

Les modèles d'apprentissage automatique reposent généralement sur l'hypothèse que les données source et cible sont indépendantes et identiquement distribuées (\textit{i.i.d.}). Pourtant, dans la pratique, les changements de distribution apparaissent naturellement dans de nombreux scénarios du monde réel. Par exemple, les voitures autonomes ont du mal à être performantes dans des conditions différentes de celles de l'entrainement, comme les variations de météo \cite{volk2019}, de lumière \cite{dai2018} et de pose d'objets \cite{Alcorn_2019_CVPR}. Pire encore, les modèles peuvent être exposés à des entrées provenant de classes non vues qu'ils tenteront de prédire malgré tout. Ces échecs peuvent passer inaperçus car ils n'entraînent pas d'erreurs explicites dans le modèle. 

Alors que les travaux précédents de la littérature traitent séparément de la détection des erreurs de classification et de la détection des entrées hors distributions (OOD), nous soutenons qu'il est nécessaire pour un système de reconnaissance d'être capable d'identifier à la fois les erreurs de classification et les entrées inconnues/invisibles au moment du test pour un déploiement sûr dans des environnements ouverts \cite{Bendale_2015_CVPR}. Nous illustrons cette tâche dans \cref{chap5:fig:simultaneous}. En particulier, nous constatons dans \cref{chap5:sec:exp} que toutes les approches précédentes ne sont pas aussi performantes sur les deux tâches de détection, ce qui atténue leur capacité sur la tâche de détection conjointe.

Pour répondre à la tâche de détection simultanée des mauvaises classifications et des échantillons OOD, une bonne mesure d'incertitude devrait discriminer les prédictions correctes et les prédictions erronées pour les échantillons issus de la même distribution que celle d'entrainement tout en augmentant les valeurs d'incertitudes pour les entrées loin de la distribution. Par conséquent, elle devrait capturer à la fois l'incertitude aléatoire et épistémique. Les approches bayésiennes \cite{Gal2016,NIPS2019_9472} et les ensembles \cite{deepensembles2017,ovadia2019} sont des méthodes qui induisent une estimation plus précise de l'incertitude épistémique. Ces techniques produisent une densité de probabilité sur la distribution catégorielle prédictive $p(y \vert \x, \cD)$ obtenue par échantillonnage comme le montre la ligne supérieure de \cref{chap2:fig:simplex}. Mais cela se fait au prix d'un coût de calcul accru.

Une classe récente de modèles, appelée réseaux de neurones évidentiel (ENN) \cite{sensoy2018,malinin2019,beingbayesian2020}, propose plutôt d'apprendre explicitement les paramètres de concentration d'une distribution de Dirichlet $q_{\btheta}(\bpi \vert \x)= \text{Dir}. \big ( \bpi \vert \balpha \big ) $ sur les probabilités de sortie. Il a été démontré qu'elles améliorent la généralisation \cite{beingbayesian2020} et la détection des OOD \cite{maxgap2020}. L'apprentissage des ENN est formulé comme une approximation variationnelle visant à minimiser la divergence de Kullback-Leibler (KL) entre la distribution $q_{\btheta}(\bpi \vert \x) $ et la vraie distribution postérieure $p(\bpi \vert \x, y)$. En suivant \cite{beingbayesian2020}, nous utilisons un prior uniforme $p(\bpi \vert \x) ~{=}~ \text{Dir} \big ( \bpi \vert \boldsymbol{1} \big )$. La fonction de perte d'entraînement d'un ENN est :
\begin{equation}
    \cL_{\text{var}}(\btheta;\cD) =
    \frac{1}{N} \sum_{(\x, y) \in \cD} \Big (\psi(\alpha_y) - \psi(\alpha_0) \\\
    + \lambda \text{KL} \big(\text{Dir}( \bpi \vert \balpha)~\|\text{Dir}( \bpi \vert \mathbf{1} ) \big) \Big ),
    \label{eq:dev_loss_vi}
\end{equation}
avec l'hyperparamètre $\lambda >0$. En particulier, la minimisation de cette fonction de perte impose que la précision de l'échantillon d'apprentissage $\alpha_0$ reste proche de $C~{+}~1/\lambda$. \\

Basés sur le cadre de la logique subjective \cite{josan2016sublogic}, les modèles évidentiels capturent différentes sources d'incertitude. L'incertitude de premier ordre concerne l'espérance de la distribution de Dirichlet et est causée par des preuves contradictoires, par exemple la confusion entre classes. L'incertitude de second ordre exprime le manque d'évidence dans une prédiction \cite{Shi2020MultifacetedUE}, qui est caractérisée par la dispersion de la distribution de Dirichlet. Par exemple, les huskies partagent de nombreuses caractéristiques avec les loups bien qu'ils soient une race de chien, ce qui entraîne une grande incertitude du premier ordre due à la confusion de classe. En présence d'un dessin d'un husky, on s'attend à une confusion de classe similaire, mais à une quantité moindre de preuves en raison du changement de distribution.  

De manière surprenante, les précédents travaux de la littérature n'exploitent pas la distribution sur les probabilités sur le simplex pour dériver une telle mesure jointe des deux sources d'incertitude. Certaines méthodes se concentrent sur la détection des OOD en caractérisant uniquement la dispersion de la distribution, \eg, en utilisant l'information mutuelle \cite{malinin2019}. Les approches ciblant l'incertitude totale réduisent en fait les distributions de probabilité sur le simplexe à leur valeur en espérance et calculent des mesures d'incertitude du premier ordre, \eg, l'entropie prédictive \cite{sensoy2018}. Cependant, ces mesures sont invariantes à la dispersion de la distribution, alors que l'incertitude causée par la confusion de classe et le manque de preuves devrait être cumulative, une propriété naturellement remplie par la variance prédictive dans la régression bayésienne \cite{murphy2012machine}. En outre, certaines méthodes pour les modèles évidentiels utilisent des données auxiliaires pendant la formation afin d'imposer un étalement de distribution plus élevé sur les entrées OOD. Mais lorsque l'accès aux données d'entraînement OOD n'est pas envisageable, le comportement de grande dispersion n'est pas garanti pour tous les exemples OOD ~ \cite{charpentier2020,sensoy2020} et les mesures d'incertitude d'ordre deux peinent à les discriminer des exemples issus de la distributio d'entrainement.

\paragraph{KLoS.}
Nous introduisons une nouvelle mesure, appelée \emph{KLoS}, qui calcule la divergence KL entre la sortie du modèle et une distribution de Dirichlet "peaké" avec des concentrations $\boldsymbol{\gamma}_{\hat{y}}$ concentrées sur la classe \textit{prédite} $\hat{y}$ :
\begin{equation}
    \textrm{KLoS}(\boldsymbol{x}) \triangleq \textrm{KL} \Big ( \textrm{Dir} \big (\bpi \vert \balpha \big ) ~\|~ \textrm{Dir} \big ( \bpi \vert \boldsymbol{\gamma}_{\hat{y}} \big ) \Big ) ,
    \label{synth:eq:kl_pred}
\end{equation} 
où $\balpha= \exp f(\boldsymbol{x}, \btheta)$ sont la sortie du modèle et $\boldsymbol{\gamma}_{\hat{y}} = (1,\ldots,1,\tau,1,\ldots,1)$ sont les paramètres de concentration uniforme sauf pour la classe prédite avec $\tau~{=}~1/\lambda+1$. \\

Plus KLoS est petit, plus la prédiction est certaine. Les prédictions correctes auront des distributions de Dirichlet similaires à la distribution de Dirichlet du prototype $\boldsymbol{\gamma}_{\hat{y}}$ et seront donc associées à un score d'incertitude faible (\cref{chap5:fig:simplex_behavior}a). Les échantillons présentant une confusion de classe élevée présenteront une distribution de probabilité en espérance plus proche du centre du simplex que le prototype de classe en espérance
$p^*_{\hat{y}} = (\frac{1}{K-1+\tau},\cdots,\frac{\tau}{K-1+\tau},\ldots,\frac{1}{K-1+\tau})$, ce qui entraîne un score KLoS plus élevé (\cref{chap5:fig:simplex_behavior}b). De même, KLoS pénalise également les échantillons dont la précision $\alpha_0$ est différente de la précision $\alpha_0^*=\tau +K-1$ du prototype $\boldsymbol{\gamma}_{\hat{y}}$. Les échantillons dont la quantité d'évidence est inférieure (\cref{chap5:fig:simplex_behavior}c) et supérieure (\cref{chap5:fig:simplex_behavior}d) à $\alpha_0^*$ reçoivent un score KLoS plus élevé.

Puisque les échantillons de la distribution doivent avoir une précision proche de $\alpha_0^*$ pendant l'entrainement, les prototypes par classe sont des estimations fines des paramètres de concentration des données d'entrainement pour chaque classe. Par conséquent, KLoS est une métrique basée sur la divergence, qui n'a besoin que des données de la distribution pendant l'entrainement pour calculer ses prototypes. Ce comportement est illustré dans \cref{chap5:subsec:synth_exp}. La mesure proposée sera efficace pour détecter différents types d'échantillons OOD dont la précision est loin de $\alpha_0^*$. En revanche, les mesures d'incertitude de second ordre, telles que l'information mutuelle, supposent que les échantillons OOD ont des $\alpha_0$ plus petits, une propriété difficile à respecter pour les modèles entrainés uniquement avec des échantillons issus de la distribution d'entrainement (voir \cref{chap5:fig:intro_density_plot}). Dans \cref{chap5:sec:ood_training}, nous explorons plus en profondeur l'impact du choix des données d'entraînement OOD sur les valeurs réelles de $\alpha_0$ pour les échantillons OOD.

\paragraph{KLoSNet.}
Lorsque le modèle classe mal un exemple, c'est-à-dire que la classe prédite $\hat{y}$ diffère de la vérité terrain $y$, KLoS mesure la distance entre la sortie du ENN et le postérieur $p(\bpi \vert \x, \hat{y})$ estimé sur la mauvaise classe. Mesurer plutôt la distance à la distribution postérieure à la vraie classe $p(\pi \vert \x, y)$ donnerait plus probablement une plus grande valeur, reflétant le fait que le classifieur a fait une erreur. Ainsi, une meilleure mesure pour la détection des erreurs de classification serait : 
\begin{equation}
\label{eq:klos*}
    \textrm{KLoS}^*(\x, y) \triangleq \textrm{KL} \Big ( \textrm{Dir} \big (\bpi \vert \balpha \big ) ~\|~ \textrm{Dir} \big ( \bpi \vert \boldsymbol{\gamma}_{y} \big ) \Big ),
\end{equation}
où $\boldsymbol{\gamma}_{y}$ correspond aux concentrations uniformes sauf pour la \emph{vrai} classe $y$ avec $\tau ~{=}~ 1 / \lambda + 1$.

Évidemment, la vraie classe d'une prédiction n'est pas disponible lors de l'estimation de la confiance sur des échantillons de test. Nous proposons d'apprendre KLoS$^*$ en introduisant un réseau de neurone auxiliaire de confiance, appelé KLoSNet, avec des paramètres $\bomega$, qui produit une prédiction de confiance $C(\x, \bomega)$. KLoSNet consiste en un petit décodeur, composé de plusieurs couches denses attachées à l'avant-dernière couche du réseau de classification original. Pendant l'apprentissage, nous cherchons $\bomega$ tel que $C(\x, \bomega)$ soit proche de $\text{KLoS}^*(\x,y)$, en minimisant 
\begin{equation} 
\label{eq:loss-klosnet}
\cL_{\text{KLoSNet}}(\bomega ; \cD) = \frac{1}{N} \sum_{(\x,y)\in\mathcal{D}} \big\| C(\x,\bomega) - \textrm{KLoS}^*(\x, y) \big\|^2. 
\end{equation}

\paragraph{Expériences.}
Nous avons évalué notre approche sur la tâche de détection simultanée des mauvaises classifications et des échantillons OOD par rapport à diverses méthodes de référence, y compris les métriques d'incertitude de premier et de second ordre, les méthodes de post-entrainement pour la détection OOD \cite{odin2018,mahalanobis2018}) et notre précédent travail ConfidNet. Les prédictions correctes sont considérées comme des échantillons positifs tandis que les entrées mal classées et les exemples OOD constituent des échantillons négatifs. Des expériences sont menées avec les architectures VGG-16 \cite{Simonyan15} et ResNet-18 \cite{resnet2015} sur les jeux de données CIFAR-10 et CIFAR-100 \cite{Krizhevsky09}. Les résultats montrent que KLoSNet agit comme un estimateur de densité par classe et surpasse les mesures d'incertitude actuelles.

La littérature sur les modèles évidentiels ne traite que d'un ensemble d'entraînement OOD en lien avec le jeux de données d'entrainement, \eg, {CIFAR-100 pour les modèles entraînés sur CIFAR-10}. Dans \cref{chap5:fig:compa_OOD_training}, nous faisons varier l'ensemble d'entraînement OOD utilisé pour entraîner un modèle évidentiel avec la fonction de perte de divergence KL inverse \cite{malinin2019} et évaluons les performances en utilisant TinyImageNet comme ensemble de test OOD. Comme prévu, l'utilisation de CIFAR-100 comme données d'entraînement OOD améliore les performances pour chaque mesure (MCP, Mut.\,Inf. et KLoS). Cependant, l'amélioration apportée par l'entraînement avec des échantillons OOD dépend fortement de l'ensemble de données choisi. La performance de Mut.\,Inf. diminue de 92,6\% AUC avec CIFAR-100 à 82,9\% en passant à LSUN, et devient même pire avec SVHN (78,5 \%) par rapport à l'utilisation de données OOD (80,6\%). Nous constatons également que KLoS surpasse ou est à égalité avec MCP et Mut.\,Inf. dans tous les cas. Plus important encore, l'utilisation de KLoS sur des modèles sans données d'entraînement OOD donne de meilleures performances de détection que d'autres mesures prises à partir de modèles entraînés avec des échantillons OOD inappropriés, c'est-à-dire tous les jeux de données OOD autres que CIFAR-100.

\section*{Conclusion et perspectives}
\label{synthese:sec:conclusion}

La principale contribution de cette thèse est d'utiliser un modèle de confiance auxiliaire pour apprendre la confiance d'une prédiction d'un réseaux de neurone profond en classification. Étant donné un modèle de classification entraîné, le modèle de confiance apprend à estimer à partir des données un critère adéquat dérivé du classifieur, tel que la probabilité de la vrai classe pour les réseaux de neurones standard et KLoS pour les modèles évidentiels. Au moment du test, nous utilisons directement la sortie du modèle de confiance comme estimation de l'incertitude. L'un des principaux avantages de cette méthode est d'être agnostique en termes d'architecture : dans nos expériences, nous avons réussi à améliorer l'estimation de l'incertitude pour les modèles de classification avec différentes architectures d'apprentissage profond (MLP, LeNet, VGGs, ResNets). Nous avons appliqué notre approche à trois tâches : prédiction d'échec (\cref{chap3}), adaptation non supervisée de domaine pour la segmentation sémantique (\cref{chap4}), et détection simultanée des erreurs de distribution et des échantillons hors distribution (\cref{chap5}). Pour chaque tâche, il y a deux défis principaux à relever : (1) quel critère utiliser, et (2) comment entraîner efficacement le modèle de confiance. 

Discutons maintenant des directions intéressantes qui pourraient être abordées dans des travaux futurs en relation avec nos contributions.

\paragraph{Génération de données d'erreurs pour faciliter l'apprentissage de confiance}.
L'apprentissage de confiance a montré des améliorations significatives par rapport aux méthodes de références pour l'estimation de l'incertitude dans chacun des travaux considérés. Néanmoins, l'apprentissage du modèle auxiliaire dépend de la qualité du jeu de données, c'est-à-dire du nombre d'erreurs disponibles. Les réseaux de neurones modernes sont sur-paramétrés et ont tendance à sur-apprendre les données d'entraînement, atteignant ainsi une grande précision sur les jeux d'entraînement et ne laissant qu'une petite fraction d'échantillons mal classés. Nous pensons que ce problème de déséquilibre des données atténue les performances de l'apprentissage de confiance. Pour résoudre le problème du déséquilibre, une solution serait de générer artificiellement des erreurs. Les perturbations adversariales sont de petites perturbations de l'entrée qui sont presque imperceptibles pour les humains, mais qui trompent un réseau de neurone, transformant ainsi une prédiction correcte en une prédiction erronée. Un ensemble combiné d'entrées authentiques et adverses permettrait de rééquilibrer la formation. Il a été démontré que Mix-up \cite{zhang2018mixup}, et plus généralement les techniques agressives d'augmentation des données telles qu'AugMix \cite{hendrycks2020augmix} et CutMix \cite{yun2019cutmix} améliorent la robustesse et pourraient également être appliquées ici pour générer des échantillons avec des probabilités mixtes, fournissant ainsi une plus grande gamme de valeurs TCP pour l'apprentissage de la confiance.

\paragraph{Apprentissage de la confiance d'un ensemble}
En tant qu'alternative simple aux méthodes entièrement bayésiennes, les ensembles a été un sujet de recherche populaire au sein des méthodes probabilistes \cite{Dietterich00ensemblemethods,Rokach2010,deepensembles2017}. Pourtant, lorsqu'il s'agit de prédire les défaillances, les méthodes précédentes reposent sur la réduction des prédictions en un seul vecteur de probabilité moyen et sur la dérivation des mesures habituelles telles que le MCP et l'entropie. Une direction intéressante serait de combiner l'idée de l'apprentissage de confiance au contexte des ensembles. La solution la plus simple serait d'entraîner un modèle auxiliaire pour régresser la valeur TCP du vecteur de probabilité moyen. Les expériences préliminaires ont montré des difficultés à converger vers la régression de telles valeurs moyennes, car chaque modèle individuel se comporte différemment et nous ne pouvons pas nous fier aux poids initialisés d'un modèle arbitraire. Nous pourrions essayer d'attacher un modèle auxiliaire à chaque membre de l'ensemble et de les entraîner séparément avant de faire la moyenne des sorties de tous les modèles de confiance. Une approche intelligente impliquerait de tirer parti de la diversité des prédictions pour affiner le critère à estimer pendant l'apprentissage de la confiance.

\paragraph{Modèles génératifs pour la détection d'échantillons hors distribution}
Dans le \cref{chap5}, nous avons mis en évidence le comportement d'estimateur de densité par classe de KLoS, qui est une propriété cruciale en l'absence de données d'entraînement OOD pour améliorer la détection simultanée des erreurs de classification et des échantillons OOD. Parallèlement à cette contribution, nos expériences ont également révélé que si les performances des mesures d'incertitude existantes sont considérablement améliorées par l'utilisation d'échantillons OOD, elles dépendent aussi de manière critique du type de ces échantillons (\cref{chap5:sec:ood_training}). Parmi les méthodes utilisant des échantillons OOD lors de l'apprentissage de classifieur profonds, Hendrycks \textit{et al.} \cite{hendrycks2019oe} proposent d'apprendre à classer les échantillons de la distribution d'entrainement tout en produisant une entropie prédictive élevée pour les échantillons OOD. En conséquence, ils utilisent l'entropie prédictive pour distinguer les échantillons de la distribution d'entraiement des échantillons hors distribution. Cette méthode repose sur la disponibilité d'un grand ensemble de données OOD, par exemple 80 millions TinyImages avec CIFAR-10 ou CIFAR-100 comme jeu de données de distribution. Mais si trouver des échantillons OOD appropriés peut être facile pour certains jeux de données académiques, cela peut s'avérer plus problématique dans les applications du monde réel \cite{charpentier2020,sensoy2020}, avec le risque de dégrader les performances avec un choix inapproprié. La construction d'un ensemble d'entraînement OOD adapté aux tâches réelles est une perspective de recherche qui pourrait alléger le besoin d'échantillons OOD réels mais difficiles à trouver. Pour garantir une bonne généralisation à d'autres types d'anomalies, le principal défi consiste à produire des échantillons proches de la distribution, même à la limite, comme pour l'ensemble de données jouet présenté dans \cref{chap6:fig:ideal_training_ood}. Les modèles génératifs tels que proposés dans \cite{lee2018training,sensoy2020,vernekar2019} pourraient constitués une solution intéressante pour répondre à ce problème.

\paragraph{Applications supplémentaires de l'apprentissage de confiance}
De manière analogue au \cref{chap4}, l'approche de l'apprentissage de confiance pourrait être appliquée à de nouveaux contextes où la qualité des estimations d'incertitude est cruciale. L'apprentissage semi-supervisé (SSL) \cite{SSLbook} est un paradigme d'apprentissage qui vise à améliorer les performances d'apprentissage à partir de données étiquetées et des instances non étiquetées supplémentaires. Une famille d'approches pour SSL \cite{lee-icml2013,grandvalet-nips2005} propose de prédire des pseudo-étiquettes sur des données non étiquetées et de réentraîner un réseau en utilisant ces pseudo-étiquettes. La clé est de sélectionner les étiquettes les plus fiables. De toute évidence, et en raison du lien étroit entre l'adaptation de domaine et les méthodes d'apprentissage semi-supervisé, une extension naturelle de ConDA consiste à l'appliquer dans ce dernier contexte. Une alternative à l'apprentissage semi-supervisé est l'apprentissage actif dans lequel les données sont échantillonnées pour être étiquetées par des oracles humains dans le but de maximiser la performance du modèle tout en minimisant les coûts d'étiquetage. Diverses stratégies d'échantillonnage ont été proposées pour l'apprentissage actif au fil des ans, selon des perspectives différentes, par exemple l'incertitude \cite{gal-active2017} et la représentativité \cite{sener2018active}. L'apprentissage par la confiance pourrait améliorer la sélection d'échantillons utiles, comme cela est fait dans un travail connexe où les auteurs visent à apprendre la valeur de perte \cite{Yoo_2019_CVPR}.

%% file: appendixA/appendixA.tex
\section{Effect on confidence loss}
\label{appxA:sec:ablation_loss}

The influence of the loss (MSE, BCE, Focal Loss or Ranking based on TCP) is analysed for SVHN, CIFAR10 and CamVid in \cref{appxA:tab:loss_variants}. We also tested the normalized variant of the TCP confidence criterion, \textit{n}TCP. We can observe that its performance is lower than the one of TCP on small datasets such as CIFAR-10 where few errors are present, but higher on larger datasets such as CamVid where each pixel is a sample. This emphasizes once again the complexity of incorrect/correct classification training.

\begin{table}[ht]
  \caption[Full detail of the effect of the loss on failure prediction with ConfidNet]{\textbf{Effect of the loss and of the confidence criterion on the error-detection performance
of ConfidNet}. Comparison in between proposed MSE and three other alternatives, all based on TCP as confidence criterion, except last one which is MSE with normalized $\mathrm{TCP}$ (\textit{n}TCP).   
This table extends \cref{chap3:tab:loss-analysis}.}
\label{appxA:tab:loss_variants}
  \centering
  \begin{adjustbox}{max width=\textwidth}
  \begin{tabular}{clrrr}
    \toprule
    Dataset & Loss & FPR\,@\,95\%\,TPR\,$\downarrow$ & AUPR\,$\uparrow$ & AUROC \,$\uparrow$\\
    \midrule
    \multirow{5}{*}{\shortstack[c]{\ubold{SVHN} \\ SmallConvNet}} & BCE & 29.34\% & 50.00\% & 92.76\% \\
    & Focal & 28.67\% & 49.96\% & 93.01\% \\
    & Ranking & 31.04\% & 48.11\% & 92.90\% \\
    & \textit{n}$\mathrm{TCP}$  & 30.19\%& 47.04\% & 93.12\% \\
    & $\mathrm{TCP}$ & \ubold{28.58\%} & \ubold{50.72\%} & \ubold{93.44\%} \\
    \midrule
    \multirow{5}{*}{\shortstack[c]{\ubold{CIFAR-10} \\ VGG-16}} & BCE & 45.20\% & 47.95\% & 91.94\% \\
    & Focal & 45.20\% & 47.76\% & 91.93\% \\
    & Ranking & 46.99\% & 44.04\% & 91.49\% \\
    & \textit{n}$\mathrm{TCP}$ & 45.02\% & 48.78\% & 92.06\% \\
    & $\mathrm{TCP}$ & \ubold{44.94\%} & \ubold{49.94\%} & 92.12\% \\
    \midrule
    \multirow{5}{*}{\shortstack[c]{\ubold{CamVid} \\ SegNet}} 
    & BCE & 61.68\% & 48.96\% & 83.41\% \\
    & Focal & 61.64\% & 49.05\% & 84.09\% \\
    & \textit{n}$\mathrm{TCP}$  & \ubold{60.41\%} & \ubold{51.35\%} & \ubold{85.18\%} \\
    & $\mathrm{TCP}$ & 61.52\% & 50.51\% & 85.02\% \\
    \bottomrule
  \end{tabular}
  \end{adjustbox}
  \label{tab:loss_variants}
\end{table}

\section{Empirical error and success distributions}
\label{appxA:sec:plots}

In this section, we provide the plots, analogous to Figure 1 in the main paper, that show the distribution of the confidence measures over correct and incorrect predictions respectively, for each dataset and each model in our failure prediction experiments. We also include absolute numbers of incorrect and correct predictions grouped into 3 bins (`$>1/K$', `$[\frac{1}{K},\frac{1}{2}]$' and `$>1/2$') to validate our assumptions about TCP's properties. The plots are available for MNIST with MLP in Fig.~\ref{appxA:fig:density_plot_mnist_mlp}, for MNIST with a small convnet in Fig.~\ref{appxA:fig:density_plot_mnist_conv}, for SVHN with a small convnet in Fig.~\ref{appxA:fig:density_plot_svhn_conv}, for CIFAR-100 with VGG-16 in Fig.~\ref{appxA:fig:density_plot_cifar100_vgg16} and for CamVid with SegNet in Fig.~\ref{appxA:fig:density_plot_camvid_segnet}.

\begin{figure}[ht]
\centering
\caption[Distributions of MCP and TCP confidence estimates computed over correct and erroneous predictions by a trained MLP on MNIST]{Distributions of MCP and TCP confidence estimates computed over correct and erroneous predictions by a trained \textbf{MLP on MNIST.}}
\begin{minipage}[c]{0.45\linewidth}
\centering
    \includegraphics[width=\linewidth]{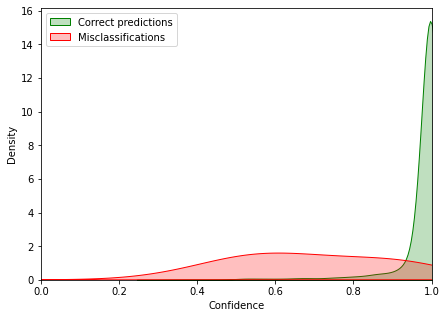}
    \subcaption{MCP}
\end{minipage}%
\begin{minipage}{0.45\linewidth}
\centering
    \includegraphics[width=\linewidth]{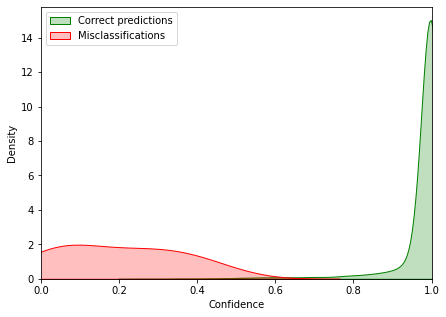}
    \subcaption{TCP}
\end{minipage}
\begin{adjustbox}{max width=\textwidth}
\begin{tabular}{l|ccc|ccc|rr}
\toprule
 Model & \multicolumn{3}{c}{Nb. of Errors} & \multicolumn{3}{c|}{Nb. of Successes} & AUPR\,$\uparrow$ & AUROC \,$\uparrow$\\
& $>1/K$ & $[\frac{1}{K},\frac{1}{2}]$ & $>1/2$ & $<1/K$ & $[\frac{1}{K},\frac{1}{2}]$ & $>1/2$ & & \\
\midrule
MCP & 0 & 25 & 170 & 0 & 28 & 9777 & 37.70\% & 97.13\% \\
TCP & 81 & 114 & 0 & 0 & 28 & 9777 & 98.77\% & 99.98\% \\
\bottomrule
\end{tabular}
\end{adjustbox}
\label{appxA:fig:density_plot_mnist_mlp}
\end{figure} 

\begin{figure}[ht!]
\centering
\caption[Distributions of MCP and TCP confidence estimates computed over correct and erroneous predictions by a trained small ConvNet model on MNIST]{Distributions of MCP and TCP confidence estimates computed over correct and erroneous predictions by a trained \textbf{small ConvNet model on MNIST.}}
\begin{minipage}[c]{0.45\linewidth}
\centering
    \includegraphics[width=\linewidth]{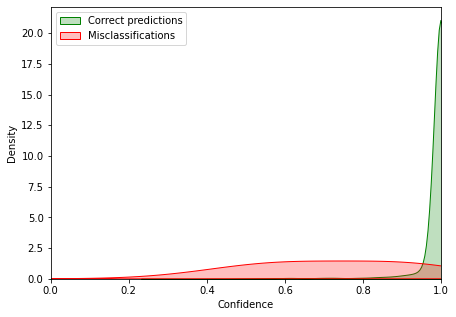}
    \subcaption{MCP}
\end{minipage}%
\begin{minipage}{0.45\linewidth}
\centering
    \includegraphics[width=\linewidth]{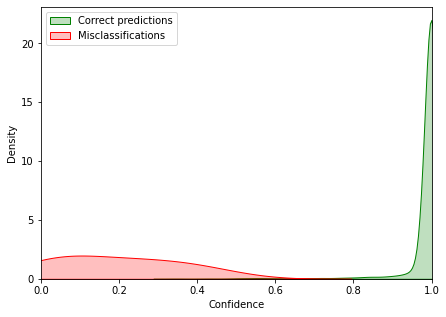}
    \subcaption{TCP}
\end{minipage}
\begin{adjustbox}{max width=\textwidth}
\begin{tabular}{l|ccc|ccc|rr}
\toprule
 Model & \multicolumn{3}{c}{Nb. of Errors} & \multicolumn{3}{c|}{Nb. of Successes} & AUPR \,$\uparrow$ & AUROC \,$\uparrow$\\
& $>1/K$ & $[\frac{1}{K},\frac{1}{2}]$ & $>1/2$ & $<1/K$ & $[\frac{1}{K},\frac{1}{2}]$ & $>1/2$ & & \\
\midrule
MCP & 0 & 8 & 82 & 0 & 11 & 9899 & 35.05\% & 98.63\% \\
TCP & 32 & 58 & 0 & 0 & 11 & 9899 & 99.41\% & 99.41\% \\
\bottomrule
\end{tabular}
\end{adjustbox}
\label{appxA:fig:density_plot_mnist_conv}
\end{figure} 

\begin{figure}[ht!]
\centering
\caption[Distributions of MCP and TCP confidence estimates computed over correct and erroneous predictions by a trained small ConvNet architecture on SVHN]{Distributions of MCP and TCP confidence estimates computed over correct and erroneous predictions by a trained \textbf{small ConvNet architecture on SVHN.}}
\begin{minipage}[c]{0.45\linewidth}
\centering
    \includegraphics[width=\linewidth]{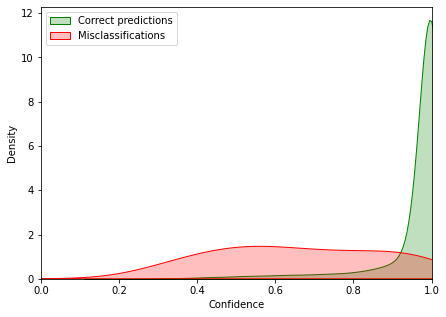}
    \subcaption{MCP}
\end{minipage}%
\begin{minipage}{0.45\linewidth}
\centering
    \includegraphics[width=\linewidth]{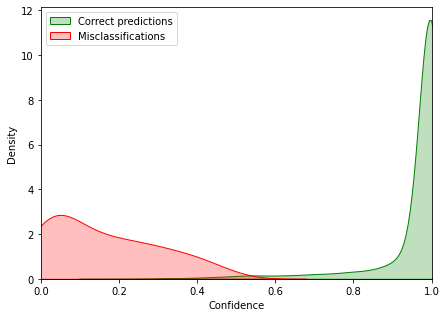}
    \subcaption{TCP}
\end{minipage}
\begin{adjustbox}{max width=\textwidth}
\begin{tabular}{l|ccc|ccc|rr}
\toprule
 Model & \multicolumn{3}{c}{Nb. of Errors} & \multicolumn{3}{c|}{Nb. of Successes} & AUPR\,$\uparrow$ & AUROC \,$\uparrow$\\
& $>1/K$ & $[\frac{1}{K},\frac{1}{2}]$ & $>1/2$ & $<1/K$ & $[\frac{1}{K},\frac{1}{2}]$ & $>1/2$ & & \\
\midrule
MCP & 0 & 329 & 857 & 0 & 206 & 24640 & 48.18\% & 93.20\% \\
TCP & 500 & 686 & 0 & 0 & 206 & 24640 & 98.93\% & 99.95\% \\
\bottomrule
\end{tabular}
\end{adjustbox}
\label{appxA:fig:density_plot_svhn_conv}
\end{figure} 

\begin{figure}[ht!]
\centering
\caption[Distributions of MCP and TCP confidence estimates computed over correct and erroneous predictions by a trained VGG-16 model on CIFAR-100]{Distributions of MCP and TCP confidence estimates computed over correct and erroneous predictions by a trained \textbf{VGG-16 model on CIFAR-100.}}
\begin{minipage}[c]{0.45\linewidth}
\centering
    \includegraphics[width=\linewidth]{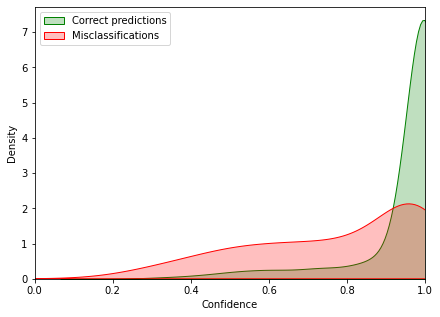}
    \subcaption{MCP}
\end{minipage}%
\begin{minipage}{0.45\linewidth}
\centering
    \includegraphics[width=\linewidth]{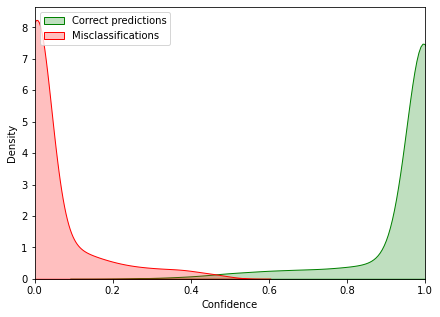}
    \subcaption{TCP}
\end{minipage}
\begin{adjustbox}{max width=\textwidth}
\begin{tabular}{l|ccc|ccc|rr}
\toprule
 Model & \multicolumn{3}{c}{Nb. of Errors} & \multicolumn{3}{c|}{Nb. of Successes} & AUPR\,$\uparrow$ & AUROC \,$\uparrow$\\
& $>1/K$ & $[\frac{1}{K},\frac{1}{2}]$ & $>1/2$ & $<1/K$ & $[\frac{1}{K},\frac{1}{2}]$ & $>1/2$ & & \\
\midrule
MCP & 0 & 603 & 2801 & 0 & 118 & 6478 & 71.99\% & 85.67\% \\
TCP & 2724 & 680 & 0 & 0 & 118 & 6478 & 99.91\% & 99.91\% \\
\bottomrule
\end{tabular}
\end{adjustbox}
\label{appxA:fig:density_plot_cifar100_vgg16}
\end{figure} 

\begin{figure}[ht!]
\centering
\caption[Distributions of MCP and TCP confidence estimates computed over correct and erroneous predictions by a trained SegNet model on CamVid]{Distributions of MCP and TCP confidence estimates computed over correct and erroneous predictions by a trained \textbf{SegNet model on CamVid.}}
\begin{minipage}[c]{0.45\linewidth}
\centering
    \includegraphics[width=\linewidth]{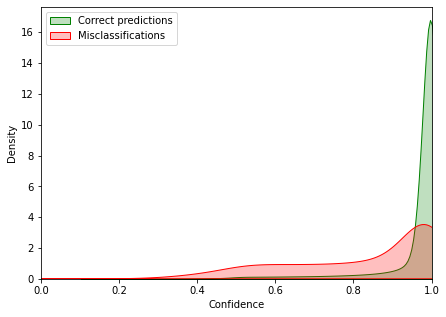}
    \subcaption{MCP}
\end{minipage}%
\begin{minipage}{0.45\linewidth}
\centering
    \includegraphics[width=\linewidth]{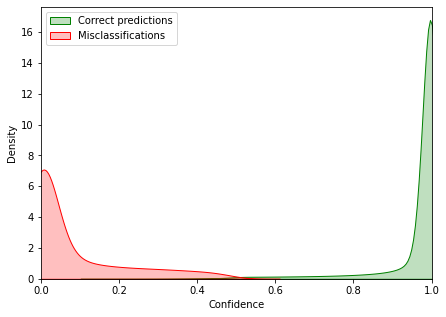}
    \subcaption{TCP}
\end{minipage}
\begin{adjustbox}{max width=\textwidth}
\begin{tabular}{l|ccc|ccc|rr}
\toprule
 Model & \multicolumn{3}{c}{Nb. of Errors} & \multicolumn{3}{c|}{Nb. of Successes} & AUPR\,$\uparrow$ & AUROC\,$\uparrow$ \\
& $>1/K$ & $[\frac{1}{K},\frac{1}{2}]$ & $>1/2$ & $<1/K$ & $[\frac{1}{K},\frac{1}{2}]$ & $>1/2$ & & \\
\midrule
MCP & 0 & 401,573 & 55,506,172 & 0 & 188,128 & 34,166,526 & 48.53\% & 84.42\% \\
TCP & 41,84,875 & 1,722,871 & 0 & 0 & 188,128 & 34,166,526 & 99.92\% & 99.99\% \\
\bottomrule
\end{tabular}
\end{adjustbox}
\label{appxA:fig:density_plot_camvid_segnet}
\end{figure}

%% file: appendixB/appendixB.tex
\section{Experimental Setup}
\label{appxB:sec:experimental_setup}

In this section, we provide comprehensive details about the datasets, the implementation and the hyperparameters of the experiments shown in \cref{chap5}.

\subsection{Image Classification Datasets.}
\label{appxB:subsection:datasets}

In \cref{chap5:sec:exp} to \cref{chap5:sec:ood_training}, the experiments are conducted using CIFAR-10 and CIFAR-100 datasets \cite{Krizhevsky09}. They consist in $32\!\times\!32$ natural images featuring 10 object classes for CIFAR-10 and 100 classes for CIFAR-100. Both datasets are composed with 50,000 training samples and 10,000 test samples. We further randomly split the training set to create a validation set of 10,000 images. 

OOD datasets are TinyImageNet\footnote{https://tiny-imagenet.herokuapp.com/} -- a subset of ImageNet (10,000 test images with 200 classes) --, LSUN \cite{yu15lsun} -- a scene classification dataset (10,000 test images of 10 scenes) --, STL-10 -- a dataset similar to CIFAR-10 but with different classes, and SVHN \cite{svhn-dataset} -- an RGB dataset of $28\!\times\!28$ house-number images (73,257 training and 26,032 test images with 10 digits) --. We downsample each image of TinyImageNet, LSUN and STL-10 to size $32\!\times\!32$.

\paragraph{Training Details.} 
We implemented in PyTorch \cite{pytorch_paper} a VGG-16 architecture \cite{Simonyan15} in line with the previous works of \cite{charpentier2020,malinin2019,maxgap2020}, with fully-connected layers reduced to 512 units. Models are trained for 200 epochs with a batch size of 128 images, using a stochastic gradient descent with Nesterov momentum of $0.9$ and weight decay $5\text{e-}4$. The learning rate is initialized at 0.1 and reduced by a factor of 10 at 50\% and 75\% of the training progress. Images are randomly horizontally flipped and shifted by $\pm4$ pixels as a form of data augmentation. 

\paragraph{Balancing Misclassification and OOD Detection.}
Most neural networks used in our experiments tend to overfit, which leaves very few training errors available.

\begin{wraptable}{r}{5cm}
	\resizebox{\linewidth}{!}{%
    	\begin{tabular}{l|cc}
    		\toprule
    		& CIFAR-10 & CIFAR-100 \\
    		\midrule
    		Train & 99.0 {\small $\pm 0.1$} & 91.2 {\small $\pm 0.2$} \\
    		Val & 93.6 {\small $\pm 0.1$} & 70.6 {\small $\pm 0.3$} \\
    		Test & 93.0 {\small $\pm 0.3$} & 70.1 {\small $\pm 0.4$} \\
    		\bottomrule
    	\end{tabular}
    }
    \caption[Accuracies of evidential neural networks trained on CIFAR-10 and CIFAR-100 datasets]{Mean accuracies (\%) and std. over five runs.}
    \label{appxB:tab:accuracies}
\vspace{-0.3cm}
\end{wraptable}
We provide accuracies on training, validation and test sets in \cref{appxB:tab:accuracies}. With such high predictive performances, the number of misclassifications is usually lower than the number of OOD samples ($\sim$10,000). Hence, the oversampling approach proposed in the paper helps to better balance misclassification detection performances and OOD detection performances in the reported metrics. 

\paragraph{KLoSNet.} We start from the pre-trained evidential model described above. As detailed in Section 3.2 of the main paper, KLoSNet consists of a small decoder attached to the penultimate layer of the main network. In CIFAR experiments, this corresponds to VGG-16's fc1 layer of size 512. This auxiliary neural network is composed of five fully-connected layers of size 400, except for the last layer obviously. KLoSNet decoder's weights $\omega$ are trained for 100 epochs with $\ell_2$ loss (Eq.\,8 in the main paper) and with Adam optimizer with learning rate $1\text{e-}4$. As KLoS$^*$ ranges from zero to large positive values ($>$1000), one may encounter some issues when training KLoSNet. Consequently, we apply a sigmoid function, $\sigma(x) = \frac{1}{1 + e^{-x}}$, after computing the KL-divergence between the NN's output and $\gamma_{y}$. To prevent over-fitting, training is stopped when validation AUC metric for misclassification detection starts decreasing. Then, a second training step is performed by initializing new encoder $E'$ such that $\theta_{E'} = \theta_E$ initially and by optimizing weights $(\theta_{E'}, \omega)$ for 30 epochs with Adam optimizer with learning rate $1\text{e-}6$. We stop training once again based on the validation AUC metric.

\section{Additional Results}
\label{appxB:sec:add_results}

\subsection{Detailed Results for Synthetic Experiments}
\label{appxB:subsec:detail_results_synth}
We detail in \cref{appxB:tab:exp_synthetic} the quantitative results for the task of simultaneous detection of misclassifications and of OOD samples for the synthetic experiment presented in Section 4.1 of the paper. First-order uncertainty measures such as MCP and Entropy perform obviously well on the first task with $80.2\%$ AUC for MCP. However, their OOD performance drops to $\sim$15\% AUC on this dataset. On the other hand, Mahalanobis is adapted to detect OOD samples but not as good for misclassifications. KLoS achieves comparable performances to best methods in misclassification detection and in OOD detection (79.4\% for Mis. and 98.8\% for OOD). As a result, when detecting both inputs simultaneously, KLoS improves all baselines, reaching 89.2\% AUC. 

\begin{table}[ht]
\centering
\resizebox{0.5\linewidth}{!}{%
\begin{tabular}{l|ccc}
	\toprule
	\textbf{Method} &  Mis. ($\uparrow$) & OOD ($\uparrow$) & Mis+OOD ($\uparrow$) \\
	\midrule
	MCP & \ubold{80.2} {\small $\pm 1.1$} & 15.9 {\small $\pm 0.7$} & 48.6 {\small $\pm 1.9$} \\
	Entropy & \ubold{78.4} {\small $\pm 1.5$} & 11.0 {\small $\pm 0.3$} & 45.7 {\small $\pm 1.0$} \\
	Mut.\, Inf. & 75.0 {\small $\pm 2.3$} & 2.2 {\small $\pm 0.2$} & 38.8 {\small $\pm 1.2$} \\
	Diff. Ent. & 74.2 {\small $\pm 2.7$} & 1.9 {\small $\pm 1.0$} & 38.0 {\small $\pm 1.3$} \\
	Mahalanobis & 51.5 {\small $\pm 2.8$} & \ubold{98.5} {\small $\pm 0.3$} & 75.0 {\small $\pm 1.4$} \\
	\midrule
	KLoS & \ubold{79.4} {\small $\pm 1.2$} & \ubold{98.8} {\small $\pm 0.3$} & \ubold{89.2} {\small $\pm 0.5$} \\ 
	\bottomrule
\end{tabular}
}
\caption[Quantitative results for synthetic experiment]{Synthetic experiment: misclassification (Mis.), out-of-distribution detection (OOD) and simultaneous detection (Mis+OOD) (mean \% AUC and std. over 5 runs). Bold type indicates significant top performance ($p<0.05$) according to paired t-test.}
\label{appxB:tab:exp_synthetic}
\end{table}

\subsection{Results with SVHN as OOD test dataset}

We report in \cref{appxB:tab:results_svhn} all the results when evaluating with SVHN \cite{svhn-dataset} as OOD dataset. Along with simultaneous detection results, we also provide separate results for misclassifications detection and OOD detection respectively. Similarly to the comparative results in the main paper, KLoSNet outperforms all the baselines in every simultaneous detection benchmark, with Mahalanobis being second.

\begin{figure}[ht]
\centering
\begin{minipage}{0.45\linewidth}
    \centering
    \subcaption{CIFAR-10 with VGG-16}
    \vspace{-0.2cm}
	\resizebox{\linewidth}{!}{%
    	\begin{tabular}{l|c|cc}
    		\toprule
    		& \textbf{CIFAR-10} & \multicolumn{2}{c}{\textbf{SVHN}}  \\
    		\textbf{Method} & Mis. ($\uparrow$) & OOD ($\uparrow$) & Mis+OOD ($\uparrow$) \\
    		\midrule
    		MCP & 87.6 {\small $\pm 1.6$} & 87.3 {\small $\pm 2.2$} & 88.9 {\small $\pm 0.5$} \\
    		Entropy & 83.5 {\small $\pm 2.4$} & 85.5 {\small $\pm 2.3$} & 86.9 {\small $\pm 1.9$} \\
    		ConfidNet & 90.2 {\small $\pm 0.8$} & \ubold{89.0} {\small $\pm 3.1$} & 91.0 {\small $\pm 1.1$} \\
    		Mut.\, Inf. & 84.1 {\small $\pm 1.5$} & 80.0 {\small $\pm 3.9$} & 83.2 {\small $\pm 1.7$} \\
    		Diff. Ent. & 86.8 {\small $\pm 1.0$} & 86.0 {\small $\pm 2.0$} & 87.6 {\small $\pm 0.9$} \\
    		EPKL & 83.9 {\small $\pm 1.5$} & 79.4 {\small $\pm 4.2$} & 82.8 {\small $\pm 1.9$} \\
    		ODIN & 86.0 {\small $\pm 2.0$} & 86.8 {\small $\pm 2.2$} & 87.7 {\small $\pm 1.0$} \\
    		Mahalanobis & 91.2 {\small $\pm 0.3$} & \ubold{89.1} {\small $\pm 2.8$} & 91.5 {\small $\pm 1.1$} \\
    		\midrule
    		KLoSNet (Ours) & \ubold{92.5} {\small $\pm 0.6$} & \ubold{89.8} {\small $\pm 3.0$} & \ubold{92.7} {\small $\pm 1.2$} \\
    		\bottomrule
    	\end{tabular}
    }
    \label{appxB:tab:results_svhn_cifar10_vgg16}
\end{minipage}
\hspace{0.1cm}
\begin{minipage}{0.45\linewidth}
    \centering
    \subcaption{CIFAR-10 with ResNet18}
    \vspace{-0.2cm}
	\resizebox{\linewidth}{!}{%
    	\begin{tabular}{l|c|cc}
    		\toprule
    		& \textbf{CIFAR-100} & \multicolumn{2}{c}{\textbf{SVHN}}  \\
    		\textbf{Method} & Mis. ($\uparrow$) & OOD ($\uparrow$) & Mis+OOD ($\uparrow$) \\
    		\midrule
    		MCP & 84.9 {\small $\pm 0.8$} & 79.6 {\small $\pm 1.0$} & 83.0 {\small $\pm 0.9$} \\
    		Entropy & 84.6 {\small $\pm 0.8$} & 79.6 {\small $\pm 1.1$} & 82.8 {\small $\pm 0.9$} \\
    		ConfidNet & 90.7 {\small $\pm 0.4$} & 84.6 {\small $\pm 1.1$} & 88.6 {\small $\pm 0.6$} \\
    		Mut. Inf & 80.6 {\small $\pm 0.6$} & 77.0 {\small $\pm 1.2$} & 79.4 {\small $\pm 0.9$} \\
    		Diff. Ent & 82.7 {\small $\pm 0.6$} & 78.3 {\small $\pm 1.2$} & 81.1 {\small $\pm 0.9$} \\
    		EPKL & 80.2 {\small $\pm 0.6$} & 76.8 {\small $\pm 1.3$} & 79.0 {\small $\pm 0.9$} \\
    		ODIN & 83.7 {\small $\pm 0.7$} & 78.9 {\small $\pm 1.0$} & 81.9 {\small $\pm 0.9$} \\
    		Mahalanobis & 91.2 {\small $\pm 0.4$} & 90.7 {\small $\pm 0.4$} & 91.8 {\small $\pm 0.3$} \\
    		\midrule
    		KLoSNet (Ours) & \ubold{93.9} {\small $\pm 0.4$} & \ubold{93.1} {\small $\pm 1.1$} & \ubold{94.4} {\small $\pm 0.3$} \\
    		\bottomrule
    	\end{tabular}
    }
    \label{appxB:tab:results_svhn_cifar10_resnet18}
\end{minipage}

\vspace{0.3cm}

\begin{minipage}{0.45\linewidth}
    \centering
    \subcaption{CIFAR-100 with VGG-16}
    \vspace{-0.2cm}
	\resizebox{\linewidth}{!}{%
   	    \begin{tabular}{l|c|cc}
    		\toprule
    		& \textbf{CIFAR-10} & \multicolumn{2}{c}{\textbf{SVHN}}  \\
    		\textbf{Method} & Mis. ($\uparrow$) & OOD ($\uparrow$) & Mis+OOD ($\uparrow$) \\
    		\midrule
    		MCP & 82.9 {\small $\pm 0.8$} & 70.8 {\small $\pm 3.9$} & \ubold{81.3} {\small $\pm 2.0$} \\
    		Entropy & 82.2 {\small $\pm 0.8$} & \ubold{72.9} {\small $\pm 3.9$} & \ubold{81.5} {\small $\pm 2.0$} \\
    		ConfidNet & 84.4 {\small $\pm 0.6$} & 68.0 {\small $\pm 3.4$} & 80.8 {\small $\pm 2.0$} \\
    		Mut.\, Inf. & 78.9 {\small $\pm 0.8$} & \ubold{72.7} {\small $\pm 4.9$} & 79.5 {\small $\pm 2.5$} \\
    		Diff. Ent. & 80.2 {\small $\pm 0.8$} & \ubold{72.4} {\small $\pm 4.9$} & 80.2 {\small $\pm 2.5$} \\
    		EPKL & 78.8 {\small $\pm 0.8$} & \ubold{72.7} {\small $\pm 4.8$} & 79.4 {\small $\pm 2.4$} \\
    		ODIN & 82.1 {\small $\pm 0.8$} & \ubold{72.0} {\small $\pm 3.8$} & \ubold{81.3} {\small $\pm 1.9$} \\
    		Mahalanobis & 84.0 {\small $\pm 0.2$} & \ubold{73.4} {\small $\pm 5.6$} & \ubold{83.2} {\small $\pm 2.5$} \\
            \midrule
    		KLoSNet (Ours) & \ubold{86.7} {\small $\pm 0.4$} & 70.4 {\small $\pm 5.7$} & \ubold{83.5} {\small $\pm 2.8$} \\
    		\bottomrule
    	\end{tabular}
    }
    \label{appxB:tab:results_svhn_cifar100_vgg16}
\end{minipage}
\hspace{0.1cm}
\begin{minipage}{0.45\linewidth}
    \centering
    \subcaption{CIFAR-100 with ResNet18}
    \vspace{-0.2cm}
	\resizebox{\linewidth}{!}{%
    	\begin{tabular}{l|c|cc}
    		\toprule
    		& \textbf{CIFAR-100} & \multicolumn{2}{c}{\textbf{SVHN}}  \\
    		\textbf{Method} & Mis. ($\uparrow$) & OOD ($\uparrow$) & Mis+OOD ($\uparrow$) \\
    		\midrule
    		MCP & 84.9 {\small $\pm 0.8$} & 79.6 {\small $\pm 1.0$} & 83.0 {\small $\pm 0.9$} \\
    		Entropy & 84.6 {\small $\pm 0.8$} & 79.6 {\small $\pm 1.1$} & 82.8 {\small $\pm 0.9$} \\
    		ConfidNet & 90.7 {\small $\pm 0.4$} & 84.6 {\small $\pm 1.1$} & 88.6 {\small $\pm 0.6$} \\
    		Mut. Inf & 80.6 {\small $\pm 0.6$} & 77.0 {\small $\pm 1.2$} & 79.4 {\small $\pm 0.9$} \\
    		Diff. Ent & 82.7 {\small $\pm 0.6$} & 78.3 {\small $\pm 1.2$} & 81.1 {\small $\pm 0.9$} \\
    		EPKL & 80.2 {\small $\pm 0.6$} & 76.8 {\small $\pm 1.3$} & 79.0 {\small $\pm 0.9$} \\
    		ODIN & 83.7 {\small $\pm 0.7$} & 78.9 {\small $\pm 1.0$} & 81.9 {\small $\pm 0.9$} \\
    		Mahalanobis & 91.2 {\small $\pm 0.4$} & 90.7 {\small $\pm 0.4$} & 91.8 {\small $\pm 0.3$} \\
    		\midrule
    		KLoSNet (Ours) & \ubold{93.9} {\small $\pm 0.4$} & \ubold{93.1} {\small $\pm 1.1$} & \ubold{94.4} {\small $\pm 0.3$} \\
    		\bottomrule
    	\end{tabular}
    }
    \label{appxB:tab:results_svhn_cifar100_resnet18}
\end{minipage}
    \caption[Simultaneous detection results with SVHN as OOD dataset]{\textbf{Results with SVHN as OOD dataset} (\% mean AUROC and std. over 5 runs).}
    \label{appxB:tab:results_svhn}
\end{figure}

\subsection{Detail results of selective classification}
\label{appxB:subsec:detailed_selective}

In addition to aggregated results shown in \cref{chap5:subsec:detail_results_selective}, we provide detailed results of selective classification in presence of domain shifts by corruption for CIFAR-10-C (\cref{appxB:tab:CIFAR-10_exp_select}) and CIFAR-100-C (\cref{appxB:tab:CIFAR-100_exp_select}). For almost every corruption, KLoSNet outperforms other methods. When averaged on all corruptions, KLoSNet scores 48.6\% AURC while the second best, ConfidNet, reaches 49.0\% AURC. 

\begin{table}[ht]
\resizebox{\linewidth}{!}{%
    \begin{tabular}{l|c|ccc|cccc|cccc|cccc|c}
    \midrule
     &  & \multicolumn{3}{c|}{\textbf{Noise}} & \multicolumn{4}{c|}{\textbf{Blur}} & \multicolumn{4}{c|}{\textbf{Weather}} & \multicolumn{4}{c|}{\textbf{Digital}} &  \\
    \textbf{Method} & Clean & Gaussian & Shot & Impulse & Defocus & Glass & Motion & Zoom & Snow & Frost & Fog & Bright & Contrast & Elastic & Pixel & JPEG & Mean \\
    \midrule
    MCP & 48.3\% & 46.7\% & 48.0\% & 43.7\% & 48.6\% & 45.7\% & 45.4\% & 43.4\% & 44.4\% & 42.5\% & 40.4\% & 46.5\% & 43.3\% & 43.3\% & 43.0\% & 1.9\% & 42.2\% \\
Entropy & 48.8\% & 47.2\% & 48.5\% & 43.9\% & 49.1\% & 45.9\% & 45.6\% & 43.6\% & 44.7\% & 42.7\% & 40.6\% & 46.9\% & 43.5\% & 43.7\% & 43.2\% & 2.0\% & 42.5\% \\
ConfidNet & 47.4\% & 45.8\% & 47.1\% & 42.9\% & 48.1\% & 45.2\% & 44.9\% & 42.5\% & 43.5\% & 41.6\% & 39.1\% & 45.9\% & 42.6\% & 42.0\% & 42.1\% & 1.3\% & 41.4\% \\
Mut.\, Inf. & 53.6\% & 52.3\% & 53.2\% & 48.3\% & 53.0\% & 49.3\% & 49.1\% & 48.8\% & 49.8\% & 47.8\% & 46.9\% & 51.5\% & 47.9\% & 49.8\% & 48.2\% & 4.8\% & 47.1\% \\
ODIN & 49.3\% & 47.7\% & 48.9\% & 44.3\% & 49.5\% & 46.3\% & 46.0\% & 44.1\% & 45.2\% & 43.2\% & 41.2\% & 47.3\% & 43.9\% & 44.3\% & 43.7\% & 2.2\% & 42.9\% \\
Mahalanobis & 48.9\% & 46.9\% & 48.6\% & 42.5\% & 49.7\% & 45.3\% & 44.8\% & 43.1\% & 44.1\% & 41.4\% & 38.6\% & 45.8\% & 42.6\% & 42.9\% & 42.4\% & 1.0\% & 41.8\% \\
KLoSNet & \ubold{47.0\%} & \ubold{45.3\%} & \ubold{46.9\%} & \ubold{42.3\%} & \ubold{48.0\%} & \ubold{45.0\%} & \ubold{44.6\%} & \ubold{42.2\%} & \ubold{43.1\%} & \ubold{41.1\%} & \ubold{38.1\%} & \ubold{45.2\%} & \ubold{42.4\%} & \ubold{41.8\%} & \ubold{42.0\%} & \ubold{0.9\%} & \ubold{41.0\%} \\
    \bottomrule
    \end{tabular}
}
\label{appxB:tab:CIFAR-10_exp_select}
    \caption[Detailed results for selective classification on CIFAR-10-C]{\textbf{Detailed results for selective classification on CIFAR-10-C}. Comparative performance in AURC (\%) of classification with the option to reject misclassified test samples and samples from shifted distributions. Results are average on 5 runs (mean $\pm$ std.).}
\end{table}

\begin{table}[ht]
\resizebox{\linewidth}{!}{%
    \begin{tabular}{l|c|ccc|cccc|cccc|cccc|c}
    \midrule
     &  & \multicolumn{3}{c|}{\textbf{Noise}} & \multicolumn{4}{c|}{\textbf{Blur}} & \multicolumn{4}{c|}{\textbf{Weather}} & \multicolumn{4}{c|}{\textbf{Digital}} &  \\
    \textbf{Method} & Clean & Gaussian & Shot & Impulse & Defocus & Glass & Motion & Zoom & Snow & Frost & Fog & Bright & Contrast & Elastic & Pixel & JPEG & Mean \\
    \midrule
MCP & \ubold{53.3\%} & \ubold{52.7\%} & 55.2\% & 50.8\% & 54.0\% & 52.8\% & 52.5\% & 51.3\% & 52.0\% & 50.5\% & 47.9\% & 52.1\% & 50.7\% & 50.4\% & 51.3\% & 13.1\% & 49.4\% \\
Entropy & 53.5\% & 53.0\% & 55.8\% & 51.3\% & 54.8\% & 53.3\% & 53.1\% & 51.8\% & 52.6\% & 51.0\% & 48.2\% & 52.7\% & 51.1\% & 50.9\% & 51.7\% & 13.5\% & 49.9\% \\
ConfidNet & 53.5\% & 52.8\% & \ubold{55.0\%} & 50.1\% & 53.7\% & 52.3\% & 51.9\% & 50.8\% & 51.6\% & 50.0\% & 47.2\% & 51.7\% & 50.3\% & 49.8\% & 50.9\% & 11.7\% & 49.0\% \\
Mut.\, Inf. & 54.5\% & 54.0\% & 57.1\% & 53.1\% & 56.2\% & 54.9\% & 54.7\% & 53.5\% & 54.3\% & 52.6\% & 50.0\% & 54.4\% & 52.4\% & 53.0\% & 53.2\% & 16.5\% & 51.5\% \\
ODIN & 53.4\% & 52.9\% & 55.5\% & 51.2\% & 54.5\% & 53.2\% & 52.9\% & 51.7\% & 52.4\% & 50.9\% & 48.2\% & 52.5\% & 51.0\% & 50.8\% & 51.7\% & 13.7\% & 49.8\% \\
Mahalanobis & 54.5\% & 53.9\% & 56.7\% & 50.9\% & 55.5\% & 52.8\% & 52.9\% & 52.2\% & 52.5\% & 50.7\% & 47.8\% & 52.4\% & 51.0\% & 51.4\% & 51.9\% & 11.0\% & 49.9\% \\
\midrule
KLoSNet & 54.0\% & 53.1\% & 55.5\% & \ubold{49.5\%} & \ubold{53.6\%} & \ubold{51.6\%} & \ubold{51.4\%} & \ubold{50.7\%} & \ubold{51.1\%} & \ubold{49.2\%} & \ubold{46.5\%} & \ubold{51.3\%} & \ubold{49.9\%} & \ubold{49.7\%} & \ubold{50.7\%} & \ubold{9.7\%} & \ubold{48.6\%} \\
\bottomrule
    \end{tabular}
}
    \label{appxB:tab:CIFAR-100_exp_select}
    \caption[Detailed results for selective classification on CIFAR-100-C]{\textbf{Detailed results for selective classification on CIFAR-100-C}. Comparative performance in AURC (\%) of classification with the option to reject misclassified test samples and samples from shifted distributions. Results are average on 5 runs (mean $\pm$ std.).}
\end{table}

